%% file: main.tex
\documentclass[11pt]{article}
\usepackage{ifthen}
\usepackage{tabulary}
\newboolean{arxiv}
\setboolean{arxiv}{true}
\usepackage{ifthen}
\usepackage[top=1in, bottom=1in, left=0.8in, right=0.8in]{geometry}

\usepackage{arxiv_def}

\usepackage[toc,page]{appendix}
\usepackage{xr}

\makeatletter
\newcommand{\vast}{\bBigg@{3}}
\newcommand{\Vast}{\bBigg@{4}}
\makeatother

\begin{document}
\title{Local minima of the empirical risk in high dimension:\\
General theorems and convex examples}

\author{Kiana Asgari\thanks{Department of Management Science and Engineering, Stanford University} \;\;
\and\;\;
Andrea Montanari\thanks{Department of Statistics and Department of Mathematics, Stanford University} 
	\and 
	Basil Saeed\thanks{Department of Electrical Engineering, Stanford University}
}

\maketitle
\input{abstract}

\tableofcontents

\input{main_paper_text}

\input{Appendix_THM2}

\subsection*{Acknowledgments}
\input{acknowledge}

\newpage
\appendix
\input{Appendix_A_RMT}

\input{Appendix_B_THM1}

\input{Appendix_C_CONVEX}
\input{Appendix_D_EXAMPLES}
\input{Appendix_E_NUMERICAL}
\clearpage
\bibliographystyle{amsalpha}
\bibliography{all-bibliography}

\end{document}

%% file: abstract.tex
\begin{abstract}
We consider a general model for high-dimensional empirical risk minimization whereby the data $\bx_i$ are $d$-dimensional Gaussian vectors, the model is parametrized by $\bTheta\in \reals^{d\times k}$, and the loss depends on the data via the projection $\bTheta^{\sT}\bx_i$. This setting covers as special cases 
classical statistics methods (e.g. multinomial regression and other generalized linear models), but also two-layer fully connected neural networks with $k$ hidden neurons.

We use the Kac-Rice formula from Gaussian process theory to derive a bound on the expected number of local minima of this empirical risk,
under the proportional asymptotics in which $n,d\to\infty$, with $n\asymp d$. Via Markov's inequality, this bound allows to determine the positions of these minimizers (with exponential deviation bounds) and hence derive sharp asymptotics on the estimation and prediction error.

As a special case, we apply our characterization to convex losses. 
We show that our approach is tight and
allows to prove previously conjectured results. In addition, we characterize the spectrum of the Hessian at the minimizer. A companion paper applies our general result to non-convex examples.
\end{abstract}

%% file: main_paper_text.tex
\section{Introduction}
\label{sec:Intro}

Empirical risk minimization (ERM) is by far the most popular
technique for parameter estimation in high-dimensional statistics and 
machine learning. While empirical process theory provides a
fairly accurate picture of its behavior when the sample size is sufficiently large \cite{van1998asymptotic,vershynin2018high}, a wealth of new phenomena arises in the so-called \emph{proportional regime} where the number of 
model parameters per sample is of non-vanishing order \cite{montanari2018mean}.
Examples of such phenomena are
exact or weak recovery phase transitions \cite{DMM09,BayatiMontanariLASSO,lelarge2019fundamental,barbier2019optimal,mignacco2020role},
information-computation gaps \cite{celentano2022fundamental,schramm2022computational},
benign overfitting and double descent
\cite{hastie2022surprises}. 

The proportional regime is increasingly of interest because of the adoption of
ever-more expressive families of statistical models by practitioners.
For instance, scaling laws in AI development imply that optimal predictive performances are achieved when the number of model parameters scales roughly proportionally to the number of samples \cite{kaplan2020scaling,hoffmann2022training}.

Unfortunately, the mathematical tools at our disposal to understand ERM in the
proportional regime are far less robust and general than textbook 
asymptotic statistics. Approximate message passing (AMP) algorithms 
admit a general high-dimensional characterization
\cite{bayati2011dynamics,BayatiMontanariLASSO,donoho2016high},
and give access to certain local 
minimizers of the empirical risk, but not all of them, 
and typically require case-by-case technical analysis.
Gaussian 
comparison inequalities provide a very elegant route \cite{ThrampoulidisOyHa15,thrampoulidis2018precise}, but 
give two-sided bounds only for convex problems, and only succeed 
when a simple comparison process exists that yields a sharp 
inequality. Leave-one out techniques \cite{el2018impact} can be powerful,  but 
also crucially rely on the condition that perturbing the ERM problem by leaving out one datapoint does not significantly affect the minimizer. This in turn is challenging to guarantee without convexity. 

The main objective of this paper is to develop a new approach 
to ERM analysis by leveraging the celebrated Kac-Rice formula for the expected number of zeros of a Gaussian process
\cite{Kac1943,Rice1945}. Given a smooth Gaussian process $\bz:\reals^m\to\reals^m$,
and an open set $\cB\subseteq\reals^m$
the Kac-Rice formula provides an expression for the expected number of zeros 
of $\bz$ in $\cB$: $Z(\cB):=\{\btheta\in \cB:\, \bz(\btheta)=\bzero\}$.
By Markov's inequality,  the expected number of zeros 
$\E[Z(\cB)]$ is an upper bound on the probability that 
$\cB$ contains a zero of the process. By evaluating this upper bound for
suitable collections of sets $\cB$, it is possible to
establish certain typical properties of the zeros of $\bz$.

In a nutshell, the Kac-Rice approach to empirical risk minimization constructs a Gaussian process
$\bz$ whose zeros are in one-to-one correspondence with stationary points of the empirical 
risk. By further restricting the formula to local minima, 
this approach can characterize the estimation error, empirical risk, and other statistical properties of these local minima.
Carrying out this program requires addressing three main technical challenges:
$(i)$~Constructing the process $\bz$;
$(ii)$~Evaluating the high-dimensional asymptotics of the Kac-Rice formula to derive a rate function
for the expected number of local minima (this mainly requires controlling the expected determinant of a certain random matrix); $(iii)$~Extracting information from the resulting rate function.

Before giving a more detailed explanation of this approach, we will introduce the general empirical risk minimization problem that we will study.
In the following sections, we will then carry out tasks $(i)$ and $(ii)$ in a certain generality, proving a general formula that covers the non-convex case. To give a concrete
example of the analysis of the rate function (point $(iii)$ above), we will consider the case of convex ERM problems. 
Even in this case, the Kac-Rice approach yields new results

We will informally discuss the phenomenon of rate trivialization. When this occurs, the Kac-Rice approach provides a sharp asymptotic characterization of non-convex ERM
as well. Establishing rate trivialization in specific examples is deferred to a forthcoming publication.

\subsection{Setting}

We consider a general supervised learning problem: we 
are given $n$ i.i.d. pairs $\{(\bx_i,\by_i):i\le n\}$,  $(\bx_i,\by_i) \in \R^d\times\R^{q}$, 
with $\bx_i$ a covariates/feature vector, and $\by_i$ a 
response variable (possibly vector valued).  
Our analysis will assume the covariates vectors $\bx_i$ to be  Gaussian $\bx_i\sim\normal(\bzero,\bSigma)$
and the response variables to be distributed according to a general multi-index model.
Namely, there exists a parameters vector $\bTheta_0^{d\times k_0}$ such that
\begin{align}
\nonumber
\P(\by_i\in S|\bx_i) = \rP(S |\bTheta_0^{\sT}\bx_i)\, ,
\end{align}
where $\rP$
is a probability kernel. More concretely, and equivalently, we will write
$\by_i = \bphi(\bTheta_0^{\sT}\bx_i,w_i)$ for some  measurable function $\bphi:\reals^{k_0}\times\reals\to\reals^q$ and
noise variables $w_i$ independent of $\bx_i$. 
In words, the response (or label) $\by_i$ only depends on a
$k_0$-dimensional projection $\bTheta_0^{\sT}\bx_i$ of the covariates vector $\bx_i$
(each coordinate of $\bTheta_0^{\sT}\bx_i$ is referred to as an `index').

The statistician's objective is to estimate $\bTheta_0$ from 
data $\{(\bx_i,\by_i):i\le n\}$. This is often achieved by minimizing an
empirical risk of the form
\begin{equation}
\label{eq:erm_obj_0}
    \hat R_n(\bTheta) := \frac1n \sum_{i=1}^n L(\bTheta^\sT \bx_i, \by_i) + \frac1d\sum_{i=1}^d \sum_{j=1}^k\rho(\sqrt{d}\Theta_{i,j})
\end{equation}
over the parameters $\bTheta = (\Theta_{i,j})_{i\in[d],j\in[k]}$. 
Here, $L:\R^{k+q}\to \R$ is a loss function and $\rho :\R\to\R$ is a regularizer.
Classically, the true parameter $\bTheta_0$ is estimated by a global  minimizer
\begin{align}
\hbTheta \in \argmin_{\bTheta\in\reals^{d\times k}}\hR_n(\bTheta)\, .
\end{align}
In many modern applications, the empirical risk $\hR_n(\bTheta)$ is non-convex and hence
we content ourselves with a local minimum which is computed by a low-complexity optimization algorithm,
such as gradient descent.  This motivates our goal to 
characterize all local minima and their properties in a unified fashion.

We will focus on the following 
high-dimensional asymptotics:
\begin{equation}
  n,d\to\infty   \quad\textrm{with}\quad \frac{n}{d}\to\alpha\in(0,\infty), \quad\textrm{while} \quad k,k_0 =O(1).
\end{equation}

In Section~\ref{sec:assumptions}, we will give the precise assumptions for our results to hold. At this point, however, it may be useful to consider some examples that fit into the above formulation.

\vspace{1em}
\begin{example}[Multinomial regression]\label{ex:Multinomial}
Consider classification with $(k+1)$ classes: the covariates vector
$\bx_i \in\R^d$ represents features of an object (e.g. an image), and the response variable
$\by_i$ takes  $k+1$ possible values or class labels. We want to learn a model that is capable to assign labels to unseen objects.

We  use one-hot encoding of class labels. Namely
$\by_i\in\R^k$ takes value in 
$\{\be_0:=\bzero,\be_1,\dots,\be_k\}$ (where $\be_i$ is the $i$-th element of the canonical basis.)

The multinomial regression model prescribes class label to depend on a
$k$-dimensional projection of the covariates: 
\begin{equation}
\P(\by_i=\be_j|\bx_i) = p_j(\bTheta_0^{\sT}\bx_i) := 
\frac{\exp\{ \be_j^\sT\bTheta_0^\sT \bx_i\} }{ 1  + \sum_{j'=1}^k \exp\{\be_{j'}^\sT\bTheta_0^\sT\bx_{j'}\}},\quad\quad
j\in\{1,\dots,k\}.
\label{eq:MultiNomialDef}
\end{equation}
The maximum likelihood estimator is obtained by minimizing the empirical risk $\hR_n(\bTheta)$ of Eq.~\eqref{eq:erm_obj_0} 
with $k=k_0$ and cross-entropy loss
\begin{align}
\nonumber
L(\bv, \by) := -\sum_{j=0}^k y_j\log p_j(\bv)\, .
\end{align}
The case of two classes ($k=1$) corresponds to logistic regression.
The high-dimensional asymptotics of the maximum likelihood estimator
in logistic regression has been the object of significant recent work
\cite{sur2019modern,montanari2019generalization,deng2022model}, with important questions still open.

The case of more than two classes ($k>2$) poses additional challenges:
we will apply our general theory in Section
\ref{sec:Multinomial} (to which we refer also for further discussion of related work). 
\end{example}

\begin{example}[Two-layer neural networks]
Consider a binary classification problem
whereby $y_i\in\{0,1\}$, with $\P(y_i=1|\bx_i) = \varphi(\bTheta_0^{\sT}\bx_i)$. 
With $\bTheta_{\cdot,i}$ denoting the $i$-th column of $\bTheta$,
consider fitting 
a two-layer neural network
\begin{align}
\nonumber
f(\bx;\bTheta) = \sum_{i=1}^ka_i \sigma(\bTheta_{\cdot,i}^{\sT}\bx)\, 
\end{align}
with $k\ge k_0$, and \emph{cross entropy} loss.
For simplicity, we can think of the second layer weights $(a_i)$
as fixed. 
The ERM problem can be recast in the form \eqref{eq:erm_obj_0}
by setting:
\begin{align}
\nonumber
L(\bv, y) := - y f_0(\bv) + \log (1+e^{f_0(\bv)})\, ,\;\;\;\;\;
f_0(\bv):= \sum_{i=1}^k a_i\sigma(v_i)\, .
\end{align}
We will use our general theory to treat this example in future work \cite{OursInPreparation}.
\end{example}

\subsection{Summary of approach and results: Number of local minima and topological trivialization}
In this section, we give an overview of our approach and a summary of our results.
We will omit several technical details, and focus on giving an informal explanation of the approach.

Let $\cuP(\R^{m})$ denote the set of probability distributions on $m$-dimensional vectors.
For a given $\bTheta\in\R^{d\times k}$, define the empirical distributions
\begin{equation}
    \hmu(\bTheta):= \frac1d\sum_{j=1}^d \delta_{\sqrt{d} 
    (\bSigma^{1/2}[\bTheta,\bTheta_0])_j} \in\cuP(\R^{k+k_0})\, ,
    \quad\quad
    \hnu(\bTheta):= \frac1n\sum_{i=1}^n \delta_{(\bTheta^\sT\bx_i, \by_i)} \in\cuP(\R^{k+q}).
\end{equation}
Here, $\bSigma$ denotes the covariance of $\bx$ while
$(\bSigma^{1/2}[\bTheta,\bTheta_0])_j$ denotes the $j$th row of the $d\times (k+k_0)$ dimensional matrix $\bSigma^{1/2}[\bTheta,\bTheta_0].$
The  empirical distributions  $\hmu(\bTheta)$,
$\hnu(\bTheta)$ encode many properties of the candidate minimizer
$\bTheta$. For instance, the unregularized empirical risk is:
\begin{align}
\int\! L(\bv,\by)\, \hnu(\bTheta)(\de\bv,\de\by) = \frac1n \sum_{i=1}^nL(\bTheta^{\sT}\bx_i,\by_i)\, .
\end{align}

Given two (well-behaved) sets of probability distributions $\cuA\subseteq \cuP(\R^{k+k_0}),\cuB\subseteq\cuP(\R^{k+q})$, we 
consider the following set of local minimizers $\bTheta$ of the empirical risk $\hR_n(\bTheta)$
\begin{align}
\nonumber
\cZ'_n(\cuA, \cuB):=\Big\{\mbox{ local minimizers of }  \hR_n(\bTheta)
\mbox{ s.t.  }  \hmu(\bTheta) \in\cuA, \hnu(\bTheta)\in\cuB \Big\}\, 
\end{align}
%
In general, the number of local minima  $|\cZ'_n(\cuA, \cuB)|$ can be exponentially large 
in $n,d$
(see \cite{arous2019landscape} for an example of this type), and therefore, it makes sense to 
study the exponential growth rate of its expectation, i.e. the asymptotycs 
of $n^{-1}\log\E\left[ |\cZ'_n(\cuA, \cuB)|\right]$.  Our main result will be an 
upper bound of the form (see Theorem \ref{thm:general} for a definition of the rate $\Phi(\mu,\nu)$)
\begin{align}
    \lim_{n,d\to\infty}\frac{1}{n}\log \E\left[ |\cZ'_n(\cuA, \cuB)|\right]
    \le -\inf_{\mu\in\cuA,\nu\in\cuB}\Phi(\mu,\nu)\, .\label{eq:Summary}
\end{align}
An application of Markov's inequality then allows us to
identify typical properties of local minimizers of $\hat R_n(\bTheta)$ in the following way:
given
a test function 
$\psi:\reals^k\times\reals^{k_0}\to \reals$, and an interval $I\subseteq \R$,
we can define $\cuA(I,\psi): = \{\mu: \int \psi(\bt,\bt_0)\mu(\de\bt,\de\bt_0)\not \in I\}$
which is a subset of $\cuP(\R^{k+k_0})$.
For any local minimizer $\hbTheta$, Eq.~\eqref{eq:Summary} implies
\begin{align}
\label{eq:markov_localization}
\P\left(\frac{1}{d}\sum_{j=1}^d\psi\big(
\sqrt{d}(\bSigma^{1/2}[\hat\bTheta,\bTheta_0])_j \big)
\not\in I\right)
&= \P\left( \hmu(\hat\bTheta) \in \cuA(I,\psi)
\right)\\
&\le
\exp\left\{-n\inf_{\mu\in\cuA(I,\psi), \nu}
\Phi(\mu,\nu)+o(n)\right\}\, .
\nonumber
\end{align}
A similar bound can be obtained with linear functions of $\hnu(\hat\bTheta)$.

We expect  the bound of Eq.~\eqref{eq:Summary} 
(with rate function as specified in Theorem \ref{thm:general})
to be tight in certain cases of interest (e.g. when rate trivialization occurs), but we do not attempt to prove it (see Remark \ref{rmk:Tightness} below).
Crucially, we observe that in a number of
settings, our bound allows us to identify 
the precise limit of $\hmu(\hat\bTheta)$ and $\hnu(\hat\bTheta)$. Namely,
in many cases of interest there exist $(\mu_\star,\nu_\star)$ such that:
\begin{equation}\label{eq:Trivialization}
\begin{split}
    &\Phi(\mu,\nu)\ge 0 \;\;\;\forall \mu,\nu\, ,\\
    &\Phi(\mu,\nu) = 0 \;\;\Leftrightarrow  (\mu,\nu) = (\mu_\star,\nu_\star)\, ,
    \end{split}
\end{equation}
which implies that $\mu_\star,\nu_\star$ is the unique limit:
\begin{align}
\nonumber
\hmu(\hat\bTheta)\weakc \mu_\star\,,\;\;\;\;\;\;
\hnu(\hat\bTheta)\weakc\nu_\star\, .
\end{align}
(Here and below $\weakc$ denotes weak convergence of probability measures and the limit is
understood to hold almost surely.)
We will refer to Eq.~\eqref{eq:Trivialization} as the \emph{rate trivialization property. Note that, when rate trivialization occurs,
it must be true that the rate function $\Phi(\mu,\nu)$ that we compute is sharp,
i.e.  Eq.~\eqref{eq:Summary} holds with equality 
(as explained in Remark \ref{rmk:Tightness} below, this is expected to hold  more generally.) }

 As a further concrete example, if rate trivialization holds 
 (and assuming for simplicity $k=k_0$),
 then we can characterize the asymptotics of $\|\bSigma^{1/2}(\hbTheta-\bTheta_0)\|_F^2$
 to within $o_n(1)$ error. Namely, taking $\psi(\bt,\bt_0) := \|\bt - \bt_0\|_2^2$, and 
 using Eq.~\eqref{eq:markov_localization} with $I = [E_\star - \eps, E_\star +\eps]$, we obtain that for any $\eps>0$
 there exists $c(\eps)>0$ such that
 \begin{align}
 \P\Big(\|\bSigma^{1/2}(\hbTheta-\bTheta_0)\|_F^2 \in [E_\star-\eps,E_\star+\eps]\Big)
 \ge 1- 2e^{-n\, c(\eps)}\, ,\;\;\;\;
 E_\star:=\int \|\bt - \bt_0\|_2^2\, \mu_\star (\de\bt,\bt_0)\, .
 \end{align}
(This implies $\lim_{n,d\to\infty}
\|\bSigma^{1/2}(\hbTheta-\bTheta_0)\|_F^2 =E_{\star}$ almost surely, but is stronger.e

In this paper, we will prove that rate trivialization
takes place for strictly convex ERM problems. 
In these cases, we will provide fairly explicit formulas for $\mu_\star$.
Although strictly convex ERM problems have at most a unique minimizer,
rate trivialization is far from obvious because of the inequality 
in Eq.~\eqref{eq:Summary}. 
Hence, convex examples provide an important test of our general theory, 
and also a  very useful domain of application.

In our companion paper we will provide sufficient conditions under which the  
rate trivialization condition of Eq.~\eqref{eq:Trivialization} holds for non-convex losses.
Under these conditions, we obtain a sharp characterization of estimators defined by non-convex ERM
problems.

While the rate trivialization phenomenon --whenever it occurs-- allows
us to prove sharp results about the typical properties of local minima, 
we emphasize that the exponential
growth rate of the expected number of local minima 
$\E\left[ |\cZ'_n(\cuA, \cuB)|\right]$ does not coincide (in general) with the the exponential
growth rate of the `typical'  number of local minima
(e.g., the growth rate of the median of $|\cZ'_n(\cuA, \cuB)|$). 
Hence, if $\Phi(\mu,\nu)<0$, say, for all $\mu\in\cuA$, $\nu\in \cuB$, this \emph{does not imply}
that there is a local minimizer with $\hmu(\bTheta)\in\cuA$, $\hnu(\bTheta)\in\cuB$, with probability bounded away from zero\footnote{In particular, proving a lower bound matching the upper bound of Eq.~\eqref{eq:Summary}, would not yield new information about typical minima. We also emphasize that proving a matching lower bound involves significant technical challenges.}.

\paragraph*{Summary of main results:}
\begin{enumerate}
\item Under only mild smoothness assumptions on the loss, we prove an explicit upper bound of the form~\eqref{eq:Summary} on the exponential growth rate
of the expected number of local minima of $\hR_n(\bTheta)$ 
in any regular specified domains $\cuA,\cuB$ of the parameter space (Theorem~\ref{thm:general}).
\item We present an analysis of this upper bound in the case of convex losses,
showing that it implies a sharp asymptotic characterization of the empirical risk minimizer. Namely, we prove the trivialization property of Eq.~\eqref{eq:Trivialization}, giving an analytic description of $\mu_\star$ and $\nu_\star$ in terms of a finite-dimensional system of equations (Theorems~\ref{thm:convexity} and~\ref{thm:global_min}).
\item
We give explicit conditions on the problem setup that guarantee that the aforementioned system of equations has a unique solution. This is done by deriving an infinite-dimensional variational problem that can thought of as the~\emph{asymptotic equivalent} of the ERM problem in the proportional limit (Theorem~\ref{thm:simple_critical_point_variational_formula}).
\item We demonstrate how these general results can be applied to yield concrete predictions on specific problems.
We specialize to exponential families~(Proposition~\ref{propo:Exponential}) and even more explicitly to multinomial regression (Proposition~\ref{prop:multinomial}).
\end{enumerate}
Our most important result is the general 
upper bound of point 1. 
Analyzing it in the case of non-convex losses is somewhat more challenging.
In a forthcoming  paper \cite{OursInPreparation} we establish a general 
sufficient condition for rate trivialization: whenever this condition holds,
the upper bound yields sharp asymptotics for non-convex ERM. 

\subsection{Related work}
A substantial line of work studies the existence and properties
of local minima of the empirical risk for a variety of statistics 
and machine learning problems, see e.g. 
\cite{mei2018landscape,sun2018geometric,soltanolkotabi2018theoretical}. However,
concentration techniques used in these works typically require $n/d\to\infty$.

The Kac-Rice approach has a long history in statistical
physics where it has been used to characterize the landscape of 
mean field spin glasses 
\cite{bray1980metastable}. 
We refer in particular to the seminal works
of Fyodorov \cite{fyodorov2004complexity} and Auffinger, Ben Arous, Cern\`y \cite{auffinger2013random}, as well as to the recent papers
\cite{subag2017complexity,ros2023high} and references therein.
It has also a history in statistics, where it has been used for statistical analysis of Gaussian fields \cite{adler2009random}.

Within high-dimensional statistics,
the Kac-Rice approach was first used
in \cite{arous2019landscape} to characterize 
topology of the likelihood function for Tensor PCA. 
In particular, these authors showed that the expected number of modes of the likelihood can grow exponential in the dimension, hence 
making optimization intractable.
The Kac-Rice approach was used in \cite{fan2021tap} to show that Bayes-optimal estimation 
in high dimension can be performed for $\integers_{2}$-synchronization via minimization of the so-called Thouless-Anderson-Palmer (TAP) free energy. These results  were substantially strengthened in \cite{celentano2023local} which proved local convexity of TAP free energy in a neighborhood of
the global minimum. 

None of the above papers studied the ERM problem that is the focus of the present paper.
The crucial challenge to applying the Kac-Rice formula to
the empirical risk of Eq.~\eqref{eq:erm_obj_0} lies in the fact that
the empirical risk function $\hR_n(\,\cdot\,)$ is not a Gaussian process
(even for Gaussian covariates $\bx_i$).
In contrast, 
\cite{arous2019landscape,fan2021tap,celentano2023local}
treat problems  for which $\hR_n$ is itself Gaussian.
In the recent review paper \cite{ros2023high}, Fyodorov and Ros 
mention non-Gaussian landscapes as an outstanding challenge even at the non-rigorous level of theoretical physics

While in principle the Kac-Rice formula can be extended to non-Gaussian processes,
this produces an additional technical challenges. Namely, one has to prove that the 
gradient $\nabla \hR_n(\bTheta)$ has a density which is bounded in a neighborhood of $\bzero$.
In a notable paper, Maillard, Ben Arous, Biroli \cite{maillard2020landscape} 
followed this route to study the ERM problem
of Eq.~\eqref{eq:erm_obj_0} in the case $k=1$, $k_0=0$ when the loss function takes the form 
$L(\btheta^{\sT}\bx,y)= \phi(y\btheta^{\sT}\bx)$, with $y$ independent of $\bx$ (and sketch an extension to $k=1$, $k_0=1$,
with $L(\btheta^{\sT}\bx,y)= (y-\phi(\btheta^{\sT}\bx))^2$,
$y = \phi(\btheta_0^{\sT}\bx)$.)
They derive a formula for the exponential
growth rate of the expected number of local minima, which is equivalent to (a specialization of)
the formula in our Theorem \ref{thm:general}. Their proof requires an important
technical assumption on the Hessian $\nabla^{2}\hR_n(\btheta)$. Roughly speaking: the Hessian is assumed
not to have eigenvalues too close to $0$, with probability extremely close to one,
uniformly over a certain domain. While it stands to reason that this assumption holds in many cases of interest, proving it is non-trivial.

In order to overcome the problem of dealing with a non-Gaussian gradient $\nabla \hR_n(\bTheta)$,
we apply the Kac-Rice formula to a process in $(n+d)k + n k_0$ dimensions,
which we refer to as the `gradient process,' and which 
`lifts' the gradient $\nabla \hR_n(\bTheta)$. This approach is related to the one
of  \cite{maillard2020landscape}, but differs in the technical part.

The gradient process has two important properties: $(i)$~it is Gaussian;
$(ii)$~zeros of the gradient process (with certain additional conditions) correspond to local minima of the empirical risk. We believe this route presents some advantages, in that it simplifies
bounding the density. Despite these advantages, applying Kac-Rice to the gradient process requires some care
since the latter is defined on a manifold. 


We finally note that, in the case of convex losses, alternative proof techniques have been used
to characterize the empirical risk minimizer for $k>1$,
based on the analysis of
approximate message passing (AMP) algorithms. This approach was initially developed to
analyze the Lasso \cite{BayatiMontanariLASSO} and subsequently refined and extended, see e.g.
\cite{donoho2016high,sur2019modern,loureiro2021learning,ccakmak2024convergence}.
While our main motivation is to move beyond convexity, our approach recovers and unifies earlier results,
with some technical improvements. 

\subsection*{Notations}
We denote by $\cuP(\Omega)$ the set of (Borel) probability measure on 
 $\Omega$ (which will always be Polish equipped with its Borel $\sigma$-field). 
Additionally, we denote by $\cuP_n(\Omega)\subset \cuP(\Omega)$ the set of empirical measures on $\Omega$ of $n$ atoms.
 Throughout we assume $\cuP(\Omega)$ is endowed with the topology induced by the Lipschitz bounded metric, defined by
\begin{equation}
\nonumber
\dBL(\mu,\nu) = \sup_{\vartheta\in \cF_{\textrm{LU}}} 
    \left|\int \vartheta(\omega) \,\mu(\de\omega) - \int \vartheta(\omega) \,\nu(\de\omega)\right|
    \, ,
\end{equation}
where $\cF_{\mathrm{LU}}$ is the class of Lipschitz continuous functions $\vartheta:\Omega\to \R$ with Lipschitz constant at most 1 and uniformly  bounded by 1.
We will also use $W_2(\mu,\nu)$ to denote the Wasserstein 2-distance.

Letting $\Omega = \Omega_1\times\Omega_2$, $\ba_2 $ a generic point in $\Omega_2$, and $\mu\in \cuP(\Omega)$, we denote by $\mu_{\cdot|\ba_2}\in\cuP(\Omega_1)$ to be the regular conditional distribution of $\mu$ given $\ba_2$.  
We also denote by $\mu_{(\ba_2)}\in\cuP(\Omega_2)$  the restriction of $\mu$ to $\Omega_2$ (refer to  Section D.3 \cite{dembo2009large} for more about product spaces).

We let $\sS^k$ be the
set of $k\times k$ symmetric matrices with entries in $\R$, and 
$\sS^k_{\succeq 0},\sS^k_{\succ 0}$ the subsets of positive semidefinite, positive definite matrices, respectively.
For a matrix $\bZ \in\C^{k\times k}$, define
\begin{equation}
\nonumber
\Re(\bZ) =  \frac12 \left(\bZ + \bZ^*\right),\quad
    \Im(\bZ) := \frac1{2i } \left(\bZ - \bZ^*\right)\, ,
\end{equation}
and let 
\begin{equation}
\nonumber
\bbH_+^k := \left\{\bZ \in\C^{k\times k} : {\Im(\bZ)} \succ \bzero\right\},
\quad 
\bbH_-^k := \left\{\bZ \in\C^{k\times k} : \Im(\bZ) \prec \bzero\right\}.
\end{equation}
Note that $\Re(\bZ)$ and $\Im(\bZ)$ are self-adjoint for any $\bZ$.
Given a self-adjoint matrix $\bA\in\C^{n\times n}$, we denote by $\lambda_1(\bA)\ge,
\dots,\ge \lambda_n(\bA)\in\reals$ its  eigenvalues, and by
$\lambda_{\min}(\bA), \lambda_{\max}(\bA)$ the minimum and  maximum
eigenvalues, respectively. 
For $m\le n$, we use $\sigma_1(\bA)\ge \dots \ge\sigma_{m}(\bA)$ for the singular values of a matrix $\bA\in\C^{m\times n}$, and $\sigma_{\max}=\sigma_1,\sigma_{\min}=\sigma_m$ to denote the minimum and maximum of these. For a square block matrix $\bA = (\bA_{i,j})_{i,j \in[p]}$, we denote the Schur complement with respect to square block $\bA_{j,j}$ by $\bA/\bA_{j,j}$.

For a function $g:\R\to\R$ and $\bA\in\R^{m\times n}$, we use $g(\bA)$ to denote the element-wise application of $g$ to $\bA$.
We use $\grad_\ba f \in \R^m,\grad_\ba^2 f \in\R^{m\times m}$ to denote the Euclidean gradient and Hessian of $f:\R^m \to \R$ with respect to $\ba\in\R^m$, for some $m$. Furthermore, for function $f :\R^{m\times n} \to \R$, we denote by $\grad^2_\bA f(\bA)  \in\R^{mn\times mn}$
its Hessian after identifying $\R^{m\times n}$ with $\R^{mn}$.
Similarly, for a function $\boldf :\R^{m} \to \R^{n},$  we denote by $\bJ_\ba \boldf \in\R^{n \times m}$ its Jacobian matrix with respect to $\ba\in\R^m$. For $\bF:\R^{n\times m} \to \R^{k\times p},$ 
we use $\bJ_{\bA} \bF \in\R^{kp \times nm}$ to denote the Jacobian with respect to $\bA \in\R^{n\times m}$ after vectorizing the input and output by concatenating the columns. We'll often omit the argument in the subscript when it is clear from context.
For matrices $\bA\in\R^{n\times m},\bB\in\R^{n\times k}$, we will use $[\bA,\bB]\in\R^{n\times (m+k)}$ to denote the matrix whose rows are the concatenation of the rows of $\bA$,$\bB$. For vectors $\ba\in\R^k,\boldb \in\R^m$, we will often use $(\ba,\boldb)\in\R^{k+m}$ to denote their concatenation, instead of the more cumbersome notation $[\ba^\sT,\boldb^\sT]^\sT$.

We use $|S|$ to denote the cardinality of set $S$.
We denote the Euclidean ball of radius $r$ and center $\ba$ in $\reals^d$ by $\Ball^d_r(\ba)$. We similarly use $\Ball^{n\times d}_r(\bA)$ to denote the Frobenius norm ball of matrices centered at $\bA\in\R^{n\times d}$.
For two distributions $\nu,\mu$, we use
$\KL(\nu\|\mu)$ to denote their KL-divergence.
%
\section{Main results I: General empirical risk minimization}
In this section, we present our result on the rate function of a general empirical risk minimization problem. We begin with the next section summarizing the definitions necessary to state our results.
\subsection{Definitions}
\label{sec:definitions}

\paragraph*{General.}
In what follows,
    $\bx_i\in\R^d,
    \bTheta = (\Theta_{i,j})_{i\in[d],j\in[k]} \in \R^{d\times k}$, $\bTheta_0\in \R^{d \times k_0}$,
    $\bw:= (w_1,\dots,w_n) \in \R^n$, $\ell : \R^{k+k_0+1} \to \R$,
     $\bphi : \R^{k_0 + 1} \to \R,$
$\rho :\R \to \R$.
We will use $(\btheta_j)_{j\in[k]}$,$(\btheta_{0,j})_{j\in[k_0]}$ to denote the $d$-dimensional columns of $\bTheta,\bTheta_0$, respectively.
We will explicate the dependence of $\by_i$ on $(\bTheta_0^\sT\bx_i,w_i)$ and write for each $i\in[n]$,
\begin{equation}
\nonumber
    \ell(\bTheta^\sT \bx_i, \bTheta_0{^\sT} \bx_i, w_i) = L(\bTheta^\sT\bx_i, \bphi(\bTheta_0^\sT\bx_i,w_i)).
\end{equation}
Hence, from a mathematical viewpoint, the empirical risk \eqref{eq:erm_obj_0} is equivalent 
to
\begin{equation}
\label{eq:erm_obj}
    \hat R_n(\bTheta) = \frac1n \sum_{i=1}^n \ell(\bTheta^\sT \bx_i, \bTheta_0{^\sT} \bx_i, w_i) + \frac1d\sum_{j=1}^d \sum_{i=1}^k\rho(\sqrt{d}\Theta_{i,j})\, .
\end{equation}
Of course, from the statistical viewpoint, estimation proceeds by minimizing \eqref{eq:erm_obj_0}, without knowledge of $\bTheta_0$.


Given $\mu \in\cuP(\R^{k} \times \R^{k_0})$, we denote its second moment matrix by
\begin{align}\label{eq:Rmatrix_Def}
\bR(\mu) :=\int \bt\bt^{\sT} \mu(\de\bt) =  \begin{bmatrix}
\bR_{11} & \bR_{10}\\
\bR_{01} & \bR_{00}\\
\end{bmatrix},\quad\quad \bR_{11} \in\R^{k\times k} ,\bR_{00} \in \R^{k_0\times k_0}.
\end{align}
Given $\bTheta\in\R^{d\times k}$, we define the empirical measures (note that $\hnu(\bTheta)$
is redefined as to encode potentially more information):
\begin{equation}\label{eq:First-Emp-Dist}
    \hmu(\bTheta):= \frac1d\sum_{j=1}^d \delta_{\sqrt{d} 
    (\bSigma^{1/2}[\bTheta,\bTheta_0])_j} \in\cuP(\R^{k+k_0})
    \quad\quad
    \hnu(\bTheta):= \frac1n\sum_{i=1}^n \delta_{[\bTheta^\sT\bx_i, \bTheta_0^\sT\bx_i, w_i]} \in\cuP(\R^{k+k_0+1}).
\end{equation}
In particular  
\begin{align}\label{eq:Rmatrix_Def_empirical}
\bR(\hat\mu(\bTheta)) :=\begin{bmatrix}
\bTheta^{\sT}\bSigma\bTheta &\;\; \bTheta^{\sT}\bSigma\bTheta_0\\
\bTheta_0^{\sT}\bSigma\bTheta &\;\; \bTheta_0^{\sT}\bSigma\bTheta_0\\
\end{bmatrix}.
\end{align}

Finally, for $\bV\in \R^{n\times k}, \bV_0 \in\R^{n\times k_{0}}$ and $\bw\in\R^n$, with $(\bV)_i \in\R^k,(\bV_{0})_i\in\R^{k_0}$ denoting the $i$-th row of $\bV,\bV_0$ respectively, 
we introduce the following notation for the derivatives of the loss and regularizer
\begin{equation}
\label{eq:def_bL_bRho}
  \bL(\bV,\bV_0;\bw)=\left(\frac{\partial}{\partial{z_j}} \ell(\bz,(\bV_{0})_i, w_i)\Big|_{\bz = (\bV)_i}\right)_{\substack{i\in[n]\\j\in[k]}},
  \;\quad
  \bRho(\bTheta) :=  \frac1{\sqrt{d}}\left(\rho'(\sqrt{d} \Theta_{i,j})\right)_{i\in[d],j\in[k]},
\end{equation}
and let
\begin{equation}
\label{eq:def_G}
\bG(\bV,\bV_0,\bTheta;\bw) := \frac1n \bL(\bV,\bV_0;\bw)^\sT[\bV,\bV_0] + \bRho(\bTheta)^\sT[\bTheta,\bTheta_0] \in\R^{k\times(k+k_0)}.
\end{equation}
In particular, notice that $\bG(\bX\bTheta, \bX\bTheta_0,\bTheta;\bw) = 
\bzero$ is equivalent to the set of equations 
\begin{equation}
\btheta_{j}^\sT\grad_{\btheta_i}\hat R_n(\bTheta) =0,\quad\quad
\btheta_{0,l}^\sT\grad_{\btheta_i}\hat R_n(\bTheta) =0,\quad\quad i,j\in[k], j\in[k_0]. \label{eq:Thetagrad}
\end{equation}
\paragraph*{Rationale for the definition of $\bG(\bV,\bV_0,\bTheta;\bw)$} 
In our application of the Kac-Rice formula, we will identify a minimizer of the 
empirical risk with a pair $(\bV,\bTheta)\in\reals^{n\times k}\times\reals^{d\times k}$
(where $\bV=\bX\bTheta$) and define
the gradient process as a function of $(\bV,\bTheta)$. It follows from Eq.~\eqref{eq:Thetagrad}
 that all stationary points must belong to the deterministic submanifold of $\reals^{n\times k}\times\reals^{d\times k}$ defined by
$\{(\bV,\bTheta):\, \bG(\bV,\bV_0,\bTheta;\bw) =\bzero\}$. Conditioning on $(\bV_0,\bw)$,
this is a non-random manifold, and it will be important to view 
the gradient process a stochastic process defined 
on this manifold.


%
%
\paragraph*{The set of local minimizers}
We will consider local minimizers of the ERM problem satisfying a set of constraints. 
Namely, for $\cuA\subseteq \cuP(\R^{k+k_0}),\cuB\subseteq\cuP(\R^{k+k_0}\times \R)$, 
we define 
\begin{align}
\label{eq:set_of_zeros_main}
\cZ_n&(\cuA,\cuB):=
\Big\{\bTheta\in \R^{d\times k}:\; 
(\hmu(\bTheta),\hnu(\bTheta))\in \cuA\times \cuB,
\nabla \hR_n(\bTheta)=\bzero,\; 
\nabla^2 \hR_n(\bTheta)\succeq\bzero
\Big\} \cap \cG_n
\end{align}
where $\cG_n$ is the following set of `well-behaved' vectors  $\bTheta$:
\begin{align}\label{eq:G-Def}
\cG_n   :=
\big\{ \bTheta:  \;  \sfa_{R}^2\prec \bR(\hmu(\bTheta))\prec \sfA_{R}^2,\;
&\sigma_{\min}(\bL(\bX\bTheta,\bX\bTheta_0))> \sqrt{n}\,\sfa_{L},\;\\
&\quad\quad\sigma_{\min}\left( \bJ\bG(\bX\bTheta,\bX\bTheta_0,\bTheta) \right) > \,\sfa_{G,n}
\big\},
\nonumber
\end{align}
where $\sfa_{G,n}\sfa_R,\sfA_R,\sfa_L$ are strictly positive real parameters and
$\bJ \bG := \bJ_{[\bV,\bV_0,\bTheta]}\bG$ denotes the Jacobian of the map defined in Eq.~\eqref{eq:def_G}. We note that, once lifted to the space $(\bTheta,\bV,\bV_0)$,
with $\bV=\bX\bTheta$, $\bV_0=\bX\bTheta_0$, the conditions
$\cG_n$ are purely deterministic. For instance the last one can be rewritten as 
$\sigma_{\min}\left( \bJ\bG(\bV,\bV_0,\bTheta) \right) > \sfa_{G,n}$, and amounts to require that $(\bV,\bTheta)$
is not a near-singular point of
the  manifold $\{(\bV,\bTheta): \bG(\bV,\bV_0,\bTheta)  =\bzero\}$. The other two conditions have similar interpretations of implying regularity of the gradient process. In practice these conditions mean that near-singular values of $(\bTheta,\bV)$
must be treated case-by-case.

Both $\cZ_n$ and $\cG_n$ are functions of $\bTheta_0$ and $\bw$, as well the parameters:
    $\sPi := ( \sfa_{G,n}\sfa_R,\sfA_R,\sfa_L)$, 
    but we suppress this in the notation.
We denote the cardinality of $\cZ_n$ by
\begin{align}
\label{eq:number_of_zeros_main}
& Z_n(\cuA,\cuB) :=  |\cZ_n(\cuA,\cuB)|.
\end{align}
%
%
\paragraph*{Random matrix theory}
Our general result on the exponential growth rate  
of the expected number of local minima will be stated in terms of the asymptotic spectral distribution of the Hessian matrix at a local minimizer and its Stieltjes transform.
It is useful to write down the structure of the Hessian
\begin{equation}\label{eq:FirstHessian}
\nabla^2\hR_n(\bTheta) := \frac{1}{n}\left(\bX\otimes \bI_k\right)^\sT \bK\left(\bX\otimes\bI_k\right) +  \grad^2 \rho(\bTheta)\, ,\;\;
\bK:=\Diag\left(\nabla^2\ell(\bv_i,\bv_{0,i},w_i)\right)_{i\le n}\,, 
\end{equation}
 where $\bv_i := \bTheta^{\sT}\bx_i$, 
 $\bv_{0,i}:=\bTheta_0^{\sT}\bx_i$, and $\Diag$ denotes a block-diagonal matrix with the indicated blocks. (In the proofs
 we will use an equivalent but slightly different convention.)
 
To define the asymptotic Stieltjes transform,
for $z\in\C$ with $\Im(z) >0$, $\bS\in\C^{k\times k}$ a symmetric matrix
and $\nu\in\cuP({\R^{k+k_0+1}})$,
we let  $\bF_z(\bS; \nu)\in\C^{k\times k}$
\begin{equation}
\label{eq:bF_def}
     \bF_z(\bS; \nu) := \Big( \E_\nu[(\bI + \grad^{2}\ell(\bv,\bv_0,w)\bS )^{-1}
     \grad^{2}\ell(\bv,\bv_0,w)
     ] - z \bI \Big)^{-1}\,.
\end{equation}
We let $\bS_\star(\nu;z)$ be the the unique solution to 
\begin{equation}
\label{eq:fp_eq}
\bF_z(\bS_\star;\nu )= \alpha\bS_\star\quad\quad
\Im(\bS_\star) \succ\bzero,
\end{equation}
and define $\bS_0(\nu) = \lim_{\eps\to 0} \bS_\star(\nu;i\eps)$.
Existence and uniqueness of $\bS_\star(\nu;z)$ and  $\bS_0(\nu)$ are proven in 
Appendix~\ref{sec:RMT}, which extends classical results on the Marchenko-Pastur law. 
Appendix~\ref{sec:RMT} also proves that the function
\begin{equation}
\nonumber
    z \mapsto \frac1k \Tr(\alpha\bS_\star(\nu;z))
\end{equation}
is the  Stieltjes transform on a unique measure on $\reals$ which we denote by
$\mu_{\MP}(\nu;\alpha)\in\cuP(\R)$. 
In the case $k=1$, $\mu_{\MP}(\nu;\alpha)$ is the free multiplicative convolution of the law of $\nabla^2\ell_{\#}\nu$ (the push-forward of $\nu$ under $(\bv,\bv_0,w)\mapsto \nabla^2\ell(\bv;\bv_0,w)$) and asymptotic spectral distribution of a Wishart matrix (a `free Poisson'). Freeness emerges because,
conditional on $\bX\bTheta$, $\bX\bTheta_0$, 
$\bX$ is free with $\Diag\left(\nabla^2\ell(\bv_i,\bv_{0,i},w_i) \right)_{i\le n}$, cf. Eq.~\eqref{eq:FirstHessian}. For general $k$, 
our result can also be stated in the language of operator-valued free probability \cite{mingo2017free}.

For any fixed $\bTheta$, the empirical spectral distribution of the regularization term
(as an element of $\R^{dk \times dk}$) is given by
\begin{equation}
\nonumber
    \mu_{\rho}(\bTheta) :=  \frac1{dk}\sum_{j=1}^k \sum_{i=1}^d \delta_{\rho''(\sqrt{d}\Theta_{i,j})}
    = \frac1k \sum_{j=1}^k \rho''_{\#} \hmu_{j} ,\quad\quad\textrm{where}\quad\quad
    \hmu_j := \frac1d \sum_{i=1}^d \delta_{\sqrt{d} \Theta_{i,j}}.
\end{equation}
We will often abuse notation and denote this by $\mu_{\rho}( \hmu(\bTheta))$ since this is purely a function of the empirical distribution of $\bTheta.$
Finally, we let 
\begin{equation}
\nonumber
    \mu_\star(\nu,\mu) := \mu_{\MP}(\nu;\alpha) \boxplus \mu_{\rho}(\mu) 
\end{equation}
where $\boxplus$ denotes the free additive convolution.

\subsection{General assumptions}
\label{sec:assumptions}
\begin{assumption}[Regime]
\label{ass:regime}
We assume the proportional asymptotics, i.e., $d := d_n$ such that
   \begin{equation}
       \alpha_n := \frac{n}{d_n} \to \alpha \in (0, \infty), \quad\quad\textrm{and} \quad\quad k,k_0 \le C
   \end{equation}
   for some universal constant $C$ independent of $n$.
\end{assumption}

\begin{assumption}[Loss]
\label{ass:loss}
\label{ass:density}
The following derivatives of $\ell:\R^{k+k_0+ 1} \to\R$ exist and are Lipschitz:
\begin{equation}
\nonumber
    \frac{\partial}{\partial z_l}\ell(\bz),\quad\quad
    \frac{\partial^2}{\partial {z_j} \partial{z_l}} \ell(\bz),\quad\quad
    \frac{\partial^3}{\partial z_i \partial z_j \partial z_l}\ell(\bz), \quad\quad\textrm{for}\quad\quad i,j\in[k+k_0], l\in[k], 
\end{equation}
\end{assumption}

\begin{assumption}[Regularizer]
\label{ass:regularizer}
$\rho:\R \to \R$ 
is three times continuously differentiable. 
\end{assumption}

\begin{assumption}[Data distribution]\label{ass:Data}
The covariates $\bx_i$ are i.i.d. with $\bx_i \sim\cN(\bzero,\bSigma)$ for $\bSigma$ invertible.
If $\rho\neq 0$,
we assume $\bSigma = \bI$.
The random noise vector $\bw\in\R^n$ is independent of $\{\bx_i\}_{i\in[n]}$
and satisfies
\begin{equation}
    \hnu_\bw := \frac1n \sum_{i=1}^n \delta_{w_i} \stackrel{W_2}{\Rightarrow} \P_w
\end{equation}
in probability, for some distribution $\P_w$. 
In particular, this implies that for some $\sfA_w >0$, 
\begin{equation}
    \lim_{n\to\infty}\P\left(\|\bw\|_2 > \sfA_w \sqrt{n}\right)   = 0.
\end{equation}

\end{assumption}

\begin{assumption}[Distribution of $\bTheta_0$]
\label{ass:theta_0}
   For some $\mu_0 \in\cuP(\R^{k_0})$ with $\bR_{00} :=
   \int\bt_0\bt_0^\sT\mu_0(\de\bt_0)\in\R^{k_0\times k_0},$ we have the convergence 
   \begin{equation}
\nonumber
       \hmu_0 := \frac1d \sum_{j=1}^d \delta_{\sqrt{d}(\bSigma^{1/2}\bTheta_{0})_{j}}\stackrel{W_2}{\Rightarrow} \mu_0 \in \cuP(\R^{k_0}).
   \end{equation}
\end{assumption}

\begin{assumption}[The parameters $\sPi$]
\label{ass:params}
 The parameters $\sfa_{L},\sfa_R,\sfA_{R}$ are fixed independent of $n,d$ in the proportional asymptotics. 
 The parameter $\sfa_{G,n}$ can depend on $n$, subject to
 $\sfa_{G,n} = e^{-o(n)}$.
\end{assumption}


\begin{assumption}[Constraint sets]
\label{ass:sets}
The constraint sets $\cuA\subseteq \cuP(\R^{k+k_0}),\cuB\subseteq \cuP(\R^{k+k_0+1})$ are measurable, and 
for all $n,d$, the set
\begin{equation}
\nonumber
    \{\bTheta :
    \in\R^{d\times k} : \hmu(\bTheta) \in \cuA\},\quad\quad
    \Big\{\bbV \in\R^{n\times (k+k_0)} : \frac1n \sum_{i=1}^n \delta_{(\bbV)_{i},w_i} \in\cuB\Big\} 
\end{equation}
are open for any fixed $\bw \in \R^n$, where $(\bbV)_{i}$ denotes the $i$th row of $\bbV.$
\end{assumption}
\begin{remark}
   Since $\bTheta\mapsto \hmu(\bTheta)$ is continuous under $W_2$ on $\cuP(\R^{k+k_0}),$ a sufficient condition or the last assumption would be for $\cuA,\cuB$ to be open in the topology induced by $W_2$.
\end{remark}

\subsection{General variational upper bound on the number of critical points}
We state our main theorem, which gives a general formula that upper bounds the number of critical points of the risk~\eqref{eq:erm_obj}.

In what follows, when considering probability measures $\mu \in\cuP(\R^{k}\times \R^{k_0})$ and $\nu\in\cuP(\R^{k} \times \R^{k_0} \times \R),$
we use $(\bt,\bt_0) \in\R^{k} \times \R^{k_0}$ and $(\bv,\bv_0,w)\in\R^{k}\times \R^{k_0}\times \R$
to denote random variables with these distributions, respectively.

For a given $\cuA\subseteq\cuP(\R^{k+k_0})$ and $\cuB \subseteq \cuP(\R^{k+k_0+1})$, define the set of probability measures
\begin{align}
\nonumber
 \cuV(\cuA,\cuB) := \Big\{&(\mu,\nu) \in \cuA\times \cuB :\;
\mu_{(\bt_0)} = \smu_{0},\; \nu_{(w)} = \P_w,\;
\mu_{\star}(\mu,\nu)((-\infty,0)) = 0,\\
&\hspace{10mm}
\E_\nu[\grad_\bv \ell(\bv,\bv_0,w)(\bv,\bv_0)^\sT] + 
     \E_\mu[\rho'(\bt) (\bt, \bt_0)^\sT] =   \bzero_{k\times (k+k_0)}
\Big\}.\label{eq:GeneralSet}
\end{align}
Further, define $\Phi_\gen:\cuP(\R^{k+k_0})\times \cuP(\R^{k+k_0+1})\times\sS^{k+k_0}\to\reals $
via
\begin{align}
\Phi_\gen(\mu,\nu,\bR)&:=-\frac{k}{2\alpha}\log(\alpha)
-\frac{k}{\alpha}
\int \log(\lambda ) \mu_{\star}(\nu,\mu)(\de \lambda)
  + \frac{1}{2\alpha}\log \det\left( \E_{\nu}[\grad_\bv \ell\grad_\bv\ell^\sT]\right) 
 \nonumber \\
   &\quad
+\frac1{2}
\Tr\left( (\E_\nu[\grad_\bv\ell\grad_\bv\ell^\sT])^{-1}\E_\mu[\grad \rho \grad \rho^\sT]\right)
- \frac{1}{2\alpha} \Tr\left(\bR_{11}\right)
\label{eq:PhiGen} \\
  &\quad- \frac1{2}\Tr\left(\E_\nu[(\bv,\bv_0)\grad_\bv\ell^\sT] \E_\nu[\grad_\bv\ell\grad_\bv\ell^\sT]^{-1}\E_\nu[(\grad_\bv\ell) 
  \;(\bv,\bv_0)^\sT] \bR^{-1}\right)\nonumber\\
    &\quad
   +\KL( \nu_{\cdot |w} \|\cN(0, \bR) ) + \frac1\alpha\KL(\mu_{\cdot | {\bt_0}}\| \cN(\bzero, \bI_{k})).\nonumber
\end{align}
\begin{theorem}
\label{thm:general}
Let $\bZ = \bX\bSigma^{-1/2}$ (a matrix with i.i.d. $\normal(0,1)$
entries).
Fix $\sfA_{Z} > 1+\alpha^{-1/2}$,
For $\delta >0$, define the event $\Omega_\delta$, (with $\P(\Omega_\delta)=1-o_n(1)$)

\begin{equation}\label{eq:OmegaDeltaDef}
    \Omega_\delta := 
    \Big\{\bw  \in\Ball_{\sfA_w\sqrt{n}}^n (\bzero) ,\; \dBL(\hnu_\bw, \P_w) < \delta,\; 
    \sigma_{\max}(\bZ) \le \sqrt{n}\sfA_Z
    \Big\}\, .
\end{equation}

Under  Assumptions \ref{ass:regime} to \ref{ass:params} of Section~\ref{sec:assumptions}, we have, for any $\cuA,\cuB$ as in Assumption~\ref{ass:sets},
\begin{align}\label{eq:MainResult-Bound}
 \lim_{\delta\to0}\lim_{n\to\infty}\frac{1}{n}\log\E_{\bX,\bw}\left[Z_n(\cuA,\cuB) \one_{\Omega_\delta}\right]
    \le- \inf_{(\mu,\nu) \in \cuV(\cuA,\cuB) \cap \cuG} \Phi_\gen(\mu,\nu, \bR(\mu)),
\end{align}
where 
\begin{equation}
   \cuG 
   := 
\left\{(\mu,\nu) : \sfa_{R}^2 \prec \bR(\mu) \prec\sfA_R^2,\;
\E_{\nu}[\grad_\bv\ell \grad_\bv\ell^\sT] \succ \sfa^2_{L},\;
\E_{\nu}[(\bv,\bv_0)(\bv,\bv_0)^\sT] 
\prec
\sfA_R^2 \sfA_{Z}^2 
\right\}\, . \label{eq:cuG_Def}
\end{equation}
\end{theorem}

 The above theorem holds \emph{without} any convexity assumption.
 It can be used to characterize $\hmu(\bTheta)$ and
$\hnu(\bTheta)$,
for $\bTheta$ a local minimizer of the ERM problem
inside general sets $\cuA,\cuB$, if we can lower bound the infimum of $\Phi$ over
$\cuV(\cuA,\cuB)\cap\cuG$.
In particular, this strategy can be used to determine the typical value
coordinate averages as per Eq.~\eqref{eq:markov_localization}.

While this analysis may be done on a case by case basis, we will instead establish general sufficient conditions for rate trivialization, cf. Eq.~\eqref{eq:Trivialization}.
In what follows, we present the result of this analysis for convex losses. We defer the case of non-convex losses to~\cite{OursInPreparation}.

\begin{remark}\label{remark:Gaussian_isotropic_change_of_variable}
By the change of variables $\bTheta'= \bSigma^{1/2}\bTheta$,
$\bTheta_0'= \bSigma^{1/2}\bTheta_0$, the case $\bSigma\neq \bI$
can be reduced to $\bSigma = \bI$, when $\rho=0$. We shall therefore always assume the latter in the proofs.
\end{remark}

\begin{remark}
 The bounds on the largest singular value of $\bZ$ in Eq.~\eqref{eq:OmegaDeltaDef}
holds with probability at least $1-2e^{-cn}$ by standard concentration bounds
\cite{Guionnet}. Hence, letting $\Omega_{0,\delta}:= 
\{\bw \in\Ball_{\sfA_w\sqrt{n}}^n (\bzero) ,\; \dBL(\hnu_\bw, \P_w) < \delta\}$,
$\P(\Omega_{\delta})\ge 1-\P(\Omega_{0,\delta})-2 e^{-cn}$.
\end{remark}

\begin{remark}[Universality]
Theorem \ref{thm:general} assumes Gaussian covariates.
An important open direction is to generalize it to
other distributions, a natural next step being design matrices $\bX$ with i.i.d. centered, unit variance entries (under tail conditions). Recent 
universality theorems \cite{hu2022universality,montanari2022universality} imply that typical properties of the empirical risk minimizer are often universal, and hence at least some properties of the rate function of Theorem \ref{thm:general} should also be universal (e.g., the location of its minima under rate trivialization). 
The main challenge in dealing with non-Gaussian $\bX$
is the lack of an explicit formula for the density of the gradient $\nabla\hR_n(\bTheta)$.
 \end{remark}

 \begin{remark}[Tightness]\label{rmk:Tightness}
 As mentioned in the introduction, we expect the upper bound
 of Eq.~\eqref{eq:MainResult-Bound} to hold with equality in certain 
 cases of interest. More precisely, we expect this to be the case when the Hessian $\nabla^2\hR_n(\bTheta)$ conditioned
 to $\nabla\hR_n(\bTheta)=\bzero$ does not have (with high probability) negative `outlier eigenvalues'
 (i.e. eigenvalues that are negative while the lower edge of the bulk of the spectrum is bounded below by a positive constant). In principle,  an explicit condition on $\mu$, $\nu$ can be written for this to be the case.
\end{remark}

 \begin{remark}[Ridge regularization]\label{rmk:Ridge}
 In the case of ridge regularization $\rho(t) = \lambda t^2/2$, we have 
 $\grad_{\bt}\rho(\bt) = \lambda \bt$.  
Writing $\bR$ for $\bR(\mu)$,  the constraint $\E[\grad\ell \bbv^\sT + \grad \rho (\bt,\bt_0)^\sT] = \bzero$  defining $\cuV$ implies 
\begin{align*}
   &-
\Tr\left( (\E_\nu[\grad\ell\grad\ell^\sT])^{-1}\E_\mu[\grad \rho \grad \rho^\sT]\right) 
+ \Tr\left(\E_{\nu}\left[\bbv\grad\ell^\sT\right] \E_\nu[\grad\ell\grad\ell^\sT]^{-1}\E_\nu[\grad\ell\bbv^\sT] \bR^{-1}\right)\\
&=
   -\lambda^2
\Tr\left( (\E_\nu[\grad\ell\grad\ell^\sT])^{-1}\bR_{11}\right)
+\lambda^2\Tr\left((\E_\nu[\grad\ell\grad\ell^\sT])^{-1}
 [\bR_{11},\bR_{10}] \bR^{-1} [\bR_{11},\bR_{10}]^\sT
\right)\\
&= 0\, ,
\end{align*} 
Therefore, the rate function of Eq.~\eqref{eq:PhiGen} 
simplifies to $\Phi_\gen(\mu,\nu,\bR(\mu) )= \Phi_\ridge(\mu,\nu,\bR(\mu) )$, where
\begin{align}
\Phi_\ridge(\mu,\nu,\bR)&:=-\frac{k}{2\alpha}\log(\alpha)
-\frac{k}{\alpha}
\int \log(\lambda +s) \mu_{\MP}(\nu;\alpha)(\de s)
  + \frac{1}{2\alpha}\log \det \E_{\nu}[\grad_\bv \ell\grad_\bv\ell^\sT]
 \nonumber \\
   &\quad
- \frac{1}{2\alpha} \Tr\left(\bR_{11}\right)+\KL( \nu_{\cdot |w} \|\cN(0, \bR) ) + \frac1\alpha\KL(\mu_{\cdot | {\bt_0}}\| \cN(\bzero, \bI_{k}))
\label{eq:PhiRidge} \, .
\end{align}
\end{remark}

\begin{remark}[Case $\alpha\le 1$]
In general, the Kac-Rice approach requires modifications for $n<d$, unless the risk is suitably regularized.
Indeed, in that case  each 
local minimizer of the empirical risk belongs to an affine space of minimizers of dimension 
at least $d-n$. This reveals itself in the fact that $\mu_{\star}(\nu)$ has necessarily
a point mass at $0$, and hence $\Phi_{\gen}(\mu,\nu,\bR)=\infty$ identically.

On the other hand, we obtain non-trivial results from Theorem 
\ref{thm:general} (and its consequences) when $\alpha \in (0,1]$ under a convex regularizer.
\end{remark}

\section{Main results II: Convex empirical risk minimization}
We specialize our result to the case of a convex loss and ridge regularization.  
Our approach recovers and extends results obtained with alternative techniques, which we compare to in Remark~\ref{rmk:ComparisonCVX}.


\subsection{Assumptions for convex empirical risk minimization}

We will make the following additional assumptions.
\begin{assumption}[Convex loss and ridge regularizer]
\label{ass:convexity}    
We have
$\grad_{\bv}^2 \ell(\bv,\bv_0,w) \succeq\bzero_{k\times k}$ for all
$(\bv,\bv_0,w)\in\R^k\times \R^{k_0}\times \supp(\P_w)$. 
Furthermore, for some $\lambda\ge 0$ fixed, we take 
$$\rho(t) = \frac{\lambda}{2}t^2.$$
\end{assumption}

\subsection{Simplified variational upper bound for convex losses and ridge regularizer}

Recall the definition of proximal operator: 
for $\bz \in\R^k, \bv_0 \in\R^{k_0}, \bS\in\R^{k\times k}, w\in\R,$  and $f : \R^{k+k_0+ 1} \to\R,$ we let
\begin{equation}
\label{eq:Prox-Def}
    \Prox_{f(\cdot, \bv_0, w)}(\bz;\bS):=\arg\min_{\bx\in\R^k}\left( \frac12(\bx-\bz)^\sT\bS^{-1}(\bx-\bz) + f(\bx,\bv_0,w)\right)\in\R^k.
\end{equation}
The next definition plays a crucial role in what follows: 
it characterizes the candidate minimizer of the rate function $\Phi_\gen(\mu,\nu,\bR(\mu))$ of
Eq.~\eqref{eq:PhiGen} in the convex setting, in terms of the solution of a finite system of nonlinear
equations. Existence and uniqueness of solutions will be established below by expressing these parameters 
in terms of an infinite-dimensional convex optiization problem.
\begin{definition}
\label{def:opt_FP_conds}
   We say that the pair $(\mu^\opt,\nu^\opt)\in\cuP(\R^{k+k_0}) \times \cuP(\R^{k+k_0+1})$ satisfies the \emph{critical point optimality condition} if there exist $(\bR,\bS) \in \R^{(k+k_0)\times (k+k_0)} \times \R^{k\times k}$ that satisfy the following set of equations
\begin{align}
\label{eq:opt_fp_eqs}
    & \alpha \E_{\nu^\opt}[\grad\ell(\bv,\bv_0,w)\grad\ell(\bv,\bv_0,w)^\sT]
    =\bS^{-1} (\bR/\bR_{00}) \bS^{-1}
    \\
    &\E_{\nu^\opt}\left[ \grad \ell(\bv,\bv_0,w) (\bv,\bv_0)^\sT\right] + \lambda (\bR_{11},\bR_{10}) = \bzero_{k\times (k+k_0)},
\end{align}
where 
   \begin{align}
&\nu^\opt = \mathrm{Law}(
\Prox_{\ell(\,\cdot\,, \bv_0,w)}( \bg; \bS),\bv_0,w),\quad\textrm{where}\quad
(\bg,\bv_0,w) \sim \cN\left( \bzero_{k+k_0},\bR\right)\otimes \P_w,\label{eq:nuopt}
\\
&\mu^\opt_{(\bt_0)} = \mu_{0},\quad  \mu^\opt_{\cdot|\bt_0}= \cN(\bR_{10}({\bR_{00}})^{-1}\bt_0,\bR/\bR_{00}) \quad\quad\textrm{for all}\quad\quad\bt_0\in\R^{k_0}.\label{eq:muopt}
\end{align}
\end{definition}

\begin{theorem}[Rate function under convexity]
\label{thm:convexity}
Consider the setting of Theorem \ref{thm:general}.
Let Assumptions \ref{ass:regime} to \ref{ass:sets} hold, in addition to Assumption~\ref{ass:convexity}, and
define 
\begin{align}
\Phi_\cvx(\nu,\mu,\bS,\bR)
:=&
    -\lambda\Tr(\bS)  -
    \E_{\nu}[\log\det(\bI + \grad^2 \ell\;\bS) ]  +\frac1\alpha \log\det(\bS) +\frac{k}{\alpha}\\
  &
  + \frac{k}{2\alpha} \log(\alpha)
 + \frac{1}{2\alpha}\log \det\left( \E_{\nu}[\grad \ell\grad\ell^\sT]\right)
- \frac{1}{2\alpha} \Tr\left(\bR_{11}\right)
\nonumber\\
&
+  \KL(\nu_{\cdot|w} \| \cN(\bzero, \bR)) + \frac1\alpha \KL ( \mu_{\cdot| \bt_0}\| \cN(\bzero,\bI_k)).\nonumber
\end{align}
Further, define the sets 
\begin{align}
   &\cuV_\mupart(\cuA) := \{\mu \in\cuA : \mu_{(\bt_0)} = \mu_{0} \}\\
&\cuV_\nupart(\cuB;\bR) := \Big\{\nu \in  \cuB :
\E_\nu[\grad \ell(\bv,\bv_0,w)(\bv,\bv_0)^\sT] + 
     \lambda (\bR_{11},\bR_{10}) =   \bzero,\;
 \nu_{(w)} = \P_w\; \Big\}.
\end{align}
Then the following hold.
\begin{enumerate}
    \item Defining $\cuG$ and $\Omega_\delta$ as in Theorem~\ref{thm:general}, we have:
\begin{align}
\nonumber
   &\limsup_{\delta\to 0 }\limsup_{n\to\infty}
   \frac1n\log (\E[Z_n(\cuA,\cuB) \one_{\Omega_\delta}] )
   \le 
-   \hspace{-2mm}
\inf_{\substack{\mu\in \cuT(\cuA) \\
\nu\in\cuV(\bR(\mu),\cuB)\\
(\mu,\nu) \in\cuG
}}
   \sup_{\bS\succ\bzero}
   \Phi_\cvx(\nu,\mu,\bS,\bR(\mu)),
\nonumber
\end{align}
\item For any $\mu\in\cuV_\mupart(\cuA)$, $\nu \in\cuV_\nupart(\cuB;\bR(\mu))$, 
\begin{equation}
\nonumber
    \sup_{\bS \succ\bzero} \Phi_\cvx(\mu,\nu, \bR(\mu),\bS) \ge 0
\end{equation}
with equality if and only if $(\mu,\nu)$ satisfies the critical point optimality condition 
of Definition~\ref{def:opt_FP_conds}.
\end{enumerate}
\end{theorem}

\begin{remark}[Comparison with related work]\label{rmk:ComparisonCVX}
As mentioned in the introduction, several earlier papers analyzed convex ERM in settings analogous to the one considered
here using an analysis technique based on AMP algorithms \cite{bayati2011dynamics,BayatiMontanariLASSO,donoho2016high,
sur2019modern}. The idea is to construct an AMP algorithm to minimize the empirical risk and establish sharp asymptotics for the minimizer in two steps: $(1)$~Characterize the high-dimensional asymptotics of AMP via state evolution (for any constant number of iterations); 
$(2)$~Prove that AMP converges to the minimizer in $O(1)$ iterations. 
For the case $k>1$, this approach was initiated in \cite{loureiro2021learning}, with important steps completed in \cite{ccakmak2024convergence}.

Even in the convex case, the Kac-Rice approach has some advantages. 
Earlier results require strong convexity, with Hessian lower bounded uniformly in $n,d$
\cite{ccakmak2024convergence}. 
In contrast:
$(i)$~We do not require an \emph{a priori} convexity lower bound;
$(ii)$~In particular, whenever the solution $(\mu_{\star},\nu_{\star})$ of the critical point optimality condition is unique,
and we can rule out local minima $(\hmu(\bTheta),\hnu)\in \cuG^c$, 
we can conclude that $\hmu(\bTheta)\weakc \mu_{\star}$, $\hnu(\bTheta)\weakc \nu_{\star}$;
$(iii)$~In these cases, we obtain exponential concentration around the unique 
$(\mu_{\star},\nu_{\star})$ (as anticipated in Eq.~\eqref{eq:markov_localization});
$(iv)$~Even in cases in which the solution of the critical point optimality condition is not unique (and hence there might be multiple near minimizers), Theorem \ref{thm:convexity} constrains the possible limits of $\hmu(\btheta),\hnu(\btheta)$.

Most importantly, this result is obtained as a direct application 
of our master Theorem \ref{thm:general}.
\end{remark}
%
%
\subsection{Uniqueness of the critical point and sharp asymptotics of the empirical risk}
Our next theorem establishes that the critical point optimality
conditions are the stationarity conditions for a variational functional $\Risk(\bK,\bM)$,
whose minimum is the high-dimensional limit 
of the minimum empirical risk.
For context, define $\Risk_{\infty}:\sfS^k_{\succeq}\times \R^{k\times k_0}\to\reals$ via
\begin{align}
 \Risk_{\infty}(\bK,\bM):=  \E[\ell(\bK \bz_1 + \bM \bz_0, \bR_{00}^{1/2}\bz_0, w)]  + \frac{\lambda}{2} \Tr(\bK^2 + \bM\bM^\sT) \, ,
 \end{align}
 where expectation is with respect to $(\bz_1,\bz_0,w)\sim \cN(\bzero,\bI_{k})\otimes \cN(\bzero,\bI_{k_0})\otimes \P_w$.
 This function gives the value of the population
 risk at $\bTheta$, provided $\bK$, $\bM$ encode the Gram matrices of $\bTheta$, $\bTheta_0$.
 Namely, recall the definition $\bR(\hmu(\bTheta))$ in
 Eq.~\eqref{eq:Rmatrix_Def}. Then a simple calculation yields:
\begin{align}
(\bR/\bR_{00}, \bR_{10} \bR_{00}^{-1/2}) =  (\bK^2,\bM)\;\; 
\Rightarrow \;\;\E[\hR_n(\bTheta)] = \Risk_{\infty}(\bK,\bM)\, .\label{eq:R-KM}
\end{align}

We will show that the asymptotics of the minimum empirical risk is given by a
a convex function $\Risk(\bK,\bM)$ which is related to $\Risk_{\infty}(\bK,\bM)$.
Solutions of the critical point optimality condition 
corresponds to  minimizers of $\Risk(\bK,\bM)$.

Let $(\Omega,\cF,\P)$ be a probability space with independent random variables 
$\bz_0\sim\normal(\bzero,\bI_{k_0})$,
$\bz_1\sim\normal(0,\bI_k)$ and
$w\sim\P_w$ defined on it, and additional independent randomness $Z\sim \Unif([0,1))$.
Define  $L^2 := L^2(\Omega\to\R^k,\cF,\P)$, and
\begin{equation}
\nonumber
    \cS(\bK) :=
    \left\{\bu  \in L_2 : 
    \E[\bu\bu^\sT]  \preceq \alpha^{-1}\bK^2
    \right\}\,.
\end{equation}
We then define the function $\Risk:   \sfS^k_{\succeq \bzero}\times \R^{k\times k_0}\to \R$ by
\begin{align}
\Risk(\bK,\bM):= \inf_{\bu\in\cS(\bK)} \E[\ell(\bu + \bK \bz_1 + \bM \bz_0, \bR_{00}^{1/2}\bz_0, w)]  + \frac{\lambda}{2} \Tr(\bK^2 + \bM\bM^\sT) \, .\label{eq:FKM_Def}
\end{align}

It is also useful to introduce the following   min-max problem:
   \begin{align}
   \label{eq:min_max_critical_pts}
   &\min_{\bK \succeq \bzero, \bM} \max_{\bS\succeq\bzero} \cuG(\bK,\bM;\bS)\, ,\\
       &\cuG(\bK,\bM;\bS):=
       \E\left[\More_{\ell(\cdot, \bg_0,w)}(\bg;\bS)\right] - \frac1{2\alpha}\Tr(\bS^{-1}\bK^2) 
       + \frac{\lambda}{2} \Tr(\bK^2+\bM\bM^{\sT})\, . \label{eq:min_max_critical_G_def}
   \end{align}
   where 
   \begin{align}
   \label{eq:moreau_def}
       \More_{\ell(\cdot, \bg_0,w)}(\bg;\bS)
       := \min_{\bx \in\R^{k}}\left\{ \frac12(\bx - \bg)^\sT\bS^{-1}(\bx - \bg) + \ell(\bx,\bg_0 ,w) \right\}
   \end{align}
   is the \emph{Moreau envelope}, and  the expectation is over $(\bg,\bg_0) \sim \cN(\bzero,\bR), w\sim \P_w$
   (with $(\bR/\bR_{00}, \bR_{10} \bR_{00}^{-1/2}) =  (\bK^2,\bM)$). 

\begin{theorem}
\label{prop:simple_critical_point_variational_formula}
\label{thm:simple_critical_point_variational_formula}
Under Assumption~\ref{ass:convexity}, any solution $\bR^\opt$ of the critical point optimality condition of 
Definition~\ref{def:opt_FP_conds}   corresponds to a minimizer of the convex function 
$(\bK,\bM)\mapsto \Risk(\bK,\bM)$. 

Namely:
\begin{enumerate}
\item $\Risk$ is convex on $\sfS^k_{\succeq \bzero}\times \R^{k\times k_0}$.
\item The minimizers of $\Risk$ are in one-to-one correspondence with the solutions $\bR^\opt$ of the
critical point optimality condition of  Definition~\ref{def:opt_FP_conds}, via 
\begin{equation}
\label{eq:Corr_Ropt_Kopt}
 (\bR^{\opt}/\bR_{00}, \bR^\opt_{10} \bR_{00}^{-1/2}) =  ((\bK^\opt)^2,\bM^\opt)\, 
\end{equation}
where $\bK^\opt,\bM^\opt$ are the minimizers of $\Risk(\bK,\bM).$
\item  $\Risk$ has a unique minimizer if
either of the following conditions hold:
\begin{enumerate}
    \item $\lambda>0$. In this case the minimizer $(\bK,\bM)$ exists and is unique.
    \item $\lambda =0$ and $\ell(\, \cdot\, , \bv_0, w)$ is strictly convex for all $\bv_0,w$.
    In this case either the minimizer $(\bK,\bM)$ exists uniquely, or any minimizing sequence diverges.
    \item There exist Borel sets $\cE_0, \cE_1$ and a point $\bx_0\in\reals^k$
    such that $\P(\bR_{00}^{1/2}\bz_0\in \cE_0)>0, \P(w\in \cE_1)>0$ and, for $(\bR_{00}^{1/2}\bz_0,w)\in\cE_0\times \cE_1$, $\ell(\,\cdot\,; \bR_{00}^{1/2}\bz_0,w)$ 
    is strictly convex at $\bx_0$. 
    Further, $\Risk$ admits a minimizer $(\bK^\opt,\bM^{\opt})$ such that $\bK^{\opt}\succ \bzero$
    and 
    the function $\bS\mapsto \cuG(\bK^{\opt},\bM^{\opt};\bS)$ is maximized at some $\bS^{\opt}\succ \bzero$.
\end{enumerate}
\item 
We have 
\begin{align}
\lim_{r\to\infty}\lim_{n,d\to\infty}\inf_{\|\bTheta\|_F\le r}\hR_n(\bTheta)=\inf_{\bK\succeq \bzero}
\Risk(\bK,\bM)\, .\label{eq:IntepretationRisk}
\end{align}
\end{enumerate}
\end{theorem}
\begin{remark}
Under uniform convergence, Eq.~\eqref{eq:IntepretationRisk} would hold with $\Risk$
replaced by $\Risk_{\infty}$. The function $\Risk(\bK,\bM)$ has the interpretation
of the the minimum train error over $\bTheta$ satisfying 
$(\bR/\bR_{00}, \bR_{10} \bR_{00}^{-1/2}) =  (\bK^2,\bM)$, while 
$\Risk_{\infty}(\bK,\bM)$ is the corresponding test error. 
Note that the difference $\Risk_{\infty}(\bK,\bM)-\Risk(\bK,\bM)$
is always non-negative and is due to the minimization over the random variable $\bu$ in
Eq.~\eqref{eq:FKM_Def}. Since $\bu$ is of order $1/\sqrt{\alpha}$ we recover 
the uniform convergence result that the generalization error is of order $\sqrt{d/n}$.
\end{remark}
%
%
\subsection{Consequences of rate trivialization}

\begin{theorem}
\label{thm:global_min}
With the definitions of Theorem~\ref{thm:convexity}, 
let Assumptions~\ref{ass:regime},~\ref{ass:loss},~\ref{ass:Data},~\ref{ass:theta_0} and ~\ref{ass:convexity}
hold.

Assume there exists unique $\mu^\opt,\nu^\opt$ that satisfy 
 the critical point optimality condition of Definition~\ref{def:opt_FP_conds}
 (see Theorem \ref{prop:simple_critical_point_variational_formula} for sufficient conditions), 
and 
there exist constants $C,c>0$ and (for each $n$) a minimizer $\hat\bTheta_n\in\reals^{d\times k}$
such that, with high probability, 
\begin{align}\nonumber
\hat\bTheta_n \in \cE &:= \Big\{
\bTheta : \, \bM_{\bTheta}^\sT\nabla^2 \hR_n(\bTheta)
\bM_{\bTheta}\succeq e^{-o(n)},\;\;
\E_{\hnu(\bTheta)}[\grad\ell \grad\ell^\sT] \succ c\bI
\Big\}\,.\label{eq:SetUniqueness}\\
\bM_{\bTheta}&:=\bI_k\otimes[\bTheta,\bTheta_0]\, .\
\end{align}

Then the following hold:
\begin{enumerate}
    \item We have, in probability, 
    \begin{equation}
    \hmu(\hat\bTheta_n) \stackrel{W_2}{\Rightarrow} \mu^\opt ,\quad\quad
    \hnu(\hat\bTheta_n) \stackrel{W_2}{\Rightarrow} \nu^\opt\, .
    \end{equation}
    \item The empirical spectral distribution of the rescaled Hessian 
    $\bSigma_k^{-1/2}\grad^2\hat R_n(\hat\bTheta_n)\bSigma_k^{-1/2}$
    (where 
    $\bSigma^{-1/2}_k := \bSigma^{-1/2}\otimes \bI_k$)
    at the minimizer converges weakly to $\mu_\star(\nu^\opt)$ in probability. 
  \end{enumerate}
\end{theorem}

\begin{remark}
   Notice that the set in Eq.~\eqref{eq:SetUniqueness} does not include any constraint on the
 Jacobian of $\bG$, which instead enters the set \eqref{eq:GeneralSet} of Theorem~\ref{thm:general}. 
The proof Theorem \ref{thm:global_min}   lower bounds the minimum singular value of the Jacobian of $\bG$ in terms of the minimum singular values of $[\hat\bTheta,\bTheta_0]$ and the eigenvalues of $\grad^2\hat R_n$ in the direction of $[\hat\bTheta,\bTheta_0]$
 (See Section~\ref{sec:pf_convex_results} and
 Lemma~\ref{lemma:jacobian_lb} in the appendix).

 Note that, defining $\obv_i :=(\bv_i, \bv_{0,i})$, and recalling the
 definition of $\bK$ from Eq.~\eqref{eq:FirstHessian}, we have
 \begin{align*}
\bM_{\bTheta}^\sT\nabla^2 \hR_n(\bTheta)\bM_{\bTheta} =
 \frac{1}{n}(\bI_k\otimes [\bV, \bV_0])^{\sT}\bK(\bI_k\otimes [\bV, \bV_0])
  = \frac{1}{n}\sum_{i=1}^n \nabla^2\ell(\bv_i,\bv_{0,i},w_i)\obv_i \obv_i^{\sT}\, .
 \end{align*}
 Therefore the condition on the minimum eigenvalue along $[\bTheta,\bTheta_0]$ is a condition
 about the empirical distribution $\hnu(\hbTheta_n)$. 
\end{remark}

%
%

\section{Application: Exponential families and multinomial regression}
We demonstrate how to apply the theory in the last section to the regularized maximum-likelihood
estimator (MLE) in exponential families. A large number of earlier works
studied the case of exponential families with a single parameter, a prominent example being logistic
regression \cite{sur2019modern,candes2020phase,montanari2019generalization,deng2022model,zhao2022asymptotic}.

Here we obtain sharp high-dimensional asymptotics in the general case. 

\subsection{General exponential families}

Fix $q$ and $k = k_0$.
Given $\btau:\R^{q} \to \R^{k}$, and a reference measure $\nu_0$ on $\reals^q$, 
define the probability distribution on $\R^q$ 
\begin{align}
 \rP(\de\by|\bfeta)=  e^{\bfeta^\sT \btau(\by) - a(\bfeta)}
 \nu_0(\de\by)
 \, ,\;\;\;\;
 a(\bfeta) := \log\left\{\int e^{\bfeta^\sT \btau(\by)}
 \nu_0(\de\by)\right\}\, .\label{eq:ExpoDef1}
 \end{align}
 
We  assume to be given i.i.d. samples $(\bx_i,\by_i) \in \R^d\times \R^{q}$  for $i\le n$,
where $\bx_i\sim\normal(\bzero,\bSigma)$ and 
\begin{align}
    \P(\by_i\in S|\bx_i) = \rP(S|\bTheta_0^{\sT}\bx_i)\, .\label{eq:ExpoDef2}
\end{align}
We consider the regularized MLE defined 
as the minimizer of the following risk (for $\lambda\ge 0$)
\begin{equation}
     \hR_n(\bTheta) := \frac1n \sum_{i=1}^n \left\{ a(\bTheta^\sT \bx_i) -
      \<\bTheta^{\sT}\bx_i  ,\btau(\by_i)\>
      \right\} + \frac{\lambda}{2} \|\bTheta\|_F^2\, .\label{eq:RiskExpo}
\end{equation}

In this case, the critical point optimality conditions of Definition \ref{def:opt_FP_conds}
reduce to the following set of equations for  $\bR\in\sfS^{2k}_{\succeq \bzero}$ 
and $\bS \in \sfS^k_{\succ \bzero}$:
   \begin{align}
    &\alpha \; \E[(\grad a(\bv) - \btau(\by)) (\grad a(\bv) -\btau(\by))^\sT]  = \bS^{-1}(\bR/ \bR_{00}) \bS^{-1}\, ,\label{eq:ExpoFP1}\\
    &\E\left[ (\grad a(\bv) - \btau(\by))(\bv,\bg_0)^\sT\right] + \lambda (\bR_{11},\bR_{10}) = \bzero_{k\times (k+k)},\label{eq:ExpoFP2}
\end{align}
where, letting $\ell(\bv,\by):= a(\bv) - \bv^\sT \btau(\by)$,
\begin{align}
\bv  = \Prox_{\ell(\,\cdot\,,\by)}( \bg; \bS),\quad\quad
\by \sim \rP(\by | \bfeta=\bg_0),\quad
(\bg,\bg_0) \sim \cN\left( \bzero_{k+k},\bR\right). \label{eq:ExpoFP3}
\end{align}

Here we state a simple result for exponential families 
under strong convexity.
In the next section we show that our general results 
(Theorems \ref{thm:convexity} and \ref{thm:global_min}) can also be applied when strong 
convexity fails, by considering the case of multinomial regression.
\begin{proposition}\label{propo:Exponential}
Consider the exponential family of Eqs.~\eqref{eq:ExpoDef1}, \eqref{eq:ExpoDef2}, 
and assume that  $a_1\bI_m\preceq \nabla^2 a(\bfeta) \preceq a_2\bI_m$
for some constant $0\le a_1\le a_2$.
Assume that either $(i)$~$\lambda=0$ with $a_1>0$ and $n/d\to\alpha\in(1,\infty)$;
or $(ii)$~$\lambda>0$  with $a_1\ge 0$ and $n/d\to\alpha\in(0,\infty)$.
Further assume that $\bTheta^{\sT}_0\bSigma\bTheta_0\to\bR_{00}\succ \bzero_{k\times k}$ as $n,d\to\infty$.
Let $\hbTheta_n$ the the regularized  MLE  \eqref{eq:RiskExpo} (almost surely 
unique for all $n$ large enough), and further assume $\bSigma=\bI$ if $\lambda>0$.

Then \eqref{eq:ExpoFP1}, \eqref{eq:ExpoFP2} admit a unique solution $\bR^\opt, \bS^\opt$.
Further, if for
some $a_3 > 0$ independent of $n$, we have either
\begin{enumerate}
    \item 
    $\lambda_{\min}(\E_{\hnu}[\grad\ell \grad \ell^\sT])\ge a_3$ ; or
   \item  
   $\alpha > k$ and
   $\lambda_{\min}(\hbTheta_n^{\sT}\hbTheta_n)\ge a_3$,
\end{enumerate}
then we have, in probability, 
    \begin{equation}
    \frac1n\sum_{i=1}^n \delta_{\hat\bTheta_n^\sT\bx_i, \by_i} \stackrel{W_2}{\Rightarrow} \nu^\opt\, ,
    \end{equation}
   where $\nu^\opt={\rm Law}(\bv,\by)$,
 and $\bv,\by$ are the random variables in Eq.~\eqref{eq:ExpoFP3}.
 
     If  in addition $\hmu_0 \stackrel{W_2}{\Rightarrow} \mu_0$ (as in Assumption~\ref{ass:theta_0}),
     then the following holds,
    with $\mu^\opt$ determined by $\bR^\opt$ as per Eq.~\eqref{eq:muopt},
     \begin{equation}
    \hat\mu(\hat\bTheta_n) \stackrel{W_2}{\Rightarrow} \mu^\opt .
     \end{equation}
\end{proposition}

\subsection{Revisiting multinomial regression}
\label{sec:Multinomial}

As a special case of the exponential family, we consider multinomial regression 
with $k+1$ classes labeled $\{0,\dots,k\}$, already 
defined in Section \ref{sec:Intro}, Example \ref{ex:Multinomial}.

Recall that, for $j\in\{1,\dots, k\}$,  $\be_j$ denotes the $j$-th canonical basis vector in $\R^{k}$ 
and we let $\be_0 = \bzero_k$.  We encode class labels by letting 
$\by_i\in \{\be_0,\dots,\be_k\}$.
The regularized MLE minimizes the following risk function:
\begin{align}
\hR_n(\bTheta) = 
\frac1n \sum_{i=1}^n \bigg\{ -  \<\bTheta^\sT\bx_i,\by_i\>+
   \log\Big( 1+\sum_{j=1}^k e^{\<\be_j,\bTheta^{\sT}\bx_i\>}\Big)  
    \bigg\} +\frac{\lambda}{2}\|\bTheta\|_F^2\, .\label{eq:RiskMultinomial}
\end{align}
We also define the moment generating function $a:\reals^k\to\reals$ via
\begin{align}
    a(\bv) := \log\Big( 1+\sum_{j=1}^k e^{v_j}\Big)  \, .
\end{align}

For multinomial regression, 
Eqs.~\eqref{eq:ExpoFP1}, \eqref{eq:ExpoFP2}  take the even more explicit form
\begin{align}
\label{eq:FP_multinomial}
   \alpha \; \E[   (\bp(\bv)- \by )(\bp(\bv) - \by)^\sT ]   &=  \bS^{-1}(\bR/ \bR_{00})\bS^{-1},\\
   \E[   (\bp(\bv)- \by )(\bv,\bg_0)^\sT] &+ \lambda (\bR_{11},\bR_{10})= \bzero\, ,\nonumber
\end{align}
where $\bp(\bv) := \big( p_{j}(\bv) \big)_{j\in[k]}$ for $p_j$ defined in Eq.~\eqref{eq:MultiNomialDef}, and the random variables $\bv, \by$ 
have joint distribution defined by
\begin{align}
\bv  = \Prox_{a(\, \cdot\, )}( \bg + \bS \by; \bS),\quad
   \P\left(\by = \be_j\right)  = p_{j}(\bg_0),\; j\in [k],\quad
(\bg,\bg_0) \sim \cN\left( \bzero_{k+k_0},\bR\right) ,\label{eq:FP_Multi3}
  \end{align}
  and, of course, $\P(\by = \bzero) = 1 - \sum_j p_j(\bg_0)$.
  %

\begin{proposition}
\label{prop:multinomial}
Consider multinomial regression under the model of Eq.~\eqref{eq:MultiNomialDef} with risk function \eqref{eq:RiskMultinomial} for $\lambda\ge 0$. 
Assume that $n,d\to\infty$ with $n/d\to\alpha\in(1,\infty)$
and that $\bTheta^{\sT}_0\bSigma\bTheta_0\to\bR_{00}\in\R^{k\times k}$ as $n,d\to\infty$, with 
$\bR_{00}\succ \bzero$.

Then the following hold:
\begin{enumerate}
    \item 
If the system~\eqref{eq:FP_multinomial} has a solution $(\bR^\opt,\bS^\opt)$, then 
\begin{enumerate}
    \item  $(\bR^\opt,\bS^\opt)$  is
the unique solution of Eq.~\eqref{eq:FP_multinomial}.
\item We have,
for  $\hbTheta:= \argmin \hR_n(\bTheta)$ and some finite $C>0$, 
\begin{equation}
\label{eq:mle_exists_condition}
    \lim_{n\to\infty } \P\left( \hbTheta  \;\mbox{\rm exists},\; \|\hat\bTheta\|_F < C\right) = 1.
\end{equation}
\item Letting $\mu^\opt$ be  determined by $\mu_0,\bR^\opt$ as per Eq.~\eqref{eq:muopt},
and $\nu^\opt={\rm Law}(\bv,\by)$ with $\bv,\by$ defined by Eq.~\eqref{eq:FP_Multi3}, 
if
$\hmu_0 \stackrel{W_2}{\Rightarrow} \mu_0$ as in Assumption~\ref{ass:theta_0},
we have
    \begin{equation}
    \hmu(\hat\bTheta_n) \stackrel{W_2}{\Rightarrow} \mu^\opt ,\quad\quad
    \frac1n\sum_{i=1}^n \delta_{\hat\bTheta_n^\sT\bx_i, \by_i} \stackrel{W_2}{\Rightarrow} \nu^\opt\, . 
    \end{equation}
    \item The empirical spectral distribution of the rescaled Hessian
    $\bSigma_k^{-1/2}\grad^2\hat R_n(\hat\bTheta_n)\bSigma_k^{-1/2}$
    (where 
    $\bSigma^{-1/2}_k := \bSigma^{-1/2}\otimes \bI_k$) at the minimizer 
    converges weakly to $\mu_\star(\nu^\opt)$ in probability. 
    \end{enumerate}
\item Conversely, if the system~\eqref{eq:FP_multinomial} does not have a solution, then, for all $C>0$,
\begin{equation}
\lim_{n\to\infty } \P\left( \hat\bTheta  \;\mbox{\rm exists},\; \|\hat\bTheta\|_F < C\right) = 0.
\end{equation}
\end{enumerate}
\end{proposition}

\begin{remark}
    A detailed study of high-dimensional asymptotics in multinomial regression was recently carried out in 
    \cite{tan2024multinomial}. However the techniques of \cite{tan2024multinomial} only allows characterizing
    the distribution of the MLE on null covariates. 
  This problem was also considered in \cite{cornacchia2023learning,ccakmak2024convergence},
  but proofs in these papers require strong convexity, which does not hold here.
\end{remark}

\begin{figure}[t]
    \centering
     \begin{subfigure}[t]{0.45\textwidth}
        \includegraphics[width=\textwidth]{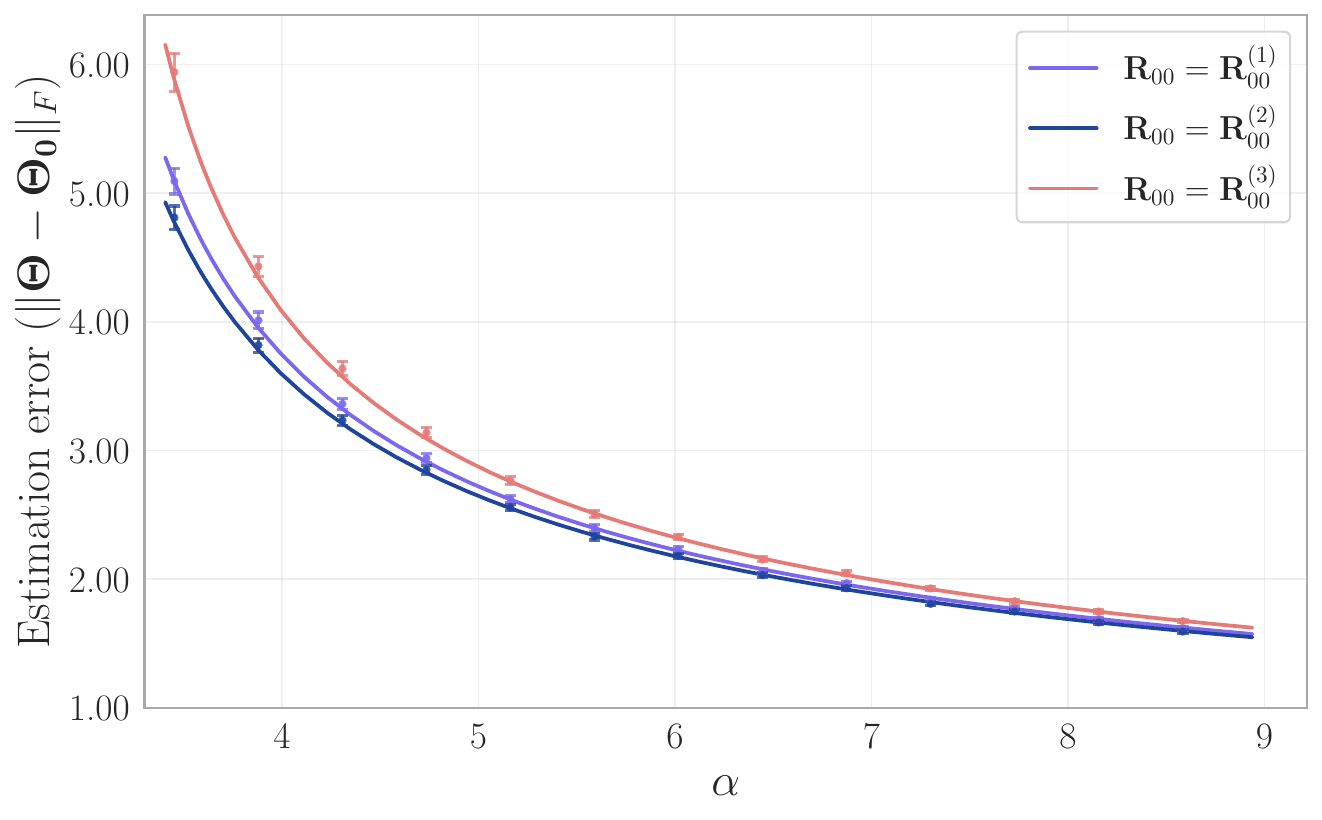}
    \end{subfigure}
    \begin{subfigure}[t]{0.45\textwidth}
        \includegraphics[width=\textwidth]{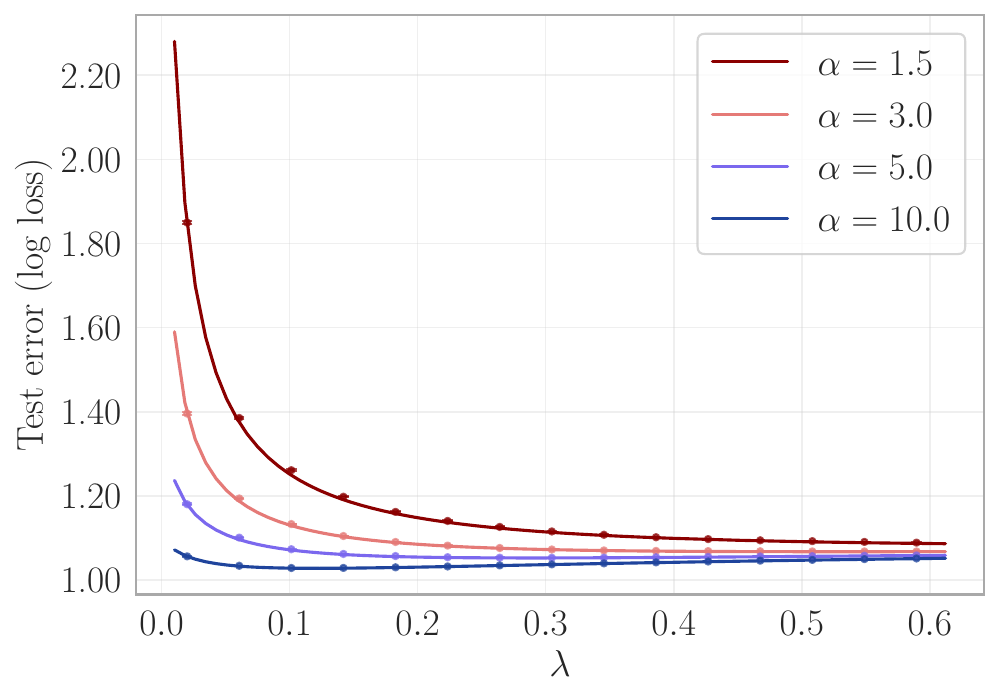}
    \end{subfigure}
    \begin{subfigure}[t]{0.45\textwidth}
        \centering
        \includegraphics[width=\textwidth]{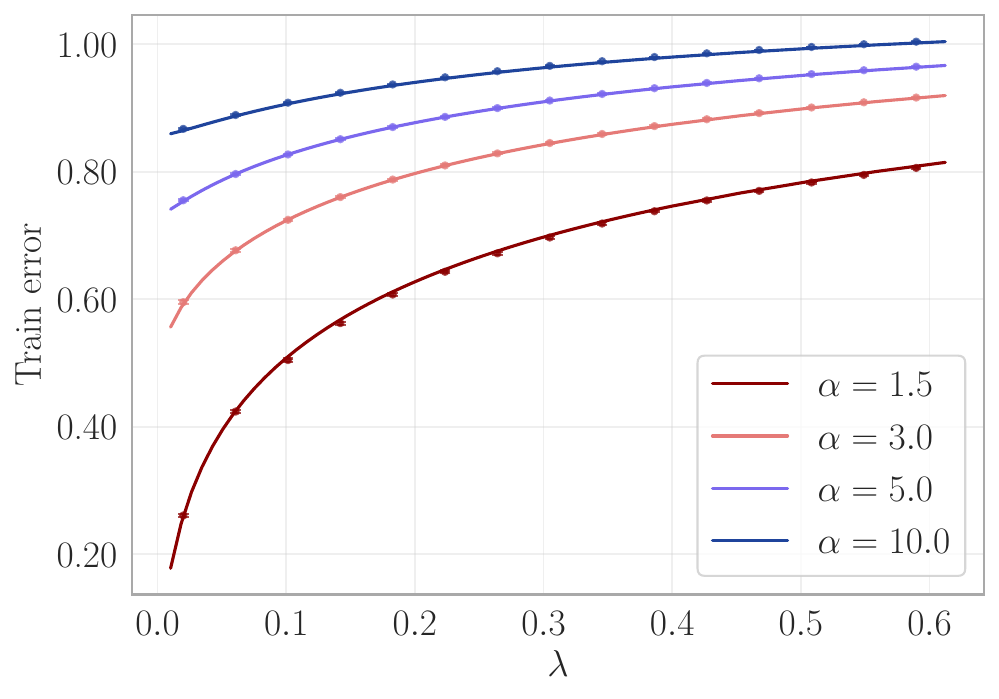}
    \end{subfigure}
    \begin{subfigure}[t]{0.45\textwidth}
       \centering
        \includegraphics[width=\textwidth]{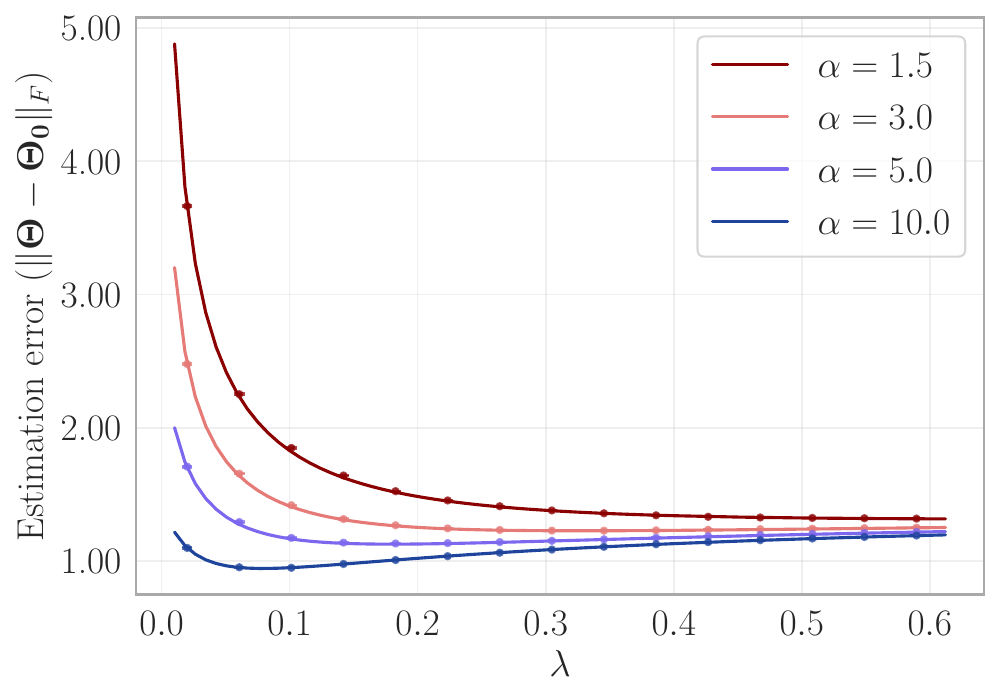}
    \end{subfigure}

    \caption{Train/test error (log loss), estimation error, and classification error
    of ridge regularized multinomial regression, for $(k+1)=3$ symmetric classes,
    as a function of the regularization parameter $\lambda$ for several values of $\alpha$. 
    Empirical results are averaged over 100 independent trials,  with $d = 250$. 
    Continuous lines are theoretical predictions obtained by solving numerically the system \eqref{eq:FP_multinomial}.
   }
    \label{fig:regularized_error}
\end{figure} 

\begin{figure}[t]
    \centering
    
     \begin{subfigure}[t]{0.45\textwidth}
        \includegraphics[width=\textwidth]{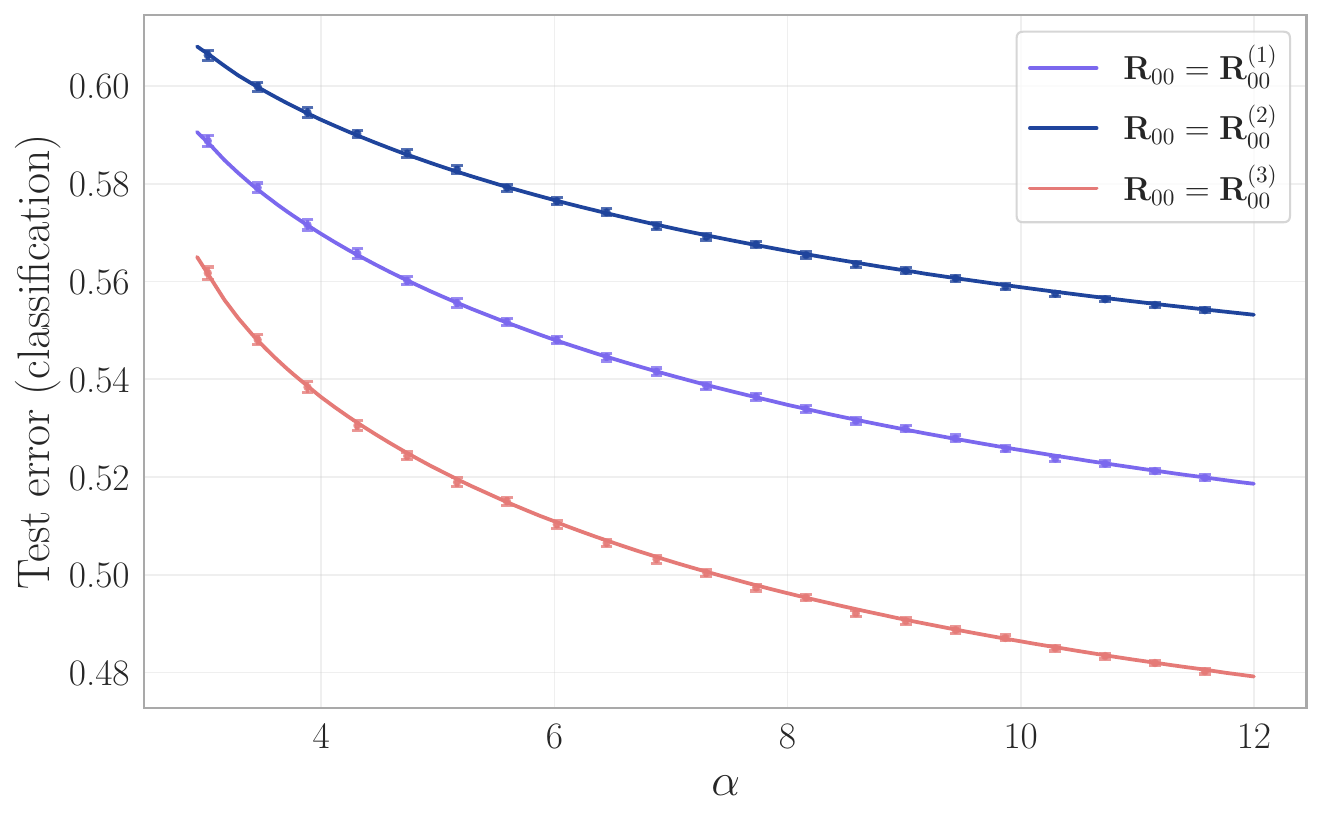}
    \end{subfigure}
    \begin{subfigure}[t]{0.45\textwidth}
        \includegraphics[width=\textwidth]{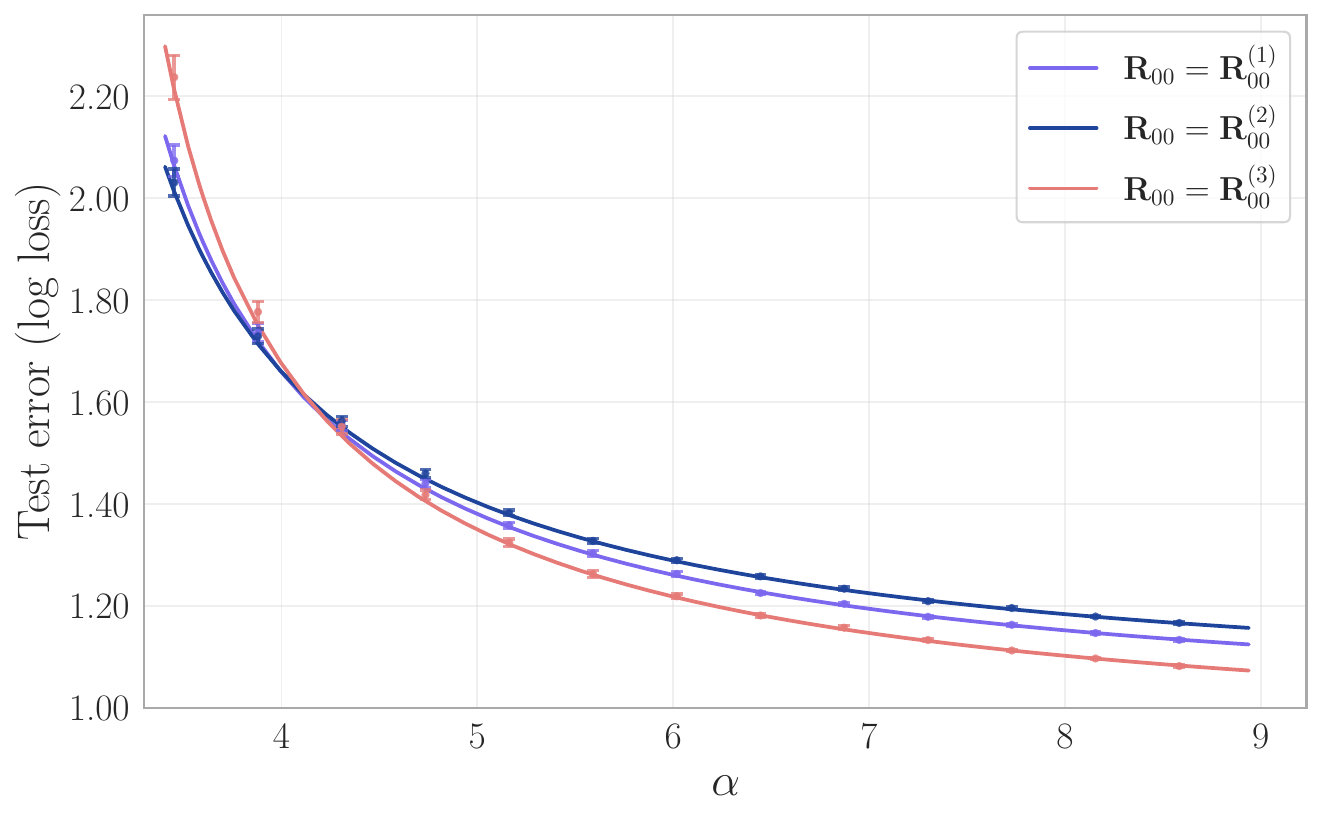}
    \end{subfigure}
    \begin{subfigure}[t]{0.45\textwidth}
       \centering
        \includegraphics[width=\textwidth]{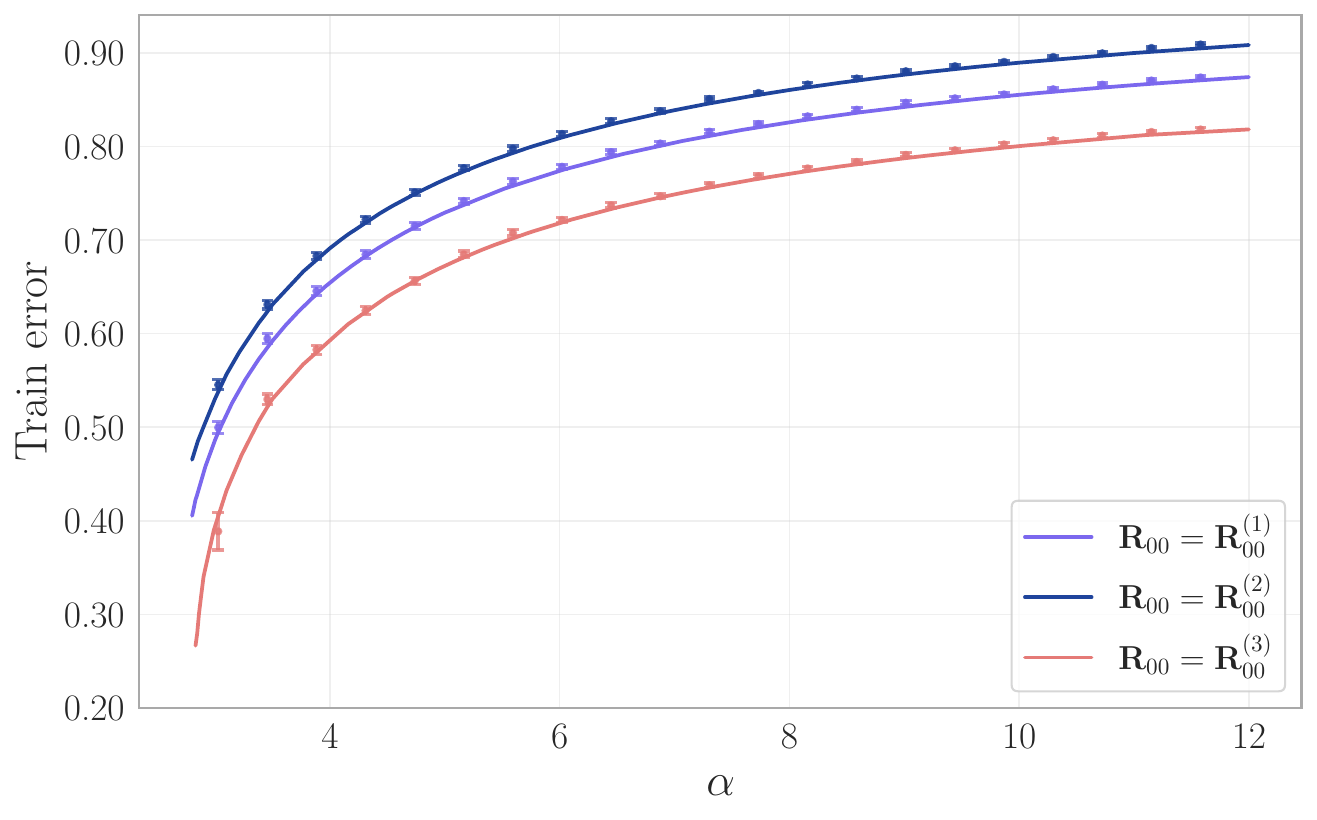}
    \end{subfigure}
    \begin{subfigure}[t]{0.45\textwidth}
        \centering
        \includegraphics[width=\textwidth]{figures/multinomial_3classes/error_vs_alpha/F_norm_vs_alpha.pdf}
    \end{subfigure}
    
    \caption{Train/test error (log loss), estimation error, and classification error
    of unregularized multinomial regression, for $(k+1)=3$ symmetric classes, as a function of $\alpha$ for different values of $\bR_{00}$ specified in the text.
    Empirical results are averaged over 100 independent trials,  with $d = 250$. 
   }
    \label{fig:error_vs_alpha}
\end{figure}

\begin{figure}[t]
    \centering
    \begin{subfigure}[t]{0.45\textwidth}
        \includegraphics[width=\textwidth]{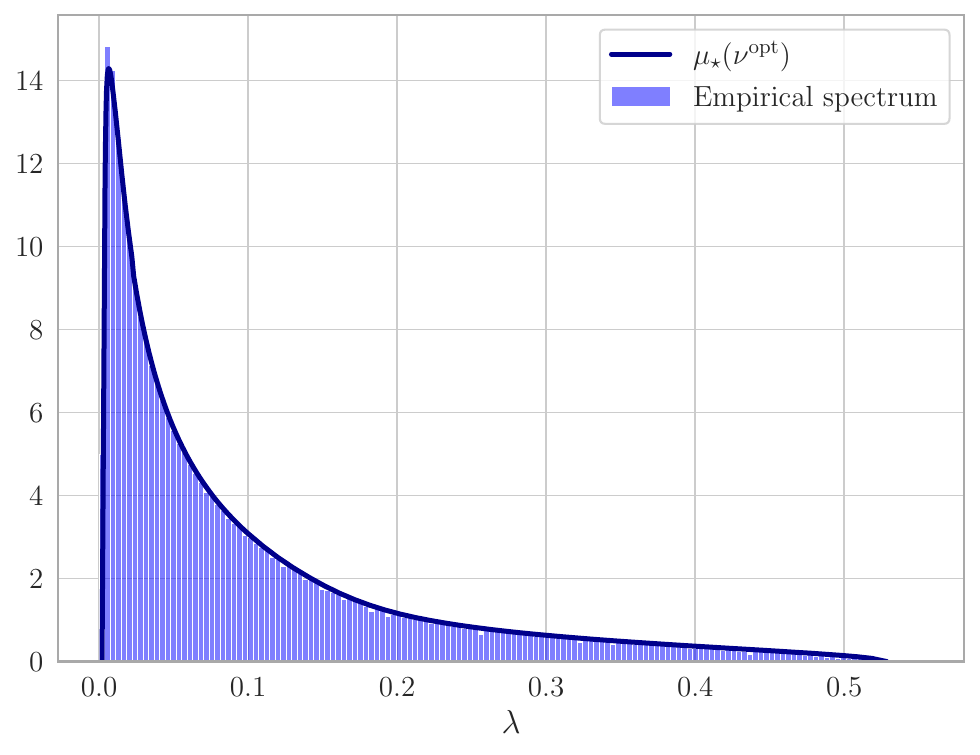}
        \label{fig:subfig1}
    \end{subfigure}
    \begin{subfigure}[t]{0.45\textwidth}
       \centering
        \includegraphics[width=\textwidth]{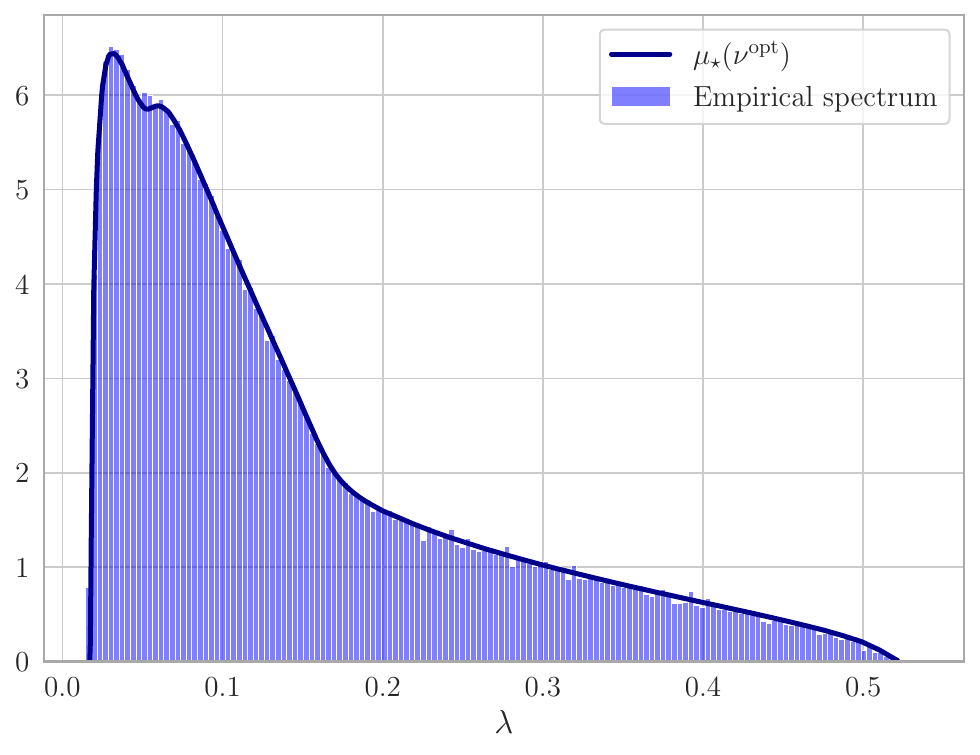}
        \label{fig:subfig2}
    \end{subfigure}
    \begin{subfigure}[t]{0.45\textwidth}
        \centering
        \includegraphics[width=\textwidth]{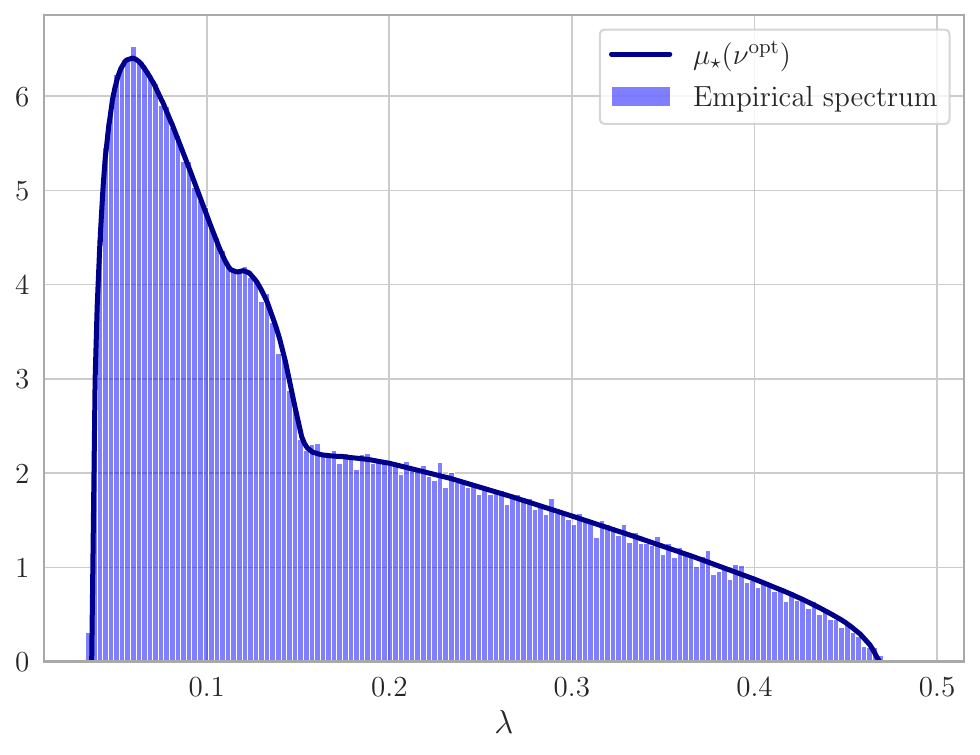}
    \end{subfigure}
    \begin{subfigure}[t]{0.45\textwidth}
        \includegraphics[width=\textwidth]{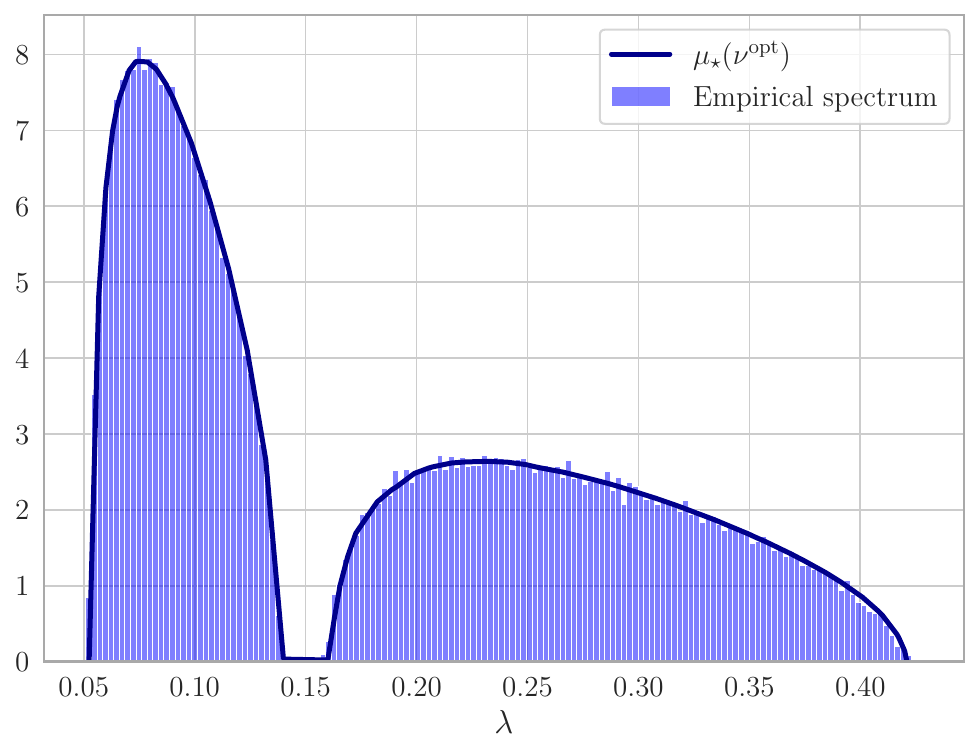}
        \label{fig:subfig4}
    \end{subfigure}
    
    \caption{Histograms of the empirical spectral distribution of the Hessian at the MLE
    for multinomial regression with three symmetric classes, in $d=250$ dimensions, aggregated over $100$ 
    independent realizations.
    From left to right, $\alpha = 3$, $\alpha = 5$, $\alpha=10$, and $\alpha = 20$. Blue lines represent the theoretical distribution derived from Proposition \ref{prop:multinomial}.}
    \label{fig:Spectrum}
\end{figure}

\section{Numerical experiments}

In this section, we compare the predictions of our theory to numerical simulations 
for ridge-regularized multinomial regression, presented in Section \ref{sec:Multinomial}.
We present the results of two types of experiments:
\begin{enumerate}
\item Experiments with synthetic data, distributed according to the multinomial model used in our analysis. We present here
result with $k+1=3$ classes, while additional results with $k+1=5$ classes\footnote{We note that the complexity of solving the critical point optimality condition \eqref{eq:FP_multinomial} increases with $k$.} are presented in Appendix \ref{app:Numerical}.
In this case, we observe very close agreement between experiments and asymptotic predictions already
when $d\gtrsim 250$. 
\item Experiments with the image-classification dataset Fashion-MNIST \cite{xiao2017fashion}. In this case, we construct the feature vectors $\bx_i$ by passing the images through a one-layer random neural network \cite{rahimi2008random}. Of course, the resulting vectors $\bx_i$ are non-Gaussian, but we nevertheless observe encouraging agreement with the predictions.
\end{enumerate}

\subsection{Synthetic data}
In Fig.~\ref{fig:regularized_error}, we consider the case of $(k+1)= 3$ classes which
are completely symmetrical (the optimal decision regions are congruent). A simple 
calculation reveals that this corresponds to $\bR_{00}= \bR_{00}^s(c)$
for some $c>0$,
where
\begin{align}
\bR_{00}^s(c) :=
\begin{bmatrix}
        c & c/2 \\ 
        c/2 & c
    \end{bmatrix}\,.
\end{align}
In the simulations  of Fig.~\ref{fig:regularized_error}, we use $c=1$.
    We compare empirical results
    for estimation error, test error and train error (in log-loss). We observe that theory matches well with numerical simulations even for moderate dimensions.

In Fig.~\ref{fig:error_vs_alpha}, we once again compare the same empirical and theoretical quantities, for different values of the ground truth parameters (encoded in $\bR_{00}$) as a function of $\alpha$ (here $\lambda=0$). We consider 3 different values of $\bR_{00}$:
\begin{equation}
    \bR_{00}^{(1)} 
    := \begin{bmatrix}
        1&1/2\\1/2&1
    \end{bmatrix},
    \quad\quad
      \bR_{00}^{(2)}  
      := \begin{bmatrix}
        1&0.9\\0.9&1
      \end{bmatrix},\quad\quad
      \bR_{00}^{(3)} := 
      \begin{bmatrix}
        1&-1/2\\-1/2&1
    \end{bmatrix}.
\end{equation}


In Fig.~\ref{fig:Spectrum}, we consider the first setting above with $\bR_{00} = \bR_{00}^s(1)$, $\lambda =0$
and plot the empirical spectral distribution of the Hessian $\nabla^2\hR_n(\hbTheta)$
at the MLE $\hbTheta$. We compare this with the prediction of Proposition \ref{prop:multinomial}.
Again, the agreement is excellent. 

We observe that the spectrum structure changes significantly around $\alpha\gtrsim 5$,
developing two `bumps.' This is related to the structure of the population Hessian,
which is easy to derive:
\begin{align}
\nabla^2 R(\bTheta)\big|_{\bTheta=\bTheta_0} = \E[\bA(\bTheta_0^{\sT}\bx)]\otimes \bI_d
+\sum_{a,b=1}^k\E[\partial^2_{a,b} \bA(\bTheta_0^{\sT}\bx)]\otimes \bTheta_{\cdot,a}
\bTheta_{\cdot,b}^{\sT}\, ,\label{eq:HessianDecomposition}
\end{align}
where  $\bA:\R^k\to\R^{k\times k}$ is defined by
\begin{align}
 \bA(\bv) := \nabla_{\bv}^2\ell(\bv,\bv_{0},w)\big|_{\bv_0= \bv}\, .
\end{align}
Hence, neglecting the second term in Eq.~\eqref{eq:HessianDecomposition}
(which is low rank and contributes at most $k^2$ outlier eigenvalues), the eigenvalues should concentrate (for large $n$) around  the $k$ eigenvalues of $\E[\bA(\bTheta_0^{\sT}\bx)]$.
The actual structure of $\grad^2\hat R_n(\hat\bTheta)$ 
in Fig.~\ref{fig:Spectrum} is significantly different: eigenvalues do not concentrate, but 
we begin seeing a $k$-mode structure emerging for $\alpha\gtrsim 10$.

In Appendix \ref{app:Numerical} we report results of similar numerical experiments
with $(k+1)=5$ classes, confirming the above observations. 

We summarize two qualitative findings that emerge from both experiments and theory:
\begin{enumerate}
\item In the noisy regime considered here,
unregularized multinomial regression is substantially suboptimal
even when the number of samples per dimension is quite large ($\alpha=10$). 
See Fig.~\ref{fig:regularized_error}.
\item The structure of the Hessian and the empirical risk minimizer, is very different from classical theory, even when $\alpha = 20$. See Fig.~\ref{fig:Spectrum}.
\end{enumerate}

\subsection{Fashion-MNIST data}
\label{sec:Fashion-MNIST}

\begin{figure}[htbp]
    \centering

    {\tiny  
    \makebox[0.32\textwidth]{ 250 features}  
    \makebox[0.32\textwidth]{ 350 features}  
    \makebox[0.32\textwidth]{ 500 features}  
    }

    \vspace{0.15cm}  

    \begin{tabular}{@{}c@{} c@{} c@{}}   
        \begin{subfigure}[b]{0.335\textwidth}
            \centering
            \includegraphics[width=\textwidth]{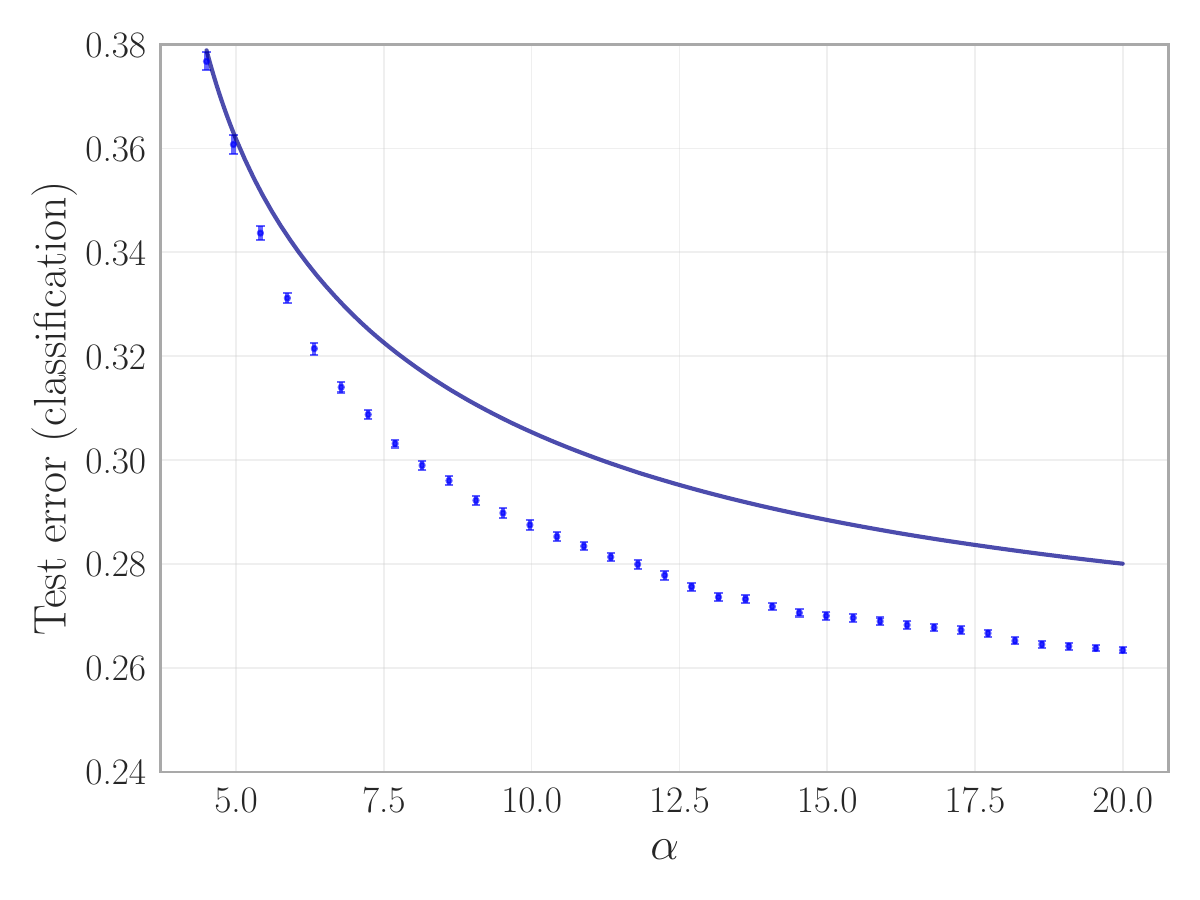}
        \end{subfigure} &
        \begin{subfigure}[b]{0.335\textwidth}
            \centering
            \includegraphics[width=\textwidth]{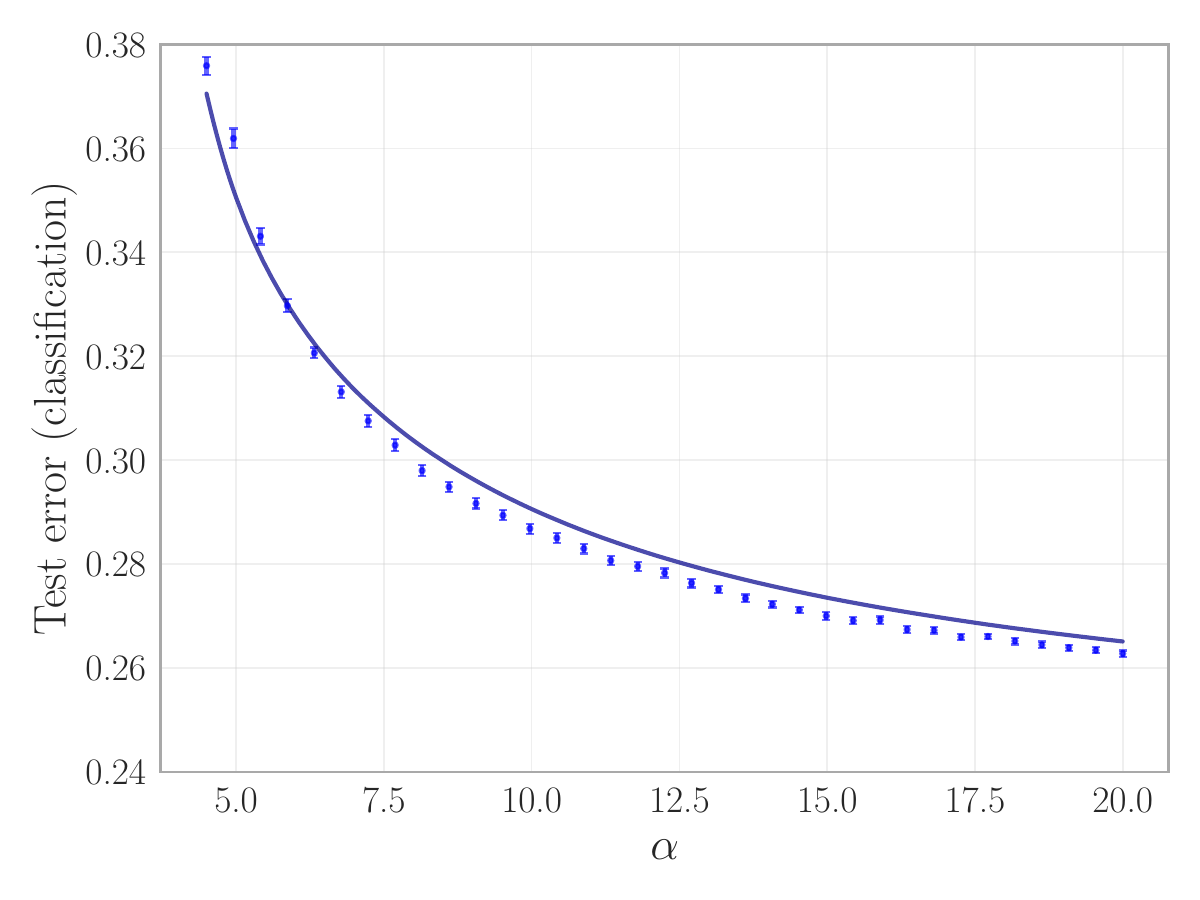}
        \end{subfigure} &
        \begin{subfigure}[b]{0.335\textwidth}
            \centering
            \includegraphics[width=\textwidth]{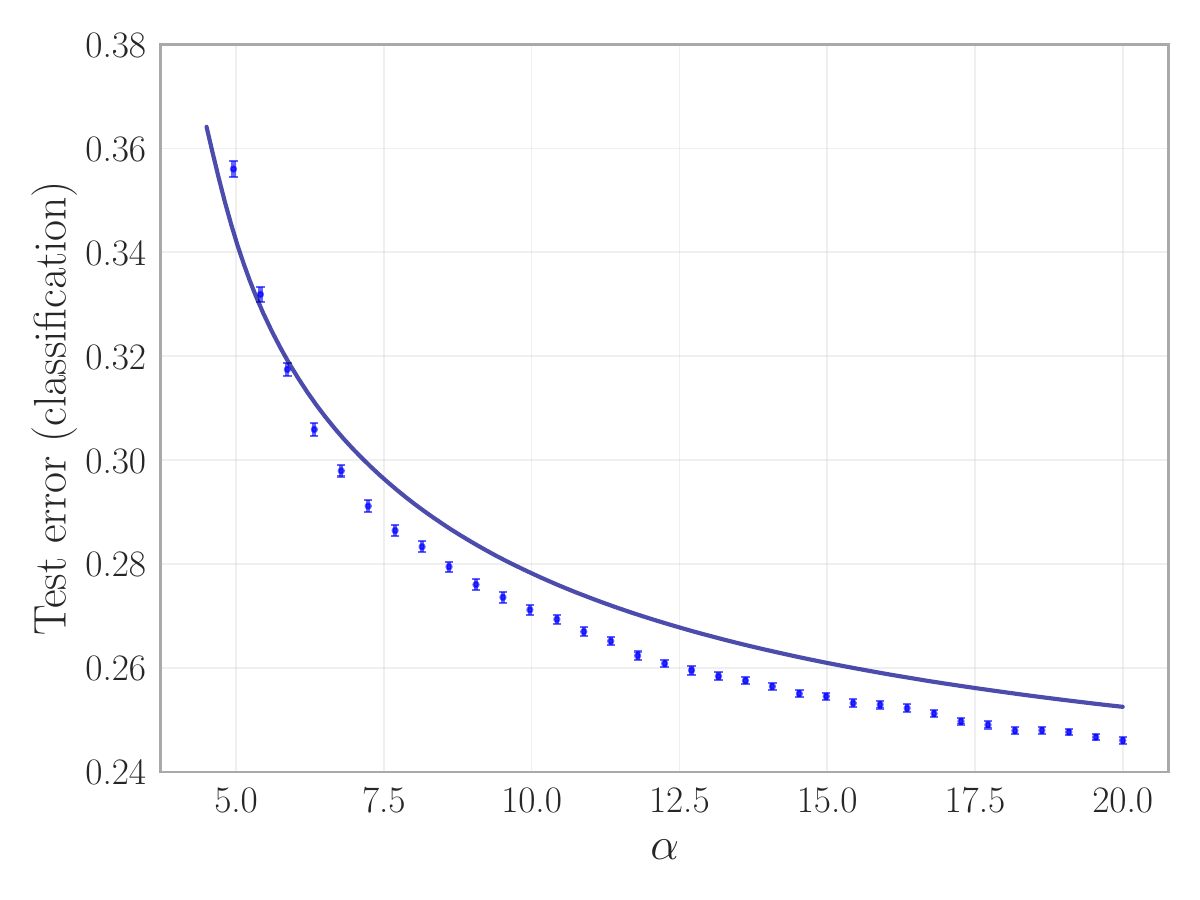}
        \end{subfigure} \\

        \vspace{0.3cm}  

        \begin{subfigure}[b]{0.335\textwidth}
            \centering
            \includegraphics[width=\textwidth]{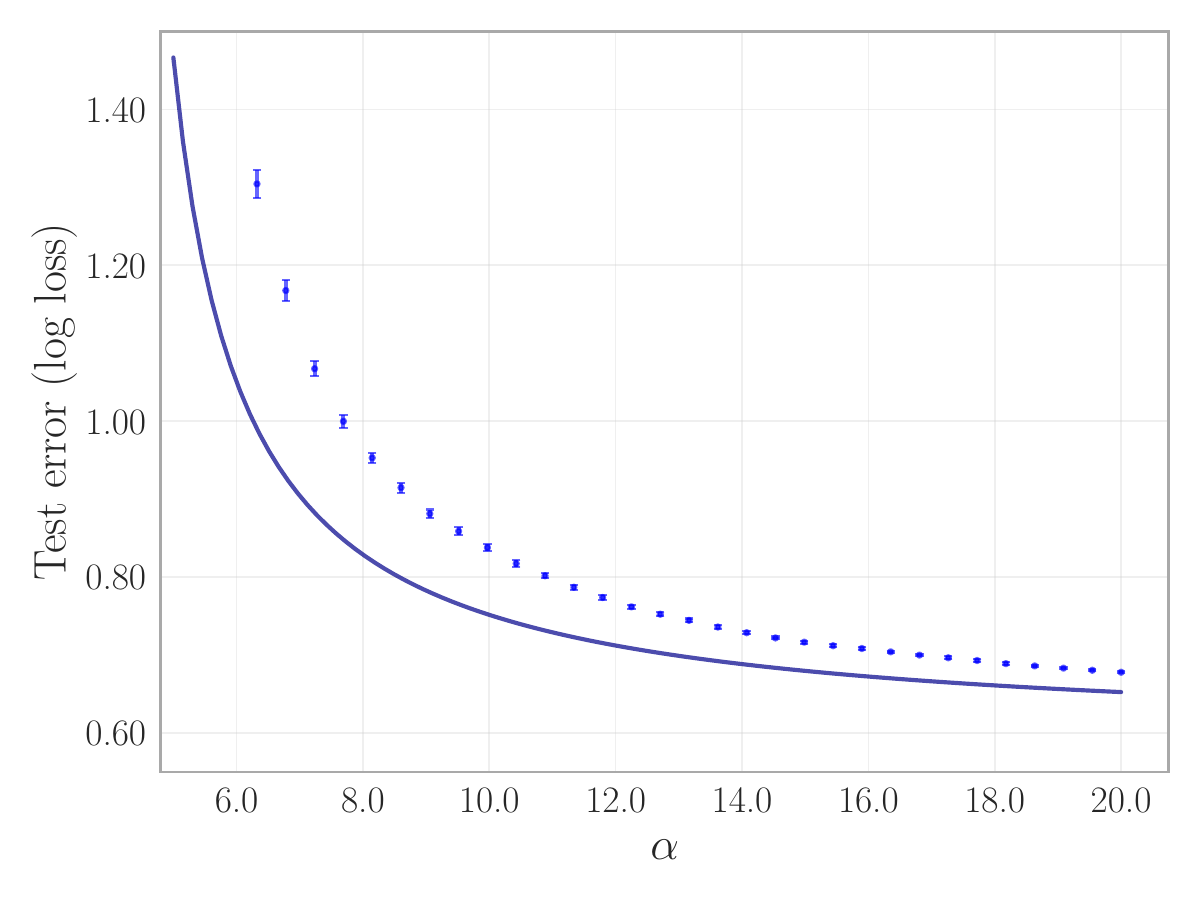}
        \end{subfigure} &
        \begin{subfigure}[b]{0.335\textwidth}
            \centering
            \includegraphics[width=\textwidth]{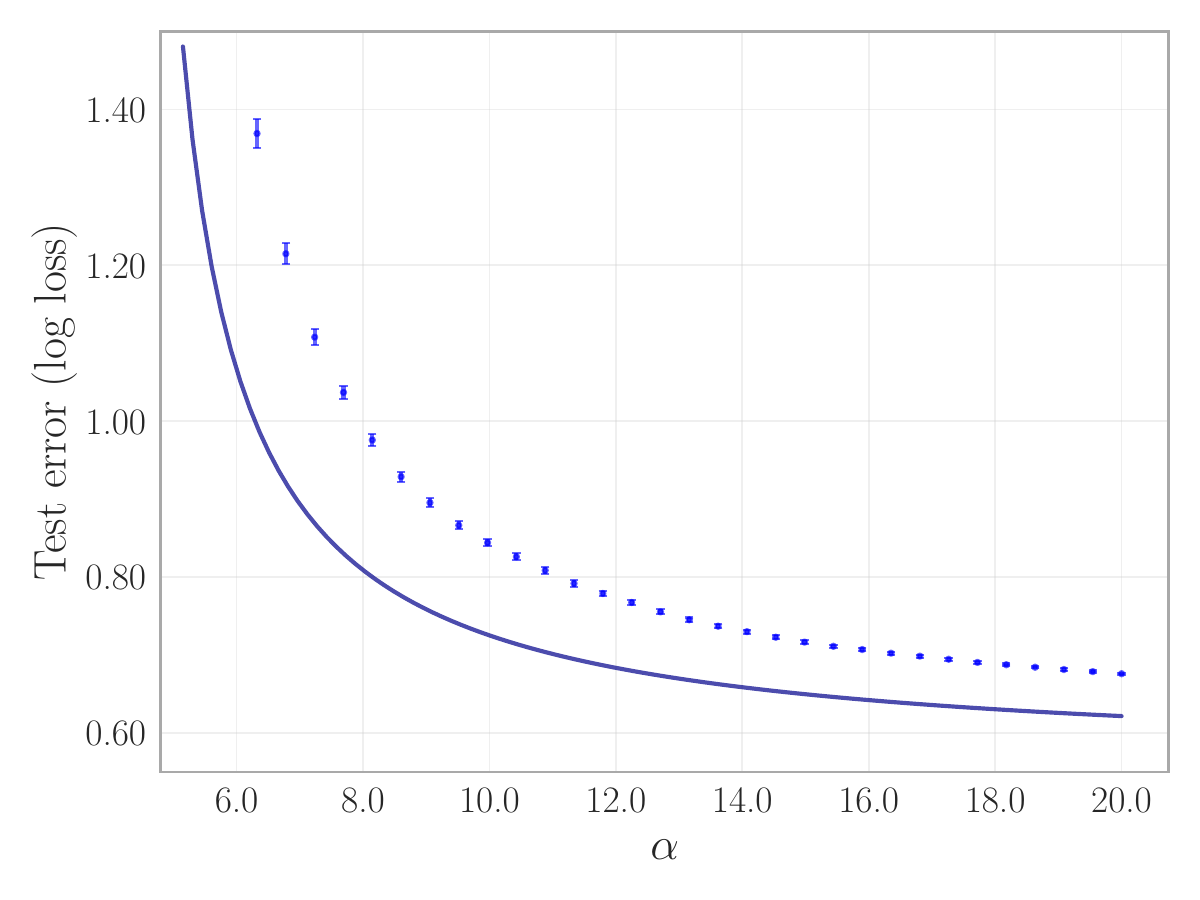}
        \end{subfigure} &
        \begin{subfigure}[b]{0.335\textwidth}
            \centering
            \includegraphics[width=\textwidth]{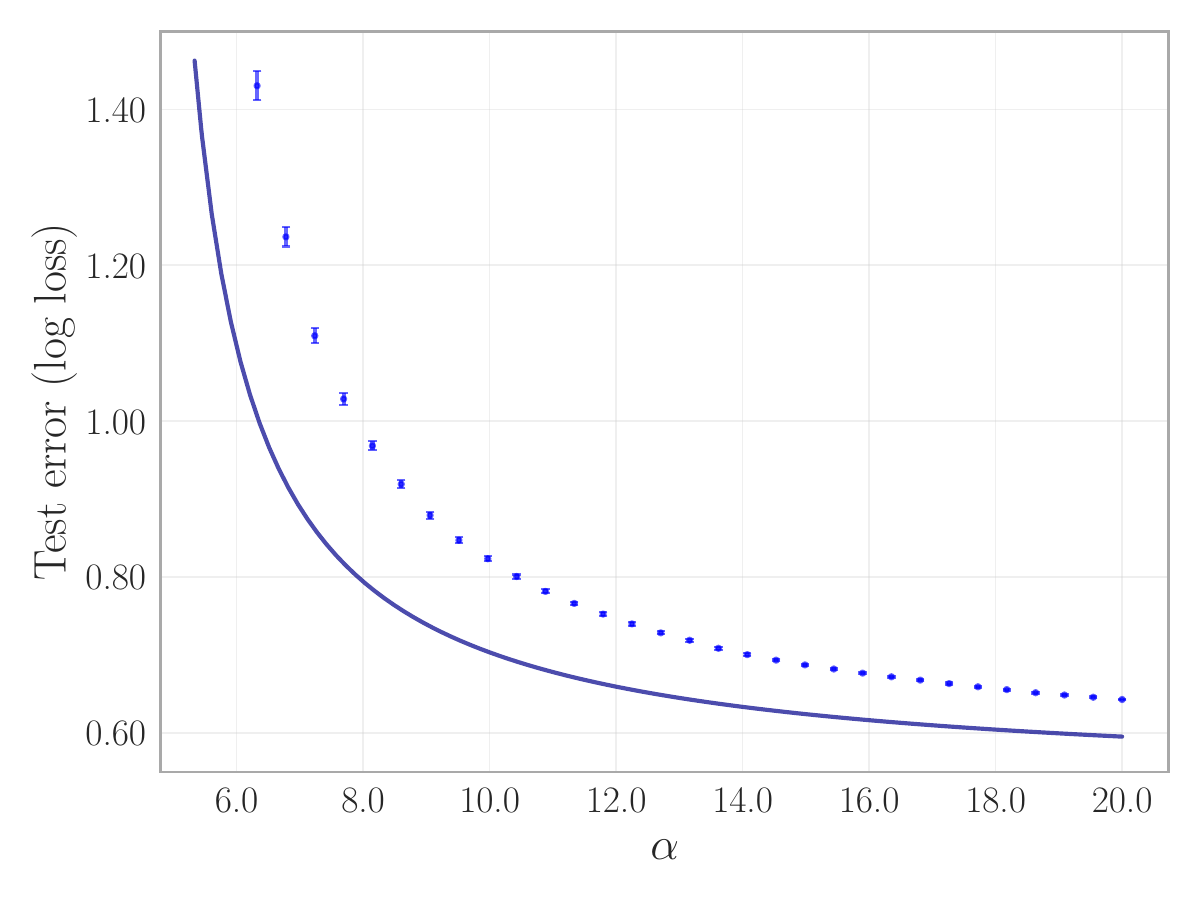}
        \end{subfigure} \\

        \vspace{0.3cm}  

        \begin{subfigure}[b]{0.335\textwidth}
            \centering
            \includegraphics[width=\textwidth]{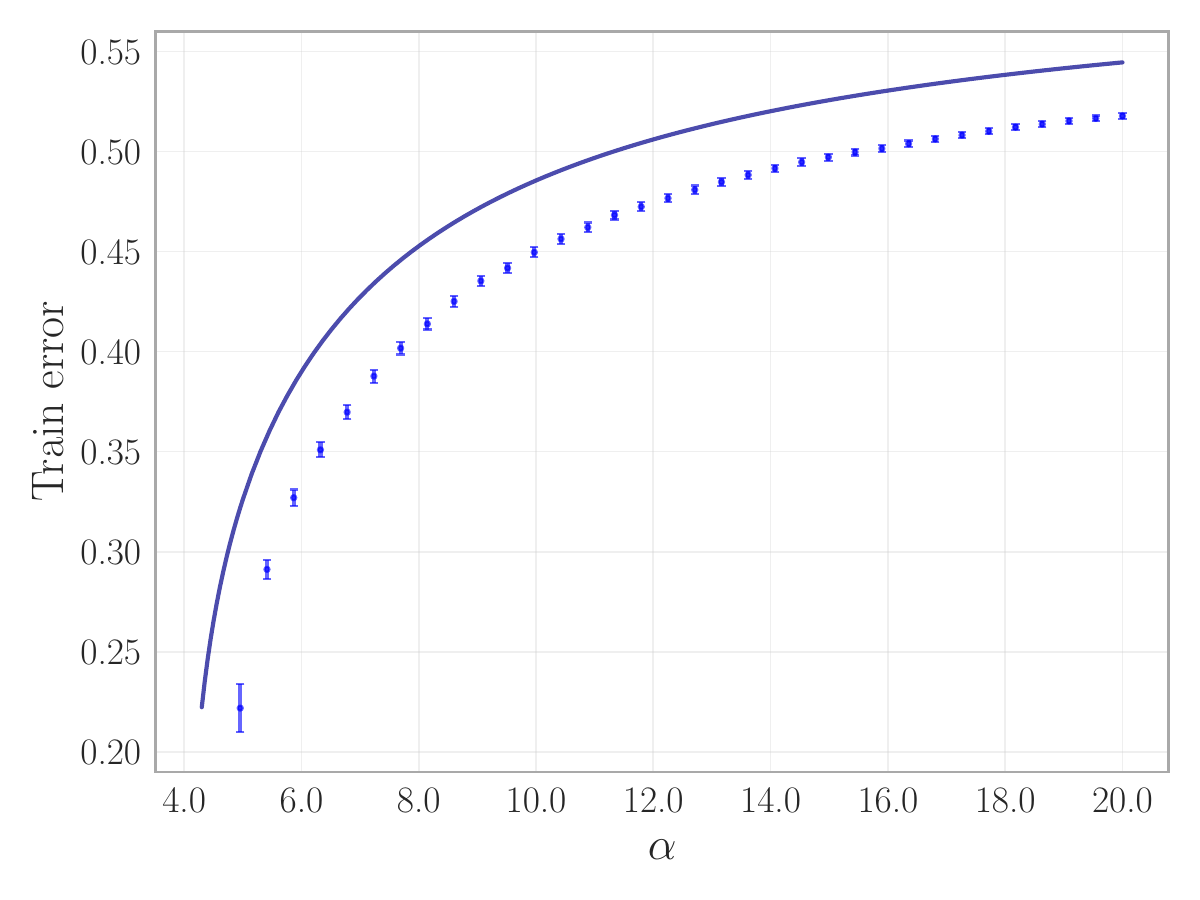}
        \end{subfigure} &
                \begin{subfigure}[b]{0.335\textwidth}
            \centering
            \includegraphics[width=\textwidth]{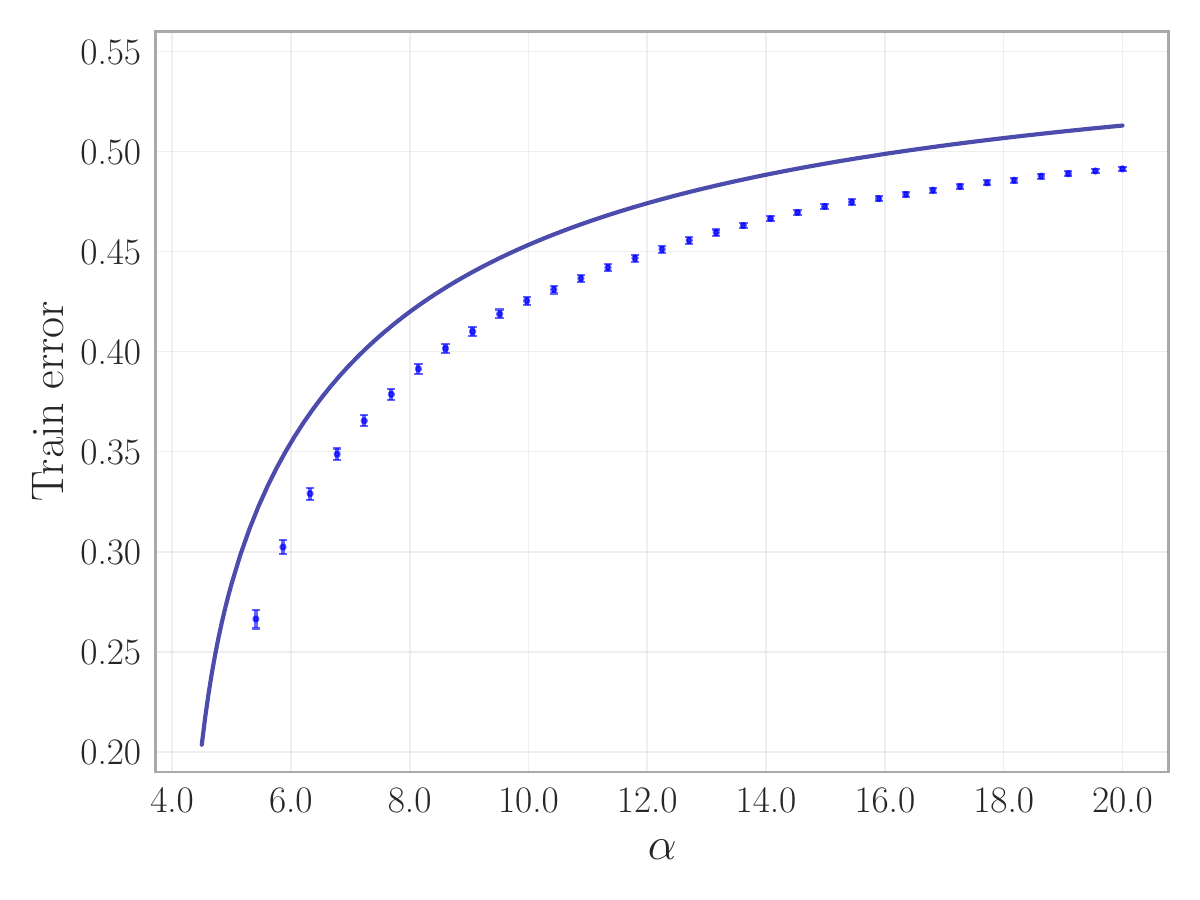}
        \end{subfigure} &
        \begin{subfigure}[b]{0.335\textwidth}
            \centering
            \includegraphics[width=\textwidth]{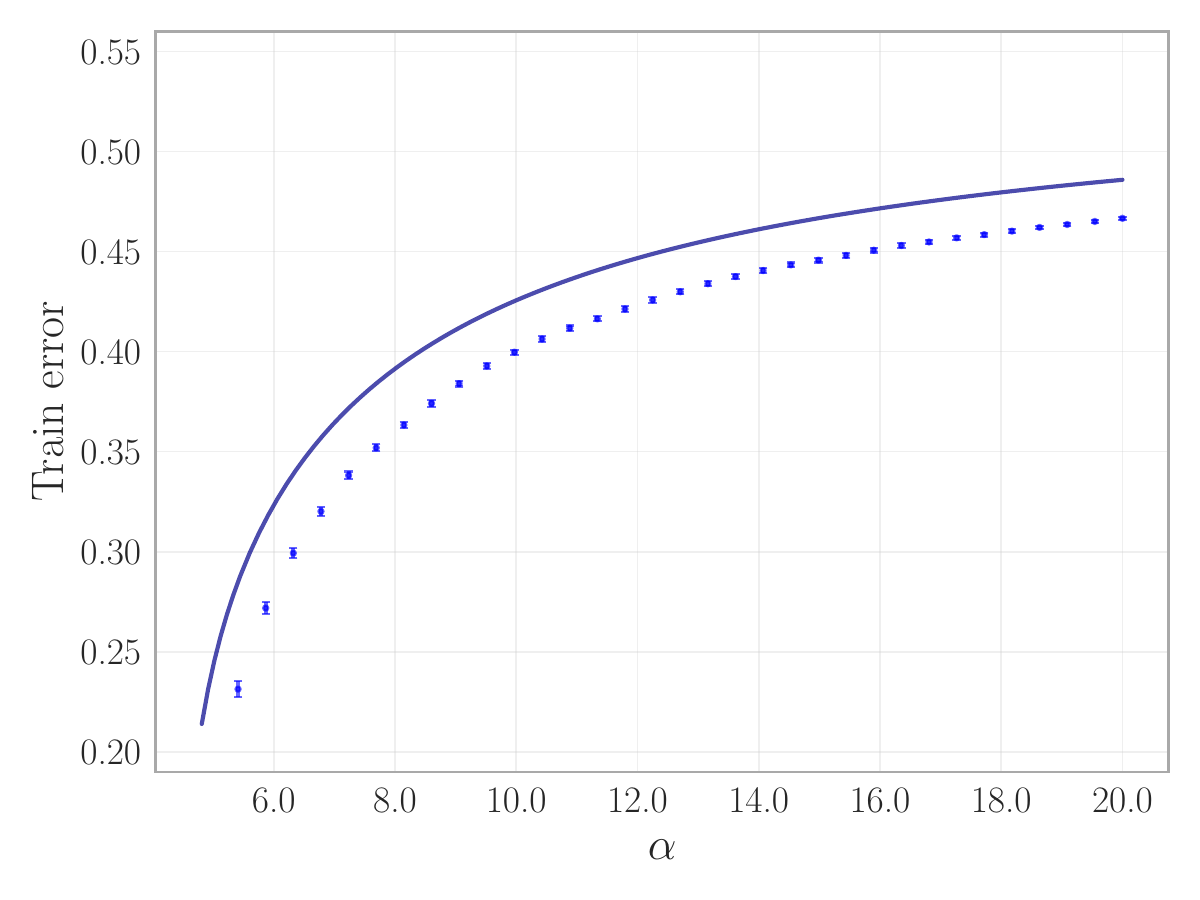}
        \end{subfigure} \\

    \end{tabular}

    \caption{Performance of multinomial regression on the Fashion-MNIST dataset 
    for $(k+1)=3$ classes, as a function of $\alpha$. We construct feature vectors $\bx_i$ using a random one-layer neural network, as discussed in 
    Section \ref{sec:Fashion-MNIST}.}
    \label{fig:mnist_tanh}
\end{figure}

\begin{figure}[!ht]
    \centering
    \begin{tabular}{c@{}c@{}}  
        \begin{subfigure}[t]{0.52\textwidth}
            \centering
            \includegraphics[width=\textwidth]{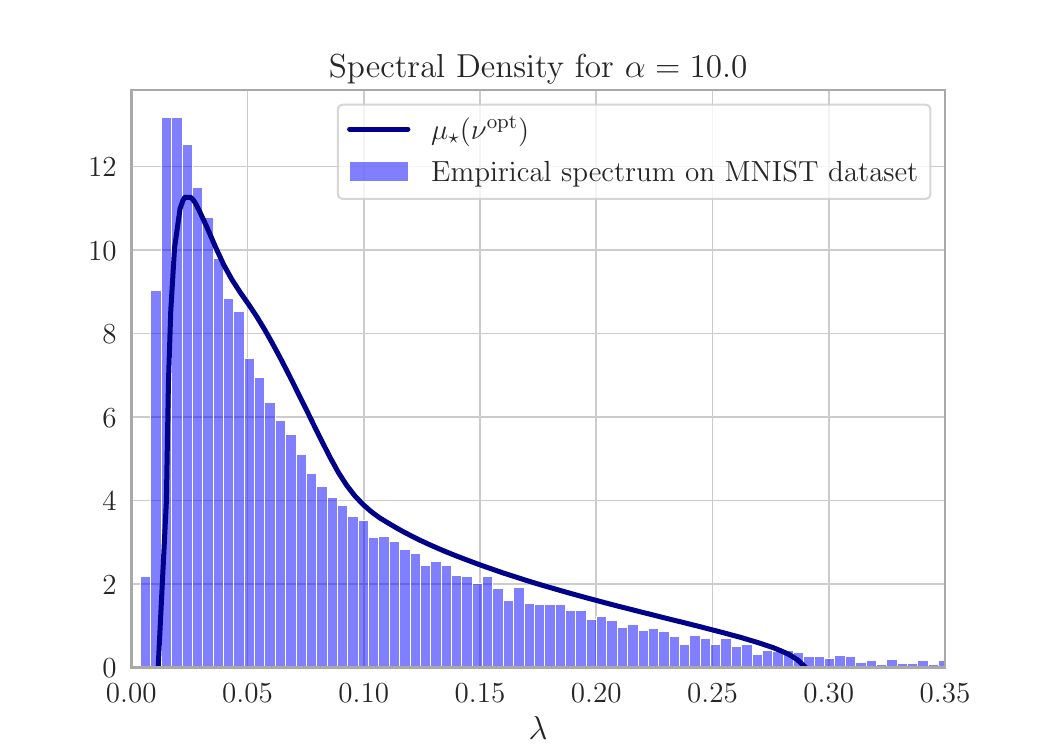}
        \end{subfigure} &
        \begin{subfigure}[t]{0.52\textwidth}
            \centering
            \includegraphics[width=\textwidth]{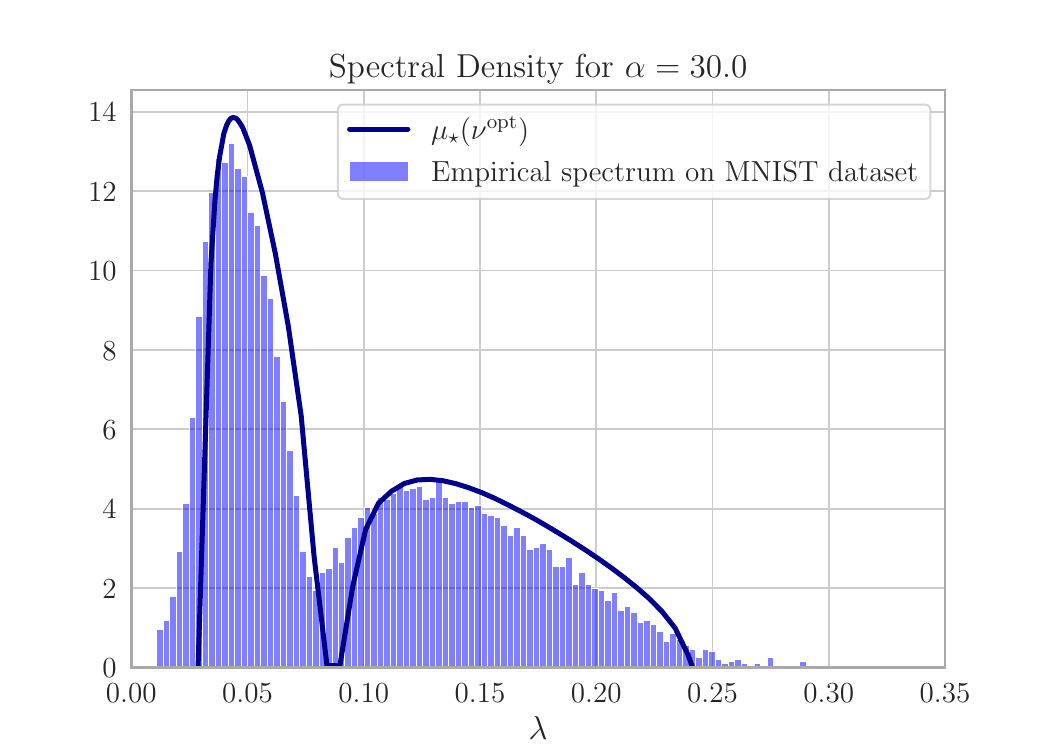}
        \end{subfigure} 
    \end{tabular}

    \caption{Histograms: Empirical spectral distribution of the Hessian at the MLE. Here we use use Fashion-MNIST data for $(k+1)=3$ classes, and construct feature vectors $\bx_i$ using a random one-layer neural network, as discussed in 
    Section \ref{sec:Fashion-MNIST}.}
    \label{fig:ESD_mnist}
\end{figure}

In Fig.~\ref{fig:mnist_tanh}, we report the result of our experiments on the Fashion MNIST
dataset \cite{xiao2017fashion}. This consists of overall $70,000$ grayscale images of dimension $28\times 28$ belonging to $10$
classes (split into $60,000$ images for training and $10,000$ images for testing).
We select $k+1=3$ classes (pullovers, coats, and shirts), for a total of $N=18,000$ training images. We standardize the entries of these images to get 
vectors $\{\bz_i\}_{i\le N}$, $\bz_i\in\reals^{d_0}$, $d_0=784$.
We then construct feature vectors $\bx_i$ by using a one-layer neural network with random weights. Namely
\begin{align}
\obx_i = \sigma(\bW\bz_i)\, ,
\end{align}
where $\bW= (W_{jl})_{j\le d,l\le d_0}$ is a matrix with i.i.d. entries $W_{jl}\sim\normal(0,1/d_0)$,
and $\sigma:\R\to\R$ is a (nonlinear) activation function which acts on vectors entrywise.
We then construct vectors $(\bx_i)_{i\le N}$ by `whitening' the $\{\obx_i\}_{i\le N}$. 
Namely, letting $\hbSigma:= N^{-1}\sum_{i\le N}\obx_i\obx_i^{\sT}$,
we define $\bx_i = \hbSigma^{-1/2}\obx_i$.
In Fig.~\ref{fig:mnist_tanh} we use $d\in\{250, 350,500\}$ random features and
$\sigma(x) = \tanh(x)$, but similar results are obtained with other activations.

For several values of $n$, we subsample $n$ out of the $N$ samples and fit multinomial regression and
average the observed test/train error and classification error to obtain the empirical data
in Fig.~\ref{fig:mnist_tanh}.

Since we do not know the ground truth, we fit multinomial regression to the whole 
$N$ samples $\{(y_i,\bx_i)\}_{i\le N}$ thus constructed, and assume that the resulting estimate coincides with
$\bTheta_0$. We then extract $\bR_{00} = \bTheta_0^{\sT}\bTheta_0$, which is the only unknown quantity
to evaluate the predictions of Proposition \ref{prop:multinomial}.

In Fig.~\ref{fig:ESD_mnist} we plot the empirical spectral distribution of the Hessian
at the MLE, and compare it with the theoretical prediction of Proposition \ref{prop:multinomial}.

These experiments suggest the following conclusions:
\begin{enumerate}
    \item Figures \ref{fig:mnist_tanh} and  \ref{fig:ESD_mnist} show reasonable quantitative agreement between theoretical predictions and experiments with real data. Given  the Gaussian covariates of Proposition \ref{prop:multinomial}, such agreement is surprising. While recent universality results \cite{hu2022universality,montanari2022universality,pesce2023gaussian} points in this direction, there is still much unexplained in this agreement.
    \item The train loss appears to be systematically smaller than predicted by the theory while test 
    loss is systematically larger and the test error systematically smaller. 
    In other words, multinomial model is overfitting the data
    a bit more than expected from the Gaussian theory, but still performing better than expected in classification. 
    This suggests that the data is better separated than in he Gaussian model. Also, the Hessian spectrum appears somewhat shifted to the left with respect to Gaussian theory.
    \item We constructed isotropic feature vectors $\bx_i$ through  the `whitening'
     step $\bx_i = \hbSigma^{-1/2}\obx_i$. We expect it to be  possible to generalize 
     our results to non-isotropic feature vectors with non-zero regularization, but leave 
     this extension for future work.
\end{enumerate}

%
%
\section{General empirical risk minimization: Proof of Theorem~\ref{thm:general}}
\label{sec:pf_thm1}

We begin by noting that, as stated in Remark~\ref{remark:Gaussian_isotropic_change_of_variable}, it is sufficient to prove the theorem for the
case $\bSigma=\bI_d$. The general case holds (for $\rho=0$) by noting that
$\bX=\bZ\bSigma^{1/2}$ for $\bZ$ with isotropic Gaussian rows, and suitably
redefining $\bTheta, \bTheta_0$.

\subsection{Applying the Kac-Rice integral formula}
\label{sec:pf_thm1_kr_integral}
As discussed in the introduction, the proof of Theorem~\ref{thm:general} relies on the Kac-Rice formula
for counting the number of zeros of a Gaussian process. 
We recall the generic Kac-Rice formula in the following theorem, which is an adaptation of Theorem 6.2 from \cite{azais2009level}. The only modification is the introduction of a process $\bh$ that is jointly Gaussian with the process $\bz$ whose number of zeros is to be studied. 
Its proof is essentially the same as the one in  \cite{azais2009level}.
\begin{theorem}[Modification of Theorem 6.2, \cite{azais2009level}]
\label{thm:kac_rice}
Let $\cT$ be an open subset of $\R^m$, and $\cH \subseteq \R^M$ be a second set that 
coincides with the closure of its interior.
Let $\bz: \cT \to \R^m$ 
and $\bh:\cT  \to \R^M$ be random fields.
Assume that 
\begin{enumerate}[label=(\arabic*)]
    \item \label{thm:kac_rice_cond1}$\bz,\bh$ are jointly Gaussian.
    \item \label{thm:kac_rice_cond2}Almost surely the function $\bt\mapsto \bz(\bt)$ is of class $C^1$.
    \item \label{thm:kac_rice_cond3}For each $\bt\in \cT$, $\bz(\bt)$ has a nondegenerate distribution (i.e., positive definite covariance).
    \item \label{thm:kac_rice_cond4}We have 
    $\P(\exists \bt \in \cT, \bz(\bt) = \bzero, \bh(\bt) \in \partial\cH) = 0$.
    \item \label{thm:kac_rice_cond5}We have $\P(\exists \bt \in \cT, \bz(\bt) = \bzero, \det(\bJ_\bt \bz(\bt)) = \bzero) = 0$.
\end{enumerate}
Then for every Borel set $\cB$ contained in $\cT$, denoting 
$N_0(\cB) := \left|\left\{ \bt\in\cB : \bz(\bt) =  \bzero, \; \bh(\bt) \in \cH \right\}\right|,$
we have
\begin{equation}
\E[N_0(\cB)]  = \int_{\cB} \E\left[|\det \bJ_\bt \bz(\bt)| \one_{\bh(\bt) \in \cH} \big| \bz(t) = \bzero\right] p_{\bz(\bt)}(\bzero) \de \bt,
\end{equation}
where $p_{\bz(\bt)}$ denotes the density of $\bz(\bt)$.
\end{theorem}
Our goal is to apply this formula to a Gaussian process whose zeros correspond to the critical points of the ERM problem of~\eqref{eq:erm_obj}.
Before we introduce this process, let us fix some notation to streamline the exposition.
For $j\in[k]$, $l\in [k_0]$, let $\bell_j(\bV,\bV_0;\bw) \in\R^n$
and $\rho_j(\bTheta) \in\R^{d}$ 
be the columns of the matrices $\bL(\bV,\bV_0;\bw)$ and $\bRho(\bTheta)$
of~\eqref{eq:def_bL_bRho}, respectively,
and $\bv_j, \bv_{0,l}$ be the columns of $\bV \in\R^{n\times k},\bV_0\in\R^{n\times k_0}$.
We will often use the notation $\bbV := [\bV,\bV_0]$, and suppress the dependence on the arguments in the notation whenever it does not cause confusion.
Furthermore, it will often be convenient to work with the empirical distributions $\hmu(\bTheta)$ of $\sqrt{d}[\bTheta,\bTheta_0]$ and that of the rows of $[\bbV,\bw]$.
With a slight abuse of notation, we denote by $\hmu$ and $\hnu$ the empirical distributions the rows of $\sqrt{d}[\bTheta,\bTheta_0]$ and of the rows of $[\bbV,\bw]$, respectively (these coincide
with Eq.~\eqref{eq:First-Emp-Dist} for $\bV=\bX\bTheta$, $\bV_0=\bX\bTheta_0$). 

Let $m_n := dk +nk +nk_0$. For fixed $\bw\in\R^n$, define the \emph{gradient process}
$\bzeta(\,\cdot\,;\bw) : \R^{m_n} \to \R^{m_n},$ 
\begin{equation}
\nonumber
    \bzeta(\bTheta,\bV,\bV_0;\bw) :=
    \begin{bmatrix}
     \bX^\sT\bell_1(\bV,\bV_0,\bw) + n \brho_1(\bTheta)\\
     \vdots\\
     \bX^\sT \bell_k(\bV,\bV_0,\bw) + n \brho_k(\bTheta)\\
     \bX\btheta_1 - \bv_1\\
     \vdots\\
     \bX\btheta_k - \bv_k\\
     \bX\btheta_{0,1} - \bv_{0,1}\\
     \vdots\\
     \bX\btheta_{0,k_0} - \bv_{0,k_0}
    \end{bmatrix}.\label{eq:GradProcess}
\end{equation}
From the definition of $\bzeta$, it's easy to note that for any $\bw$,
\begin{equation}
    \{(\bTheta,\bbV) : \bzeta(\bTheta,\bbV; \bw) = \bzero \} = \{(\bTheta,\bbV): \grad_\bTheta \hat R_n(\bTheta) = \bzero, \bbV = \bX [\bTheta,\bTheta_0]\}.
\end{equation}
Let
\begin{equation}
\label{eq:bH_def}
\bH(\bTheta,\bbV;\bw) := \bH_0(\bbV; \bw) + n \grad^2 \rho(\bTheta), \quad
    \bH_0(\bbV;\bw) := \left(\bI_k \otimes \bX\right)^\sT \bSec(\bbV;\bw)\left(\bI_k \otimes \bX\right),
\end{equation}
\begin{equation}
   \bSec(\bbV;\bw) := \begin{pmatrix}
\bSec_{i,j}(\bbV;\bw)
   \end{pmatrix}_{i,j \in[k]}
,\quad
    \bSec_{i,j}(\bbV;\bw):= \Diag\left\{\partial_{i,j}\ell(\bbV;\bw)\right\}\in\R^{n\times n}. 
    \label{eq:SecDef}
\end{equation}
Recall the notation for Hessian of the regularizer
 $\grad^2\rho(\bTheta) = \Diag\{\rho''(\sqrt{d}\Theta_{i,j})\}_{i,j \in [d]\times [k]} \in\R^{dk\times dk}$.
Observe the relation between these quantities and the Hessian of the empirical risk at the points $\bzeta = \bzero$. We have
\begin{align}
\label{eq:inclusion_1_f=0}
\{\bzeta =\bzero \} \;\; \subseteq \;\;
\{\bbV = \bX(\bTheta,\bTheta_0) \} \;\;\subseteq\;\;
\left\{\nabla^2\hR_n(\bTheta) = \frac1n \bH(\bTheta,\bbV;\bw) \right\}\,.
\end{align}

In order to study the expected size of $Z_n$ defined in Eq.~\eqref{eq:number_of_zeros_main}, we
would like to apply
Theorem~\ref{thm:kac_rice} to $\bzeta$ with the constraint $\bH(\bTheta,\bbV;\bw) \succeq \bzero$ on the Hessian, along with the additional constraints on $\cG_n$ appearing in the definition of $\cZ_n$ in~\eqref{eq:set_of_zeros_main}.
However, the process $\bzeta$ is \emph{degenerate}: as we show in Appendix~\ref{section:kac_rice}, its covariance has rank $m_n -r_k$ for $r_k := k(k+k_0)$
while Theorem~\ref{thm:kac_rice} requires the dimension of the index set to be the same as the rank of the covariance of the process. On the other hand, by the KKT conditions, the points $(\bTheta,\bbV)$ corresponding to critical points of $\hat R_n$ belong 
to the $m_n -r_k$ dimensional manifold 
$\cM_0 := \{\bG(\bTheta,\bbV) = 0\}$
where $\bG$ was defined in~\eqref{eq:def_G}.
Furthermore, with some algebra (Lemma~\ref{lemma:eig_vecs_NS_Sigma}), one can show that the mean $\bmu(\bTheta,\bbV)$ of the process $\bzeta(\bTheta,\bbV)$ is orthogonal to the nullspace of the covariance 
 at any point $(\bTheta,\bbV)$ in this manifold, so that restricting to this manifold gives a process of dimension $m_n - r_k$ to which we apply Theorem~\ref{thm:kac_rice}.
The following lemma provides the necessary extension of
Theorem~\ref{thm:kac_rice}.

\begin{lemma}[Kac-Rice on the manifold]
\label{prop:kac_rice_manifold}
\label{lemma:kac_rice_manifold}
Let $\sfA_{R}, \sfA_{R},  \sfa_{L}, \sfa_{G,n}$ be as in the definition of $\cG_n$ in Eq.~\eqref{eq:G-Def} 
and $\sfA_Z$ be as in the statement of Theorem~\ref{thm:general}. Further let
\begin{equation}
    \sfA_V := \sfA_Z \sfA_R
\end{equation}
Fix $\bw \in\R^n$ (suppressed in the notation).
Recall the definition of $\bL$ from Eq.~\eqref{eq:def_bL_bRho}.
For $\cuA,\cuB$ as in Assumption~\ref{ass:sets}, define the \emph{parameter manifold}

\begin{align}
\nonumber\cM(\cuA,\cuB) := \Big\{&
   (\bTheta,\bbV) : \hmu\in\cuA,\;
   \hnu\in\cuB,\;
   \bG(\bbV,\bTheta) =  \bzero,\;
   \sfa_R^2\prec \bR(\hmu) \prec\sfA_R^2,\;\\
   \label{eq:param_manifold_def} &
 \bbV^\sT\bbV \prec \sfA_{V}^2n, \;
\sigma_{\min}(\bL(\bbV)) \succ \sqrt{n}\,\sfa_{L},\;
\sigma_{\min}\left( \bJ_{(\bbV,\bTheta)} \bG\right) > \,\sfa_{G,n}
    \Big\}.
\end{align}
Let $\bB_{\bLambda(\bTheta,\bbV)}$ be an orthonormal basis matrix for the column space of the covariance of $\bzeta$ at $(\bTheta,\bbV)$ (defined in Corollary~\ref{cor:proj} of the appendix), and 
$\bz(\bTheta,\bbV) := \bB_{\bLambda(\bTheta,\bbV)}^\sT \bzeta(\bTheta,\bbV).$
Let $p_{\bTheta,\bbV}(\bzero)$ be the density of $\bz(\bTheta,\bbV)$, and finally, let
\begin{align}
\nonumber
Z_{0,n}(\cuA,\cuB) 
:= \big|\{(\bTheta,\bbV) \in\cM(\cuA,\cuB) : \bzeta = \bzero,\; \bH \succeq\bzero\}\big|\, .
\end{align}
Then under Assumptions~\ref{ass:loss},~\ref{ass:regularizer},~\ref{ass:sets}, we have 
\begin{align}
\label{eq:kr_eq_manifold}
\E[Z_{0,n}(\cuA,\cuB)|\bw] 
    &=\int
    \E\bigr[\left| \det (\de \bz(\bTheta,\bbV) )\right|
    \one_{\bH \succeq \bzero}
    \big| \bz  = \bzero, \bw\bigl] p_{\bTheta,\bbV}(\bzero)  \de_\cM V
\end{align}
where the latter is an integral over the manifold $\cM(\cuA,\cuB)$ 
(with the volume element denoted by $\de_{\cM}V$), and $\de \bz(\bTheta,\bbV) : T_{(\bTheta,\bbV)}\cM(\cuA,\cuB) \to \R^{m_n - r_k}$ is the differential which we identify with a $\R^{(m_n - r_k)\times (m_n - r_k)}$ Euclidean matrix.
\end{lemma}
The proof of this lemma is deferred to Section~\ref{sec:proof_of_kac_rice_on_manifold}.
Directly from the definitions above, we have
\begin{equation}
\label{eq:inclusion_2_f=0}
    \{\bTheta: \grad \hat R_n(\bTheta) = \bzero \} \;\; =\;\; \{\bTheta: \exists \bbV\;\textrm{s.t.}\;\bzeta =\bzero,\; \bG(\bbV,\bTheta) = \bzero\}.
\end{equation}
This equality along with the inclusion of 
Eq.~\eqref{eq:inclusion_1_f=0}
implies then that for $\cuA,\cuB$
in their respective domains, we have 
on the event $\Omega_\delta$ defined in Eq.~\eqref{eq:OmegaDeltaDef}
\begin{equation}
      Z_n(\cuA,\cuB)
       = Z_{0,n}(\cuA,\cuB) \label{eq:Z=Z0}
\end{equation}
for $Z_n$ as in Eq.~\eqref{eq:number_of_zeros_main}.
So to derive the bound of Theorem~\ref{thm:general}, we will study the asymptotics (up to first order in the exponent) of the integrand in Eq.~\eqref{eq:kr_eq_manifold}.

\subsection{Integration over the manifold}
\label{sec:integration_over_manifold}
To control the integral in Eq.~\eqref{eq:kr_eq_manifold} over the parameter manifold, we will upper bound the integral by a volume integral over the \emph{$\beta$-blow up} of $\cM$ defined by
\begin{equation}
    \cM^\up{\beta}(\cuA,\cuB) := \{\bp \in\R^{m_n}: \exists\;\bp_0\in\cM(\cuA,\cuB),\quad \norm{\bp-\bp_0}_2\le \beta\}
\end{equation}
for some $\beta >0$. 
To compute the asymptotics of the integral in Eq.~\eqref{eq:kr_eq_manifold},  we will choose a sequence $\{\beta_n\}_n$ so that 
the volume integral over the proper subset of $\R^{m_n}$ is a good approximation of the manifold integral for all $n$.

This blow-up approximation goes through thanks to 
Assumption~\ref{ass:params}, which states that the minimum singular value of the Jacobian of the constraint $\bG$ defining $\cM$ is lower bounded by $\sfa_{G,n}> 0$. The approximation of manifold integrals by volume integrals is formalized  by
the following lemma.
\begin{lemma}[Manifold integral lemma]
\label{lemma:manifold_integral} 
Let $r_k := k(k+k_0)$, $m_n:= nk + nk_0 + dk$.
Let $f: \cM^\up{1} \subseteq \R^{m_n}\rightarrow \R$
be a nonnegative continuous function. 
There exists a constant $C = C(\sfA_{V},\sfA_{R})>0$ that depends only on $(\sfA_{V},\sfA_{R})$, such that for any positive sequence $\{\beta_n\}_n$ satisfying
\begin{equation}
   \beta_n  \le \frac{C(\sfA_{V},\sfA_{R}) \, \sfa_{G,n}}{r_k^2} 
\end{equation}
we have
\begin{equation} 
\label{eq:Lemma2_integration_over_blowup_form}
        \int_{(\bTheta,\bbV)\in \cM} f(\bTheta,\bbV) \de_\cM V
        \le
        \Err_{\sblowup}(n)
         \,
        e^{\beta_n\;\norm{\log f}_{\Lip,\cM^{(1)}}}
        \int_{(\bTheta,\bbV)\in \cM^{(\beta_n)}} f(\bTheta,\bbV)\de(\bTheta,\bbV),
\end{equation}
where the multiplicative error $\Err_{\sblowup}(n)$ is given explicitly in Lemma~\ref{lem:intg-tube}. Further
 when 
$  \beta_n = e^{-o(n)}\text{ and }
  \beta_n = o(\sfa_{G,n})$,
the multiplicative error satisfies
$\lim_{n\to\infty}  \frac1n \log \Err_{\sblowup}(n) = 0$.
\end{lemma}
The proof of this lemma is deferred to Appendix~\ref{section:manifold_integration}. The determinant term appearing in Eq.~\eqref{eq:kr_eq_manifold}
is that of the differential $\de \bz(\bTheta,\bbV)$ defined on the tangent space of $\cM$. 
We relate this to the Euclidean Jacobian of $\bzeta(\bTheta,\bbV)$ which will be defined on $\cM^\up{1} \subseteq R^{m_n}$.
Let $\bB_{\bT(\bTheta,\bbV)}$ and $\bB_{\bLambda(\bTheta,\bbV)}$ be  a basis for the tangent space of $\cM$ and the column space of the covariance of $\bzeta$    at $(\bTheta,\bbV)$, respectively (we will denote the
covariance by $\bLambda(\bTheta,\bbV)$). Further let
$\bB_{\bT^c(\bTheta,\bbV)}$ and $\bB_{\bLambda^c(\bTheta,\bbV)}$ be basis matrices for the orthogonal complement of these spaces.
Suppressing $(\bTheta,\bbV)$ in the arguments, we have in this notation $\det(\de\bz)  = \det\left(\bB_{\bLambda}^\sT \bJ \bzeta \bB_\bT\right)$.
Since the codimension of  the tangent space of $\cM$, and of the column space of $\bLambda$ are $r_k=O(1)$, we expect that $\log |\det(\de \bz(\bTheta,\bbV))| = \log \det|\bJ \bzeta(\bTheta,\bbV)|+O(1)$ for large $n$. In turn, we can directly compute
    \begin{equation}\label{eq:zeta_jacobian}
        \bJ \bzeta = \begin{bmatrix}
            n\grad^2 \rho& (\bI_k\otimes\bX)^\sT\bSec &
            (\bI_k\otimes \bX)^\sT\tilde\bSec\\
            \bI_k\otimes\bX& -\bI&\bzero\\
            \bzero &\bzero &-\bI
        \end{bmatrix}
\end{equation}
to see that $|\det(\bJ \bzeta)| = |\det\left(\bH \right)|$,
where $\bH$ was defined in~\eqref{eq:bH_def}.
Characterizing the asymptotic spectral density of $\bH$ will allow us to determine the asymptotics of the determinant using Gaussian concentration. 
Of course, we need to modify such calculation to correctly account for the 
conditioning on $\bz=\bzero$.

We will formalize this analysis below, while deferring most technical details to the appendix.

\subsubsection{Relating the differential to the Hessian}
As noted previously (and stated in Lemma~\ref{lemma:eig_vecs_NS_Sigma} of the appendix), the projection of $\bzeta$ onto the nullspace of $\bLambda$ vanishes,
so that
$\bB_{\bLambda^c}^\sT \bJ\bzeta\bB_{\bT} = \bzero$ and hence
\begin{equation}
\label{eq:det_projection}
\det(\de\bz)  = \det\left(\bB_{\bLambda}^\sT \bJ \bzeta \bB_\bT\right)  =\frac{ \det\left( \bJ \bzeta\right)}
   {\det( \bB_{{\bLambda}^c}^\sT \bJ \bzeta \bB_{\bT^c})}\, .
\end{equation}

The next lemma expresses the determinant in the denominator of the last display in terms of more familiar quantities.
Its proof is deferred to Appendix~\ref{sec:determinant_bound}. 
\begin{lemma}
   \label{lemma:det_complement} 
Under Assumption~\ref{ass:loss} and~\ref{ass:regularizer}, we have for any $(\bTheta,\bbV)\in \cM(\cuA,\cuB),$ 
on the event $\bzeta = \bzero$, 
\begin{equation}
\label{eq:first_claim_det_reduction_lemma}
    \det( \bB_{{\bLambda}^c}^\sT \bJ \bzeta \bB_{\bT^c}) 
    = \frac{ n^{r_k}\det\left(\bJ \bG (\bJ \bG)^\sT\right)^{1/2}}{\det\left([\bTheta,\bTheta_0]^\sT [\bTheta,\bTheta_0]\otimes \bI_k + \bI_{k+k_0} \otimes \bL^\sT\bL \right)^{1/2}}.
\end{equation}
Consequently, 
\begin{equation}
|\det( \bB_{{\bLambda}^c}^\sT \bJ \bzeta \bB_{\bT^c}) | \ge 
\frac{\sfa_G^{r_k}}
{C(k, \sfA_R,\sfA_V) }
\end{equation}
for some constant $C>0$ depending only on $\sfA_R,\sfA_V$ and $k$.
\end{lemma}
Combined with Eq.~\eqref{eq:det_projection}, and the fact that $|\det(\bJ \bzeta)|=|\det(\bH)|$,
Lemma~\ref{lemma:det_complement} then gives the bound 
\begin{equation}
\label{eq:simplfied_det_bound}
\E\left[\left| \det (\de \bz(\bTheta,\bbV) )\right|
    \one_{\bH\succeq \bzero}
    \big| \bz  = \bzero\right] 
    \le 
\frac{ C(k,\sfA_R,\sfA_V)\, }{\sfa_{G}^{r_k}}
\E\left[\left| \det (
\bH)\right|
    \one_{\bH \succeq \bzero} 
    \big| \bz = \bzero\right].
\end{equation}
Now letting
\begin{equation}
\label{eq:integrand_of_interest}
  f_0(\bTheta,\bbV)  :=
\frac{C(k,\sfA_R,\sfA_V)}{\sfa_{G}^{r_k}}
\E\left[\left| \det (
\bH)\right|
    \one_{\bH \succeq \bzero} 
    \big| \bz = \bzero\right]
    p_{\bTheta,\bbV}(\bzero),
\end{equation}
the function $f_0,$ defined on $\cM^\up{1}$, is a point-wise upper bound for the integrand in Eq.~\eqref{eq:kr_eq_manifold} holding for $(\bTheta,\bbV) \in \cM$.
We will obtain further
a point-wise upper bound for the density term $p_{\bTheta,\bbV}(\bzero)$ and the conditional expectation term in the next section holding for $(\bTheta,\bbV)\in \cM$, before applying Lemma~\ref{lemma:manifold_integral} to obtain an integral over the blow-up $\cM^\up{\beta_n}$ of the resulting upper bound.

   %


\subsection{Computing and bounding the density term}
For any  
$(\bTheta,\bbV)$, $p_{\bTheta,\bbV}(\bzero)$ is the density function of a Gaussian  random variable at $\bzero$.
With some algebra, this can be shown to be  (cf. Lemma~\ref{lemma:density})
\begin{align}\label{eq:DensityAtZero}
  p_{\bTheta,\bbV}(\bzero) 
    :=
    \exp\bigg\{-\frac{1}2\left(\Tr(\bbV \bR^{-1}\bbV) + n\Tr\left(\bRho (\bL^\sT\bL)^{-1}\bL^\sT \bbV \bR^{-1}(\bTheta,\bTheta_0)^\sT\right)
    \right)\\\nonumber
    +n^2
\Tr\left(\bRho (\bL^\sT\bL)^{-1}\bRho\right)
\bigg\}
    {
\det^*(2\pi\bLambda(\bTheta,\bbV))^{-1/2}
}
    \, ,
\end{align}
with $\det^*$ denoting the product of the non-zero eigenvalues, and the covariance $\bLambda(\bTheta,\bbV)$ is given by
\begin{equation}
\bLambda:= 
   \begin{bmatrix}
       \bL^\sT \bL \otimes \bI_{d}  & \bM & \bM_0\\
       \bM^\sT  & \bTheta^\sT \bTheta \otimes \bI_n & \bTheta^\sT\bTheta_0 \otimes \bI_n\\
       \bM_0^{\sT} & \bTheta_0^\sT\bTheta  \otimes \bI_n& \bTheta_0^{\sT}\bTheta_0 \otimes \bI_n
   \end{bmatrix} ,
\end{equation}
where

\begin{equation}
    \bM :=  \begin{bmatrix}
        \btheta_1\bell_1^\sT & \dots  & \btheta_k \bell_1^\sT\\
        \vdots  &   & \vdots \\
        \btheta_1\bell_k^\sT & \dots  & \btheta_k \bell_k^\sT\\
    \end{bmatrix}
    \in \R^{d k\times n k},\quad
    \bM_0 :=
    \begin{bmatrix}
        \btheta_{0,1}\bell_1^\sT & \dots  & \btheta_{0,k_0} \bell_1^\sT\\
        \vdots  &   & \vdots \\
        \btheta_{0,1}\bell_k^\sT & \dots  & \btheta_{0,k_0} \bell_k^\sT\\
        \end{bmatrix} \in \R^{d{k} \times n k_0 }.
\end{equation}

From the low-rank structure of $\bM$ and $\bM_0$, one would expect that 
$\det^*(\bLambda(\bTheta,\bbV))$ is approximately given by $\det(\bL^\sT\bL \otimes \bI_d)\cdot\det(\bR(\bTheta) \otimes \bI_n)$ for large $n$ and fixed $k,k_0$. 
If we use this heuristic in Eq.~\eqref{eq:DensityAtZero} and rewrite
various quantities 
in terms of the empirical measures $\hmu$ and $\hnu$, we reach the conclusion
of the following lemma. (We refer to Appendix~\ref{sec:density_bound} for its proof.)
 \begin{lemma}[Bounding the density]
\label{lemma:density_bounds}
Define the Gaussian densities
\begin{align}
    p_{1}(\bTheta) := \frac{d^{dk/2}}{(2\pi)^{dk/2}} \exp\left\{-\frac{d}{2} \Tr\left(\bTheta^\sT\bTheta\right)\right\},\quad
    p_{2}(\bbV) := 
    \frac1{(2\pi)^{n(k+k_0)/2}}
\exp\left\{-\frac12
    \Tr\left(\bbV^\sT
    \bbV\right)
    \right\}.
\end{align}
Under Assumptions \ref{ass:regime} to \ref{ass:params} of Section~\ref{sec:assumptions}, there exist constants $n_0, C >0$  
 such that for all $(\bTheta,\bbV) \in\cM(\cuA,\cuB)$ and $n> n_0,$
\begin{equation}
\label{eq:density_bound_main_eq}
   p_{\bTheta,\bbV}(\bzero)\le (C\, n\, \sfA_V)^{r_k} \,
   (\alpha_n^{1/2} d)^{-dk}
   e^{n h_0(\hmu,\hnu;\alpha_n)}  \;
    p_1(\bTheta) p_2(\bbV)
\end{equation}
where $\alpha_n:=n/d$ and 
\begin{align}
\nonumber
h_0(\mu,\nu;\alpha) :=& - \frac{1}{2\alpha} \log \det\E_\nu[\grad\ell\grad\ell^\sT] 
+\frac{1}{2 \alpha}\Tr\bR_{11}(\mu)
-
\frac{1}2 \Tr(\E_\mu[\grad \rho\grad\rho^\sT] \E_\nu[\grad\ell\grad\ell^\sT]^{-1})\\\nonumber
&- \frac{1}2 \log\det \bR(\mu)
+\frac{1}2 \Tr\big(\E_\nu[\bbv\grad\ell^\sT] \E_\nu[\grad\ell\grad\ell^\sT]^{-1} 
\E_\nu[\grad\ell\bbv^\sT]
\bR(\mu)^{-1}
\big)  \\\nonumber
&+ \frac{1}{2}\Tr\big((\bI_k - \bR(\mu)^{-1}) \E_\nu[\bbv\bbv^\sT]\big).
\end{align}
Here, $\grad \ell$ is the gradient with respect to the first $k$ arguments.
 \end{lemma} 

In the above lemma, we deliberately wrote the bound in terms of product of a density over $(\bTheta,\bbV)$  (the term $p_1(\bbV) p_2(\bTheta)$), 
and a term that depends on $(\bTheta,\bbV)$ only through their empirical distributions.

\subsection{Analysis of the determinant term}

\subsubsection{Asymptotic spectral density of the Hessian}
Recall now the definition of $\mu_\star(\nu,\mu)$ introduced in Section~\ref{sec:definitions}.
The following proposition affirms that $\mu_\star$ is the limit of the empirical spectral distribution of $\bH$, uniformly over $\bbV,\bTheta.$ We recall that $\bH_0=\bH_0(\bbV;\bw)$ is the rescaled Hessian of the loss part of
the risk, cf. Eq.~\eqref{eq:bH_def}. 
%

\begin{proposition}
\label{prop:uniform_convergence_lipschitz_test_functions} 
Under Assumptions of Section~\ref{sec:assumptions}, we have for any Lipschitz function $g:\R\to\R$, and constant $C>0$ independent of $n$,
  \begin{equation}
  \lim_{\substack{n\to\infty\\n/d \to \alpha}}
  \sup_{\substack{\bw\in\R^n,\\\bbV \in\R^{n\times k } \|\bTheta\|_F\le C }}\left|\frac1{dk}\E\left[\Tr \,g\left(\frac1n\bH(\bTheta,\bbV;\bw)\right)\right]
      - \int g(\lambda) \mu_\star(\hnu,\hmu)(\de \lambda)
      \right| = 0.
  \end{equation}
  Moreover, if $\hnu \Rightarrow \nu$ weakly in probability for some $\nu\in\cuP(\R^{k+k_0+1})$,
  we have for any fixed $z\in\bbH_+$ the convergence in probability
\begin{equation}
\label{eq:ST_convergence_in_P_seq_measures}
     \frac1{dk} (\bI_k \otimes \Tr) \left(\bH_0(\bbV;\bw) - z\bI_{dk}\right)^{-1} \to \alpha \bS_\star(\nu,z)\, ,
\end{equation}
where for any $z\in \bbH_+$,  $\bS_{\star}(\nu,z)$ is defined as the unique solution of  Eq.~\eqref{eq:fp_eq}.
\end{proposition}
The proof of this proposition is deferred to Appendix~\ref{sec:RMT},
which also proves uniqueness of  $\bS_\star(\nu,z)$ in Appendix \ref{app:sec:UniquenessSstar}. There, we analyze the empirical Stieltjes transform of $\bH_0$  
via a leave-one-out approach that is similar to the one used in deriving the asymptotic density of Wishart matrices $\bX^\sT\bX$~\cite{BaiSilverstein}. Of course the difference here is that the Stieltjes transform is an element of $\C^{k\times k}$ (often called the operator-valued Stieltjes transform in the free probability literature~\cite{speicher2019non}). As a result, the analysis requires additional  care compared to the scalar case to deal with the additional complications arising from non-commutativity. 
Finally, the empirical spectral distribution of $\bH$ follows from the one of $\bH_0$
via a free probability argument.
See Appendix~\ref{sec:RMT} for details.

\subsubsection{Conditioning on the process and concentration of the determinant.}
We outline the main steps in bounding the conditional expectation of the determinant appearing in the right-hand side of Eq.~\eqref{eq:simplfied_det_bound}.
We leave most technical details to Appendix~\ref{sec:determinant_bound}.

First, note that
since the mean of $\bzeta$ is in the column space of $\bLambda(\bTheta,\bbV)$ for any $(\bTheta,\bbV)\in\cM$, conditioning on $\bz = \bzero$ is equivalent to conditioning 
on $\bzeta = \bzero$. 
The latter, meanwhile, is equivalent to conditioning on
$\{\bL^\sT\bX = - n\bRho^\sT,\; \bX[\bTheta,\bTheta_0] = \bbV\}$. 
So letting $\bP_{\bTheta}, \bP_\bL$ be the projections onto the columns spaces of $[\bTheta,\bTheta_0],\bL$ respectively, we have on $\{\bzeta = 0\}$
\begin{equation}
    \bX = \bP_\bL^\perp \bX \bP_\bTheta^\perp - n\bL(\bL^\sT\bL)^{-1} \bRho^\sT \bP_{\bTheta}^\perp  + \bbV\bR^{-1} [\bTheta,\bTheta_0]^\sT.
\end{equation}

Since  $\bP_\bL^\perp \bX \bP_\bTheta^\perp$ is independent of
$\bL^\sT\bX$, $\bX[\bTheta,\bTheta_0]$, for any measurable function $g$,
\begin{equation}
\label{eq:conditioning_generic}
    \E[g(\bX) | \bzeta = 0] = \E[g(\bX + \bDelta_{0,k})]
\end{equation}
for some matrix $\bDelta_{0,k} = \bDelta_{0,k}(\bTheta,\bbV)$ 
satisfying
\begin{align}
    &\rank(\bDelta_{0,k}) \le 4(k+k_0),\\
    &\norm{\bDelta_{0,k}}_\op \le C\sqrt{n} \max\bigg(\frac{\norm{\bX}_\op}{\sqrt{n}} ,
    C(\sfA_{V},\sfA_{R},\sfa_R,\sfa_L)
    \bigg)
\end{align}
for some constant $C>0$ independent of $n$, uniformly over $(\bTheta,\bbV)\in\cM.$

The identity \eqref{eq:conditioning_generic} allows us to compute the conditional expectation as 
\begin{align}
\label{eq:concentration_decomp_0_main}
\E\left[\left| \det (\bH )\right|
    \one_{\bH \succeq \bzero}
    \big| \bz  = \bzero, \bw\right]
&=
\E\left[  \big|\det\left(\bH + \bDelta_{1,k}\right)\big|    \one_{\{\bH + \bDelta_{1,k} \succeq \bzero\}}
\big| \bw
\right]
\end{align}
for some $\bDelta_{1,k}$ of rank at most some $r'_k = O(r_k)$. 
Meanwhile, the interlacing theorem implies that, for any $i\ge dk-r'_k$,
\begin{equation}
\label{eq:interlacing}
    \lambda_{i+r'_k}(\bH+  \bDelta_{1,k})\leq \lambda_{i}(\bH),\quad \lambda_i(\bH + \bDelta_{1,k})\leq \|\bH + \bDelta_{1,k}\|_\op.
\end{equation}
As a consequence, the determinant  in Eq.~\eqref{eq:concentration_decomp_0_main} can be estimated as follows.
For any $\tau_1 >0$,
  \begin{align}
\label{eq:log_det_to_log_eps}
      \det((\bH  + \bDelta_{1,k})/n)=& \exp\Big\{\sum_{i=1}^{dk}\log 
      (\lambda_i(\bH+  \bDelta_{1,k} )/n)\Big\}\\
      \leq& \exp\Big\{\sum_{i=1}^{dk -r'_k}\log \lambda_i(\bH /n) +
      r_k'\log\Big( \frac{\|\bH + \bDelta_{1,k}\|_\op}{n}\Big)\Big\}\\
      \leq& \exp\Big\{\Tr\left(\log^{(\tau_1)}(\bH /n) \right)\Big\} \cdot
\Big( \frac{\|\bH+ \bDelta_{1,k}\|_\op}{n \tau_1}\Big)^{r'_k},
  \end{align}
  where $\log^\up{\tau_1}(t) := \log(\tau_1 \vee t).$
As we show in Section~\ref{sec:pf_lemma_CE_bound} of the appendix, 
the term $\|\bH+ \bDelta_{1,k}\|_\op$ is at most polynomial in $n$ uniformly over $(\bTheta,\bbV) \in \cM(\cuA,\cuB)$ 
with high probability.
Therefore, for the indicator involving the minimal singular value of $\bH$
in Eq.~\eqref{eq:concentration_decomp_0_main}, we note that 
$$\{\bH + \bDelta_{1,k} \succ \bzero\}\subseteq
    \left\{
    \big|\left\{ \lambda \in \spec\left(\bH /n \right) : \lambda \leq 0\right\} \big| < r'_k
    \right\}= \{ \hmu_{\bH}(-\infty,0) < r'_k/n\}$$
    where $\hmu_{\bH}$ is the empirical spectral measure of $\bH/n$.
    
With the proper formalization of the above, along with a concentration argument showing that Lipschitz functions of $\bH$ concentrate super-exponentially, we reach the following lemma which summarizes the results of analyzing the determinant.

\begin{lemma}[Bounding the conditional expectation of the determinant]
 \label{lemma:CE_bound}
Fix $\tau_0,\tau_1 \in (0,1)$, and $\bw$ satisfying $\|\bw\|_2 \le \sfA_{w}\sqrt{n}$.
Then under Assumptions \ref{ass:regime} to \ref{ass:params} of Section~\ref{sec:assumptions}, there exist constants $C,c>0$,
and $C_0(\tau_0)$ depending on $\tau_0$, both independent of $n$, such that for all $n > C_0(\tau_0)$, 
\begin{enumerate}
\item For any $(\bbV,\bTheta) \in\cM(\cuA,\cuB)$ satisfying $\mu_{\star}(\hnu,\hmu)((-\infty, -\tau_0)) < \tau_0,$ we have 
\begin{align}
\E[|\det \big(\bH\big)|\one_{\bH\succeq \bzero} \big|\bzeta=\bzero,\bw]
\le
n^{dk}
\Big(&\exp\Big\{
\E\left[\Tr\log^{(\tau_1)}\left(\bH/n\right)\Big| \bw\right] +  \frac{C n^{3/4}}{\tau_1}
\Big\}
+e^{-n^{5/4}}
\Big)
\end{align}
\item For $(\bbV,\bTheta)\in\cM(\cuA,\cuB)$ 
satisfying 
$\mu_{\star}(\hnu,\hmu)((-\infty, -\tau_0)) \ge \tau_0,$
\begin{align}
\E[|\det \big(\bH\big)|\one_{\bH \succeq \bzero} \big|\bzeta = 0 ,\bw]
& \le C \exp \left\{ 
    -c 
    \tau_0^2
    n^{3/2}
    \right\}.
\end{align}
\end{enumerate}
\end{lemma}

\subsection{Asymptotics of the integral}
Finally, the analysis of the previous three sections can be applied to upper bound the asymptotics of the integral of Eq.~\eqref{eq:kr_eq_manifold}
by first recalling that $Z_n = Z_{0,n}$ (see Eq.~\eqref{eq:Z=Z0}), combining the bound on the density of 
Eq.~\eqref{eq:density_bound_main_eq}
along with the bounds of Eq.~\eqref{eq:simplfied_det_bound} and Lemma~\ref{lemma:CE_bound} on the conditional expectation term (holding point-wise over $\cM$), approximating the integral of the resulting bound by an integral over the blow-up, then applying Proposition~\ref{prop:uniform_convergence_lipschitz_test_functions} to pass to the asymptotic spectral density.
The exact details are left to Appendix~\ref{sec:proof_prop_asymp_1}, where we detail the proof of the following lemma summarizing the result.
%
\begin{lemma}[Upper bound on the Kac-Rice integral]
\label{prop:asymp_1}
\label{lemma:asymp_1}
Fix $\tau_0,\tau_1\in (0,1)$. Let $\Omega_\delta,\cuG$ be as defined in Theorem~\ref{thm:general}.
Define 
\begin{equation}
\nonumber
\phi_{\tau_1}(\nu,\mu; \alpha)
:=
 \frac{k}{2\alpha}\log(\alpha)+
\frac{k}{\alpha}
\int \log^{(\tau_1)}(\lambda) \mu_{\star}(\nu,\mu)(\de \lambda) + h_0(\mu,\nu; \alpha),\quad\quad\textrm{and}
\end{equation}
\begin{equation}
\nonumber
    \cuM^{(\beta)}(\cuA,\cuB) := \left\{ (\mu,\nu) :  \exists\;  (\mu_0,\nu_0) \in \cuM(\cuA,\cuB) \; \textrm{s.t.} \;  W_2(\nu,\nu_0) < \beta, W_2(\mu,\mu_0) < \beta \right\}
\end{equation}
where (for $\cuG$ defined in Eq.~\eqref{eq:cuG_Def})
\begin{align}
\nonumber
\cuM(\cuA,\cuB) := \Big\{&(\mu,\nu) \in \cuA\times \cuB :\;
\E_\nu[\grad_\bv \ell(\bv,\bv_0,w)(\bv,\bv_0)^\sT]+ 
     \E_\mu[\rho'(\bt) (\bt, \bt_0)^\sT] =   \bzero
\Big\} \cap \cuG.
\end{align}
%
Under Assumptions~\ref{ass:regime},\ref{ass:loss},\ref{ass:regularizer} and~\ref{ass:sets}, there exists a constant $C(\sfA_{R},\sfA_{V})>0$ depending only on $(\sfA_{R},\sfA_{V})$, such that
for any $\beta\in(0, 1)$,
we have
\begin{align*}
&\limsup_{n\to\infty}\frac1n\log\E[Z_n(\cuA,\cuB) \one_{\Omega_\delta}]\\
&\le
   \limsup_{n\to\infty}\frac1n\log
   \E_\bw\E_{\substack{\bTheta\sim p_1\\ \bbV\sim p_2}}\left[e^{n\phi_{\tau_1}(\hmu,\hnu)}
   \one_{\{\mu_{\star}(\hmu,\hnu)((-\infty, -\tau_0]  ) < \tau_0\} \cap \cuM^{(\beta)}}
   \one_{\Omega_\delta}
   \right] .
\end{align*}
\end{lemma}

At this point, 
the statement of Theorem~\ref{thm:general} follows from this lemma as a consequence of Sanov's theorem and Varadhan's lemma, from which the $\KL$-divergence terms in the formula \eqref{eq:PhiGen} for  $\Phi_{\gen}$ 
appear as the rate function for large deviations of the empirical measures  $\hmu$ and $\hnu$.
The final details of this are left to Appendix~\ref{sec:proof_thm1_large_deviations}.

%% file: Appendix_THM2.tex
\section{Convex empirical risk minimization: Proofs of Theorems \ref{thm:convexity}, \ref{thm:global_min}}
\label{sec:pf_convex_results}

As for Theorem \ref{thm:general}, it is sufficient to prove the theorems of
this section for the
case $\bSigma=\bI_d$. The general case (for $\lambda=0$) follows by the same change 
of variables.

\subsection{Proof of Theorem~\ref{thm:convexity}}

\subsubsection{Simplified variational upper bound under convexity: Proof of point 
\textit{1}}
The upper bound  of Theorem~\ref{thm:convexity} is a special case of Theorem~\ref{thm:general} under the additional assumption of convexity.
In order to derive $\Phi_\cvx$ from $\Phi_\gen$ of Theorem~\ref{thm:general} 
under Assumption~\ref{ass:convexity}, first note that, by Remark
\ref{rmk:Ridge}, we can replace $\Phi_\gen$ by $\Phi_\ridge$ of Eq.~\eqref{eq:PhiRidge}.

Comparing $\Phi_\cvx$ with $\Phi_\ridge$, we see that what remains now is to obtain the following bound on the logarithmic potential.
\begin{lemma}[Variational principle for the log potential]
\label{lemma:variational_log_pot}
Let
\begin{equation}
    \mu_{\star,\lambda}(\nu) :=
    \mu_{\MP}(\nu;\alpha) \boxplus  \delta_{\lambda}.
\end{equation}
   Under Assumption~\ref{ass:convexity} , for any $\nu\in\cuB$, and $\lambda \ge0$,
   we have
    \begin{equation}
    \label{eq:variational_log_pot}
        k\int\log(\zeta ) \mu_{\star,\lambda}(\nu) (\de\zeta)
\le \inf_{\bS\succ\bzero} K_{-\lambda}(\bS;\nu),
    \end{equation}  
    where
\begin{equation}
    K_z(\bQ;\nu):= -\alpha z \Tr(\bQ) + \alpha \E_{\nu}[\log\det(\bI + \grad^2 \ell(\bv,\bv_0, w)\bQ) ]  - \log\det(\bQ) - k (\log(\alpha) + 1).
\end{equation}
\end{lemma}
Since $\rho''(t) = \lambda$ under Assumption~\ref{ass:convexity}, note that for any $(\nu,\mu) \in\cuP(\R^{k+k_0+1})\times \cuP(\R^{k+k_0})$, $\mu_\star(\mu,\nu) = \mu_{\star,\lambda}(\nu).$
This finally shows $\Phi_\gen(\nu,\mu,\bR) \le \sup_{\bS\succ\bzero}\Phi_\cvx(\nu,\mu,\bR,\bS)$ as claimed, giving the claim of point~\textit{1} of the theorem.

\begin{remark}
Note that for $z\in\bbH^+$, by directly differentiating $K_z$ of Lemma~\ref{lemma:variational_log_pot} with respect to $\bQ$ we can easily see that $\bS_\star$ defined by~Eq.~\eqref{eq:fp_eq} is a critical point of $K_z(\bQ;\nu)$.
In the proof of Lemma~\ref{lemma:variational_log_pot} which is deferred to 
Section~\ref{sec:log_pot_proof} of the appendix,
 we show that under the convexity assumption of Assumption~\ref{ass:convexity} this critical point is the  minimizer of $K_z(\bQ).$
\end{remark}

\subsubsection{The critical point optimality condition: Proof of point~\texorpdfstring{$\textit{2}$}{2}}
Let $\bK^2 = \bK(\bR)^2 := \bR/\bR_{00}$ for ease of notation.
Let us decompose $\Phi_\cvx$ as $\Phi_\cvx = \Phi_{\cvx,\mupart} + \Phi_{\cvx,\nupart}$ for
\begin{align}
\Phi_{\cvx,\mupart}(\mu,\bR)
&:=\frac{k}{2\alpha}
- \frac{1}{2\alpha} \Tr\left(\bR_{11}\right)
+ \frac1{2\alpha} \log \det(\bK(\bR)^2)
 + \frac1\alpha \KL ( \mu_{\cdot| \bt_0}\| \cN(\bzero,\bI_k)).\\
    \Phi_{\cvx,\nupart}(\nu,\bR,\bS) &:=
    -\lambda\Tr(\bS)  +
    \frac{1}{2\alpha} \log\det\left(\E_\nu\left[ \grad \ell\grad\ell^\sT\right]\right)
+ 
\frac1{\alpha}\log\det\bS \nonumber\\
&\quad- \E_{\nu}\left[\log \det \left(\bI_k + \grad^2 \ell^{1/2} \bS \grad^2 \ell^{1/2}\right)\right]\nonumber\\
&\quad+\frac{k}{2\alpha} 
+ \frac{k}{2\alpha} \log(\alpha)
-\frac1{2\alpha} \log\det\left(
\bK(\bR)^2
\right)   + \KL\left(\nu_{\,\cdot\, | w}\| \cN(\bzero, \bR)\right).
\end{align}

It is straightforward to verify that 
\begin{equation}
\label{eq:cvx_mupart_KL}
   \Phi_{\cvx,\mupart}(\mu,\bR) =  \frac1{\alpha}\KL\Big( \mu_{\cdot | \bt_0}\|  \gamma_{\cdot|\bt_0}(\bR)\Big),\quad
   \gamma_{\cdot|\bt_0}(\bR) = \cN\big(\bR_{10} \bR_{00}^{-1} \bt_0, \bK(\bR)^2 \big)
\end{equation}
where it is understood that the `outer' expectation in the conditional divergence is taken with respect to measure $\mu_{(\bt_0)} = \mu_0$.

We will next rewrite the divergence term appearing in $\Phi_{\cvx,\nupart}$ as a divergence involving the distribution of the proximal operator in the Definition~\ref{def:opt_FP_conds}.
Note that for $f(\bv,\bv_0, w)$ convex in $\bv$ for fixed $\bv_0,w$,
 the map $\bz \mapsto \Prox_{f(\cdot, \bv_0, w)}(\bz; \bS)$  is invertible  for any $\bS \succeq \bzero_{k\times k}$, with inverse given by
\begin{equation}
\Prox_{f(\cdot,\bv_0,w)}^{-1}(\bv; \bS) = \bS\grad f(\bv,\bv_0,w)+\bv,
\end{equation}
which can be derived from the first order conditions
\begin{align}
     \bS\grad f(\Prox_{f(\cdot;\bv_0,w)}(\bz;\bS),\bv_0,w)=\bz- \Prox_{f(\cdot,\bv_0,w)}(\bz; \bS)
\end{align}
where $\grad f \in\R^{k}$  is the gradient of $f$ with respect to the first $k$ variables.

Let $\bg,\bg_0$ be jointly Gaussian as in Definition~\ref{def:opt_FP_conds}, and $w\sim\P_w$. 
For any $\bS,\bR \succ\bzero$,
denoting
$p^\opt_{\bS,\bR}(\bv| \bv_0, w)$ the conditional density of 
$\Prox_{\ell(\,\cdot\,;\bg_0,w)}(\bg; \bS)$ 
given $w$,$\bv_0=\bg_0$,
we find that
\begin{align}
\label{eq:prox_density}
    & p^\opt_{\bS,\bR}(\bv|\bv_0, w)
   = 
 (2\pi)^{-k/2}\det(\bK(\bR)^2)^{-1/2}
     \det\left(\bI_k+\grad^2\ell(\bv,\bv_0,w)^{1/2}\bS \grad^2\ell(\bv,\bv_0,w)^{1/2}\right)
    \\
    &\cdot\exp\left\{ -\frac12 (\bS\grad\ell(\bv,\bv_0, w )+\bv-\bmu(\bv_0,\bR))^\sT\bK(\bR)^{-2}(\bS\grad\ell(\bv,\bv_0, w )+\bv-\bmu(\bv_0,\bR))\right\}\nonumber
\end{align}
where $\bmu := \bmu(\bv_0, \bR) := \bR_{10}\bR_{00}^{-1}\bv_0$.
Using this identity, the KL
divergence of the conditional measure $\nu_{\bv|\bv_0,w}$ with respect to $\cN(\bmu,\bK^2)$ can be written as
\begin{align*}
    \KL\left(\nu_{\cdot|\bv_0,w}\|  \cN(\bmu,\bK^2)\right) 
    &=\KL\left(\nu_{\cdot|\bv_0,w}\Big\|  p^\opt_{\bS,\bR}(\cdot | \bv_0, w) \right) 
    +
     \E_\nu\left[\log\det\left(\bI_k+\grad^2\ell^{1/2}\bS \grad^2\ell^{1/2}\right)\right]\\
     &\quad\quad-\frac12\E_\nu\left[\grad \ell^\sT \bS \bK(\bR)^{-2} \bS\grad \ell\right]
     -\E_\nu\left[
    \grad \ell^\sT \bS \bK(\bR)^{-2} \left(\bv - \bmu\right)
     \right].
\end{align*}
Recall that $\nu\in\cuV(\bR)$
implies that $\E[\grad \ell \cdot (\bv,\bv_0)^\sT] + \lambda (\bR_{00},\bR_{01}) = \bzero,$ whence
\begin{align*}
\E_\nu\left[
    \grad \ell^\sT \bS\, \bK(\bR)^{-2} \left(\bv - \bmu\right)
     \right]  &=  \Tr\left( \bS \bK(\bR)^{-2} \E_{\nu}[\bv \grad\ell^\sT - \bR_{10}\bR_{00}^{-1}\bv_0\grad\ell^\sT]\right)  
= -\lambda \Tr(\bS).
\end{align*}
This, along with the chain rule for the KL-divergence and the expansion of the conditional $\KL$ above gives
\begin{align}
\nonumber
  \KL\left(\nu_{ \cdot|w}\|  \cN(\bzero,\bR)\right) &= 
\KL\left(\nu_{\cdot|\bv_0,w}\|  p_{\bS,\bR}^\opt(\cdot | \bv_0, w) \right) 
    +
     \E_\nu\left[\log\det\left(\bI_k+\grad^2\ell^{1/2}\bS \grad^2\ell^{1/2}\right)\right]\\
  &\quad
  -\frac12\E_\nu\left[\grad \ell^\sT \bS \bK(\bR)^{-2} \bS\grad \ell\right]
  +\KL\left(\nu_{(\bv_0)|w}\|  \cN(\bzero,\bR_{00})\right) + \lambda \Tr(\bS).
\end{align}
By substituting this equality for the KL term into $\Phi_{\cvx,\nupart}$ and carrying out the appropriate cancellations and combining with Eq.~\eqref{eq:cvx_mupart_KL},
we conclude that for any $\mu,\nu$ as in the statement, 
\begin{align}
   &\Phi_{\cvx,\mupart}(\mu,\bR) + \sup_{\bS\succ\bzero}\Phi_{\cvx,\nupart}(\nu,\bR,\bS) 
     =
     \frac1{\alpha}\KL\big( \mu_{\cdot | \bt_0}\| \gamma_{\cdot|\bt_0}(\bR)  \big)\\
     &\hspace{35mm}+
    \sup_{\bS\succ\bzero} \bigg\{
M(\bS;\nu,\bR) 
+\KL(\nu_{\cdot|\bv_0,w}\|p^\opt_{\bS,\bR})+\KL(\nu_{(\bv_0)|w}\|\cN(\bzero,\bR_{00}))\bigg\}\nonumber\\
&\hspace{35mm}\stackrel{(a)}{\ge}
    \sup_{\bS\succ\bzero} \, M(\bS;\nu, \bR) \, ,\nonumber
\end{align}
with (for $\bK=\bK(\bR)$)
\begin{equation}
    M(\bS;\nu, \bR) = \frac1{2\alpha} \log\det \left(
    \E_{\nu}[\grad \ell \grad \ell^\sT ] \bS^2 \bK^{-2}
    \right)-\frac12\E_{\nu}[\grad\ell^\sT \bS \bK^{-2}\bS \grad\ell]+\frac k{2\alpha}\log(\alpha e).
\end{equation}
One can check that $M(\bS;\nu,\bR)$ is strictly concave in $\bS$ and is uniquely maximized at 
\begin{equation}
\label{eq:S_from_phi}
    \bS= \bS^\opt(\nu,\bR) =\frac1{\sqrt\alpha}\bK\left(\bK^{-1}\E_{\nu}[\grad\bell\grad\bell^\sT]^{-1} \bK^{-1}\right)^{1/2}\bK^{1},
\end{equation}
with $M(\bS^\opt(\nu,\bR);\nu, \bR) = 0$.
This shows that $\Phi_{\cvx} \ge 0$ on  $\cuV_{\mupart}\times \cuV_{\nupart}$.

Finally, the inequality in $(a)$ holds with equality if and only if $\de\nu_{\cdot| \bv_0,w}(\bv) = p_{\bS,\bR(\mu)}^\opt\de\bv$,
$\nu_{\cdot|w} = \cN(0,\bR_{00})$ and $\mu_{\cdot|\bt_0} = \gamma_{\cdot|\bt_0}(\bR)$. Combining these with $\bR=\bR(\mu)$ and 
the constraints defining $\cuV_{\mupart}$ and $\cuV_{\nupart}$ along with Eq.~\eqref{eq:S_from_phi} retrieves Definition~\ref{def:opt_FP_conds}.

\subsection{Proof of Theorem~\ref{thm:global_min}}
To prove point \textit{1}, we'll look at critical points $\hat\bTheta_n$ of the 
empirical risk $\hat R_n$ such that $\hmu(\hat\bTheta_n)$
and $\hnu(\hat\bTheta_n)$ belong to the complement  of the sets
\begin{equation}
    \cuA_\eps:=  \{\mu : W_2(\mu,\mu^\opt) \le\eps \},\quad
    \quad\quad
    \cuB_\eps := 
     \{\nu : W_2(\nu,\nu^\opt) \le\eps \},
\end{equation}
for fixed $\eps >0$. We will then apply Theorem~\ref{thm:convexity} to deduce that the probability that there exist such critical points vanishes under the high-dimensional asymptotics.
Let $\tilde\Omega_0 := \{\hat\bTheta_n \in \cE\}$, and $\tilde\Omega_1 := \{ \bX^\sT\bX \prec
n C_0(\alpha) 
\}$.
We cite two results from Appendix~\ref{sec:simplifying_constraint_set} allowing to simplify the set in Eq.~\eqref{eq:set_of_zeros_main} to the set $\cE$ defined in Theorem~\ref{thm:convexity}.
First, 
under the conditions of the theorem, Lemma~\ref{lemma:jacobian_lb} gives the deterministic bound
\begin{equation}
    \sigma_{\min}\left( \bJ_{(\bbV,\bTheta)} \bG\right) \ge 
     \frac{\sigma_{\min}( (\bI_k\otimes [\bTheta,\bTheta_0])^\sT\grad^2 \hat R_n(\bTheta) (\bI_k \otimes [\bTheta,\bTheta_0] )
     }{
     \|\bX[\bTheta,\bTheta_0]\|_\op + \|[\bTheta,\bTheta_0]\|_\op
     }.
\end{equation}
%
%
Hence,
we have $\sigma_{\min}(\bJ_{(\bbV,\bTheta)} \bG) = e^{-o(n)}$ for all $\bTheta \in\cE$ of Eq.~\eqref{eq:SetUniqueness}.
Further, under the same conditions, Lemma~\ref{lemma:min_sv_Theta} of the Appendix shows that for any $C,c>0$, there exists $c_0>0$ so that
\begin{equation}
\tilde\Omega_2 := \left\{ \forall\; \bTheta 
\;\textrm{with}\;
\|\bTheta\|_F \le C\;
\textrm{and}\;
\grad \hat R_n(\bTheta)  = \bzero,
\;\textrm{if}\;
\sigma_{\min}(\bL) \ge  c \sqrt{n}\; 
\;\textrm{then}\;
\sigma_{\min}([\bTheta,\bTheta_0]) \ge c_0\;  
\right\}
\end{equation}
is a high probability event. 
Hence, on $\tilde \Omega_3 := \tilde\Omega_0 \cap\tilde\Omega_1\cap\tilde\Omega_2$, we have $\hat\bTheta_n \in \cZ_n$ of Eq.~\eqref{eq:set_of_zeros_main} for some choice of $\sPi$ satisfying Assumption~\ref{ass:params}.
By Theorem~\ref{thm:convexity} 
there exists some $c_1(\eps) >0$ such that if 
$\mu \in\cuV_\mupart(\cuA^c_\eps)$ or $\nu \in\cuV_\nupart(\bR(\mu),\cuB^c_\eps)$ for some $\mu\in\cuP(\R^{k+k_0})$, we have $\sup_{\bS\succ\bzero}\Phi_\cvx(\mu,\nu,\bR(\mu)) > c_1(\eps)$ uniformly. 
So
using the shorthand 
$\hmu,\hnu$ for $\hmu(\hat\bTheta_n),\hnu(\hat\bTheta_n)$ respectively,
 we can bound for any $\delta>0$,
\begin{align*}
    \P\left( W_2(\hmu, \mu^\opt) + W_2(\hnu, \nu^\opt) > \eps
    \right)
    &\le  \P\big(\{(\hmu,\hnu) \in\cuA_\eps^c \times \cuB_\eps^c\} \cap \Omega_\delta\cap \tilde\Omega_3\big) + \P(\Omega_\delta^c)+ \P(\tilde\Omega_3^c)\\
&\le \E[\one_{\hat\bTheta_n \in \cZ_n(\cuA_\eps^c,\cuB_\eps^c)} \one_{\Omega_\delta}]
+ \P(\Omega_\delta^c)
+ \P(\tilde\Omega_3^c)
\end{align*}
Now $\lim_{n\to\infty}\P(\tilde\Omega_3^c) = 0$ since $\tilde\Omega_3$ is a high probability set. 
Furthermore, for any $\delta>0$, 
$\lim_{n\to\infty}\P(\Omega_\delta^c) = 0$ by 
Assumption~\ref{ass:Data} on $\bw$.
So the asymptotic bound on the rate function in Theorem~\ref{thm:convexity} then gives that, for any $\eps>0$,
\begin{align}
    \lim_{n\to\infty}\P\left( W_2(\hmu, \mu^\opt) + W_2(\hnu, \nu^\opt) > \eps
    \right)
    \le
\lim_{\delta\to0}\lim_{n\to\infty}
\E[\cZ_{n}(\cuA_\eps^c, \cuB_\eps^c)\one_{\Omega_\delta} ]
\le \lim_{n\to\infty} e^{- n c(\eps)}=0 
\end{align}
giving the statement of \textit{1} of the theorem.

Claim \textit{2} now follows from the convergence in point~\textit{1} and Proposition~\ref{prop:uniform_convergence_lipschitz_test_functions}. This concludes the proof of the theorem.

%% file: acknowledge.tex
This work was supported by the NSF through award DMS-2031883, the Simons Foundation through
Award 814639 for the Collaboration on the Theoretical Foundations of Deep Learning, 
and the ONR grant N00014-18-1-2729.

%% file: Appendix_A_RMT.tex
\newpage
\section{Random matrix theory: the asymptotics of the Hessian}
\label{sec:RMT}
The goal of this section is to study the asymptotics of the spectrum of the Hessian $\bH$
originally defined in Eq.~\eqref{eq:bH_def} of Section~\ref{sec:pf_thm1}, and recalled bellow for convenience.
The section will culminate in the proof of Proposition~\ref{prop:uniform_convergence_lipschitz_test_functions} of Section~\ref{sec:pf_thm1}.

Recall the definitions
\begin{equation*}
\bH(\bTheta,\bbV;\bw) = \bH_0(\bbV; \bw) + n \grad^2 \rho(\bTheta), \quad\quad
    \bH_0(\bbV;\bw) = \left(\bI_k \otimes \bX\right)^\sT \bSec(\bbV;\bw)\left(\bI_k \otimes \bX\right)
\end{equation*}
where
\begin{equation*}
   \bSec(\bbV;\bw) = \begin{pmatrix}
\bSec_{i,j}(\bbV;\bw)
   \end{pmatrix}_{i,j \in[k]}
,\quad
    \bSec_{i,j}(\bbV;\bw)= \Diag\left\{\left(\frac{\partial^2}{\partial {v_i}\partial v_j}\ell(\bbV;\bw)\right)\right\}\in\R^{n\times n}.
\end{equation*}

As always, we'll often find it convenient to suppress the dependence on $\bbV$ and/or $\bw$ in the notation. For example, we write $\bH_0$ or
$\bH_0(\bbV)$ for the matrix $\bH_0(\bbV;\bw).$

The main object of analysis will be the empirical matrix-valued Stieltjes transform defined for $z\in \bbH^+$
\begin{equation}
\label{eq:Sn_def}
\bS_n(z;\bbV) :=  (\bI_k \otimes \Tr)(\bH_0(\bbV) - z n \bI_{dk})^{-1}.
\end{equation}

\subsection{Preliminary results}
This subsection summarizes some preliminary results useful for proofs of this section. 
\subsubsection{Some general properties of  \texorpdfstring{${(\id \otimes \Tr)},\Re$ and $\Im$}{(I x Tr),Re and Im}}

Let us present some properties of the operators $\Re,\Im$ and $\id \otimes \Tr$ (or $\bI\otimes \Tr)$ that will be useful for this section. Some proofs are deferred to Section~\ref{section:RMT_appendix_technical_results}.
\begin{lemma}[Properties of $\Re$ and $\Im$.]
\label{lemma:re_im_properties}
Let $\bZ \in\bbH^+_k$. Then,
\begin{enumerate}
\item 
$\bZ$ is invertible,
\begin{equation*}
    \Im(\bZ^{-1}) = - \bZ^{-1} \Im(\bZ) \bZ^{*-1}\prec\bzero, \quad\quad
\|\bZ^{-1}\|_\op \le \|\Im(\bZ)^{-1}\|_\op, \quad
\norm{\Im(\bZ)}_\op \le \norm{\bZ}_\op.
\end{equation*}

\item For any $\bW$ self-adjoint, we have
\begin{align*}
    \Im((\bI+\bW \bZ)^{-1}\bW ) &= -((\bI + \bW \bZ)^{-1}\bW)\Im(\bZ)((\bI + \bW \bZ)^{-1}\bW)^*,
\end{align*}
and
\begin{align*}
   \norm{(\bI + \bW\bZ)^{-1}\bW}_\op &\le \norm{\Im(\bZ)^{-1}}_\op.
\end{align*}
\end{enumerate}
\end{lemma}
The proof of the lemma above is deferred to Section~\ref{sec:proof_lemma_re_im_properties}.
\begin{lemma}[Properties of $(\bI\otimes\Tr)$]
\label{lemma:tensor_trace_norm_bounds}
\label{lemma:tensor_trace_properties}
Let $\bM \in\C^{dk\times dk}$. Then the following hold.
\begin{enumerate}[1.]
    \item We have the bounds
\begin{align*}
    \norm{(\bI_k \otimes \Tr)\bM}_{\Fnorm} \le  \sqrt{d}\norm{\bM}_\Fnorm
    \quad\textrm{and}\quad
    \norm{(\bI_k \otimes \Tr)\bM}_\op \le  d\norm{\bM}_\op.
\end{align*}
\item If $\bM^*  =\bM  \succ  \bzero$, then $(\bI_k \otimes \Tr)\bM \succ\bzero.$
The same statement holds if we replace both strict relations 
$(\succ)$ with non-strict ones $(\succeq)$.
\item If $\bM^* = \bM \succeq \bzero$, then
\begin{equation*}
    \lambda_{\min}\left( (\bI_k \otimes \Tr)\bM \right) \ge d \lambda_{\min}(\bM).
\end{equation*}
\item We have
\begin{equation*}
   \Im\left(
   (\bI_k \otimes \Tr) \bM
   \right) =  (\bI_k \otimes \Tr) \Im(\bM).
\end{equation*}
\end{enumerate}
\end{lemma}
The proof is deferred to Section~\ref{sec:proof_lemma_tensor_trace_properties}.

\subsubsection{Definitions, relevant norm bounds, and algebraic identities}
We give some definitions that will be used throughout this section. For $j\in [n]$, let
\begin{equation*}
    \bW_j := \grad^2 \ell(\bv_j, \bu_j, w_j) \in \R^{k\times k},
\quad\bxi_j := (\bI_k \otimes \bx_j)\in\R^{dk\times k}.
\end{equation*}
Let $\sfK := \sup_{\bv,\bv_0,w} \norm{\grad^2 \ell(\bv,\bv_0,w)}_{\op}$.
With this notation,  we can write 
   $\bH_0 =  \sum_{j=1}^n \bxi_j \bW_j \bxi_j^{\sT}.$
For $i\in[n],$ let $\bH_i := \sum_{j\neq i} \bxi_j \bW_j \bxi_j^\sT$,
and define the (normalized) resolvent and the leave-one-out resolvent as
\begin{equation*}
\bR(z) := \left(\bH_0 - z n \bI_{dk}\right)^{-1}\quad
\textrm{and}\quad
\bR_i(z) := \left(\bH_i  - z n \bI_{dk}\right)^{-1},
\end{equation*}
respectively.
We present the following algebraic identities that will be used in the leave-one-out approach we follow in Section~\ref{sec:approx_ST_FP}. The proof of these is deferred to Section~\ref{sec:proof_lemma_algebra_lemma}.
\begin{lemma}[Woodbury and algebraic identities]
\label{lemma:algebra_lemma}
    For all $i\in[n]$ and $z\in\bbH_+$, we have
    \begin{equation}
\label{eq:alg_id1}
        \left(\bI_k \otimes \Tr\right)\bxi_i \bW_i \bxi_i^\sT \bR(z)  =  \left( \bI_k + \bW_i \bxi_i^\sT \bR_i(z) \bxi_i\right)^{-1}
        \bW_i\bxi_i^\sT \bR_i(z)\bxi_i,
    \end{equation}
and
\begin{equation}
\label{eq:alg_id2}
    \left(\bI_k \otimes \Tr\right) \left(\bR_i(z)-  \bR(z)\right)  =  \bxi_i^\sT \bR_i(z) (\bW_i \otimes \bI_d) \bR(z)\bxi_i.
\end{equation}

\end{lemma}

 The next lemma summarizes some \emph{a priori} bounds on the matrices
involved in the upcoming proofs.
The proof is deferred to Section~\ref{sec:proof_lemma_as_norm_bounds}.
\begin{lemma}[Deterministic norm bounds]
\label{lemma:as_norm_bounds}
For all $i\in[n]$ and $z\in\bbH_+$, we have, 
\begin{align}
\label{eq:det_norm_bound_lemma_eq123}
     &\norm{\bR(z)}_\Fnorm^2 \vee \norm{\bR_i(z)}_\Fnorm^2 \le \frac{dk}{n^2} \frac{1}{\Im(z)^2},
     \quad
  \norm{\bR(z)}_\op \vee\norm{\bR_i(z)}_\op  \le  \frac{1}{n} \frac1{\Im(z)},\\
  \nonumber
  &\norm{\bSec}_\op \le \sfK,
   \quad
   \norm{\bH_0}_\op \le \sfK \norm{\bX}_\op^2.
\end{align}
Further, for $z \in\bbH_+$, we have
\begin{equation}
\label{eq:det_norm_bound_lemma_eq4}
    \|\Im((\bI_k\otimes\Tr)\bR(z))^{-1}\|_\op 
    \le   \frac{1}{\Im(z)} \left(\frac{1}{n}\norm{\bH_0}_\op  + |z|\right)^2
\end{equation}
and
\begin{equation}
\label{eq:det_norm_bound_lemma_eq5}
    \|\Im(\bxi_i^\sT\bR_i(z)\bxi_i)^{-1}\|_\op \le \frac{n}{\norm{\bx_i}_2^2} \frac{1}{\Im(z)} \left(\frac{1}{n}\norm{\bH_i}_\op  + |z|\right)^2.
\end{equation}
\end{lemma}

We also recall the following textbook fact 
for future reference.
\begin{lemma}[Operator norm bounds for Gaussian matrices
\cite{BaiSilverstein}]
\label{lemma:standard_norm_bounds}
 Let
 $$\Omega_0 := \{\norm{\bX}_\op \le 2 d^{1/2}(1 + \sqrt{\alpha_n})),\; \norm{\bx_i}_2 \in [d^{1/2}/{2}, 2d^{1/2}] \quad\textrm{for all}\quad i\in[n]\}.$$
Then 
\begin{equation}
    \P(\Omega_0^c ) \le C \exp\{- c d\}
\end{equation}
for some universal $C,c>0$.
\end{lemma}

\subsubsection{Concentration of tensor quadratic forms}

Finally, we end this section with the following consequence of the Hanson-Wright inequality for the concentration of quadratic forms of (sub)Gaussian random variables.
%
%
\begin{lemma}
\label{lemma:hanson-wright}
Let $\bx \sim \normal(\bzero,\bI_d), \bxi := (\bI_k \otimes \bx)$.
   Let $\bM \in\C^{dk\times {dk}}$ be independent of $\bx$, and 
   set $k_+(d):= k\vee \log d$. Then,
   for any $L\ge 1$, we have
\begin{equation*}
    \norm{\bxi^\sT \bM  \bxi -  (\bI_k \otimes \Tr)\bM}_\op \le
     C L \left( k_+(d)^{1/2} d^{1/2} \norm{\bM}_\op \vee k_+(d) \norm{\bM}_\op \right)
\end{equation*}
with probability at least
\begin{equation*}
    1 - 2\min\Big(e^{-cLk}, d^{-cL}\Big)\,
\end{equation*}
where $C,c > 0$ are universal constants.
\end{lemma}

\begin{proof}
Let $\cS$ be a minimal $1/4$-net of the unit ball in $\C^k$.
Then we have
\begin{equation*}
   \|\bxi^\sT\bM \bxi - (\bI_k \otimes\Tr) \bM\|_\op \le
2 \sup_{\bu,\bv\in\cS}  \bu^\sT\left(\bxi^\sT\bM \bxi - (\bI_k \otimes\Tr) \bM\right)\bv.
\end{equation*}
Meanwhile, for any fixed $\bu,\bv\in\cS$,
note that $(\bI_k\otimes\bx)\bu = (\bu \otimes\bI_d)\bx$ so that Hanson-Wright gives
    \begin{align*}
        \bu^\sT\bxi^\sT\bM\bxi\bv -  \bu^\sT(\bI_k \otimes \Tr)\bM\bv
        &=
        \bu^\sT\bxi^\sT\bM\bxi\bv -
\E_\bx\left[\bu^\sT\bxi^\sT\bM\bxi\bv \right]
        \\
        &= \bx^\sT(\bu \otimes\bI_d)^\sT \bM (\bv \otimes\bI_d) \bx
        - \Tr\left((\bu \otimes\bI_d)^\sT \bM (\bv \otimes\bI_d)\right) \\
        &\le s
    \end{align*}
with probability larger than 
\begin{equation*}
1 - 2\exp\left\{  -c_0\left(\frac{s^2}{\norm{\bM}_\op^2 d} \wedge 
\frac{s}{\norm{\bM}_\op}
\right)
\right\}
\end{equation*}
for a universal constant $c_0>0$,
where we used that $\norm{\bu\otimes\bI_d}_\op \le \norm{\bu}_2 \le1$ (and same for $\bv$) to deduce
\begin{equation*}
    \|(\bu \otimes \bI_d)^\sT \bM (\bv\otimes \bI_d)\|_F \le 
     \sqrt{d}
\norm{\bM}_\op, \quad
    \|(\bu \otimes \bI_d)^\sT \bM (\bv\otimes \bI_d)\|_\op \le 
\norm{\bM}_\op.
\end{equation*}
A standard result \cite{vershynin2018high} gives that the size of $\cS$ is at most $C_0^k$ for some $C_0 > 0$.
So taking
\begin{equation*}
    s = 
 \left( L^{1/2} k_+(d)^{1/2} d^{1/2} \norm{\bM}_\op \vee L k_+(d) \norm{\bM}_\op \right),
\end{equation*}
we obtain via a union bound
\begin{align*}
    \P\left(\|\bxi^\sT\bM\bxi - (\bI_k\otimes\Tr)\bM\|_\op \ge 2 s \right)
    &\le 
    \P\left( \sup_{\bv,\bu\in\cS} \bu^\sT(\bxi^\sT\bM\bxi - (\bI\otimes\Tr)\bM) \bv \ge s  \right)\\
    &\le 2C_0^k \big (e^{-c L k}\wedge d^{-cL}\big)
\end{align*}
for some constant $c>0$. The claim follows by taking $L$ a sufficiently large universal constant. 
Redefining the universal constants $c$ and $C_0$ allow us to take $L\ge 1$ as in the statement.

\end{proof}

\subsection{Approximate solution to the fixed point equations}
\label{sec:approx_ST_FP}
We begin with a concentration result.
\begin{lemma}[Concentration of the leave-one-out quadratic forms]
\label{lemma:concentration_loo_quad_form}
There exist absolute constant $c,C>0$, such that the following holds.
Let $k_+(d):= k\vee \log d$ and define the event (for $L\ge 1$)
\begin{equation}
\label{eq:concentration_loo_quad_form_2}
   \Omega_1(L) := \left\{\norm{\bxi_i^\sT \bR_i \bxi_i - (\bI_k \otimes \Tr)\bR}_\op
   \le  C L\sqrt{\frac{k_+(d)}{n \alpha_n}} \frac1{\Im(z)} 
   +  \frac{\sfK\norm{\bx_i}_2^2}{n^2 \Im(z)^2}\quad\textrm{for all}\quad i\in[n]\right\}.
\end{equation}
Then for  $n\ge \alpha_n k_+(d)$, $n\le d^{10}$ we have
for some universal constant $c>0$.
\begin{equation}
\nonumber
    \P(\Omega_1(L)^c) \le 2 (e^{- cLk}\vee d^{-cL}).
\end{equation}
\end{lemma}
\begin{proof}

For all $i\in[n]$, we have by Lemma~\ref{lemma:hanson-wright} and the bounds on the norms of $\bR_i$ in Lemma~\ref{lemma:as_norm_bounds}, 
\begin{equation}
\nonumber
    \norm{\bxi_i^\sT\bR_i \bxi_i - \left(\bI_k \otimes\Tr\right) \bR_i}_\op \le 
    C L\left(
    \sqrt{\frac{k_+(d)}{\alpha_n n}}
     \vee \frac{k_+(d)}{n} \right)
    \frac1{\Im(z)}
\end{equation}
with probability  at least $1-2 (e^{- cLk}\vee d^{-cL})$.
Meanwhile, we have by
Lemma~\ref{lemma:algebra_lemma} and the bound of Lemma~\ref{lemma:as_norm_bounds} once again that
    \begin{align}
    \nonumber
        \norm{(\bI_k \otimes \Tr)\left(\bR_i - \bR\right) }_\op &= 
\norm{\bxi_i^\sT \bR_i (\bW_i \otimes \bI_d) \bR\bxi_i}_\op\\
&\le \sfK \norm{\bx_i}_2^2 \norm{\bR}_\op\norm{\bR_i}_\op
\le \sfK \frac{\norm{\bx_i}^2}{n^2}\frac1{\Im(z)^2}.
    \end{align}
A triangle inequality followed by a union bound gives the result.
\end{proof}

We are ready to prove an approximate 
fixed point equation for $\bS = \bS_n$ defined in Eq.~\eqref{eq:Sn_def}.
\begin{lemma}[Fixed point equation for the Stieltjis transform]
\label{lemma:fix_point_rate}
Recall the definition of $\bF_z$ in~\eqref{eq:bF_def}.
Let $\Omega_0,\Omega_1(L)$ be the events of Lemmas~\ref{lemma:standard_norm_bounds} and~\ref{lemma:concentration_loo_quad_form} respectively.
For any $[\bbV,\bw] \in \R^{n\times (k+k_0+1)}$,
$z \in\bbH_+$, $L\ge 1$, we have on $\Omega_0 \cap \Omega_1(L)$,
with
\begin{equation}
 \bS_n(z) := \bS_n(z;\bbV) ,\quad\quad
 \hnu := \hnu_{[\bbV,\bw]},
\end{equation}
we have
\begin{align}
    \norm{
\frac1{\alpha_n}\bI_k - \bF_z(\bS_n(z);\hnu)^{-1}
\bS_n(z)
    }_\op
&\le   
\Err_{\FP}(z; n,k,L)
\end{align}
where, letting $k_+(d) = k\vee \log d$,
\begin{equation}
   \Err_{\FP}(z;n,k,K) := C(\sfK)\frac{(1 + |z|^4)}{\Im(z)^4}\alpha_n^{-1}(1 + \alpha_n^{-1})^2 \left(
   \frac{L\sqrt{k_+(d) \alpha_n}}{ \sqrt{n}}  + \frac{1}{  n \Im(z)} 
   \right)
\end{equation}

for some $C(\sfK)>0$ depending only on $\sfK$.
\end{lemma}

\begin{proof}
The proof proceeds via a leave-one-out argument as in the scalar case.
Namely, suppressing the argument $z \in\bbH_+$,
we write
\begin{align}
\label{eq:decomposition_resolvent_eq}
\frac{d}{n} \cdot \bI_k 
= \frac1n\left(\bI_k \otimes \Tr\right)
\bR^{-1} \bR
&= \left(\bI_k \otimes \Tr\right)
\left(\frac1n\sum_{i=1}^n \bxi_i \bW_i\bxi_i^\sT \bR 
 -z  \bR\right)\\
&= 
\left(\frac1n\sum_{i=1}^n 
\left(\bI_k \otimes \Tr\right)\left(
\bxi_i \bW_i\bxi_i^\sT \bR \right)
 - z  \bS_n  \right).
\end{align}

Letting $\bA_i = \bxi_i^\sT \bR_i\bxi_i$ and recalling the definition $\bS_n = (\bI_k \otimes \Tr)\bR$, we bound
\begin{align*}
        \bDelta_i&:=
        \left(\bI_k \otimes \Tr\right)\bxi_i \bW_i \bxi_i^\sT \bR - \bW_i\bS_n(\bI_k + \bW_i\bS_n)^{-1} \\
        &=\left( \bI_k + \bW_i\bA_i\right)^{-1}\bW_i\bA_i
        - (\bI_k + \bW_i\bS_n)^{-1}\bW_i\bS_n\\
        &=
        \left( \bI_k + \bW_i\bA_i\right)^{-1}\bW_i\bA_i
        - (\bI_k + \bW_i\bS_n)^{-1}\bW_i\bA_i
       +\left( \bI_k + \bW_i\bS_n\right)^{-1}\bW_i\bA_i\\
        &\quad- (\bI_k + \bW_i\bS_n)^{-1}\bW_i\bS_n\\
&= (\bI_k +\bW_i\bA_i)^{-1}\bW_i (\bS_n -\bA_i) (\bI + \bW_i \bS_n)^{-1} \bW_i \bA_i
+ (\bI_k + \bW_i\bS_n)^{-1}\bW_i (\bA_i -\bS_n).
\end{align*}
where the first equality follows from Lemma~\ref{lemma:algebra_lemma}.

Lemma~\ref{lemma:concentration_loo_quad_form} above provides a bound for $\norm{\bA_i -\bS_n}_\op$ on the event $\Omega_1(L)$. Meanwhile, by Lemma~\ref{lemma:tensor_trace_properties}, $\Im(\bS_n) = (\bI_k \otimes \Tr)\Im(\bR) \succ\bzero $ since $\Im(\bR) \succ\bzero$. So we have by Lemmas~\ref{lemma:re_im_properties} and~\ref{lemma:as_norm_bounds} on $\Omega_0$
\begin{equation}
\nonumber
    \norm{(\bI + \bW_i\bS_n)^{-1}\bW_i}_\op 
    \le \norm{\Im(\bS_n)^{-1}}_\op
    \le   \frac{1}{\Im(z)} \left( \frac{\sfK}{n} \norm{\bX}_\op^2 + |z|\right)^2 \le \frac{C_1}{\Im(z)} (\sfK^2(\alpha_n^{-1} + 1) + |z|^2)
\end{equation}
for some $C_1>0$,
and similarly and by the same lemmas, we conclude on $\Omega_0$
\begin{equation}
\nonumber
    \norm{(\bI + \bW_i\bA_i)^{-1}\bW_i}_\op 
    \le   \frac{1}{\Im(z)} \frac{n}{\norm{\bx_i}_2^2} \left(\frac{\sfK}{n} \norm{\bX}_\op^2 + |z|\right)^2 
    \le \frac{C_2 \alpha_n}{\Im(z)} (\sfK^2 (\alpha_n^{-1} + 1) + |z|^2),
\end{equation}
for some $C_2 >0$.
Finally, on $\Omega_0$, for some $C_3>0$,
\begin{equation}
\nonumber
   \norm{\bA_i}_\op {\le} \norm{\bx_i}_2^2 \norm{\bR_i}_\op \le \frac{C_3}{\alpha_n \Im(z)}.
\end{equation}

Combining these bounds along with the one in Lemma~\ref{lemma:concentration_loo_quad_form} gives on $\Omega_0 \cap\Omega_1(L)$,
\begin{align*}
    \norm{\bDelta_i}_F
  &\le\norm{(\bI_k + \bW_i \bS_n)^{-1}\bW_i}_\op \norm{\bA_i -\bS_n}_\op \left(\norm{(\bI_k + \bW_i \bA_i)^{-1}\bW_i}_\op\norm{\bA_i}_\op + 1\right)
   \\ 
   &\le 
   C_4(\sfK)\frac{1}{\Im(z)}(1 + \alpha_n^{-1})( 1 + |z|^2) \left(
   \frac{L\sqrt{k_+(d)}}{ \sqrt{n \alpha_n}} \frac1{\Im(z)} + \frac{\sfK}{ \alpha_n n \Im(z)^2} 
   \right)
   \left(\frac{1}{\Im(z)^2} (\sfK^2(\alpha_n^{-1} + 1) + |z|^2)  + 1\right)\\
   &\le
   C_5(\sfK)\frac{(1 + |z|^4)}{\Im(z)^4}\alpha_n^{-1}(1 + \alpha_n^{-1})^2 \left(
   \frac{L\sqrt{k_+(d) \alpha_n}}{ \sqrt{n}}  + \frac{1}{  n \Im(z)} 
   \right)\\
   &=: \Err_{\FP}(z;n,k,L)
\end{align*}
for some $C_4,C_5 >0$ depending only on $\sfK$.
%
Since this holds for all $i\in[n]$, using Eq.~\eqref{eq:decomposition_resolvent_eq} we conclude that
on $\Omega_0\cap\Omega_1$ that
\begin{align}
\nonumber
&\norm{\frac1\alpha_n\bI_k + z \bS_n - \frac1n \sum_{i=1}^n\bW_i \bS_n(\bI_k + \bW_i \bS_n)^{-1} }_\op
\le
 \frac1n\sum_{i=1}^n \norm{\bDelta_i}_\op\le 
\Err_{\FP}(z;n,k,L).
\end{align}
\end{proof}

\subsection{The asymptotic matrix-valued Stieltjis transform}
\label{sec:AsymptoticST}
In the last section, we showed that $\bS_n(z;\bbV)$ approximately solves the Stieltjes transform fixed point
equation in~\eqref{eq:fp_eq} with the measure $\nu =\hnu_{[\bbV,\bw]}$. We would like to show that this implies the convergence
of $\bS_n(z)$ to the solution $\bS_\star(z ;\nu)$ when $\hnu \Rightarrow \nu$. For this, we will first need to show the existence
and uniqueness of such a solution. In this section, we show its existence and study some of its properties which will facilitate proving
its uniqueness and the desired convergence.

\subsubsection{Free probability preliminaries}

We collect some relevant free probability background. Most of what follows can be found in
\cite{nica2006lectures,mingo2017free}.
Let $(\cA,\tau)$ be a $C^*$-probability space.
An element $M\in\cA$ is said to have a free Poisson distribution with rate $\alpha_0$ if the moments of $M$ under $\tau$ correspond to the moments of the Marchenko-Pastur law with aspect ratio $\alpha_0$. For $M,T \in\cA$, if $M$ is a free Poisson element and $T$ is self-adjoint, then $M^{1/2} T M^{1/2}$ has distribution given by the multiplicative free convolution of the free Poisson distribution and the distribution of $T$.

Given a distribution $\nu_{0}$ on $(\bv,\bv_0, w) \in\R^{k+k_0+1}$, 
consider a sequence (indexed by $m$)
of collections of deterministic real diagonal matrices $\left\{\bar\bSec_{a,b}^\up{m}\right\}_{a,b\in[k]}$ with
$\bar\bSec_{a,b}^\up{m}\in \R^{m\times m}$,
$\bar\bSec_{a,b}^\up{m} = \diag((\bar K_i)_{a,b}:\; i\le m)$
such that the following hold:
The entries of these matrices are uniformly bounded and
the convergence
\begin{align}
\nonumber
\frac{1}{m}\sum_{i=1}^m\delta_{\bar K_i} \Rightarrow 
\sP_{\nabla^2 \ell}\,
\end{align}
holds, where
$\sP_{\nabla^2 \ell}$ denotes the probability distribution
of $\nabla^2\ell(\bv,\bv_0, w)$ when $(\bv,\bv_0, w) \sim\nu_0$.
Equivalently, we have for any set of pairs 
$\cP= \{ (a_1,b_1),(a_2,b_2),\dots, (a_L,b_L)\}$, 
\begin{equation}
\label{eq:K_empirical_limit}
    \lim_{m\to\infty} \frac1m \Tr\left( \prod_{(i,j) \in \cP} 
   \bar \bSec_{i,j}^\up{m}
    \right) = \int \prod_{(i,j)\in\cP}  \partial_{i,j}\ell(\bv,\bv_0, w)  \, \, \de\nu_{0}(\bv,\bv_0,w).
\end{equation}
Now define $T_{i,j}\in(\cA,\tau)$ for each $i,j\in [k]$
so that for any $\cP$,
\begin{equation}
\label{eq:taus_of_prods_of_D}
     \tau\left( \prod_{(i,j) \in \cP} 
   T_{i,j}
    \right) = \int \prod_{(i,j)\in\cP}  \partial_{i,j}\ell(\bv,\bv_0,w) \,  \de\nu_{0}(\bv,\bv_0, w)
\end{equation}
and that $\{T_{i,j}\}_{i,j\in [k]}\cup \{M\}$ are freely independent.

Let $\{\bX^\up{m}\}_{m\ge 1}$ be a sequence of $m\times p$ matrices of i.i.d standard normal entries so that
$m/p \to \alpha_0 > 0$.
Then for any noncommutative polynomial $Q$ on $q$ variables, and any $(i_1,j_1),\dots (i_q,j_q) \in [k]\times [k]$,
\begin{equation}
\label{eq:moment_convergence}
    \lim_{m\to\infty}  \E\left[\frac1m\Tr\; Q\big( ( m^{-1}\bX^{\up{m}\sT} \bar \bSec_{i_l,j_l}^\up{m}\bX^\up{m})_{l\in[q]}\big)\right] =  \tau\left( Q\big( (M^{1/2} T_{i_l,j_l} M^{1/2})_{l\in[q]}\big) \right).
\end{equation}
 The quantity on the right hand side of the last equation is completely determined by the moments of the 
 form~\eqref{eq:taus_of_prods_of_D}, since $\{T_{i,j}\}_{i,j} \cup \{M\}$ are freely independent.

\subsubsection{Characterization of the asymptotic matrix-valued Stieltjis transform}

Let $\cM^{k\times k}(\cA)$ denote the set of $k\times k$ matrices with entries in $\cA$.
For $z\in\bbH_+$, $\nu_0\in\cuP(\R^{k+k_0+1})$, and $\alpha_0 >0$, define 
$\bH_{*}(\alpha_0,\nu_0),\bR_\star(z; \alpha_0, \nu_0) \in 
\cM^{k\times k}(\cA)$ via
\begin{align}
\label{eq:def_H_star}
\bH_\star(\alpha_0, \nu_0) &:= (M^{1/2}T_{i,j} M^{1/2})_{i,j\in[k]} \, ,\\
   \bR_\star(z; \alpha_0, \nu_0) &:= \left(
   \bH_\star(\alpha_0, \nu_0)- z (\bI_k \otimes \aid )\right)^{-1} ,
\end{align}
where, for $i,j\in[k]$, $T_{i,j}\in(\cA,\tau)$ satisfy Eq.~\eqref{eq:taus_of_prods_of_D} for $\nu_0$ and the aforementioned freeness; here $M \in(\cA,\tau)$ is
a free Poisson element with rate $\alpha_0$. 
Further, let
\begin{equation}
\label{eq:def_S_star}
    \bS_\star(z; \alpha_0, \nu_0) := \left(\bI_k \otimes \tau\right) \bR_\star(z; \alpha_0, \nu_0) \in\C^{k\times k}.
\end{equation}
The following lemma shows that $\bS_\star$ satisfies 
the fixed point equation for the Stieltjes transform, as one might expect. Note however the utility of this result: it allows one to decouple the asymptotics of the Gaussian matrix $\bX$ from that of the measure $\nu_0.$
\begin{lemma}[Asymptotic solution of the fixed point equation]
\label{lemma:asymp_ST}
Fix $\nu_0\in\cuP(\R^{k+k_0+1})$.
Assume $\|\grad^2 \rho(\bv,\bv_0,w)\|_{\op}\le \sfK$ with 
probability one under $\nu_0$.
For any fixed positive integer $k$, $z\in \bbH_+$,  
and $\alpha_0 >0$, we have
    \begin{equation}
        \alpha_0 \bS_\star(z; \alpha_0, \nu_0) = 
        \bF_z(\bS_\star(z; \alpha_0, \nu_0);\nu_0).
    \end{equation}
\end{lemma}
\begin{proof}
Fix $k$ throughout and let $\bT \in\cM^{k\times k}(\cA)$ be defined by $\bT := ( T_{i,j})_{i,j\in[k]}$,
where, for $i,j\in[k]$, $T_{i,j}\in(\cA,\tau)$ satisfy Eq.~\eqref{eq:taus_of_prods_of_D} for $\nu_0$. 

Our goal is to use Lemma~\ref{lemma:fix_point_rate} to show that $\bS_\star(z;\alpha_0,\nu_0)$ satisfies the desired fixed point equation. Since this lemma is stated in terms of empirical measures, we define 
$\{\hnu_{0,m}\}_m$,
$\hnu_{0,m}\in\cuP_m(\R^{k+k_0+1})$
to be a sequence of empirical measures satisfying $\hnu_{0,m}\Rightarrow \nu_0$.
These in turn define  a sequence $\bar \bSec^\up{m} = \left(\bar \bSec^\up{m}_{i,j}\right)_{i,j \in[k]}$
of deterministic matrices 
satisfying~\eqref{eq:K_empirical_limit} for the given $\nu_0$.

Write $\bS_\star(z)$ via its power expansion:
We have, by boundedness of $\bT$,
for $|z|$ sufficiently large
\begin{align}
\nonumber
\bS_\star(z) :=
\bS_\star(z; \alpha_0,\nu_0)
&= (\bI_k \otimes \tau)\left[ \frac1{z}\left( z^{-1}(\bI_k \otimes M^{1/2})\bT(\bI_k \otimes M^{1/2} )   - (\bI_k \otimes \id) \right)^{-1}\right]
\\
&= \sum_{a=0}^{\infty} (-z)^{-{(a+1)}}(\bI_k \otimes \tau)\left[\left( 
(\bI_k \otimes M^{1/2})\bT (\bI_k \otimes M^{1/2} )
\right)^a\right].\label{eq:ExpSstar}
\end{align}
Then, Eq.~\eqref{eq:moment_convergence} gives
\begin{align*}
\nonumber
   & (\bI_k \otimes \tau)\left[\left( 
(\bI_k \otimes M^{1/2})\bT (\bI_k \otimes M^{1/2} )
\right)^a\right] \\
&\quad= \lim_{m\to\infty} 
    (\bI_k \otimes \frac1m \Tr)\left[\left( m^{-1} 
(\bI_k \otimes \bX^{\up{m}\sT})\bar\bSec^\up{m} (\bI_k \otimes \bX^\up{m} )
\right)^a\right],
\end{align*}
where $\bX^\up{m}$ is $m\times p_m$ dimensional with $m/p_m\to\alpha_0$.
Hence, there exists $r_0$, possibly dependent on $k$, such that for $|z| > r_0$, 
\begin{align*}
    \bS_\star(z) &= \lim_{m\to\infty}
\sum_{a=0}^{\infty}
    (-z)^{-(a+1)}(\bI_k \otimes \frac1m \Tr)\left[\left( m^{-1} 
(\bI_k \otimes \bX^{\up{m}\sT})\bar\bSec^\up{m} (\bI_k \otimes \bX^\up{m} )
\right)^a\right]\\
&= \lim_{m\to\infty}\bS_m(z, \hnu_{0,m})
\end{align*}
element-wise.
It follows again from  Eq.~\eqref{eq:ExpSstar} that there exist constants
$C, r_1$ such that $\|\bS_\star\|_{\op}\le C/|z|$ for $|z|\ge r_1$.
Finally, for any $r_2>0$ there exists  $\eta=\eta(r_2)>0$, such that
$\bS\mapsto \bF_z(\bS;\nu_0)$ is continuous 
on $\|\bS\|_{\op}\le \eta$ if $|z|\ge r_2$. 
Eventually increasing $r_0$,
Lemma~\ref{lemma:fix_point_rate} implies that, for $|z|\ge r_0$
\begin{equation} 
\nonumber
    \bF_z(\bS_\star(z);\nu_0) = \lim_{m\to\infty}  \bF_z(\bS_m(z);\hnu_{0,m})
     =  \lim_{m\to\infty} \frac{m}{p_m}\bS_m(z; \hnu_{0,m}) = \alpha_0\bS_\star(z).
\end{equation}
Since $z\mapsto \bF_z(\bS_\star(z))$ and $z\mapsto \bS_\star(z)$ are both analytic on $\bbH_+$ 
we conclude the claim by analytic continuation.
\end{proof}

We now give a lower bound on the minimum singular value of $\bS_\star$. This will be necessary to establish that $\bS_\star$ is the unique solution of the Stieltjes transform fixed point equation.
\begin{lemma}
\label{lemma:smallest_singular_value_Sstar}
For any $z\in\bbH_+$, $\nu_0\in\cuP(\R^{k+k_0+1})$ and $\alpha_0 >0$, we have
\begin{equation}
\nonumber
    \norm{\Im(\bS_\star(z; \alpha_0, \nu_0))^{-1}}_\op \le \frac1{\Im(z)} \left(\sfK (1+ \alpha_0^{-1/2})^2 + |z|\right)^2.
\end{equation}
\end{lemma}

\begin{proof}
By the Gelfand-Naimark-Segal construction
\cite[Lecture 7]{nica2006lectures}, there exists a Hilbert space $\cH$ and a $*$-representation of $\cA$, $\pi :\cA \to\cB(\cH)$ and some $\psi_0 \in \cH$ with $\norm{\psi_0}_\cH = 1$ such that for any $A \in\cA$, 
\begin{equation}
    \tau(A) = \inner{\psi_0,\pi(A) \psi_0}_\cH.
\end{equation}
Let us identify $A$ with $\pi(A)$ in the notation below.
Recall the product space $\cH^{k} := \cH \times \cdots \times \cH$ with Hilbert inner product given by
   \begin{equation}
        \inner{
        (\xi_1,\cdots,\xi_k),(\bar\xi_1,\cdots,\bar\xi_k)}_{\cH^k} = \sum_{i=1}^k \inner{\xi_i, \bar \xi_i}_\cH
   \end{equation}
for $\xi_i,\bar \xi_i \in\cH$ for $i\in[k]$.
Now we can bound 
$\sigma_{\min}(\Im(\bS_\star))$ by the variational characterization 
\begin{align}
\sigma_{\min}(\Im(\bS_\star)) 
&= 
   \lambda_{\min}\left((\bI_k \otimes \tau)\Im(\bR_\star)\right)
   =\hspace{-2mm} \inf_{\substack{\bu\in\C^k\\\norm{\bu}_2 = 1}}  \bu^* 
   \left(\bI_k\otimes \psi_0\right)^* \Im(\bR_\star)(\bI_k \otimes \psi_0) \bu
  \nonumber 
   \\
&= \hspace{-2mm}\inf_{\substack{\bu\in\C^k\\\norm{\bu}_2 = 1}} \psi_0^* (\bu\otimes \aid)^* 
    \Im(\bR_\star)(\bu \otimes \aid) \psi_0
    \nonumber
   \\
&\ge  \norm{\psi_0}_\cH^2 \norm{\bu \otimes \aid}_{\cH^k}^2 
\lambda_{\min}\left(\Im(\bR_\star)\right) 
= 
\lambda_{\min}\left(\Im(\bR_\star)\right).
\nonumber
\end{align}
So by Lemma~\ref{lemma:re_im_properties}, 
\begin{align}
\sigma_{\min}(\Im(\bS_\star)) 
&= \Im(z)\lambda_{\min}\left(\bR_\star^* \bR_\star\right)
\ge  \Im(z) \left( \norm{(\bI_k \otimes M^{1/2}) \bar \bSec (\bI_k \otimes M^{1/2})}_{\cB(\cH^k)} + |z| \right)^{-2}
\label{eq:lb_on_Sstar}
\end{align}
To bound the norm in the term above, 
note that for any nonnegative integer $p$, we have
   \begin{align*}
   \nonumber
       \left(\frac1k \Tr \otimes \tau\right)\left(|\bar \bSec|^{2p}\right) 
&\stackrel{(a)}{=}
       \left(\frac1k \Tr \otimes \tau\right)\left(\bar \bSec^{2p}\right) \\
&= 
\lim_{n\to\infty}\left(\frac1k \Tr \otimes \frac1n\Tr\right)\left((\bar \bSec^\up{n})^{2p}\right) \stackrel{(b)}{\le} \lim_{n\to\infty}\norm{\bar \bSec^\up{n}}_\op^{2p} \stackrel{(c)}{\le} \sfK^{2p}
   \end{align*}
where $(a)$ follows from self-adjointness, $(b)$ follows from monotonicity of $L^p$ norms and $(c)$ follows from Lemma~\ref{lemma:as_norm_bounds}. Now taking both sides to the power $1/(2p)$ and sending $p\to\infty$ gives the bound
\begin{equation}
\nonumber
    \norm{\bar \bSec}_{\cB(\cH^k)} \le \sfK.
\end{equation}
Meanwhile, by definition of the free Poisson element $M$, we have
\begin{equation}
\nonumber
\norm{\bI_k \otimes M^{1/2}}_{\cH^k}^2 \le (1+ \alpha_0^{-1/2})^2.
\end{equation}
Combining the previous two displays with Eq.~\eqref{eq:lb_on_Sstar} gives the claim.
\end{proof}
As a summary of this section, we have the following corollary.
\begin{corollary}
\label{cor:S_star_min_singular_value_bound}
For any $z\in\bbH_+,\nu\in\cuP(\R^{k+k_0+1}),$ any positive integer $k$ and any $\alpha_0 > 0$, the fixed point equation 
\begin{equation}
    \alpha_0 \bS = \bF_z(\bS; \nu)
\end{equation}
has a solution $\bS_\star(z; \alpha_0, \nu) \in \bbH_+^k$ satisfying
\begin{equation}
    \norm{\Im(\bS_\star(z; \alpha_0, \nu))^{-1}}_\op \le \frac1{\Im(z)} \left(\sfK (1+ \alpha_0^{-1/2})^2 + |z|\right)^2.
\end{equation}
Furthermore, $\mu_{\MP,\alpha_0}(\nu)$, the probability distribution whose Stieltjes transform is given by $z\mapsto k^{-1}\Tr(\bS_\star(z;\alpha_0,\nu))$, is compactly supported uniformly in $\nu$ and $\alpha_0$ in any compact interval of $(1,\infty).$

\end{corollary}
\begin{proof}
  The first part of the statement follows directly from Lemma~\ref{lemma:smallest_singular_value_Sstar}.
  That $\mu_{\MP,\alpha_0}(\nu)$ is compactly supported follows from the definition of $\mu_\star$ as the measure whose Stieltjes transform is $\bS_\star$ of Eq.~\eqref{eq:def_S_star} and the definition in Eq.~\eqref{eq:def_H_star} and uniform bounds on the operator norm of $\bH_\star.$
\end{proof}

\subsection{Uniqueness of \texorpdfstring{${\boldsymbol S}_\star$}{S*} and convergence of \texorpdfstring{${\boldsymbol S}_n$}{Sn} to \texorpdfstring{${\boldsymbol S}_\star$}{S*}}
\label{sec:UniquenessSstar}

The goal of this section is to show that the asymptotic Stieltjes transform defined 
in Section \ref{sec:AsymptoticST} is the unique solution of the fixed point equation of Eq.~\eqref{eq:fp_eq}, and to derive a bound on the difference of this quantity and the empirical Stieltjes transform.
Our approach is to study a certain linear operator whose invertibility implies uniqueness.
To define this operator, first introduce the notation
\begin{equation}
        \boldeta(\bS, \bW) := (\bI + \bW\bS)^{-1}\bW.
\end{equation}
For a given $\bS\in\C^{k\times k}$ with $\Im(\bS)\succeq \bzero$, $z\in\bbH_+$, $\alpha >0, \nu\in\cuP(\R^{k+k_0+1})$,
define $\bT_\bS(\;\cdot\;; z,\alpha, \nu):\C^{k\times k}\to \C^{k \times k}$ 
\begin{equation}
\label{eq:def_T}
    \bT_\bS(\bDelta;
    z,\alpha,\nu)
    := \bF_z(\bS_\star;\nu) \E[\boldeta(\bS_\star,\bW) \bDelta  \boldeta(\bS,\bW) ] \bF_z(\bS;\nu),
\end{equation}
where
$\bS_\star = \bS_\star(z; \alpha,\nu)$.
Now, the significance of this operator is highlighted by the following relation: 
letting $\bS_\star$ be as defined above and suppressing the dependence on $z,\alpha,\nu$,
we have for any $\bS \in \bbH_+^k$
\begin{align*}
   \bF_z(\bS_\star) - \bF_z(\bS) &=
        \left(\E[\bfeta(\bS_\star,\bW)] - z\bI\right)^{-1}
        -
        \left(\E[\bfeta(\bS, \bW)] - z\bI\right)^{-1}\\
        &=
        -\left(\E[\bfeta(\bS_\star,\bW)] - z\bI\right)^{-1}
      \E[\bfeta(\bS_\star, \bW) - 
      \bfeta(\bS,\bW)] 
        \left(\E[\bfeta(\bS, \bW)] - z\bI\right)^{-1}\\
        &=
        \bF_z(\bS_\star) 
      \E[\bfeta(\bS_\star, \bW)(\bS_\star - \bS) 
      \bfeta(\bS,\bW))] 
     \bF_z(\bS) \\
   &=\bT_{\bS}(\bS_\star - \bS).
\end{align*}
Summarizing, we have
\begin{align}
\label{eq:relation_F_T}
   \bF_z(\bS_\star) - \bF_z(\bS) =\bT_{\bS}(\bS_\star - \bS).
\end{align}

To prove uniqueness of $\bS_\star$, we'll consider $\bS_0$ to be any solution of $\alpha \bS = \bF_z(\bS)$ in $\bbH_+^k$ and show that $\bT_0 := \bT_{\bS_0}$ has a convergent Neumann series. 
Similarly we will derive the rate of convergence of $\bS_n$ to $\bS_\star$ by bounding the norm of $(\id - \bT_{\bS_n})^{-1}$.

Our first lemma of this section gives a deterministic bound on $\norm{\bT_\bS^p}_{\op\to\op}$ for all positive integer $p$, which will later allow us to assert the convergence of the Neumann series of $\bT_\bS$.
Here, $\bT^p$ is the $p$-fold composition of $\bT$,
namely $\bT^p(\bA) = \bT(\bT^{p-1}(\bA))$.

\begin{lemma}
\label{lemma:op_norm_bound_power_T}
Fix $\bS\in\C^{k\times k}$ with $\Im(\bS)\succ \bzero,$ $z\in\bbH_+ ,\alpha >0$, and $\nu\in\cuP(\R^{k+k_0+1})$ such that $\bT_\bS(\;\cdot\;, z,\alpha,\nu)$ of Eq.~\eqref{eq:def_T} is defined.
We have for any
$\bB,\bB_\star \succ \bzero$,
and integer $p>0$, we have
\begin{align}
    \norm{\bT_\bS^p}_{\op \to\op}&\le 
    \left(
 \norm{\E[\bB^{-1/2} \bF\bfeta \bB\bfeta^* \bF^* \bB^{-1/2}]}_\op^p
\norm{\bB^{-1}}_\op
\norm{\bB}_\op
\right)^{1/2}
\nonumber
\\
&\hspace{2cm}\left(
 \norm{\E[\bB_\star^{-1/2} \bF_\star\bfeta_\star \bB_\star\bfeta_\star^* \bF_\star^* \bB_\star^{-1/2}]}_\op^{p}
\norm{\bB_\star^{-1}}_\op
\norm{\bB_\star}_\op
\right)^{1/2},
\nonumber
\end{align}
where $\bS_\star$ is as defined in Eq.~\eqref{eq:def_T} and 
\begin{equation}
\bF_\star := \bF_z(\bS_\star;\nu),\quad \bF:=\bF_z(\bS;\nu),\quad \bfeta_\star  :=\bfeta(\bS_\star, \bW),\quad \bfeta := \bfeta(\bS,\bW).
\end{equation}

\end{lemma}

\begin{proof}
In what follows, 
let $\bW_1,\dots,\bW_p, \widetilde \bW_1,\dots,\widetilde\bW_p$ be i.i.d. copies of $\bW$.
We use the shorthand $\bfeta_{\star,i} \equiv \bfeta(\bS_\star;\bG_i)$ and 
$\widetilde\bfeta_{\star,i} \equiv \bfeta(\bS_\star;\widetilde\bG_i)$. Similarly define $\bfeta_i, \widetilde\bfeta_i$ for $\bS$ replacing $\bS_\star$.
Fix any $\bv,\bu\in \C^{k}$ and $\bDelta \in \C^{k\times k}$ and write
\begin{align*}
\left|\bv^* \bT^p(\bDelta) \bu\right|^2
&=  \bu^* \bT^p(\bDelta)^* \bv \bv^*  \bT^p(\bDelta) \bu\\
&=  \Tr\left(
\bu^*\E\left[
\prod_{i=p}^1 (\bF \bfeta_i)\bDelta \prod_{i=1}^p(\bfeta_{\star,i} \bF_\star)
\bv\bv^*
\prod_{i=p}^1 ( \bF_\star^*\widetilde \bfeta_{\star,i}^*)
\bDelta^*
\prod_{i=1}^p ( \widetilde \bfeta_{i}^*\bF^*)
\right]\bu
\right)
\\
&= 
\E\left[
\Tr\left(
\bDelta \prod_{i=1}^p(\bfeta_{\star,i} \bF_\star)
\bv\bv^*
\prod_{i=p}^1 ( \bF_\star^*\widetilde \bfeta_{\star,i}^* )
\bDelta^*
\prod_{i=1}^p ( \widetilde \bfeta_{i}^*\bF^*)
\bu\bu^*
\prod_{i=p}^1 (\bF \bfeta_i)
\right)
\right]
\\
&\le 
\E\left[
\Tr\left(
\left|
\bDelta \prod_{i=1}^p(\bfeta_{\star,i} \bF_\star)
\bv\bv^*
\prod_{i=p}^1 ( \bF_\star^*\widetilde \bfeta_{\star,i}^* )
\right|^2\right) \right]^{1/2}\\
&\hspace{10mm}
\E\left[\Tr\left(
\left|
\bDelta^*
\prod_{i=1}^p ( \widetilde \bfeta_{i}^*\bF^*)\bu\bu^*
\prod_{i=p}^1 (\bF \bfeta_i)
\right|^2
\right)
\right]^{1/2}
\end{align*}
where the inequality follows by Cauchy-Schwarz for (random) 
matrices.
We bound the second expectation as 
\begin{align*}
&\E\left[\Tr\left(
\left|
\bDelta^*
\prod_{i=1}^p ( \widetilde \bfeta_{i}^*\bF^*)
\bu\bu^*
\prod_{i=p}^1 (\bF \bfeta_i)
\right|^2
\right)
\right]\\
&=
\E\left[\Tr\left(
\bDelta^*
\prod_{i=1}^p ( \widetilde \bfeta_{i}^*\bF^*)
\bu\bu^*
\prod_{i=p}^1 (\bF \bfeta_i)
\prod_{i=1}^p (\bfeta_i^*\bF^* )
\bu\bu^*
\prod_{i=p}^1 (\bF \widetilde \bfeta_{i})
\bDelta
\right)
\right]\\
&=
\E\left[
\bu^*
\prod_{i=p}^1 (\bF \widetilde \bfeta_{i})
\bDelta
\bDelta^*
\prod_{i=1}^p ( \widetilde \bfeta_{i}^*\bF^*)
\bu
\right]
\E\left[
\bu^*
\prod_{i=p}^1 (\bF \bfeta_i)
\prod_{i=1}^p (\bfeta_i^*\bF^* )
\bu
\right]
\\
&\le
\E\left[
\bu^*
\prod_{i=p}^1 (\bF \widetilde \bfeta_{i})
\bDelta
\bDelta^*
\prod_{i=1}^p ( \widetilde \bfeta_{i}^*\bF^*)
\bu
\right]
\norm{\bB^{-1}}_\op \norm{\E[\bB^{-1/2} \bF\bfeta \bB\bfeta^* \bF^* \bB^{-1/2}]}_\op^p
\norm{\bB}_\op \norm{\bu}_2^2\\
&\le
\norm{\bDelta}_\op^2
 \norm{\E[\bB^{-1/2} \bF\bfeta \bB\bfeta^* \bF^* \bB^{-1/2}]}_\op^{2p}
\norm{\bB^{-1}}_\op^2
\norm{\bB}_\op^2 \norm{\bu}_2^4,
\end{align*}
where the last two inequalities can be proven by induction over $p$.
A similar computation gives the bound on the first expectation 
\begin{align*}
    &\E\left[
\Tr\left(
\left|
\bDelta \prod_{i=1}^p(\bfeta_{\star,i} \bF_\star)
\bv\bv^*
\prod_{i=p}^1 ( \bF_\star^*\widetilde \bfeta_{\star,i}^* )
\right|^2\right) \right]\\
&\quad\le
\norm{\bDelta}_\op^2
 \norm{\E[\bB_\star^{-1/2} \bF_\star\bfeta_\star \bB_\star\bfeta_\star^* \bF_\star^* \bB_\star^{-1/2}]}_\op^{2p}
\norm{\bB_\star^{-1}}_\op^2
\norm{\bB_\star}_\op^2 \norm{\bv}_2^4.
\end{align*}
Taking supremum over $\bv,\bu$ of unit norm gives that
\begin{align*}
    \frac{\norm{\bT^p(\bDelta)}_\op}{\norm{\bDelta}_\op} &\le
    \Bigg(
 \norm{\E[\bB_\star^{-1/2} \bF_\star\bfeta_\star \bB_\star\bfeta_\star^* \bF_\star^* \bB_\star^{-1/2}]}_\op^p
\norm{\bB_\star^{-1}}_\op
\norm{\bB_\star}_\op\\
&\hspace{20mm}\times\norm{\E[\bB^{-1/2} \bF\bfeta \bB\bfeta^* \bF^* \bB^{-1/2}]}_\op^{p}
\norm{\bB^{-1}}_\op
\norm{\bB}_\op
\Bigg)^{1/2}
\end{align*}
for all $\bDelta$. Taking supremum over $\norm{\bDelta}_\op = 1$ gives the result.
\end{proof}

\begin{lemma}
\label{lemma:op_norm_bound_for_sols_fp}
Fix $z \in\bbH_+$, 
$[\bbV,\bw]\in\R^{n\times (k+k_0 + 1)}$ and let
$\hnu := \hnu_{[\bbV,\bw]} \in \cuP_n(\R^{k+k_0+1})$.
Let $\bS_0$ be any solution of $\alpha_n \bS = \bF_z(\bS;\hnu)$ in    $\bbH_+^k$, and let $\bS_n(z;\bbV)$ be the quantity defined in Eq.~\eqref{eq:Sn_def}.
Use the notation
\begin{align}
\bF_0 &:= \bF_z(\bS_0;\hnu),\quad \bF_n:=\bF_z(\bS_n;\hnu),\quad \bfeta_0  :=\bfeta(\bS_0, \bW),\quad \bfeta_n := \bfeta(\bS_n,\bW),\\
\bB_n &:= \Im(\bS_n),\quad\bB_0:=\Im(\bS_0),
\end{align}
where $\bW\sim \grad^2\ell_{\# \hnu}$.
Then we have the bound
    \begin{equation}
    \nonumber
\frac1\alpha_n\norm{\E\left[\bB_0^{-1/2} \bF_0 \bfeta_0 \bB_0\bfeta_0^* \bF_0^*\bB_0^{-1/2}\right]}_\op \le 
1 -\frac{\alpha_n \Im(z)}{2} \lambda_{\min}(\bB_0^{-1/2} \bS_0 \bS_0^* \bB_0^{-1/2}).
    \end{equation}
    
Further, if 
\begin{equation}
\label{eq:n_n(z)}
     \frac{10(\sfK^2 + |z|^2)}{\Im(z)^2}(1+\alpha_n^{-1})^2\Err_{\FP}(z;n,k,L)(1 + \alpha_n\Err_{\FP}(z;n,k,L))  \le  \frac12
 \end{equation}
then the following holds,
for any $L\ge 1$,
on the event $\Omega_0\cap\Omega_1(L)$  of 
Lemmas \ref{lemma:standard_norm_bounds},
\ref{lemma:concentration_loo_quad_form}
    \begin{equation}
\frac1\alpha_n\norm{\E\left[\bB_n^{-1/2} \bF_n \bfeta_n \bB_n\bfeta_n^* \bF_n^*\bB_n^{-1/2}\right]}_\op 
    \le 1 -  \frac{\Im(z)}{2\alpha_n}\lambda_{\min} \left(\bB_n^{-1/2} \bF_n\bF_n^* \bB_n^{-1/2}\right)\, .
    \label{eq:SecondBoundFp}
    \end{equation}
    \end{lemma}

\begin{proof}
By Lemma~\ref{lemma:re_im_properties}, we have
$\Im(\bfeta_0) = -\bfeta_0\bB_0 \bfeta_0^*$, and $\Im(\bS_0^{-1}) = - \bS_0^{-1} \bB_0\bS_{0}^{*-1}.$
So rewriting the fixed point for $\bS_0$ as
    $z\bI = \E[\bfeta_0] - \alpha_n^{-1} \bS_0^{-1}$ and taking the imaginary parts
gives
\begin{equation}
\nonumber
\bzero\prec \Im(z) \bI = - \E[\bfeta_0 \bB_0 \bfeta_0^*] + \frac1\alpha_n\bS_0^{-1}\bB_0\bS_0^{*-1}.
\end{equation}
This implies
\begin{equation}
\nonumber
    \alpha_n \bB_0^{-1/2} \bS_0\E[\bfeta_0 \bB_0\bfeta_0^*] \bS_0^* \bB^{-1/2} \prec 
    \bI - \frac{\alpha_n\Im(z)}{2} \bB_0^{-1/2}\bS_0\bS^*_0 \bB_0^{-1/2}
\end{equation}
giving the first bound after substituiting $\bF_0 = \alpha_n 
\bS_0$.

For the bound \eqref{eq:SecondBoundFp}, once again let us write by definition of $\bF_n$,
   $z\bI = \E[\bfeta_n] - \bF_n^{-1}$
which gives after multiplying to the left by $\bF_n$ and to the right by $\bF_n^*$
\begin{equation}
\nonumber
    z \bF_n \bF_n^* =  \E[\bF_n \bfeta_n \bF_n^*] - \bF_n^*.
\end{equation}
Taking the imaginary part using Lemma~\ref{lemma:re_im_properties} then gives
\begin{equation}
\nonumber
    \Im(z) \bF_n \bF_n^* = - \E[\bF_n \bfeta_n\bB_n \bfeta_n^* \bF_n^*]  + \Im ( \bF_n)
\end{equation}
so that
\begin{equation}\label{eq:Fn-relation}
    \E[\bF_n \bfeta_n \bB_n \bfeta_n^* \bF_n^*] = \Im(\bF_n - \alpha_n\bS_n)  + \alpha_n\bB_n - \Im(z) \bF_n\bF_n^*.
\end{equation}
Letting $\bE_n := (\alpha_n^{-1}\bI  -\bF_n^{-1} \bS_n)$,
\begin{align*}
    \|&\Im(  \bF_n^{-1} (\bF_n - \alpha_n\bS_n) \bF_{n}^{*-1})\|_\op
    =
   \alpha_n \norm{\Im(   \bE_n \bF_{n}^{*-1})}_\op
   \le  
    \alpha_n\norm{\bE_n}_\op\norm{\bF_n^{-1}}_\op\\
    &
    \le \norm{\bE_n}_\op \norm{\bS_n^{-1}}_\op \norm{\bI - \alpha_n\bE_n}_\op\\
    &\stackrel{(a)}\le \frac1{\Im(z)} \left(\frac{1}{n^2}\norm{\bH_0}_\op^2  + |z|^2\right) \Err_{\FP}(z;n,k,L) (1 + \alpha_n\Err_{\FP}(z; n, k,L))\\
    &\stackrel{(b)}\le \frac{10}{\Im(z)} (1 + \alpha_n^{-1})^2\left(\sfK^2  + |z|^2\right) \Err_{\FP}(z;n,k,L) (1 + \alpha_n\Err_{\FP}(z; n, k,L))\\
    &\stackrel{(c)}\le \frac{\Im(z)}{2}
\end{align*}
on $\Omega_0 \cap\Omega_1(L)$,
where $(a)$ follows from Lemma~\ref{lemma:fix_point_rate}
and Lemma \ref{lemma:as_norm_bounds}, $(b)$ follows from Lemma~\ref{lemma:standard_norm_bounds}, and $(c)$ follows from the assumption in Eq.~\eqref{eq:n_n(z)}.
We conclude that
\begin{equation}
\nonumber
    \Im(\bF_n - \alpha_n\bS_n) - \frac{\Im(z)}{2}\bF_n\bF_n^* \preceq \bzero\, ,
\end{equation}
and therefore, using Eq.~\eqref{eq:Fn-relation},
and the fact that $\Im(z)>0$,
\begin{equation}
\nonumber
    \frac1{\alpha_n}\norm{\E[\bB_n^{-1/2}\bF_n \bfeta_n \bB_n \bfeta_n^* \bF_n^* \bB_n^{-1/2}]  }_\op
    \le 1 -  \frac{\Im(z)}{2\alpha_n}\sigma_{\min} \left(\bB_n^{-1/2} \bF_n\bF_n^* \bB_n^{-1/2}\right)
\end{equation}
as desired.
\end{proof}

\subsubsection{Uniqueness of \texorpdfstring{${\boldsymbol S}_\star$}{S*}}
\label{app:sec:UniquenessSstar}
We are now ready to prove uniqueness.
\begin{lemma}
\label{lemma:uniqueness_ST}
    For any $z\in\bbH_+$, $\alpha_0 >0$, $\nu\in\cuP(\R^{k+k_0+1})$,  $\bS_\star(z;\alpha_0,\nu)$ defined in Eq.~\eqref{eq:def_S_star} is the unique solution to $\alpha_0 \bS = \bF_z(\bS;\nu)$ on $\bbH_+^k$.
\end{lemma}

\begin{proof}
Recall that Lemma~\ref{lemma:asymp_ST} asserts that this is indeed a solution. We prove there that there is no other solution.

Let $\bS_0 \in\bbH_+^k$ be any solution to this fixed point equation.  Then by Eq.~\eqref{eq:relation_F_T}
    \begin{equation}
\label{eq:diff_two_sols}
0 = \bS_0 - \bS_\star - \frac1{\alpha_0} (\bF_z(\bS_0) -\bF_z(\bS_\star)) = \left(\id - \frac1{\alpha_0} \bT_{\bS_0}\right)\left(\bS_0 -\bS_\star\right).
    \end{equation}
So to conclude uniqueness, it's sufficient to show that $\left(\id - \alpha_0^{-1} \bT_{\bS_0}\right)$ is invertible.
Using Lemmas~\ref{lemma:op_norm_bound_power_T} and~\ref{lemma:op_norm_bound_for_sols_fp}, let
\begin{equation}
\nonumber
    \delta_0 := \frac{\alpha_0\Im(z)}{2} \lambda_{\min}(\bB_0^{-1/2}\bS_0 \bS_0^* \bB_0^{-1/2}),
    \quad
    \delta_\star := \frac{\alpha_0\Im(z)}{2} \lambda_{\min}(\bB_\star^{-1/2}\bS_\star \bS_\star^* \bB_\star^{-1/2}).
\end{equation}
Since $\bS_0,\bS_\star \in\bbH_+^k$, we have
by Lemma~\ref{lemma:re_im_properties} that $\delta_0,\delta_\star > 0$.
Hence by Lemmas~\ref{lemma:op_norm_bound_power_T} and~\ref{lemma:op_norm_bound_for_sols_fp} 
\begin{align*}
    \norm{\sum_{p}\alpha_0^{-p} \bT_{\bS_0}^p}_{\op\to\op}  &\le \sum_{p} (1-\delta_0)^{p/2}(1-\delta_\star)^{p/2} \left(\norm{\bB_0}_\op \norm{\bB_0^{-1}}_\op  \norm{\bB_\star}_\op \norm{\bB_\star^{-1}}_\op\right)^{1/2} \\&\le 
    \left( \frac1{\delta_0 \delta_\star}\norm{\bB_0}_\op \norm{\bB_0^{-1}}_\op  \norm{\bB_\star}_\op \norm{\bB_\star^{-1}}_\op\right)^{1/2}  
\end{align*}
implying convergence of the Neumann series, and in turn, the desired invertibility.
\end{proof}

\subsubsection{Rate of convergence}

\begin{lemma}
\label{lemma:rate_matrix_ST}
Fix $[\bbV,\bw]\in\R^{n \times (k+k_0+1)}$, and let $\hnu = \hnu_{[\bbV,\bw]}$.
Whenever
\begin{equation}
\label{eq:n_n_0_2}
     \frac{10(\sfK^2 + |z|^2)}{\Im(z)^2}(1+\alpha_n^{-1})^2\Err_{\FP}(z;n,k,L)(1 + \alpha_n\Err_{\FP}(z;n,k,L))  \le  \frac1{2\alpha_n},
 \end{equation}
we have 
\begin{equation}
\nonumber
    \sup_{\hnu \in\cuP_n(\R^{k+k_0+1})}\norm{\bS_\star(z;\alpha_n, \hnu) -\bS_n(z;\bbV)}_\op \le C(\sfK)
    \frac{1 + |z|^4}{\Im(z)^5} 
\Err_{\FP}(z; n, k,L)
\end{equation}
on the event $\Omega_0 \cap\Omega_1(L)$ of Lemmas \ref{lemma:standard_norm_bounds},
\ref{lemma:concentration_loo_quad_form}, for $L\ge 1$.
\end{lemma}
\begin{proof}
Recalling (with $\bF_{\star} := \bF_z(\bS_{\star})$,
$\bF_{n} := \bF_z(\bS_{n})$) that $\alpha_n^{-1}\bF_\star = \bS_\star$ and 
$\alpha_n^{-1}\bF_n = \bS_n +\bF_n\bE_n$ where $\bE_n := \alpha_n^{-1}\bI - \bF_n^{-1}\bS_n$,
we have by Eq.~\eqref{eq:relation_F_T},
\begin{equation}
    \bS_\star - \bS_n =
    \alpha_n^{-1}(\bF_\star - \bF_n)  + \bF_n\bE_n =
    \frac1\alpha_n \bT_{\bS_n}(\bS_\star -\bS_n) +  \bF_n\bE_n.
\end{equation}
Letting
\begin{equation}
\nonumber
    \delta_\star := \frac{\alpha_n\Im(z)}{2} \lambda_{\min}(\bB_\star^{-1/2} \bS_\star\bS_\star^* \bB_\star^{-1/2}),\quad
    \delta_n := \frac{\Im(z)}{2\alpha_n} \lambda_{\min}(\bB_n^{-1/2} \bF_n\bF_n^* \bB_n^{-1/2}),
\end{equation}
an argument similar to that of Lemma~\ref{lemma:uniqueness_ST} (making use of Lemmas~\ref{lemma:op_norm_bound_power_T} and~\ref{lemma:op_norm_bound_for_sols_fp}) 
implies that $(\id- \alpha_n^{-1}\bT_{\bS_n})$ is invertible and so
\begin{align}
\nonumber
    \|\bS_* &- \bS_n\|_\op 
    = \norm{\left(\bI - \alpha_n^{-1}\bT_{\bS_n}\right)^{-1} \bF_n \bE_n}\\
\nonumber
    &\le \sum_{p=0}^\infty \alpha_n^{-p}\norm{\bT_{\bS_n}^p}_{\op\to\op} \norm{\bF_n\bE_n}_\op\\
\nonumber
    &\le \sum_{p=0}^\infty (1-\delta)^{p/2}(1-\delta_n)^{p/2} 
    \left(\norm{\bB_n^{-1}}_\op \norm{\bB_n}_\op 
    \norm{\bB}_\op \norm{\bB^{-1}}_\op\right)^{1/2} \norm{\bF_n}_\op\norm{\bE_n}_\op \\
    &\le \left(
    \frac{1}{\delta}
    \frac{1}{\delta_n}
    \norm{\bB_n^{-1}}_\op \norm{\bB_n}_\op 
    \norm{\bB}_\op \norm{\bB^{-1}}_\op
    \right)^{1/2}
    \norm{\bF_n}_\op \norm{\bE_n}_\op.
    \label{eq:S_Sn_rate_expansion}
\end{align}

Now we collect the bounds appearing on the right-hand side of this equation. On the event $\Omega_0$ of Lemma~\ref{lemma:standard_norm_bounds}, we have Lemma~\ref{lemma:as_norm_bounds} and  Corollary~\ref{cor:S_star_min_singular_value_bound}
(recall that $\bB_n=\Im(\bS_n)$, $\bB_{\star}=\Im(\bS_{\star})$): 
\begin{equation}
\nonumber
    \norm{\bB_n^{-1}}_\op \le  \frac{C_1}{\Im(z)}(1+\alpha_n^{-1})^2 \left( \sfK^2 + |z|^2 \right),\quad\textrm{and}\quad
    \norm{\bB_\star^{-1}}_\op \le \frac{C_2}{\Im(z)}(1 + \alpha_n^{-1})^2\left(\sfK^2 + |z|^2\right),
\end{equation}
respectively.
Furthermore, by Lemma~\ref{lemma:re_im_properties}, then Lemma~\ref{lemma:as_norm_bounds} and Corollary~\ref{cor:S_star_min_singular_value_bound} respectively, we have the bounds
\begin{equation}
\nonumber
   \norm{\bB_n}_\op \le \frac1{\alpha_n\Im(z)},\quad  \norm{\bB_\star}_\op \le\frac1{\alpha_n\Im(z)}.
\end{equation}
Meanwhile, to bound the norm of $\bF_n$, we  can observe that 
    $\bF_n = \alpha_n (\bF_n \bE_n + \bS_n)$.
Since the assumption guarantees that
\begin{equation}
\label{eq:bound_En}
\norm{\bE_n}_\op \equiv \Err_{\FP} \le \frac1{2\alpha_n}\frac1{(1+\alpha_n^{-1})^2}  \frac1{1+\alpha_n \Err_{\FP}} \frac{|z|^2}{10(\sfK^2 + |z|^2)} \le \frac1{2\alpha_n},
\end{equation}
we conclude that 
\begin{equation}
\nonumber
    \norm{\bF_n}_\op \le 2\alpha_n \norm{\bS_n}_\op \stackrel{(a)}{\le} \frac{2\alpha_n}{\alpha_n \Im(z)} = \frac{2}{\Im(z)},
\end{equation}
where $(a)$ follows from Lemma~\ref{lemma:as_norm_bounds}.
Further, we have
\begin{align}
\label{eq:lb_delta_star}
\delta_{\star} &= \frac{\alpha_n \Im(z)}{2} \lambda_{\min}\left(  
(\bB_\star^{-1/2}\bA_\star\bB_\star^{-1/2} + i \bI)\bB_\star (\bB_\star^{-1/2}\bA_\star\bB_\star^{-1/2} - i\bI)
\right) \\
&\stackrel{(a)}{\ge} \frac{\alpha_n\Im(z)}{2} \lambda_{\min}\left(\bB_\star\right) \ge 
 C_3\frac{\alpha_n}{(1+\alpha_n^{-1})^2}\frac{\Im(z)^2}{\sfK^2 + |z|^2} 
 \nonumber
\end{align}
where $(a)$ holds since the spectrum of $\bB_{\star}^{-1/2}\bA_{\star}\bB_{\star}^{-1/2}$ is real.
Finally, we lower bound $\delta_n$ by writing
\begin{align*}
    \delta_n  &= \frac{\Im(z)}{2\alpha_n} \lambda_{\min}(\bB_n^{-1/2} \bS_n \bS_n^{-1}\bF_n \bF_n^*\bS_n^{*-1}\bS_n^*\bB_n^{-1/2})\\
    &\ge \frac{\Im(z)}{2\alpha_n} \sigma_{\min}(\bS_n^{-1} \bF_n) \lambda_{\min}\left( \bB_n^{-1/2} \bS_n \bS_n^* \bB_n^{-1/2}\right).
\end{align*}
Noting that
as a consequence of Eq.~\eqref{eq:bound_En} we have
    $\norm{\bS_n \bF_n^{-1}} = \norm{\alpha_n^{-1}+\bE_n}_\op \le  3/(2\alpha_n)$ gives us the lower bound on $\sigma_{\min}(\bS_n^{-1}\bF_n) \ge 2\alpha_n/3$.
This along with the decomposition of Eq.~\eqref{eq:lb_delta_star} applied to the display above gives 
\begin{equation}
\nonumber
    \delta_n \ge \frac{\Im(z)}{3}\lambda_{\min}(\bB_n^{-1/2} \bS_n\bS_n^* \bB_n^{-1/2}) \ge \frac{\Im(z)}{3} \lambda_{\min}(\bB_n) \ge
C_4\frac1{(1+\alpha_n^{-1})^2}\frac{\Im(z)^2}{\sfK^2 + |z|^2}.
\end{equation}
Using these bounds in Eq.~\eqref{eq:S_Sn_rate_expansion} above gives the claim. 
\end{proof}

\subsection{Uniform convergence under test functions: Proof of Proposition~\ref{prop:uniform_convergence_lipschitz_test_functions}
}
First, note that Eq.~\eqref{eq:ST_convergence_in_P_seq_measures} of Proposition~\ref{prop:uniform_convergence_lipschitz_test_functions}
can be deduced directly from Lemma~\ref{lemma:rate_matrix_ST}.

The following lemma allows us to deduce convergence of the expectation of a bounded Lipschitz function from convergence of the Stieltjes transform. This result, and its proof are fairly standard. The proof is included in Section~\ref{sec:proof_lemma_f_bound_st} for the sake of completion.
\begin{lemma}
\label{lemma:f_bound_st}
Fix positive integer $A>0$. 
Let $f:[-2A,2A]\to\R$ be continuous. Let $\mu_1,\mu_2$ be two probability measures on $\R$ with support in $[-A,A]$, and let $s_1,s_2$ denote their Stieltjes transforms, respectively.
Then for any $\gamma \in (0,1)$, we have
    \begin{align*}
         \bigg|\int f(x_0)\de\mu_1(x_0) - \int f(x_0)\de\mu_2(x_0)\bigg| &\le
        \frac1\pi \norm{f}_{\infty} \int_{-2A}^{2A} \Big|s_1(x+i\gamma)-s_2(x+i\gamma)\Big|\de x\\
&+
\gamma\left(
2\norm{f}_{\Lip} \log(16 A^2+ 1)
        + \frac{2 \norm{f}_{\infty}}{A}\right).
     \end{align*}

\end{lemma}

We'll apply this lemma to our setting. 
For $z\in\bbH_+$, $[\bbV,\bw]\in\R^{n\times(k+k_0+1)}$, with $\hnu  =\hnu_{[\bbV,\bw]}\in\cuP(\R^{k+k_0+1})$, let
\begin{equation}
    s_n(z;\bbV) := \frac1k \Tr\Big( \bS_n(z; \bbV)\Big),\quad
    s_\star(z ;\hnu, \alpha_n) := \frac1k \Tr\Big( \bS_\star(z; \hnu, \alpha_n)\Big),
\end{equation}
where we recall that $\bS_n$ is 
defined in Eq.~\eqref{eq:Sn_def} and
$\bS_{\star}$ is 
defined by Eq.~\eqref{eq:def_S_star}. 
Recall the definition of $\mu_{\MP} := \mu_{\MP}(\hnu;\alpha_n)$ whose Stieltjes transform is $s_\star$, and let $\mu_n = \mu_n(\bbV;\bw, \alpha_n)$ be the ESD of $\bH_0(\bbV;\bw)/n$. Note that by definition, $s_n$ is the Stieltjes transform of $\mu_n$.
Since $\hmu_{\sqrt{d}\bTheta}$, the empirical distribution of the rows of $\bTheta$, is supported on $[-C,C]$ for $\bTheta$ in the range of interest, and $\rho_0''$ is continuous and hence bounded on this range,
to deduce the first statement of Proposition~\ref{prop:uniform_convergence_lipschitz_test_functions}, 
it's sufficient to prove the following lemma.
\begin{lemma}
\label{lemma:LP_bound}
Fix $\bw\in\R^n.$ Assume $\alpha_n$ satisfies Assumption~\ref{ass:regime}.
For any Lipschitz function $f:\R\to\R$, we have
\begin{equation}
\nonumber
    \limsup_{n\to\infty} \sup_{\bbV\in\R^{n\times (k+k_0+1)}} 
    \left|
    \frac1{dk} \E\left[\Tr\;f \left(\frac1n \bH_0(\bbV;\bw)\right)\big| \bw\right] - \int f(\lambda) \mu_{\MP}(\hnu,\alpha_n)(\de\lambda)\right| = 0.
\end{equation}
\end{lemma}
\begin{proof}
We apply Lemma~\ref{lemma:f_bound_st}.
By definition of $\mu_\star$, for $n$ sufficiently large so that $\alpha_n \in [c_0,C_0]\subset (0,\infty)$, there exists a constant $A_0(\sfK) > 0$ such that
\begin{equation}
\nonumber
    \supp\left(\mu_\star(\nu, \alpha_n)\right) \subseteq [-A_0(\sfK),A_0(\sfK)].
\end{equation}
Furthermore, on the event $\Omega_0$ of Lemma~\ref{lemma:standard_norm_bounds}, we have the bound
(for a similarly bounded constant $A_1$)
\begin{equation}
\nonumber
    \frac1n\norm{\bH_0}_\op 
    \le A_1(\sfK).
    \end{equation}
%
Let  $A := A_1(\sfK) \vee A_0(\sfK)$,
and denote
\begin{equation}
\nonumber
    \hat I_n(\bbV) := 
    \frac1{dk} \Tr\;f \left(\frac1n \bH_0(\bbV)\right), \quad
    I_\star(\hnu) :=  \int f(\lambda) \mu_{\MP}(\hnu,\alpha_n)(\de\lambda).
\end{equation}
Consider the restriction of $f$ to $[-2A,2A]$ denoted by $f_{2A}$.
Since $f$ is continuous, $\|f_{2A}\|_\infty\vee \|f_{2A}\|_{\Lip} \le C_1(\sfK) < \infty$ for some $C_1$.
Then by Lemma~\ref{lemma:f_bound_st}, we have
for any $\gamma \in (0,1)$,
denoting
\begin{align*}
&\left|
    \hat I_n(\bbV)-  I_\star(\hnu)\right|\\
&\le C_2(\sfK)
\Big(
 \sup_{x \in [-2A,2A]}\left|s_n(x + i \gamma;\bbV) - s_\star(x+ i\gamma;\hnu,\alpha_n)\right|+\gamma C_3(\sfK) \Big),
\end{align*}
for some $C_2,C_3>0$.

Recall the events $\Omega_0\cap\Omega_1(L)$ on which Lemma~\ref{lemma:rate_matrix_ST} holds.
Choose $L = L_n = (\log n)^2$ so that the complement of $\Omega_0 \cap \Omega_1$ is an exponentially unlikely event.
Then, by this lemma, we have on $\Omega_0 \cap\Omega_1(L_n)$,
\begin{align*}
  \sup_{x\in[-2A,2A]} |s_n(x + i \gamma) - s_\star(x+i\gamma)| &\le 
  \sup_{x\in [-2A,2A]}\norm{\bS_n(x+i \gamma) - \bS_\star(x + i\gamma)}_\op\\
  &\le
   C_4(\sfK) \frac{1 + |A|^4+  |\gamma|^4}{\gamma^5} 
   \Err_{\FP}( 2A + i\gamma;n,k, L_n)
\end{align*}
for some $C_4>0$,
whenever Eq.~\eqref{eq:n_n_0_2} is satisfied. 
So choosing $\gamma := \gamma_n \to 0$ slow enough so that Eq.~\eqref{eq:n_n_0_2} is satisfied uniformly for all $z$ with $\Re(z) \in [-2A,2A]$, and so that $\Err_{\FP}(2A + i\gamma_n; n, k, L_n) \to 0$  as $n\to\infty$
shows that 
\begin{equation}
\label{eq:hatI_diff_Istar}
\Delta_n := \sup_{\bbV\in\R^{n\times (k+k_0+1)}}\left|
    \hat I_n(\bbV)-  I_\star(\hnu)\right| \to 0
\end{equation}
as $n\to\infty$ on $\Omega_0 \cap \Omega_1(L_n)$.

Now $\sup_{\bbV}I_\star$ is always bounded by $C_1(\sfK)$, and $\sup_\bbV\hat I_n$ is uniformly bounded on $\Omega_0$, so $\Delta_n$ is as well.
Outside of $\Omega_0$, the assumption that $f$ is Lipschitz and the exponential bound on $\P(\Omega_0^c)$ guarantees that 
$\sup_{\bbV}\E[\hat I_n(\bbV) \one_{\Omega_0^c}|\bw] \to 0$ as $n\to\infty$. So we conclude
\begin{align*}
   \left|\E[\hat I_n(\bbV)|\bw] -  I_\star(\hnu)\right| &\le 
 \E\left[\left|\hat I_n(\bbV) -  I_\star(\hnu)\right| \one_{\Omega_0 \cap \Omega_1(L_n)}\big| \bw\right] 
 +
 2 C_1(\sfK) \P\left(\Omega_1(L_n)^c \cap \Omega_0 \right) \\
 &+ C_1(\sfK) \P(\Omega_0^c) + 
 \sup_{\bbV}\E[\hat I_n(\bbV) \one_{\Omega_0^c}|\bw] \to 0
\end{align*}
as $n\to\infty$ as desired.
\end{proof}

\subsection{Proofs of technical results of this section}
\label{section:RMT_appendix_technical_results}

\subsubsection{Proof of Lemma~\ref{lemma:re_im_properties}}
\label{sec:proof_lemma_re_im_properties}
For $\bZ\in\bbH^+_k$, let $\bA = \Re(\bZ)$ and $\bB =\Im(\bZ)$. Since $\bB \succ \bzero$ we can write
\begin{equation}
\label{eq:bA_decomp}
   \bZ = (\bA + i \bB)  =  \bB^{1/2} \left(\bB^{-1/2}\bA \bB^{-1/2} + i \bI\right) \bB^{1/2}.
\end{equation}
The spectrum of $\bB^{-1/2}\bA\bB^{-1/2}$ is real since it's self-adjoint, and hence its perturbation by $i\bI$ does not contain $0$, 
proving invertibility of $\bZ$.
Now by Eq.~\eqref{eq:bA_decomp},
   \begin{align*}
      &\Im(\bZ^{-1}) 
      = 
       \Im\left(
       \bB^{-1/2}(\bB^{-1/2}\bA\bB^{-1/2} + i\bI)^{-1}\bB^{-1/2}
       \right)\\
       &=
       \Im\left(
       \bB^{-1/2}(\bB^{-1/2}\bA\bB^{-1/2} + i\bI)^{-1}
       (\bB^{-1/2}\bA\bB^{-1/2} - i\bI)
       (\bB^{-1/2}\bA\bB^{-1/2} - i\bI)^{-1}
       \bB^{-1/2}
       \right)\\
       &=
       -\bB^{-1/2}(\bB^{-1/2}\bA\bB^{-1/2} + i\bI)^{-1}
       (\bB^{-1/2}\bA\bB^{-1/2} - i\bI)^{-1}
       \bB^{-1/2}\\
       &= 
       -(\bA + i\bB)^{-1}
       \bB
   (\bA - i\bB)^{-1}\\
&= - \bZ^{-1} \Im(\bZ) \bZ^{*-1}\\
&\prec \bzero,
   \end{align*}
where in the last line we used that $\Im(\bZ)\succ\bzero.$
To prove the bounds on the operator norm in Item~\textit{1},
note that, for any vector $\bx\in \C^k$,
\begin{align}
\label{eq:modulus_is_positive}
\|(\bB^{-1/2}\bA\bB^{-1/2} + i\bI)\bx\|_2 \ge \|\bx\|_2,
\end{align}
whence, for any $\bx\in \C^k$,
$\|(\bB^{-1/2}\bA\bB^{-1/2} + i\bI)^{-1}\bx\|_2 \le \|\bx\|_2$.
Therefore, taking inverses of both sides of Eq.~\eqref{eq:bA_decomp}
we conclude that $\|\bZ^{-1}\bx\|_2 \le \|\bB^{-1}\bx\|_2 
=\|\Im(\bZ)^{-1}\bx\|_2$ as desired for the bound on $\norm{\bZ^{-1}}_\op$. Using Eq.~\eqref{eq:modulus_is_positive} once again we conclude that 
$\norm{\bZ\bx}_2 \ge \norm{\bB\bx}_2$ giving the desired bound on $\Im(\bZ).$

To prove Item~\textit{2}, first consider the case where $\bW$ is invertible. In this case, we can write
   $(\bI + \bW\bZ)^{-1} \bW = (\bW^{-1} + \bZ)^{-1}.$
Noting that $\Im(\bW^{-1} + \bZ) = \Im(\bZ) \succ \bzero$, we see that an application of Item \textit{1} gives both claims. 
For non-invertible $\bW$, let $s_{\min} := \lambda_{*}(\bW)$
be the non-zero eigenvalue of $\bW$ with the smallest absolute value,
and define $\bW_\eps := \bW + \eps |s_{\min}|$ for $\eps \in (0,1)$.
We have by the previous argument that the statement holds for $\bW$ replaced with $\bW_\eps.$ Taking $\eps\to 0$ proves it in the general non-invertible case.
\qed

\subsubsection{Proof of Lemma~\ref{lemma:tensor_trace_properties}}
\label{sec:proof_lemma_tensor_trace_properties}
Let $\bM_{i,j}\in\C^{d\times d}$ be the blocks of $\bM$ for $i,j \in[k]$.
We obtain the first bound in \textit{1} by writing
\begin{align}
\nonumber
    \norm{(\bI_k \otimes \Tr)\bM}_F^2
    &= \sum_{i,j \in[k]} \Tr(\bM_{ij})^2
    \le d \sum_{i,j\in[k]} \norm{\bM_{ij}}_F^2
    = d \norm{\bM}_F^2.
\end{align}
Now let $\bx\in\R^{d}$ be a random variable distributed uniformly on the sphere of radius $\sqrt{d}$.
For any $\bv,\bu \in \C^{k}$ we have
\begin{equation}
\label{eq:tensor_tr_to_E_sphere}
    \bu^* \left(\left(\bI_k \otimes \Tr\right)\bM \right)\bv 
   = \sum_{i,j} \overline{u_i} v_j \Tr(\bM_{i,j}) 
  =  \sum_{i,j} \overline{u_i} v_j \E[\bx^\sT\bM_{i,j} \bx] = \E[(\bu\otimes \bx)^* \bM (\bv\otimes \bx)].
\end{equation}
Optimizing over $\bv,\bu$ of unit norm gives
\begin{align}
\nonumber
   \norm{(\bI_k \otimes \Tr)\bM}_\op  
    \le 
\max_{\norm{\bv}_2 = \norm{\bu}_2 = 1}\norm{\bM}_\op  \E[\norm{\bx}_2^2] \norm{\bu}\norm{\bv}
 = d \norm{\bM}_\op
\end{align}
giving the second bound in Item \textit{1}.
For the claim in \textit{2},
take $\bv = \bu$ in Eq.~\eqref{eq:tensor_tr_to_E_sphere} to conclude the (strict) positivity of $(\bI_k\otimes \Tr)\bM$ from that of $\bM$.
For \textit{3}, once consider Eq.~\eqref{eq:tensor_tr_to_E_sphere} and minimize over $\bv = \bu$ with unit norm and use \textit{2} to write
\begin{align*}
\nonumber
   \lambda_{\min}\left((\bI_k \otimes \Tr)\bM\right)
   &= \min_{\norm{\bv} = 1} \E[(\bv \otimes \bx)^* \bM (\bv \otimes \bx)]\\
   &\ge \min_{\norm{\bv} = 1} \lambda_{\min}(\bM) \E[\norm{\bv \otimes \bx}_2^2]
    = d \lambda_{\min}(\bM),
\end{align*}
giving the claim.
Finally, \textit{4} follows by linearity of the involved operators.
\qed
\subsubsection{
Proof of Lemma~\ref{lemma:algebra_lemma}}
\label{sec:proof_lemma_algebra_lemma}
Suppress the argument $z$ in what follows.
By Woodbury, we have for each $i\in[n],$
\begin{equation}
\nonumber
\bR  = \bR_i - \bR_i\bxi_i
\left(\bI_k + \bW_i\bxi_i^\sT \bR_i \bxi_i\right)^{-1} \bW_i\bxi_i^\sT \bR_i.
\end{equation}
So
\begin{align}
\label{eq:R_to_R_i_woodburry}
 \bxi_i \bW_i\bxi_i^\sT \bR 
&= 
 \bxi_i \bW_i\bxi_i^\sT  \bR_i -
\bxi_i \bW_i\bxi_i^\sT
\bR_i\bxi_i
\left(\bI_k + \bW_i\bxi_i^\sT \bR_i \bxi_i\right)^{-1} \bW_i\bxi_i^\sT \bR_i\\
&= \nonumber
 \bxi_i \bW_i\bxi_i^\sT  \bR_i -
\bxi_i
\left(
\bI_k - 
\left(\bI_k +\bW_i \bxi_i^\sT \bR_i \bxi_i\right)^{-1}
\right)
\bW_i\bxi_i^\sT \bR_i\\
&=\nonumber
\bxi_i
\left(\bI_k + \bW_i\bxi_i^\sT \bR_i \bxi_i\right)^{-1}
\bW_i\bxi_i^\sT \bR_i.
\end{align}

To prove the identity in~Eq.~\eqref{eq:alg_id1},
for any $\bA \in\R^{k\times k}$, note that 
we have
\begin{align*}
\left(\bI_k \otimes \Tr\right) \bxi_i \bA  \bW_i\bxi_i^\sT \bR_i
   &=\left(\bI_k \otimes \Tr\right) (\bI_k \otimes \bx_i)
   \bA  
\bW_i
   (\bI_k \otimes \bx_i)^\sT \bR_i\\
   &= \left(\bI_k \otimes \Tr\right) \left(\sum_{a\in[k]} 
   (\bA \bW_i)_{j,a}
   \bx_i\bx_i^\sT  (\bR_i)_{a,l}
   \right)_{j,l \in[k]}
   \\
   &= \left( \sum_{a \in [k]}
   (\bA \bW_i)_{j,a}
   \Tr\left(\bx_i\bx_i^\sT  \left(\bR_i\right)_{a,l}
   \right)\right)_{j,l \in [k]}\\
   &= \left( \sum_{a \in [k]}
   (\bA \bW_i)_{j,a}
   \bx_i^\sT  \left(\bR_i\right)_{a,l} \bx_i
   \right)_{j,l \in [k]}\\
   &= 
   (\bA \bW_i)
   \bxi_i^\sT \bR_i\bxi_i.
\end{align*}
This identity with
    $\bA :=\left( \bI_k + \bW_i\bxi_i^\sT \bR_i \bxi_i\right)^{-1}$
    along with Eq.~\eqref{eq:R_to_R_i_woodburry}
gives the result.
To prove the identity of Eq.~\eqref{eq:alg_id2}, we write
\begin{align*}
   \left(\bI_k \otimes \Tr \right)(\bR_i - \bR) &= \left(\bI_k \otimes \Tr \right)\bR_i \bxi_i \bW_i\bxi_i^\sT \bR\\
   &= \left(\Tr\left( \sum_{b,c \in [k]} (\bR_i)_{a,b} \bx_i(\bW_i)_{b,c}\bx_i^\sT \bR_{c,d}\right)  \right)_{a,d \in [k]}\\
   &= \left( \sum_{b,c \in [k]}  \bx_i^\sT  (\bR_i)_{a,b}^\sT (\bW_i)_{b,c}\bR_{c,d}^\sT \bx_i  \right)_{a,c \in [k]}\\
   &=  \bxi_i^\sT \bR_i^\sT (\bW_i \otimes \bI_d) \bR^\sT \bxi_i.
\end{align*}
Real symmetry of the matrices $\bR_i$ and $\bR$ gives the conclusion.
\qed

\subsubsection{Proof of Lemma~\ref{lemma:as_norm_bounds}}
\label{sec:proof_lemma_as_norm_bounds}
Since the eigenvalues $\left\{\lambda_j\right\}_{j\in[dk]}$ of $\bH_i$ are  real, we have
for the first two bounds of Eq.~\eqref{eq:det_norm_bound_lemma_eq123}
\begin{align}
\nonumber
   \norm{\bR_i}_{F}^2 
   &= \sum_{j=1}^{dk} \frac{1}{|\lambda_j - z n|^2 }
\le \frac{dk}{n^2} \frac1{\Im(z)^2}\quad\textrm{and}\quad
   \norm{\bR_i}_{\op}
   = \max_{i \in [dk]} \frac{1}{|\lambda_i - z n| }
\le \frac{1}{n} \frac1{\Im(z)}.
\end{align}
The bounds on $\norm{\bR}_{F}^2$, $\norm{\bR}_{\op}$
follow similarly.

For the third bound of Eq.~\eqref{eq:det_norm_bound_lemma_eq123}, given a vector $\ba\in\R^{nk}\simeq \R^{n}\otimes \R^k$, 
we write its entries as $\ba = (a_{i,l}: \; i\in[n],
l\in [k])$.
Then write
\begin{align}
\nonumber
    \<\ba,\bSec\ba\> 
= \sum_{i=1}^n \< \ba_{i,\cdot}, \nabla^2_{\bv}\ell(\bv_i,\bu_i,\eps_i) \ba_{i,\cdot}\> 
    =  \sum_{i=1}^n \< \ba_{i,\cdot}, \bW_i \ba_{i,\cdot}\>  
    \le\sfK  \sum_{i=1}^n \norm{\ba_{i,\cdot}}_2^2\, ,
\end{align}
since $\|\bW_i\|_\op \le \sfK$ by definition.
Optimizing over $\norm{\ba}_2 = 1$ gives $\norm{\bK}_\op \le \sfK$ as desired.

The inequality $\norm{\bH_0}_\op \le\sfK \norm{\bX}_\op^2$ follows directly from the previous one.

For the bound in~\eqref{eq:det_norm_bound_lemma_eq4}
recall that $\bH_0$ is self-adjoint, and hence for $z\in\bbH_+$, 
$\Im(\bH_0/n - z\bI)^{-1} \succ \bzero$. So we can bound
\begin{align*}
   \lambda_{\min}(\Im((\bI\otimes\Tr)\bR)) &\stackrel{(a)}{=} \lambda_{\min}\left( 
    \frac1n \left( \bI_k \otimes \Tr\right) \Im\left(\left(\bH_0/n - z \bI_{nk}\right)^{-1}\right)\right)\\
    &\stackrel{(b)}{\ge}  \lambda_{\min}\left(\Im((\bH_0/n - z\bI)^{-1})\right)\\
&\stackrel{(c)}{=}  \Im(z)\lambda_{\min}\left(
(\bH_0/n -z\bI)^{-1}(\bH_0/n -z^*\bI)^{-1}
\right)\\
&\ge  \Im(z)(\norm{\bH_0/n}_\op  + |z|)^{-2},
\end{align*}
where $(a)$ and $(b)$ follow by Lemma~\ref{lemma:tensor_trace_properties}, and $(c)$ follows by Lemma~\ref{lemma:re_im_properties}.
The conclusion now follows by another application of Lemma~\ref{lemma:re_im_properties}.

Finally, for the bound in~\eqref{eq:det_norm_bound_lemma_eq5} note that
\begin{equation}
    \lambda_{\min}(\Im(\bxi_i^\sT\bR_i \bxi_i)) = 
    \lambda_{\min}(\bxi_i^\sT\Im(\bR_i) \bxi_i) = \frac{\lambda_{\min}(\bxi_i^\sT\bxi_i)}{n} \lambda_{\min}
    \big(\Im( (\bH_i/n - z\bI_{n,k})^{-1})\big)
\end{equation}
and that  $\lambda_{\min}(\bxi_i^\sT\bxi_i) = \lambda_{\min}(\bI_k \otimes \bx_i^\sT\bx_i)= \norm{\bx_i}_2^2$ along with Lemma~\ref{lemma:re_im_properties} to derive the conclusion.
\qed

\subsubsection{Proof of Lemma~\ref{lemma:f_bound_st}}
\label{sec:proof_lemma_f_bound_st}
Lemma~\ref{lemma:f_bound_st} is a direct corollary of the two lemmas of this section.
Let
\begin{equation}
    \rho(x; x_0, \gamma) := \frac1\pi \Im\left( \frac{1}{x - (x_0 + i\gamma)} \right)
\end{equation}
be the density of a Cauchy distribution with location $x_0\in\R$ and scale $\gamma>0$. Note that for any continuous bounded function $f$,
$$\lim_{\gamma \to 0}\int f(x) \rho(x; x_0, \gamma) \, \de x = f(x_0).$$
The next lemma gives a quantitative version of this
fact.

\begin{lemma}
\label{lemma:quant_dirac_integral}
Fix positive reals $B > A > 0$. Define for $|x_0| \le A$, and $f :[-B,B] \to\R$ continuous,
\begin{equation}
    \Delta_{f,B,\gamma}(x_0):= f(x_0) - \int_{-B}^B  f(x)\rho(x;x_0,\gamma)\, \de x.
\end{equation}
We have the bound
\begin{equation}
\nonumber
        \sup_{x_0 \in [-A,A]}\left|\Delta_{ f,B,\gamma}(x_0)\right|\leq \norm{f}_{\Lip} \gamma\log(4B^2+\gamma^2)
        +\norm{f}_{\infty} \frac{\gamma}{2} \left(\frac1{B-A} + \frac1{B+A}\right).
\end{equation}
\end{lemma}

\begin{proof}
Fix $x_0 \in [-A,A]$. We have
    \begin{align}
    \nonumber
        \Delta_{ f,B,\gamma}(x_0) =&  f(x_0)\int_{-\infty}^\infty\rho(x;x_0,\gamma)\de x
        -\int_{-B}^B  f(x) \rho(x;x_0,\gamma)\de x\\
        \nonumber
        =&\int_{-B}^B \left(f(x_0)-f(x) \right) \rho(x;x_0,\gamma)\de x +  f(x_0)\left(1-\int_{-B}^{B} \rho(x;x_0,\gamma)\de x\right)\\
        \leq & \norm{f}_{\Lip}\left(\int_{-B}^B |x_0-x|\rho(x;x_0,\gamma)\de x\right)
        +\norm{f}_{\infty}\left(1-\int_{-B}^B \rho(x;x_0,\gamma)\de x_0\right).
        \label{eq:last_eq_in_DeltafB_bound}
    \end{align}
    By a change of variable, the first integral above can be bounded as
    \begin{align*}
        \int_{-B}^B |x_0-x|\rho(x;x_0,\gamma)\de x &= 
        \int_{-B-x_0}^{B-x_0} |x|\rho(x;0,\gamma)\de x\\
        &\leq \frac{2}{\pi}\int_{0}^{2B} x \frac{\gamma}{x^2+\gamma^2}\de x
        = \frac{\gamma}{\pi}\log\left(\frac{4B^2}{\gamma^2} + 1\right)
    \end{align*}
where we used the even symmetry of the integrand and that $|x_0 | \le B$ to deduce the inequality.
Meanwhile, the second integral in Eq.~\eqref{eq:last_eq_in_DeltafB_bound} is bounded as
    \begin{align*}
        1-\int_{-B}^B \rho(x;x_0,v)\de x
        &= 1- \frac{1}{\pi }\left[\arctan \left( \frac{B-x_0}{\gamma}\right) - \arctan \left( \frac{-B-x_0}{\gamma}\right)\right]\\
        &\le  1- \frac{1}{\pi }\left[\arctan \left( \frac{B-A}{\gamma}\right) + \arctan \left( \frac{B+A}{\gamma}\right)\right]\\
        &\le
        \frac12 \left( \frac{\gamma}{B-A}  + \frac{\gamma}{B+A}\right),
    \end{align*}
   where in the last line we used that $1 -2\pi^{-1}\arctan(t) \le t^{-1}$. This concludes the proof.
\end{proof}
\begin{lemma}
Fix positive reals $B>A \ge 0$.
Let $f:[-B,B]\to\R$ be continuous. Let $\mu_1,\mu_2$ be two probability measures on $\R$ supported in $[-A,A]$ and let $s_1,s_2$ denote their corresponding Stieltjes transforms, respectively.
Then for any $\gamma  > 0$, we have
    \begin{align*}
         &\bigg|\int f(x_0)\de\mu_1(x_0) - \int f(x_0)\de\mu_2(x_0)\bigg| \\
         &\le
        \frac1\pi \norm{f}_{\infty} \int_{-B}^B \Big|s_1(x+i\gamma)-s_2(x+i\gamma)\Big|\de x
+2\norm{f}_{\Lip} \gamma\log(4B^2+\gamma^2)\\
        &\quad +\norm{f}_{\infty} \gamma \left(\frac1{B-A} + \frac1{B+A}\right).
     \end{align*} 
\end{lemma}

\begin{proof}
Rewriting $f$ in terms of the quantity $\Delta_{f,B,\gamma}$ defined in Lemma~\ref{lemma:quant_dirac_integral}, we have
\begin{align}
\nonumber
    &\int f(x_0) \left(\de \mu_1(x_0) - \de \mu_2(x_0)\right)\\\nonumber
    &= \int\left( \int_{-B}^B f(x) \rho(x;x_0, \gamma) \de x  + \Delta_{f,B,\gamma}(x_0) \right) 
     \left(\de \mu_1(x_0) - \de \mu_2(x_0)\right) \\
    &= 
     \int_{-B}^B f(x)\left(\int  \rho(x;x_0, \gamma) 
     \left(\de \mu_1(x_0) - \de \mu_2(x_0)\right) \right)\de x
     +
\int \Delta_{f,B,\gamma}(x_0)
    \left(\de \mu_1(x_0) - \de \mu_2(x_0)\right)
     \label{eq:decomp_Ex_diff}
\end{align}
where the change of order of integration is justified by integrability of the continuous $f$ over $[-B,B]$.
Noting that
for $j\in\{1,2\}$,
\begin{equation}
\nonumber
     \int \rho(x;x_0,\gamma) \de \mu_j(x_0)=
     \frac1{\pi}\Im(s_j(x+i\gamma)),
\end{equation}
the first term in Eq.~\eqref{eq:decomp_Ex_diff} is bounded as
\begin{align}
\nonumber
     &\int_{-B}^B f(x)\left(\int  \rho(x;x_0, \gamma) 
     \left(\de \mu_1(x_0) - \de \mu_2(x_0)\right) \right)\de x \\ \nonumber
     &=
     \frac1\pi\int_{-B}^B f(x) \left(
    \Im\left(s_1(x + i \gamma) - s_2(x+ i\gamma)\right)
     \right)
     \de x\\
     &\le \frac1\pi \norm{f}_{\infty} \int_{-B}^B \left|s_1(x + i \gamma) - s_2(x+ i\gamma)\right| \de x.
     \label{eq:decom_Ex_diff_bound_1}
\end{align}

To bound the second term in Eq.~\eqref{eq:decomp_Ex_diff}, 
for each $j\in\{1,2\}$
we have
\begin{align*}
   \int \Delta_{f,B,\gamma}(x_0) \de \mu_j(x_0)  
   &\le  \int_{-A}^A \left|\Delta_{f,B,\gamma}(x_0)\right| \de \mu_j(x_0)\\
   &\le  
\sup_{x_0 \in [-A,A]} \left|\Delta_{f,B,\gamma}(x_0)\right|
\end{align*} 
Applying Lemma~\ref{lemma:quant_dirac_integral} and combining with Eq.~\eqref{eq:decomp_Ex_diff} and Eq.~\eqref{eq:decom_Ex_diff_bound_1} gives the desired bound.

\end{proof}

%% file: Appendix_B_THM1.tex
\section{Proof of Lemmas for Theorem~\ref{thm:general}}
\label{section:kac_rice}

\subsection{Deriving the Kac-Rice equation on the parameter manifold: Proof of Lemma~\ref{prop:kac_rice_manifold}}
\label{sec:proof_of_kac_rice_on_manifold}
The goal of this section is to verify the generalization of the Kac-Rice integral to our setting and derive Lemma~\ref{prop:kac_rice_manifold} of Section~\ref{sec:pf_thm1_kr_integral}.

We use $r_k:=k(k+k_0)$, $m_n:=nk+nk_0+dk$, and continue suppressing the arguments in the definitions whenever it does not cause confusion. For instance, we will often write $\bL$ for $\bL(\bbV;\bw)$ and $\bH$ for $\bH(\bTheta,\bbV).$
In this section, unless otherwise specified, we use $\hnu$ for the empirical distribution of the rows of $[\bbV,\bw].$

\subsubsection{Some properties of the parameter manifold and the gradient process}

We begin with the following lemma establishing basic regularity conditions of the manifold $\cM(\cuA,\cuB)$.
\label{sec:properties_manifold}
\begin{lemma}[Regularity of the parameter manifold]
\label{lemma:manifold_dim}
Assume that $\ell,\rho$ and $(\cuA,\cuB)$ satisfy Assumptions~\ref{ass:loss},~\ref{ass:regularizer} and~\ref{ass:sets}, respectively. 
Then for $\sfa_R,\sfa_L,\sfa_G \ge 0$,
$\cM(\cuA,\cuB)$ is a differentiable manifold of co-dimension 
$m_k-r_k$, 
and in particular is bounded and orientable.
\end{lemma}
\begin{proof}
Let 
\begin{align}
    \cV_0 :=\Big\{
   (\bTheta,\bbV) &: \hmu\in\cuA,\;
   \hnu\in\cuB,\; 
   \sfA_R^2\succ \bR(\hmu) \succ\sfa_R^2,\;\\
   &\quad 
\sfA_V^2 \succ \E_{\hnu}[\bv\bv^\sT] , \;
\E_{\hnu}[\grad\ell \grad\ell^\sT] \succ \sfa_{L}^2,\;
\sigma_{\min}\left( \bJ_{(\bbV,\bTheta)} \bG\right) > \,\sfa_{G}
    \Big\}.
\end{align}
Observe that $\cV_0$ is a bounded open set: by Assumption~\ref{ass:sets}, the constraint imposed by $\cuA$ and $\cuB$ defines an open set, and similarly for the remaining constraints, by the regularity assumptions on $\ell$ and $\rho$.
Then since $\bzero$ is a regular value of $\bG(\bbV,\bTheta)$ when restricted to $\cV_0$ for any $\sfa_G \ge 0$, this implies $\cM(\cuA,\cuB) = \{(\bTheta,\bbV):\; (\bTheta,\bbV)\in\cV_0, \; \bG(\bTheta,\bbV) = \bzero\}$ is a bounded regular level set, and the desired properties then follow. The dimension can then be found to be 
$m_n-r_k$ 
by dimension counting.
\end{proof}

Next, we move on to studying the null space of the covariance of $\bzeta$ and showing that its degeneracy is constant in $(\bTheta,\bbV) \in\cM(\cuA,\cuB)$. 
First, by straightforward computations, we obtain the following for the mean $\bmu(\bTheta,\bbV)$ and covariance $\bLambda(\bTheta,\bbV)$ of $\bzeta(\bTheta,\bbV)$:
\begin{equation}
 \bmu := (\,\brho_1,\,\dots,\,\brho_k,\,-\bv_1,\,\dots,\,-\bv_k,\, -\bv_{0,1},\,\dots,
 \,-\bv_{0,k_0}\,),
 \end{equation}
 \begin{equation}
\bLambda:= 
   \begin{bmatrix}
       \bL^\sT \bL \otimes \bI_{d}  & \bM & \bM_0\\
       \bM^\sT  & \bTheta^\sT \bTheta \otimes \bI_n & \bTheta^\sT\bTheta_0 \otimes \bI_n\\
       \bM_0^{\sT} & \bTheta_0^\sT\bTheta  \otimes \bI_n& \bTheta_0^{\sT}\bTheta_0 \otimes \bI_n
   \end{bmatrix} ,
\end{equation}
where
\begin{equation}
    \bM :=  \begin{bmatrix}
        \btheta_1\bell_1^\sT & \dots  & \btheta_k \bell_1^\sT\\
        \vdots  &   & \vdots \\
        \btheta_1\bell_k^\sT & \dots  & \btheta_k \bell_k^\sT\\
    \end{bmatrix}
    \in \R^{d k\times n k},\quad
    \bM_0 :=
    \begin{bmatrix}
        \btheta_{0,1}\bell_1^\sT & \dots  & \btheta_{0,k_0} \bell_1^\sT\\
        \vdots  &   & \vdots \\
        \btheta_{0,1}\bell_k^\sT & \dots  & \btheta_{0,k_0} \bell_k^\sT\\
        \end{bmatrix} \in \R^{d{k} \times n k_0 }.
\end{equation}
The following lemma characterizes the nullspace of the covariance space $\bLambda$.
\begin{lemma}[Eigenvectors of the nullspace of $\bLambda$]
\label{lemma:eig_vecs_NS_Sigma}
Let
\begin{equation}
    \ba_{i,j}(\bTheta,\bbV):=
    \big(\be_{k,j}\otimes\btheta_i,\,
    -\be_{k,i}\otimes\bell_j(\bbV),\,
    \bzero_{nk_0}
    \big)\in\R^{m_n}
    ,\qquad i,j\in[k],
\end{equation}
\begin{equation}
    \ba_{0,i,j}(\bbV):= \big(\be_{k,j}\otimes\btheta_{0,i},\,
    \bzero_{nk},\,
    -\be_{k_0,i}\otimes\bell_j(\bbV)
    \big)\in\R^{m_n},
    \qquad i\in[k_0],j\in[k],
\end{equation}
in which $(\be_{k,j})_{j\in[k]}$ and $(\be_{k_0,i})_{i\in[k_0]}$ are the canonical basis vectors for $\R^k$ and $\R^{k_0}$, respectively. 
Then for any $(\bTheta,\bbV) \in \cM(\cuA,\cuB)$:
\begin{enumerate}
\item We have
    \begin{equation}
     \nonumber
     \Nul(\bLambda(\bTheta,\bbV)) =   \mathrm{span}\left(\left\{\ba_{i,j}(\bTheta,\bbV): i,j \in [k]\right\}\cup \left\{\ba_{0,i,j}(\bbV): i\in[k_0], j \in[k]  \right\} \right);
  \end{equation}

\item the collection of basis vectors
$\left\{\ba_{i,j}(\bTheta,\bbV): i,j \in [k]\right\}\cup \left\{\ba_{0,i,j}(\bbV): i\in[k_0], j \in[k]  \right\}$ are linearly independent, and hence, 
\begin{equation} \nonumber
\dim(\Nul(\bLambda(\bTheta,\bbV))) = 
r_k, 
\quad
\rank(\bLambda(\bTheta,\bbV)) 
=m_n-r_k,
\end{equation}
\item the mean $\bmu(\bTheta,\bbV)$ of $\bzeta(\bTheta,\bbV)$ is orthogonal to $\Nul(\bLambda(\bTheta,\bbV))$.
\end{enumerate}
\end{lemma}

\begin{proof}
Fix $(\bTheta,\bbV)$ throughout and suppress these in the notation.
Direct computation shows that $\bLambda \ba_{i,j} = \bzero$ and $\bLambda \ba_{0,i,j} = \bzero$ and hence the span of these vectors is contained in $\Nul(\bLambda)$.
We show the converse. Let $\bc \in\R^{dk}$, $\bd \in \R^{nk + nk_0}$ such that $\bLambda [\bc^\sT, \bd^\sT]^\sT= \bzero$. We then have
\begin{align*}
    \left(\bL^\sT\bL \otimes \bI_d\right) \bc + [\bM, \bM_0] \bd &= \bzero\\
    \left[\bM, \bM_0\right]^\sT\bc + \left(\bR \otimes \bI_n\right) \bd &= \bzero.
\end{align*}
Let us define the sets
\begin{align*}
    \cS_1:= &\{
    \be_{k,j}\otimes\btheta_i:\; i,j\in[k]
    \}\cup \{\be_{k,j}\otimes \btheta_{0,i} 
    \; : i\in[k_0],j\in[k]\}\subset \R^{dk},\\
    \cS_2:=&\{
    (\be_{k,i}\otimes\bell_j,
    \bzero):\; i,j\in[k]\} \cup 
    \{(\bzero,
    \be_{k_0,i}\otimes\bell_j, ) :\;i\in[k_0],j\in[k]\}
    \subset\R^{nk+nk_0}.
\end{align*}
Note that 
$[\bM,\bM_0]\bd\in\vspan(\cS_1)$,
$[\bM,\bM_0]^\sT\bc\in \vspan(\cS_2)$, and $\bL^\sT\bL$ and $\bR$ are invertible. This implies that  
$\bc\in\vspan(\cS_1)$
and 
$\bd\in\vspan(\cS_2)$
Hence, any $\ba \in \Nul(\bLambda)$ satisfies
\begin{align}
    \ba \in
    \mathrm{span}
    \big(&\left\{\ba_{i,j}: i,j \in [k]\right\}\cup \left\{\ba_{0,i,j}: i\in[k_0], j \in[k]  \right\} \cup
    \left\{ \tilde \ba_{i,j} : i,j \in [k] \right\}\\
    &\cup
\left\{ \tilde \ba_{0,i,j} : i\in [k_0],j\in[k] \right\}
 \big)
 \nonumber
\end{align}
where 
\begin{equation*}
\tilde\ba_{i,j} := \big(
\be_{k,j}\otimes\btheta_i,
\underbrace{\bzero, \dots, \bzero}_{k+k_0}
\big)\qquad\text{and}\qquad
\tilde\ba_{0,i,j} := \big(
\be_{k,j}\otimes\btheta_{0,i},
\underbrace{\bzero, \dots, \bzero}_{k+k_0}
\big).
\end{equation*}

We'll show that the collection $\{\bLambda \tilde \ba_{i,j}\} \cap \{\bLambda \tilde \ba_{0,i,j}\}$ is linearly independent.
This will then imply the desired inclusion
 $\Nul(\bLambda) \subseteq    \mathrm{span}\left(\left\{\ba_{i,j}: i,j \in [k]\right\}\cup \left\{\ba_{0,i,j}: i\in[k_0], j \in[k]  \right\} \right)$.
One can compute 
\begin{equation*}
    \bLambda \tilde \ba_{i,j} =
     \begin{bmatrix}(\bL^\sT\bL \otimes \bI_d)(\be_j\otimes \btheta_i)\\
   (\bR(\bTheta,\bTheta_0) \otimes \bI_n) (\be_j \otimes \bell_j)
   \end{bmatrix}
    ,
    \quad
    \bLambda \tilde \ba_{0,i,j} =
     \begin{bmatrix}(\bL^\sT\bL \otimes \bI_d)(\be_j\otimes \btheta_{i,0})\\
   (\bR(\bTheta,\bTheta_0) \otimes \bI_n) (\be_j \otimes \bell_j)
   \end{bmatrix}.
\end{equation*}
Once again, the linear independence of the columns of $\bL$ and $(\bTheta,\bTheta_0)$ now implies the desired linear independence.

Finally, the statement about the mean follows by using the constraint
$\bG(\bbV,\bTheta) = \bzero$ for $(\bTheta,\bbV) \in \cM(\cuA,\cuB)$ and carrying out the computation.
\end{proof}

As a corollary, we have the following.
\begin{corollary}
\label{cor:proj}
Under Assumptions~\ref{ass:loss} and~\ref{ass:regularizer}, there exists a matrix-valued map 
$(\bTheta,\bbV)\mapsto\bB_{\bLambda(\bTheta,\bbV)}\in\R^{m_n \times (m_n-r_k)}$ 
such that its entries are twice continuously differentiable in $(\bTheta,\bbV)\in\R^{m_n}$, 
and for every $(\bTheta,\bbV)\in\cM(\cuA,\cuB)$ the columns of 
$\bB_{\bLambda(\bTheta,\bbV)}$ form an orthonormal basis of 
$\Col(\bLambda(\bTheta,\bbV))$.
\end{corollary}

\subsubsection{Concluding the proof of Lemma~\ref{prop:kac_rice_manifold}}

Let us cite the following lemma which will be useful in checking the non-degeneracy of the process required for the Kac-Rice formula to hold.
\begin{lemma}[Proposition 6.5,~\cite{azais2009level}]
\label{lemma:proc_grad_as_0}
    Let $\bh:\cU \to\R^m$ be a $C^2$ random process over open set $\cU \subseteq \R^m$. 
Let $\mathcal{K} \subseteq \cU$ be a compact subset. 
Fix $\bu \in\R^m$.
Assume that for some $\delta >0$,
\begin{equation}
    \sup_{\bt \in \mathcal{K}} \sup_{ \bs  \in B_\delta^m(\bu) } p_{\bh(\bt)}(\bs)  < \infty.
\end{equation}
Then 
    \begin{equation}
        \P\left( \exists \bt \in \mathcal{K}:  \bh(\bt) = \bu,\; \det(\bJ_\bt \bh(\bt)) = 0\right) = 0.
    \end{equation}
\end{lemma}

\begin{proof}[Proof of Lemma~\ref{lemma:kac_rice_manifold}]
Throughout, fix $\bw$ and use $\E$ and $\P$ to denote conditional expectation 
and probability. Let $(\cU,\bpsi)$ be a local chart of $\cM(\cuA,\cuB)$.  
The argument is first established on the chart domain $\cU$; since 
$\cM(\cuA,\cuB)$ is an oriented bounded manifold, a standard 
partition of unity argument, then extends the result to the whole manifold. 

Consider the subset of critical points contained in $\cU$, and define
\begin{align} \label{eq:coordinate_chart_zero_set}
    Z_{0,n}{(\cU)} 
    :=
    \left|
        \left\{
            (\bTheta,\bbV)\in \cU :
            \bzeta = \bzero,\;
            \bH\in\sS_{\succeq0}^{m_n}
        \right\}
    \right|.
\end{align}
The objective is to show that
\begin{equation}
    \E[Z_{0,n}(\cU)]
    =
    \int_{(\bTheta,\bbV)\in\cU}
    \E\!\left[
        \big|\det(\de \bz(\bTheta,\bbV))\big|\,
        \one_{ \bH\succeq 0 }
        \,\big|\,
        \bz(\bTheta,\bbV)=\bzero
    \right]
    p_{(\bTheta,\bbV)}(\bzero)\;
    \de_\cM V .
\end{equation}
Let $\bg := \bz \circ \bpsi^{-1}$ and 
$\bh := \bH \circ \bpsi^{-1}$ denote the local coordinate representations of $\bz$ and 
$\bH$ on the chart domain.  
The proof proceeds by applying Theorem~\ref{thm:kac_rice} to the random fields $\bg$ and $\bh$ on $\bpsi(\cU)$.
\\ \noindent
\textbf{Step 1: Parametrizing the zero set using the chart.}
Recall that for any $(\bTheta,\bbV)\in\cU$, the vector $\bz(\bTheta,\bbV)$ 
represents the coordinates of $\bzeta(\bTheta,\bbV)$ in the column space of the covariance ${\bLambda(\bTheta,\bbV)}$ with respect to the 
orthonormal basis $\bB_{\bLambda(\bTheta,\bbV)}$.  
By Lemma~\ref{lemma:eig_vecs_NS_Sigma}, the mean of the process $\bzeta$ is 
orthogonal to the null space of $\bLambda$.  
As a result, the condition $\bzeta(\bTheta,\bbV)=\bzero$ in 
Eq.~\eqref{eq:coordinate_chart_zero_set} can be replaced by 
$\bz(\bTheta,\bbV)=\bzero$, which by definition is equivalent to requiring
$\bg(\bpsi(\bTheta,\bbV))=\bzero$.  
This shows that counting zeros in $\cU$ is equivalent to counting zeros of the transformed field in the coordinate space $\bpsi(\cU)$, giving
\begin{equation}
    Z_{0,n}{(\cU)}  
    \;=\;
    \Bigl|\Big\{\bs\in\bpsi(\cU) : 
    \; \; 
    \bg(\bs)=\bzero,\;
    \bh(\bs)\in \sS_{\succeq0}^{m_n} \Big\}\Bigr|.
    \label{eq:N0_definition}
\end{equation}
\\ \noindent
\textbf{Step 2: Applying the Kac–Rice formula.}
The next step is to invoke Theorem~\ref{thm:kac_rice} for the random fields $\bg$ and $\bh$. First, observe that $\bpsi(\cU)\subseteq \R^{m_n-r_k}$ is an open set by definition, and $\sS_{\succeq0}^{m_n}$ coincides with its closure. It remains to verify Conditions~\ref{thm:kac_rice_cond1}–\ref{thm:kac_rice_cond5}:
\\ \\ \noindent
\emph{Conditions~\ref{thm:kac_rice_cond1}, \ref{thm:kac_rice_cond2}, and \ref{thm:kac_rice_cond3}:}
By construction, $\bg$ and $\bh$ are jointly Gaussian random fields, so 
Condition~\ref{thm:kac_rice_cond1} is immediate.  
Condition~\ref{thm:kac_rice_cond2} follows from the regularity 
Assumptions~\ref{ass:loss} and~\ref{ass:regularizer}.  
Finally, Condition~\ref{thm:kac_rice_cond3} is ensured by the definition of 
$\bz$ together with Lemma~\ref{lemma:eig_vecs_NS_Sigma}.

\noindent
\emph{Condition~\ref{thm:kac_rice_cond5}:}
Let us define
\begin{align}
    \cE_0 
    &:=
    \left\{
    \exists \bs\in\bpsi(\cU)\;:\;
    \bg(\bs)=\bzero,\; \det\big(\bJ \bg(\bs)\big)=0
    \right\}.
\end{align}
By definition of $\bB_\Lambda$, the Gaussian field $\bg$ is non-degenerate 
everywhere on $\bpsi(\cU)$.  
Since $\overline{\bpsi(\cU)}$ is bounded by Lemma~\ref{lemma:manifold_dim}, the 
random vector $\bg(\bs)$ admits a density that is bounded in a neighborhood of 
$\bzero$, uniformly in $\bs\in\bpsi(\cU)$ (with a constant depending on 
$\sfA_R,\sfA_V$).  
Assumptions~\ref{ass:loss} and~\ref{ass:regularizer} further guarantee that 
$\bs \mapsto \bg(\bs)$ is $C^2$.  
Thus Lemma~\ref{lemma:proc_grad_as_0} applies to the compact set 
$\overline{\bpsi(\cU)}$, yielding $\P(\cE_0) = 0$,
which verifies Condition~\ref{thm:kac_rice_cond5}.

\noindent
\emph{Condition~\ref{thm:kac_rice_cond4}:}
Observe that $\partial\; \sS_{\succeq 0}^{m_n}=\{\bA\in\sS_{\succeq 0}^{m_n} : \det(\bA)=0\}$, and hence  to check Condition~\ref{thm:kac_rice_cond4}, it is enough to show that 
$\P(\widetilde\cE_0)=0$, where
\begin{equation}
    \widetilde\cE_0:=\left\{ \exists \bs \in \bpsi(\cU)  \;:\;
    \bg(\bs) =0,\;
    \det\big(\bh(\bs)\big)= 0 \right\}.
    \label{eq:determinant_boundry_set_kacrice}
\end{equation}
As noted in Equations~\eqref{eq:zeta_jacobian} and~\eqref{eq:det_projection} and 
Lemma~\ref{lemma:det_complement}, for $(\bTheta,\bbV)\in\cU$,
\begin{equation}
    |\det(\bH)| = |\det(\bJ\bzeta)| 
    = \big|\det(\de\bz)\,\det(\bB_{\bLambda^c}^\sT
    \bJ\bzeta \bB_{\bT^c})\big| 
    \;\;,\quad \big|\det(\bB_{\bLambda^c}^\sT
    \bJ\bzeta \bB_{\bT^c})\big|>0.
\end{equation}
Moreover, for any $\bs\in\bpsi(\cU)$, 
$\det(\de\bz\circ\bpsi^{-1}(\bs)) = \det(\bJ(\bs))$, and combining this with the previous display gives
\begin{equation}
    |\det(\bh)| = |\det(\bJ\bg)|\cdot C,
    \quad \text{for } 
    C := \big|\det\big(\bB_{\bLambda^c}^\sT
    \bJ\bzeta\circ\bpsi^{-1}
    \bB_{\bT^c}\big)\big|>0.
\end{equation}
Hence we obtain 
$\widetilde\cE_0 = \cE_0$, and therefore
$\P(\widetilde\cE_0) = \P(\cE_0) \;=\; 0$
which verifies Condition~\ref{thm:kac_rice_cond4}.

Since all the conditions of Theorem~\ref{thm:kac_rice} are met, we conclude
\begin{equation}
    \E[Z_{0,n}(\cU)] = \int_{\bpsi(\cU)}
    \E\big[|\det(\bJ \bg(\bs))|\one_{\bh(\bs)\succeq 0} \big|\; \bg(\bs)=\bzero\big] p_{\bg(\bs)}(\bzero) \de \bs,
\end{equation}
where $p_{\bg(\bs)}$ denotes the density of $\bg(\bs)$.
\\ \noindent
\textbf{Step 3: Pullback by $\psi$. }
Lastly, we use the chart $\bpsi$ to rewrite the integral over $\bpsi(\cU)$ as an integral over $\cU$ with respect to the volume form $\de_\cM V$.
For $\bs\in\bpsi(\cU)$ we have
\begin{align*}
    \E\left[|\det \bJ \bg(\bs)| \one_{\bh(\bs)\succeq 0} \big| \bg(\bs) = \bzero \right]=
    \left|\det\left( 
\de\bpsi^{-1}(\bs)\right)\right|
    \E\Big[&
    \left| \det \left( \de \bz(\bTheta,\bbV) \big|_{(\bTheta,\bbV) = \bpsi^{-1}(\bs)}\right)
    \right| \\
    &\one_{\bH\circ\bpsi^{-1}(\bs)
    \succeq \bzero } 
\big| 
 \bz(\bpsi^{-1}(\bs)) = \bzero 
    \Big].
\end{align*}
Changing variables via $\bpsi(\bTheta,\bbV) = \bs$ and noting once again $\bz = \bg\circ\bpsi$, and recalling the density defined in Lemma~\ref{lemma:density}, we have 
\begin{align*}
\E[Z_{0,n}(\cU)] &= \int_{\bs\in\bpsi(\cU)} 
    \E\left[
\big|\det \left( \de \bz(\bTheta,\bbV) \big|_{(\bTheta,\bbV) = \bpsi^{-1}(\bs)}\right)\big|
    \one_{\bh(\bs) \in \cH}
\big| 
 \bz(\bpsi^{-1}(\bs)) = \bzero 
    \right]\\
    & \quad\quad\qquad \times p_{\bg(\bs)}(\bzero) 
   \left|\det \left(\de\bpsi^{-1}(\bs)\right)\right|
    \de \bs\\
    &=\int_{(\bTheta,\bbV) \in \cU}  \E\Big[\left| \det (\de\bz(\bTheta,\bbV))\right|
    \;\one_{\bH\succeq 0}\;
    \Big| \bz (\bTheta,\bbV) = \bzero\Big] p_{\bz(\bTheta,\bbV)}(\bzero)  \de_\cM V
\end{align*}
as desired.
\end{proof}

\subsection{Integration over the manifold: Proof of Lemma~\ref{lemma:manifold_integral}}
\label{section:manifold_integration}
In this section, we upper bound the integral on the manifold appearing in Lemma~\ref{lemma:kac_rice_manifold} by an integral over a `$\beta$-blowup' of the manifold $\cM(\cuA,\cuB)$ as stated in Lemma~\ref{lemma:manifold_integral}.

Since the content of this lemma is more general than the
specific application we are interested in, we will consider 
a slightly more abstract setting outlined in Section \ref{sec:PreliminariesManifold}.

\subsubsection{Preliminaries}
\label{sec:PreliminariesManifold}

We consider an embedded smooth manifold $\cM\subset \R^{m_n}$, 
of co-dimension $r_k$  (see \cite{lee2012smooth} for geometry background),
which is defined by $\cM= \{\bu\in\cO:\, \bg(\bu) = \bzero\}$
for $\cO\subseteq \reals^{m_n}$ an open set and $\bg:\reals^{m_n}\to\reals^{r_k}$
a smooth map. 
 The tangent space at $\bu\in\cM$, denoted by $\mathrm{T}_\bu\cM$, is the linear subspace of $\R^{m_n}$ spanned by all velocity vectors $\gamma'(0)$ of smooth curves $\gamma:[0,1]\to\cM$ with $\gamma(0)=\bu$.
The normal space at $\bu$, denoted $\mathrm{N}_\bu\cM$, is the orthogonal complement of $\mathrm{T}_\bu\cM$ in $\R^{m_n}$.  The normal bundle is the disjoint union
$\rN\cM := \coprod_{\bu\in\cM} \mathrm{N}_\bu\cM$,
whose fiber at $\bu$ is the normal space at $\bu$. 

We assume $\bJ\bg(\bu)$ to be non-singular for all 
 $\bu\in\cM$, and define 
\begin{equation}\label{eq:normal_frame_basis}
    \bE(\bu) := \bJ\bg(\bu)^\sT \big(
    \bJ\bg(\bu)\bJ\bg(\bu)^\sT\;
    \big) ^{-1/2}\in\R^{m_n\times r_k}\, .
\end{equation}
As stated formally below (see Lemma \ref{lem:regular_level_set_thm}), the columns of $\bE(\bu)$ form an orthonormal basis of 
$\mathrm{N}_\bu\cM$.  
We will adopt the identification of the normal 
bundle
\begin{equation}
    \rN\cM\cong  \cM\times \R^{r_k},
\end{equation}
under which a normal vector at $\bu$ is identified with its coordinate vector in 
$\R^{r_k}$ with respect to the frame $\bE(\bu)$. Note that this identification is
not always possible, but it is possible `locally,' i.e. by eventually 
choosing a smaller open set  $\cO$. 

The normal exponential map 
$\exp^\perp: {\cM}\times \R^{r_k}\to\R^{m_n}$ is defined by moving from a point $\bu\in{\cM}$ along a normal direction 
specified by $\bx\in\R^{r_k}$; namely
\begin{equation}
    \exp^\perp(\bu,\bx) :=\bu+\bE(\bu)\bx.
\end{equation}
A \emph{normal tubular neighborhood} is a “thickened version” of $\cM$ formed by collecting all points that admit a unique projection onto $\cM$ and are within some small normal distance from $\cM$.
Formally, letting
\begin{equation}
     \mathrm {N}^{(\tau)} {\cM}:= \cM\times \Ball_{\tau}^{r_k}(\bzero)
\end{equation}
 be the \emph{normal disk bundle} of radius $\tau>0$, the normal tubular neighborhood of ${\cM}$ is the set
\begin{equation}
    \mathrm{Tub}^{(\tau)}\cM : = \exp^\perp\big(\mathrm {N}^{(\tau)} {\cM} \big),
\end{equation}
 given that $\tau$ is small enough so that 
$\exp^\perp$ restricted to $\mathrm{N}^{(\tau)}\cM$ is a diffeomorphism onto its image. 
 
Recall from Section~\ref{sec:integration_over_manifold} that for $\tau\geq0$, the 
$\tau$-blowup $\cM^{(\tau)}\subseteq \R^{m_n}$ is the neighborhood of $\cM$ 
consisting of all points at Euclidean distance at most $\tau$ from the manifold.  
By definition, for $\tau$ small enough we immediately have $\mathrm{Tub}^{(\tau)}\cM
{\subseteq}
\cM^{(\tau)}$.

Finally, we write $\de_\cM V$ for the volume element on $\cM$, and use $\de\bx$ to denote the Lebesgue measure when the ambient space is clear from context. Using this notation, we will define the following quantities for $\tau >0$,
\begin{align} \label{eq:lower_bound_curvature_constants}
    \sfA_{g,2}^\up{\tau} & := 
    \sup_{ \bu \in \cM^\up{\tau}} \;\;
    \sup_{\bx,\by\in\Ball_1^{m_n}(\bzero)}\;
    \big\|\big(
    \bx^\sT\grad^2 g_j(\bu)\by\big)_{j\in[r_k]}
    \; \big\|_2,\\
    \sfA_{g,2} &:= 
    \sup_{ \bu \in \cM } \;\; \sup_{\bx,\by\in\Ball_1^{m_n}(\bzero)}\big\|\big(
    \bx^\sT\grad^2 g_j(\bu)\by\big)_{j\in[r_k]}
    \big\|_2 ,\\
    \sfa_g &:=  \inf_{\bu \in\cM} \sigma_{\min}(\bJ_{\bu} \bg(\bu)^\sT)>0\,.
\end{align}

We will first give a lower bound on the maximum possible radius of the normal tubes, and use $\exp^\perp$ as a global coordinate chart for the corresponding normal disk bundle. This parametrization will allow us to uniformly control the error introduced by approximating the integrals over $\cM$ by integrals over its blowup.

\subsubsection{Deriving a lower bound for the radius of the maximal normal tube}
Our main tools in this section are a generalization of Federer’s result on the 
radius of normal neighborhoods \cite{federer1959curvature} to non-compact 
manifolds, together with a direct consequence of the regular level set Theorem (Theorem 9.9 in \cite{tu2011manifolds}) that characterizes the normal bundle of $\cM$.

\begin{lemma}[Normal frame]\label{lem:regular_level_set_thm}
For every $\bu \in \cM$, the tangent and normal spaces satisfy
    \begin{equation}
        \mathrm{T}_\bu \cM = \ker\!\big(\bJ\bg(\bu)\big),
        \qquad
        \mathrm{N}_\bu \cM = \mathrm{Img}\!\big(\bJ\bg(\bu)^\sT\big).
    \end{equation}
    Moreover, the map $\bE$ defined in 
    Eq.~\eqref{eq:normal_frame_basis} is smooth, and for each 
    $\bu\in\cM$, its columns form an orthonormal basis of 
    $\mathrm{N}_\bu \cM$.
\end{lemma}
\begin{proof}
The statement follows directly from the regular level set theorem.  
In fact, as in the proof of Lemma~\ref{lemma:manifold_dim}, $\cM$ is a bounded regular 
level set of the smooth map $\bg$, so Lemma 9.10 in \cite{tu2011manifolds} applies.  
The conclusion then follows immediately from the definition of $\bE$ in Eq.~\eqref{eq:normal_frame_basis}.
\end{proof}

\begin{lemma}[Modification of Theorem 4.18 in \cite{federer1959curvature}]
\label{lem:Federer_reach_formula}
For $\bu\in\cM$, let $\bP_{\rN_\bu\cM}$ be the orthogonal projector onto $\rN_\bu\cM$. Let
\begin{equation}
    \tau_\cM:=     \inf_{\bu_1 \in \cM}
    \;\;\inf_{\substack{
    \bu_2 \in \cM \setminus \{\bu_1\},\\
    \bu_2-\bu_1\notin\rT_{\bu_1}\cM}}
    \;\;\frac{\|\bu_2 - \bu_1\|_2^{\,2}}
         {2\,\|\bP_{\rN_{\bu_1}\cM}(\bu_2-\bu_1)\|_2},
\end{equation}
and assume $\tau_\cM>0$. Then for any $0\leq \tau<\tau_\cM $, 
the normal exponential map $\exp^\perp$ restricted to $\rN^{(\tau)}\cM$ is 
injective.
\end{lemma}
\begin{proof}
Let us define $\rd^\perp:\cM\times\cM\to \R$ by 
\begin{align*}
\rd^\perp(\bu_1,\bu_2):
=\frac
{\|\bu_1-\bu_2\|_2^{\,2}}{2\,\|\bP_{\rN_{\bu_1}\cM}(\bu_2-\bu_1)\|_2}\, . 
\end{align*}
The proof follows the tangent–ball argument used in \cite{federer1959curvature}. In particular, we show that $\rd^\perp(\bu_1,\bu_2)$ is the radius of the largest ball tangent to $\cM$ at $\bu_1$ whose boundary contains $\bu_2$.  
Thus, ensuring $\tau < \rd^\perp(\bu_1,\bu_2)$ for every distinct 
$\bu_1,\bu_2\in\cM$ guarantees that the normal tube of radius $\tau$ does not ``self-intersect''.

Formally, suppose 
$(\bu_1,\bx_1),(\bu_2,\bx_2)\in \rN^{(\tau)}\cM$ satisfy  
$\exp^\perp(\bu_1,\bx_1) = \exp^\perp(\bu_2,\bx_2).$
We show that, for $\tau<\rd^\perp(\bu_1,\bu_2)$ this forces $(\bu_1,\bx_1)=(\bu_2,\bx_2)$, proving that 
$\exp^\perp$ is injective on $\rN^{(\tau)}\cM$. 

We will prove this by contradiction. 
Assume $(\bu_1,\bx_1) \neq (\bu_2,\bx_2).$
The case of $\bu_1=\bu_2$ is straightforward since we must have $\bx_1=\bx_2$.  
So it's sufficient to assume that $\bu_1\neq \bu_2$. Without loss of generality, suppose
$
\|\bE(\bu_1)\bx_1\|_2 \;\ge\; \|\bE(\bu_2)\bx_2\|_2 .
$
For brevity let us use
$r := \|\bE(\bu_1)\bx_1\|_2$ and $
\bw := \exp^\perp(\bu_1,\bx_1)$.

Since by assumption $\bw = \bu_2 + \bE(\bu_2)\,\bx_2$, we have
$ \|\bu_2 - \bw\|_2
= \|\bE(\bu_2)\bx_2\|_2
\le r$.
Thus $\bu_2$ lies in the closed ball $\overline{\Ball_r^{m_n}(\bw)}$,   
and expanding the squared distance gives
\begin{equation}
    \|\bu_2 - \bw\|_2^2
=
\|\bu_2 - \bu_1\|_2^2
+ r^2
- 2(\bu_2-\bu_1)^\sT \big(\bE(\bu_1)\bx_1\big)
\;\le\; r^2.
\end{equation}
Rearranging,
\begin{equation}\label{eq:inj_distance_rearranged}
 \frac{\|\bu_2-\bu_1\|_2^2}
     {2\,(\bu_2-\bu_1)^\sT(\bE(\bu_1)\bx_1)/r}
\;\le\; r.   
\end{equation}
Since   
$\bE(\bu_1)\bx_1\in \rN_{\bu_1}\cM$ by 
Lemma~\ref{lem:regular_level_set_thm}, we also have
\begin{equation} 
\|\bP_{\rN_{\bu_1}\cM}(\bu_1-\bu_2)\|_2
\geq \frac{(\bu_2- \bu_1)^\sT\big(\bE(\bu_1)\bx_1\big)}
{\|\bE(\bu_1)\bx_1\|_2}.
\end{equation}
Combining above with Eq.~\eqref{eq:inj_distance_rearranged} and using $r = \|\bE(\bu_1)\bx_1\|_2 \le \|\bx_1\|_2\le \tau$   yields
\begin{equation}
    \rd^\perp(\bu_1,\bu_2) \;\le\; r \;\le\; \tau.
\end{equation}
Since $\bu_1,\bu_2 \in\cM$ with $\bu_1\neq \bu_2$, this
contradicts the assumption of the lemma on $\tau$.  
Hence $(\bu_1,\bx_1)=(\bu_2,\bx_2)$, proving injectivity.
\end{proof}

\begin{lemma}[Injectivity of normal exponential map]
\label{lem:lower_bound_manifold_reach}
For any $0 < \tau$ such that
\begin{equation}
    \tau \;<\;
    \sup_{\tilde \tau>0} \bigg\{
    \tilde \tau \wedge 
    \frac{\sfa_{g}}{2\,\sfA_{g,2}^{(\tilde\tau)}}\bigg\},
\end{equation}
the normal exponential map 
$\exp^\perp$ restricted to $\rN^{(\tau)}\cM$ is injective.
\end{lemma}

\begin{proof}
     Fix two distinct points $ \bu_1\bu_2\in\cM$, and for notational convenience set $\bw:=\bu_1-\bu_2$. Let $\bP_{\rN_{\bu_1}\cM}$ be the orthogonal projector onto $\rN_{\bu_1}\cM$. 
     Let us first fix a constant $\delta>0$ and observe that if $\|\bw\|_2 \ge \delta$,
\begin{equation}
    \frac{\|\bw\|_2^2}{\| \bP_{\rN_\bv\cM}(\bw) \|_2}
    \geq \frac{\|\bw\|^2_2}{\|\bP_{\rN_\bv\cM}\|_{\op}\cdot\|\bw\|_2}
    \geq\delta.
\end{equation}
For the case $\|\bw\|_2 \le \delta$, a Taylor expansion of $\bg$ along the segment 
$t \mapsto \bu_1 + t\bw$ gives
\begin{align*} 
\nonumber \| \bJ\bg(\bu_1)\bw\|_2 \stackrel{(a)}{=}& \|\bg(\bu_1) - \bg(\bu_2) + \bJ\bg(\bu_1)\bw\|_2 \\
= &\Big\|\Big(\int_0^1(1-t)\bw^\sT\,
\nabla^2 g_j(\bu_1+t\bw)\,\bw \;\de t \Big)_{j\in[r_k]}\Big\|_2\\ 
\stackrel{(b)}{\le}& 
\sup_{\bx\in\cM^{(\delta/2)}} 
\sup_{\by\in \Ball_1^{m_n}(\bzero)}
\big\|(\by^\sT\; \nabla^2 g_j(\bx)\;\by)_{j\in[r_k]} \big\|_2 
\cdot\|\bw\|_2^2 \; = \; \sfA_{g,2}^{(\delta/2)}
\; \|\bw\|_2^2, 
\end{align*}
where in $(a)$ we used $\bg=\bzero$ on $\cM$, and in $(b)$ we used 
$\inf_{\bv\in\cM}\|\bv-(\bu_1+t\bw)\|\le \|\bw\|_2/2 \le \delta/2$ for all $t\in[0,1]$. 

Next, from Lemma~\ref{lem:regular_level_set_thm} we observe that $\bP_{\rN_{\bu_1}\cM}=\bE(\bu_1)\bE(\bu_1)^\sT$, which combined with the definition of $\bE(\bu_1)$ in Eq.~\eqref{eq:normal_frame_basis} yields
\begin{align} 
\|\bP_{\rN_{\bu_1}\cM}\,(\bw)\|_2 =&
\|\bE(\bu_1) \;\big(\bJ\bg(\bu_1)\bJ\bg(\bu_1)^\sT\big)^{-1/2}\; \bJ\bg(\bu_1)\bw\|_2\\
\le&\frac{\|\bE(\bu_1)\|_\op}{\sigma_{\min}(\bJ\bg(\bu_1))} \cdot\|\bJ\bg(\bu_1)\bw\|_2 
\le \frac{1}{\sfa_g}\cdot \sfA_{g,2}^{(\delta/2)}\|\bw\|_2^2, 
\end{align}
where we used
$\|\bE(\bu_1)\|_\op = 1$ along with
$\sfa_{g} \le \sigma_{\min}(\bJ\bg(\bu_1)^\sT) $.

Thus,
\begin{equation}
    \frac{\|\bu_2-\bu_1\|_2^2}{2\|\bP_{\rN_{\bu_1}\cM} (\bu_2-\bu_1)\|_2}
    \;\ge\;
    \frac{\delta}{2} \;\wedge\;
    \frac{\sfa_{g}}{2\sfA_{g,2}^{(\delta/2)}}.
\end{equation}
Applying Lemma~\ref{lem:Federer_reach_formula} with $\tau=\delta/2$, and taking supremum over the arbitrary fixed $\delta >0$
completes the proof.
\end{proof}

\subsubsection{Tube content}
In this section, we upper bound the surface integral $\int_\cM f \de_\cM V$ by replacing integration over $\cM$ with integration over its normal tube. To control the resulting error, we first obtain bounds on the volume distortion induced by the change of coordinates. Once the distortion is quantified, we use the pushforward of the measure under $\exp^\perp$ to carry out the change of variables. Throughout, let us fix the radius 
\begin{equation}
    \tau_{\cM,\star} =1 \wedge     \sup_{\tilde \tau>0} \bigg\{
    \tilde \tau \wedge 
    \frac{\sfa_{g}}{2\,\sfA_{g,2}^{(\tilde\tau)}}\bigg\}.
\end{equation}

\begin{lemma}[Volume factor]\label{lem:local-diffeomorphis}
Assume 
$0<\tau < \tau_{\cM,\star}
\wedge (\sfa_g/\sfA_{g,2})$. 
Then, for all $(\bu,\bx)\in\rN^{(\tau)}\cM$ we have the bound
\begin{equation}
\label{eq:normal_exponential_determinant_lower_bound}
       |\,\det(\de \exp^\perp(\bu,\bx))\,| \ge \Big(1 - \tau\;\frac{\sfA_{g,2}}{\sfa_g}\Big)^{m-r_k},
\end{equation}
and the normal exponential map $\exp^\perp$ forms a diffeomorphism from $\rN^{(\tau)}\cM$ onto $\mathrm{Tub}^{(\tau)}\cM$.
\end{lemma}

\begin{proof} 
To lower bound the determinant, we show that the differential
$\de \exp^\perp(\bu,\bx) : \rT_{(\bu,\bx)}\rN^{(\tau)}\cM \rightarrow \rT_{\exp^\perp(\bu,\bx)}\mathrm{Tub}^{(\tau)}\cM$ is a low-rank perturbation of a matrix that is close to the identity through a careful parametrization of the domain and codomain tangent spaces. The (Euclidean) Jacobian of the map $\exp^\perp$, extended to be defined on an open subset of $\R^{m_n+r_k}$, is given by
  \begin{equation}
      \bJ_{\bu,\bx}\exp^\perp(\bu,\bx)= 
      \begin{bmatrix}
          \bI_{m_n} + 
          \bJ_\bu \bE(\bu)^\sT\bx \;\;,
          & 
          \;\;
          \bE(\bu)
      \end{bmatrix}
      \in \R^{m_n\times(m_n+r_k)}.
  \end{equation}
\noindent 
\textbf{Step 1: basis of $\rT_{(\bu,\bx)}\rN^{(\tau)}\cM$. } Let $\bT\in\R^{m_n\times(m_n-r_k)}$ be an orthonormal basis of $\rT_\bu\cM$. Since $\rT_{(\bu,\bx)}\rN^{(\tau)}\cM$ is isomorphic to $\rT_{\bu}\cM\oplus \R^{r_k}$, we choose the orthonormal basis

\begin{equation}\bB := 
    \begin{bmatrix}
        \bT& \bzero_{m_n\times r_k}\\
        \bzero_{r_k\times(m_n-r_k)}& \bI_{r_k} 
    \end{bmatrix}\in\R^{(m_n+r_k)\times(m_n)},
\end{equation}
where we omit the dependence on $(\bu,\bx)$ for brevity. Since $\bB$ is an orthonormal basis of $\rT_{(\bu,\bx)}\rN^{(\tau)}\cM$ and $\rT_{\exp^\perp(\bu,\bx)}\mathrm{Tub}^{(\tau)}\cM\cong\R^{m_n}$, we have
\begin{equation}
\label{eq:det_dphi_identification}
|\det(\de \exp^\perp)|=|\det(\bA)|,\quad \text{where}\quad
    \bA :=  \bJ_{\bu,\bx} \exp^{\perp}(\bu,\bx)\cdot\bB(\bu,\bx)
    \in\R^{m_n\times m_n}.
\end{equation}

\noindent
\textbf{Step 2: basis of $\rT_{\exp^\perp(\bu,\bx)}\mathrm{Tub}^{(\tau)}\cM$.}
We next change the codomain basis from the standard basis of $\R^{m_n}$ to an orthonormal tangent–normal frame.  
Since $\rT_\bu\cM = (\rN_\bu\cM)^\perp$, the matrix
\begin{equation}
    \bO := \begin{bmatrix}
        \bT\;\;, &\;\; \bE
    \end{bmatrix}\in\R^{m_n\times m_n}
\end{equation}
is orthonormal basis of $\rT_{\exp^\perp(\bu,\bx)}\mathrm{Tub}^{(\tau)}\cM$, and hence $|\det(\bA)| = |\det(\bO^\sT \bA)|$. A direct block computation yields
\begin{equation}
   \bO^\sT\bA =  \begin{bmatrix}
        \bI_{m_n-r_k} - \bS(\bu,\bx)& 
        \;\;\; \bzero_{(m_n-r_k)\times r_k}\\
        \star & \bI_{r_k}
    \end{bmatrix},
\end{equation}
where for brevity we used $\bS(\bu,\bx):=-\bT^\sT(\bJ\bE(\bu)^\sT \bx) \bT$. This gives us a naive lower bound 
\begin{equation}\label{eq:detA_to_shape_operator}
    |\det(\bO^\sT \bA)| = |\det\big(\bI -\bS(\bu,\bx)\big)| \ge (1-\|\bS(\bu,\bx)\|_\op)^{m_n-r_k}.
\end{equation}
\noindent
\textbf{Step 3: Bounding the principle curvature.} Lastly we bound the norm of the operator $\bS(\bu,\bx)$. To simplify the notation, let us define $\bG_0= (\bJ\bg(\bu)\bJ\bg(\bu)^\sT)^{-1/2}$ and use $\bG_{0,j}$ to denote the $j$'th column of $\bG_0$.
By Lemma~\ref{lem:regular_level_set_thm}, the gradient vector $\nabla g_j(\bu)$ lies in 
the normal space $\rN_\bu\cM$ for all $j\in[r_k]$, and in particular $\bT^\sT \grad g_j(\bu)=\bzero_{m_n-r_k}$. 
Combined with the definition of $\bE(\bu)$ in Eq.~\eqref{eq:normal_frame_basis} 
and an application of the chain rule, this identity allows us to obtain

\begin{align*}
    \|\bS(\bu,\bx)\|_\op &= \|\sum_{j=1}^{r_k}\bT^\sT
    \;\bJ_\bu\big(
    \nabla g_j(\bu) \cdot \bG_{0,j}^\sT\bx
    \big)
    \; \bT\|_\op\\
    &=
    \Big\|\sum_{j=1}^{r_k}\bT^\sT
    \nabla^2 g_j(\bu) \cdot (\bG_{0,j}^\sT\bx)\bT\Big\|_\op\\ 
    &\le 
    \|\bT\|_\op^2 \cdot
    \|\sum_{j=1}^{r_k}
    \nabla^2 g_j(\bu)\cdot(\bG_{0,j}^\sT\bx) \|_\op
    \;\le \; \frac{\sfA_{g,2}}{\sfa_g}\;\tau,
\end{align*}
where in the last inequality we used  Cauchy–Schwarz to conclude 
$\|\sum_{j=1}^{r_k}
    \nabla^2 g_j(\bu)\cdot(\bG_{0,j}^\sT\bx) \|_\op \leq \sfA_{g,2}\|\bG_0\bx\|_2 $ along with
$\|\bG_0\|_\op \le 1/\sigma_{\min}(\bJ\bg(\bu)^\sT)\le 1/\sfa_g$. Combining with Equations~\eqref{eq:det_dphi_identification} and ~\eqref{eq:detA_to_shape_operator} concludes the desired lower bound on the 
determinant of Eq.~\eqref{eq:normal_exponential_determinant_lower_bound}.
Finally, the diffeomorphism property of $\exp^\perp$ follows from the injectivity established in Lemma~\ref{lem:lower_bound_manifold_reach} together with the 
inverse function theorem.
\end{proof}

\begin{lemma}[Blow-up content]
\label{lem:intg-tube} 
Let $f:\cM^{(1)}\rightarrow \R$ be a differentiable nonnegative function. Assume
$0<\beta$ satisfies $\beta < 1 \wedge \tau_{\cM,\star}$.
Then 
\begin{equation}
\int_{ \cM} f(\bu) \de_\cM V(\bu)
\le  
\Err_{\sblowup}(n,\beta)\cdot
\exp\{\beta\;\norm{\log f}_{\Lip,\cM^{(\beta)}}\}
\int_{\by\in \cM^{(\beta)}}
f(\by)\de\by
\end{equation}
    where
    \begin{equation}
       \Err_{\sblowup}(n,\beta)  :=\bigg(\frac1{1 - \beta\;\sfA_{g,2}/\sfa_g}\bigg)^{m_n-r_k}
        \left(\frac{\sqrt{r_k}}{\beta}\right)^{r_k}. 
    \end{equation}
\end{lemma}

\begin{proof}
Fix $\beta$ to satisfy the condition of the lemma, and define 
$\bpi : \rN^{(\beta)}\cM \to \cM$ 
to be the normal projection onto the manifold.
The proof proceeds via the two maps
\begin{equation*}
    \cM 
    \xrightarrow{\;\bpi^{-1}\;}
    \rN^{(\beta)}\cM 
    \xrightarrow{\;\exp^\perp\;}
    \mathrm{Tub}^{(\beta)}\cM,
\end{equation*}
together with a bound on how the values of $f$ change when its domain is enlarged from $\cM$ to its normal tube. 


\noindent
\textbf{Step 1: from $\cM$ to $\rN^{(\beta)}\cM$.}
Applying Fubini’s Theorem gives
    \begin{align}
        \label{eq:manifold_int_lemma_bound_1}
        \int_{\cM} f(\bu)\de_\cM V(\bu) & 
        = \frac{1}{\int_{\bx\in\Ball_\beta^{r_k}(\bzero)} \one \; \de\bx} \cdot
        \int_{\bu\in\cM} 
        \int_{\bx\in\Ball_\beta^{r_k}(\bzero)} 
        f(\bu)\; \de_\cM V(\bu)\,\de\bx\\ 
        \nonumber
        & \stackrel{(a)}{=} 
        \frac{1}{\vol(\Ball^{r_k}_{\beta}(\bzero))} \cdot
         \int_{\by\in \rN^{(\beta)}\cM } f(\bpi(\by)) 
        \; \de_{\rN^{(\beta)}\cM}V(\by)\\
        \nonumber
         &\stackrel{(b)}{\le}  \left(\frac{\sqrt{r_k}}{\beta}\right)^{r_k} 
        \int_{\by\in \rN^{(\beta)}\cM } f(\bpi(\by)) 
        \; \de_{\rN^{(\beta)}\cM}V(\by),
    \end{align}
   where in $(a)$ we used the identity  $\rN^{(\beta)}\cM=\cM\times \Ball_{\beta}^{r_k}(\bzero)$ and the product structure  $\de_{\cM\times \R^{r_k}}V = \de_\cM V\times \de_{\R^{r_k}} V$, and in $(b)$ we bounded the Lebesgue volume by $\vol(\Ball^{r_k}_{\beta}(\bzero))^{-1} \le \left(\sqrt{r_k}/{\beta}\right)^{r_k}$.
   
\noindent
\textbf{Step 2: from $\rN^{(\beta)}\cM$ to $\mathrm{Tub}^{(\beta)}\cM$. }
For the choice of $\beta\le \tau_{\cM,\star}$, Lemma~\ref{lem:Federer_reach_formula} implies that the normal tube $\mathrm{Tub}^{(\beta)}\cM$ is well-defined, and furthermore the exponential map $\exp^\perp$ is a diffeomorphism from $\rN^{(\beta)}$ onto $\mathrm{Tub}^{(\beta)}\cM$. Then, using the uniform upper bound on
$\big|\det\big(\de\,( \exp^{\perp})^{-1}\big)\big| = 
\big|\det  \left(\de\exp^\perp\right)^{-1}\big|$ from Lemma~\ref{lem:local-diffeomorphis}, we can bound the integral over $\rN^{(\beta)}\cM$ in the previous display as
  \begin{align}
  \nonumber
        \int_{ \rN^{(\beta)}\cM} f(\bpi(\by)) \de_{\rN^{(\beta)}\cM} V(\by)
        & \stackrel{(c)}{=} 
        \int_{\bz\in \mathrm{Tub}^{(\beta)}\cM} 
        \big| 
        \det\big(\de(\exp^\perp)^{-1}(\bz) 
        \big)\big|\cdot
        f(\bpi \circ (\exp^\perp)^{-1}(\bz))\, \de\bz\\
        &{\le}
        \bigg(\frac1{1 - \beta\;\sfA_{g,2}/\sfa_g}\bigg)^{m-r_k}
        \int_{\bz\in \mathrm{Tub}^{(\beta)}\cM} 
        f(\bpi \circ (\exp^\perp)^{-1}(\bz)) \; \de\bz,
        \label{eq:manifold_int_lemma_bound_2}
  \end{align}
where in $(c)$ we used $\cM^{\beta} = \exp^{\perp}\big(\rN^{(\beta)}\cM\big)$ and $\de_{\cM^{(\beta)}} V =\de_{\R^{r_k}}V$.

\noindent
\textbf{Step 3: Bounding the normal perturbation.}
The map $\bpi \circ (\exp^\perp)^{-1}$ sends any point $\bz\in \mathrm{Tub}^{(\beta)}\cM$ to its normal projection onto $\cM$.  
Since $0<\beta \le \tau_\cM$, this projection is uniquely defined for all $\bz \in \mathrm{Tub}^{(\beta)}\cM$. The perturbation is controlled by the normal displacement inside the tube, which is at most $\beta$.  Formally, for every $\bz \in \mathrm{Tub}^{(\beta)}\cM$, there exists a unique 
$(\bu,\bx)\in \rN^{(\beta)}\cM$ such that $\bz = \exp^\perp(\bu,\bx)$, and
\begin{equation}
    \|\bz - \bpi \circ (\exp^\perp)^{-1}(\bz)\|_2
    = \|\bE(\bu)\bx\|_2
    \;\le\; \|\bE(\bu)\|_{\op}\,\|\bx\|_2
    \;\le\; \beta .
\end{equation}
Since $\log f$ is Lipschitz on the bounded set $\cM^{(\beta)}$, we obtain
\begin{equation}
    \big|\log f(\bz) - 
           \log f\!\big(\bpi \circ (\exp^\perp)^{-1}(\bz)\big)\big|
    \;\le\;
    \beta \,\|\log f\|_{\Lip,\cM^{(\beta)}} .
\end{equation}
Exponentiating and using the nonnegativity of $f$ yields
\begin{equation}
        \int_{\bz\in \mathrm{Tub}^{(\beta)}\cM} 
        f(\bpi \circ (\exp^\perp)^{-1}(\bz))
        \;\de\bz \le
        \exp \left\{\beta \|\log f\|_{{\Lip},\cM^{(\beta)}}\right\}
        \int_{\bz\in \mathrm{Tub}^{(\beta)}\cM} f(\bz)\de\bz.
\end{equation}
Combining with the bounds in Eq.~\eqref{eq:manifold_int_lemma_bound_1} and
Eq.~\eqref{eq:manifold_int_lemma_bound_2}, together with the inclusion $\mathrm{Tub}^{(\beta)}\cM \subseteq\cM^{(\beta)} $ yields the Lemma.
\end{proof}

\subsubsection{Proof of Lemma~\ref{lemma:manifold_integral}}

We now specialize the above results to the case of
interest to prove Lemma~\ref{lemma:manifold_integral}.
In this context, $m_n := nk +nk_0 + dk$, $r_k := k(k+k_0)$, and we identify 
$\R^{d\times k}\times \R^{n\times (k+k_0)} \cong  \R^{m_n}, 
\R^{k\times (k+k_0)}\cong \R^{r_k}$.  
Let $\overline\bG:\R^{m_n}\to\R^{r_k}$  denote the vectorized representation of
 $\bG:\R^{d\times k}\times \R^{n\times (k+k_0)} \to\R^{r_k}$ under this identification. Define the constants $\sfa_{\overline G}, \sfA_{\overline G,2}^{(\tau)}$, and $\sfA_{\overline G,2}^{(\tau)}$ according to Eq.~\eqref{eq:lower_bound_curvature_constants} with the choice of $\bg=\overline\bG$.
Then by construction, $\sfa_{\overline G}$ coincides with $\sfa_{G,n}$ defined in Section~\ref{sec:definitions}.

We will show that there exists a constant $C(\sfA_V,\sfA_R)>0$, depending only on 
$\sfA_V$ and $\sfA_R$, such that
\begin{equation}
    1 \wedge \frac{C(\sfA_V,\sfA_R)\,\sfa_{G,n}}{r_k^2}
    \;\le\;
    \sup_{\tau>0}
    \Big\{\tau \wedge \frac{\sfa_{G,n}}{2\sfA_{\overline G,2}^{(\tau)}}\Big\},
\end{equation}
and hence, by Lemma~\ref{lem:intg-tube}, desired inequality in Eq.~\eqref{eq:Lemma2_integration_over_blowup_form} will follow.

By differentiating the components of $\overline G _j$, $j\in[r_k]$, and using the Lipschitz 
assumptions from Assumption~\ref{ass:loss} on the partial derivatives of $\ell$, 
together with the continuity assumptions from 
Assumption~\ref{ass:regularizer} on the partial derivatives of $\rho$, 
one verifies that there exists a constant 
$C_0(\sfA_V,\sfA_R)>0$ (depending only on $\sfA_V$ and $\sfA_R$) such that
\begin{equation}
\label{eq:from_reach_to_constants}
    \sup_{\bu\in\cM^{\mathrm{(1)}}}
    \max_{j\in[r_k]}
    \|\grad^2\; \overline G _j(\bu)\|_\op
    \;\le\;
    C_0(\sfA_V,\sfA_R)\cdot r_k .
\end{equation}
By definition of $\sfA_{\overline G,2}^{(\tau)}$, this implies
\begin{equation}
    1 
    \;\le\;
    \sfA_{\overline G,2}
    \;\le\;
    \sfA_{\overline G,2}^{(1)}
    \;\le\;
    r_k^2 \, C_0(\sfA_V,\sfA_R).
\end{equation}
Hence, there exists a constant $C_1(\sfA_V,\sfA_R)>0$ such that
\begin{equation}
    \frac{C_1(\sfA_V,\sfA_R)\, \sfa_{G, n}}{r_k^2} 
    \;\le\;
    \frac{\sfa_{G, n}}{2 \sfA_{\overline G,2}^{(1)}},
\end{equation}
which establishes Eq.~\eqref{eq:from_reach_to_constants}.
Hence we conclude that for any sequence $\{\beta_n\}_n$ satisfying 
$ \beta_n\leq {C_1(\sfA_V,\sfA_R)\, \sfa_{G, n}}/{r_k^2},$
the inequality of Eq.~\eqref{eq:Lemma2_integration_over_blowup_form} holds. 

It remains to verify that 
$    \lim_{n\to\infty}  \frac1n \log \Err_{\sblowup}(n, \beta_n) = 0.$ Using the explicit expression of $\Err_{\sblowup}(n,\beta_m)$ in Lemma~\ref{lem:intg-tube}, we get
\begin{align}
    |\frac1n \log \Err_{\sblowup}(n,\beta_n)|
     \le&  \bigg|\frac{m_n-r_k}{n} \log\bigg(1-\beta_n\; 
     \frac{\sfA_{\overline G,2}}
     {\sfa_{G,n}}\bigg)\bigg|
     +\bigg|\frac{r_k}{n}
     \log\bigg(\frac{\sqrt{r_k}}{\beta_n}\bigg)\bigg|.
\end{align}
The above upper bound converges to $0$ as $n\to\infty$ provided that $\log(\beta_n)/n
\to 0 $ and 
$\beta_n\cdot \sfA_{\overline G,2}/\sfa_{G,n} \to 0$. This completes the proof of Lemma~\ref{lemma:manifold_integral}.

\qed

\subsection{Bounding the density of the gradient process: Proof of Lemma~\ref{lemma:density_bounds}}
\label{sec:density_bound}

We begin by computing the density $p_{\bTheta,\bbV}(\bzero)$ appearing in Lemma~\ref{lemma:kac_rice_manifold}.

\begin{lemma}[Density] 
\label{lemma:density}
Let $\bB_{\bLambda}$ be as in Corollary~\ref{cor:proj}.
For $(\bTheta,\bbV)\in \cM$, the density of $\bz(\bTheta,\bbV) = \bB_{\bLambda(\bTheta,\bV)}^\sT\,
\bzeta(\bTheta,\bbV)$ at $\bzero$ is given by 
\begin{equation}
\label{eq:density}
  p_{\bTheta,\bbV}(\bzero) 
    :=
\frac{
    \exp\left\{-\frac{1}2\left(
    n^2\Tr\left(\bRho (\bL^\sT\bL)^{-1}\bRho\right) + \Tr(\bbV \bR^{-1}\bbV) + n\Tr\left(\bRho (\bL^\sT\bL)^{-1}\bL^\sT \bbV \bR^{-1}[\bTheta,\bTheta_0]^\sT\right)
    \right)\right\}}
    {
\det^*(2\pi\bLambda(\bTheta,\bbV))^{1/2}
    }
\end{equation}
where $\det^*$ denotes the product of the non-zero eigenvalues.
\end{lemma}
\begin{proof}
In what follows, let us suppress the argument $(\bTheta,\bbV)$ throughout.
Recall Lemma~\ref{lemma:eig_vecs_NS_Sigma} giving the mean and covariance of $\bzeta$.
What we need to show is that the quantity multiplying the factor $-1/2$ in the exponent of~\eqref{eq:density} is equal to 
$\bmu^\sT\bB (\bB^\sT\bLambda\bB)^{-1}\bB^\sT \bmu =
\bmu^\sT\bLambda^\dagger \bmu$. This fact follows from straightforward  (albeit tedious) algebra after applying the stationary condition $\bG = \bzero$. 

Indeed, to see this, let $\ba = (\ba_1 ,\ba_2)$ where $\ba_1 \in\R^{dk}$, $\ba_2 \in\R^{n(k+k_0)}$, such that
$\bLambda^\dagger \bmu = \ba.$ 
Since $\bmu$ is orthogonal to the null space of $\bLambda$, we must have $\bLambda \ba = \bmu$, and hence
\begin{align}
\label{eq:pinv_lin_eq}
    \left(\bL^\sT \bL \otimes \bI_d \right)\ba_1 &+ [\bM, \bM_0] \ba_2 
    =  \overline\brho\\
    [\bM,\bM_0]^\sT \ba_1 &+(\bR \otimes \bI_n) \ba_2   = -\overline \bv
\end{align}
where $\overline \bv \in \R^{n(k+k_0)}$ and $\overline \brho\in\R^{dk}$ denotes the concatenation of the columns of $\bbV$ and $n\bRho$, respectively. 
Solving the for $\ba_1$ in terms of $\ba_2$, and vice-versa for the second equation allows us to conclude that 
\begin{align}
\label{eq:muTa}
   \bmu^\sT\bLambda^\dagger \bmu = \bmu^\sT\ba &=  \overline\brho^\sT(\bL^\sT\bL \otimes \bI_d)^{-1}\overline\brho + \overline\bv^\sT(\bR\otimes \bI_n)^{-1}\overline \bv 
   \underbrace{-\overline \brho^\sT(\bL^\sT\bL \otimes \bI_d)^{-1} [\bM,\bM_0]\ba_2}_{=:\textrm{(I)}}\\
   &\quad+ \underbrace{\overline\bv^\sT(\bR\otimes \bI_n)^{-1}[\bM,\bM_0]^\sT \ba_1}_{=:\textrm{(II)}}.
   \nonumber
\end{align}
Now write
\begin{align*}
  \textrm{(I)}  &\stackrel{(a)}{=}-\overline \btheta^\sT\left(\bI_k \otimes \bRho (\bL^\sT\bL)^{-1}\bL\right)\ba_2\\
  &\stackrel{(b)}=
   \overline\bv^\sT \left(\bI_k \otimes  \bL(\bL^\sT \bL)^{-1}\bL^{\sT}\right)
\ba_2\\ 
&=  \overline\bv^\sT(\bR\otimes \bI_n)^{-1}
\left(\bR \otimes \bL(\bL^\sT \bL)^{-1}\bL^{\sT}\right)
\ba_2\\
&\stackrel{(c)}{=}   \overline\bv^\sT(\bR\otimes \bI_n)^{-1}[\bM,\bM_0]^\sT (\bL^\sT\bL \otimes \bI_d)^{-1}[\bM,\bM_0]\ba_2\\
&\stackrel{(d)}{=} -\textrm{(II)} + \overline\bv^\sT (\bR \otimes \bI_n)^{-1} [\bM,\bM_0]^\sT (\bL^\sT\bL \otimes \bI_d)^{-1} \overline\brho,
\end{align*}
where in $(a)$ we used $\overline\btheta\in\R^{d(k+k_0)}$ to denote the concatenation of the columns of $[\bTheta,\bTheta_0]$, in $(b)$ we used the constraint $\bG(\bbV,\bTheta) =\bzero$, in $(c)$ we used  the identity (easily verifiable directly from the definitions)
\begin{equation}
   [\bM,\bM_0]^\sT(\bL^\sT\bL \otimes \bI_d)^{-1}[\bM,\bM_0] = \bR\otimes \bL(\bL^\sT\bL)^{-1}\bL^\sT,
\end{equation}
and in $(d)$ we used Eq.~\eqref{eq:pinv_lin_eq} to write $\ba_1$ appearing in \textrm{(II)} in terms of $\ba_2.$
Combining with Eq.~\eqref{eq:muTa} we conclude that
\begin{equation}
\bmu^\sT \bLambda^\dagger \bmu  = n^2 \Tr\left(\bRho (\bL^\sT\bL)^{-1}\bRho\right) + \Tr(\bbV \bR^{-1}\bbV) + n\Tr\left(\bRho (\bL^\sT\bL)^{-1}\bL^\sT \bbV \bR^{-1}[\bTheta,\bTheta_0]^\sT\right).
\end{equation}
\end{proof}

Next, the following lemma bounds the pseudo determinant term appearing in the expression for the density above.
\begin{lemma}[Bounding the determinant of the covariance]
\label{lemma:det_star_bound}
Let $r_k := k^2 + k_0 k$.
Under the assumptions of Section~\ref{sec:assumptions},
for any $(\bTheta,\bbV)  \in \cM(\cuA,\cuB)$, we have
   \begin{equation}
       \det^* \left(\bLambda(\bTheta,\bbV))\right)^{-1/2} \le \det(\bR(\bTheta))^{-n/2} \det(\bL(\bbV)^{\sT}\bL(\bbV))^{-(d-r_k)/2}
   \end{equation}
\end{lemma}
\begin{proof}
We'll suppress the index $(\bTheta,\bbV)\in\cM$ throughout the proof.
   Recall the definition of $\bLambda$.
   Let $\bM_1 := [\bM,\bM_0]\in\R^{dk \times (nk + nk_0)}$ where $\bM,\bM_0$ are the off diagonal blocks of $\bLambda$ defined in that lemma.
For any $\eps >0$, we have
\begin{align*}
   \det\left(\bLambda + \eps\bI\right)  &\ge \det\left(
   \begin{pmatrix}
       \bL^\sT \bL \otimes \bI_{d} + \eps \bI_{kd}  & \bM_1 \\
       \bM_1^\sT  & \bR \otimes \bI_n 
   \end{pmatrix}
   \right)\\
&=\det\left(\bR \otimes \bI_n\right) 
\det\left(
\left(\bL^\sT\bL + \eps \bI_{k}\right)\otimes \bI_d - \bM_1 \left(\bR^{-1}\otimes \bI_n\right) \bM_1^\sT
\right).
\end{align*}
With some algebra, one can show that
\begin{equation}
   \bM_1 \left(\bR^{-1}\otimes \bI_n\right) \bM_1^\sT = \bL^\sT\bL \otimes (\bTheta,\bTheta_0)\bR^{-1}(\bTheta,\bTheta_0)^\sT.
\end{equation}
Denoting the rank $(k+k_0)$ orthogonal projector $\bP_\bR :=
(\bTheta,\bTheta_0)\bR^{-1}(\bTheta,\bTheta_0)^\sT\in\R^{d\times d}$ and using $\bP_R^\perp$ for the complementary orthogonal projector, we can compute the second determinant term in the above display as
\begin{align*}
    &\det\left(
\left(\bL^\sT\bL + \eps \bI_{k}\right)\otimes \bI_d - \bM_1\left(\bR\otimes \bI_n\right) \bM_1^\sT
\right)\\
&=  \det\left( (\bL^\sT\bL + \eps \bI_{k})\otimes \bP_\bR^\perp +
(\bL^\sT\bL + \eps\bI_{k}) \otimes \bP_\bR - \bL^\sT\bL\otimes \bP_\bR
\right)\\
&=\det\left( (\bL^\sT\bL + \eps \bI_{k})\otimes \bP_\bR^\perp +
\eps\bI_{k} \otimes \bP_\bR
\right)\\
&\stackrel{(a)}{=} \det\left(\bL^\sT\bL+ \eps \bI_{k}\right)^{d-r_k} \eps^{r_k}.
\end{align*}
To see that $(a)$ holds, observe that if $\bu$ is an eigenvector of $\eps \bI_k \otimes \bP_\bR$, then $\big((\bL^\sT\bL + \eps \bI_k)\otimes \bP_\bR^\perp\big) \bu  = \bzero.$
So we conclude that for any $\eps >0$,
\begin{equation}
   \det(\bLambda +\eps \bI) \ge \det(\bR)^n \det(\bL^\sT\bL + \eps\bI_k)^{d-r_k}  \eps^{r_k}.
\end{equation}
Using that the dimension of the nullspace of $\bLambda$ is $r_k$ by Lemma~\ref{lemma:eig_vecs_NS_Sigma}, we then have
\begin{align*}
    \det^*(\bLambda) &:= \lim_{\eps \to 0} \frac1{\eps^{r_k}} \det(\bLambda + \eps\bI)\\
    &\ge \det(\bR)^{n}  \det(\bL^\sT\bL )^{d-r_k}
\end{align*}
as claimed.
\end{proof}
\begin{proof}[Proof of Lemma~\ref{lemma:density_bounds}]
The proof is a direct corollary of Lemma~\ref{lemma:det_star_bound}.
Indeed, rewriting the expression for $p_{\bTheta,\bbV}(\bzero)$ from Lemma~\ref{lemma:density} in terms of $\hmu,\hnu$, and ignoring exponentially trivial factors for large enough $n$, we reach the statement of the lemma.
\end{proof}

Finally, for future reference, we record the following uniform bound on the density.
\begin{corollary}[Uniform bound on the density]
\label{cor:uniform_density_bound}
In the setting of Lemma~\ref{lemma:density_bounds}, we have the following uniform bound holding for all $(\bTheta,\bbV) \in\cM$:
\begin{equation}
    p_{\bTheta,\bbV}(\bzero) \le \frac{\sfa_{R}^{-nk} \sfa_{L}^{-(d-r_k)}}{ (2\pi)^{(dk + nk + nk_0 -r_k)/2} n^{dk/2}}
\end{equation}
\end{corollary}


\subsection{Analysis of the determinant: Proofs of  Lemma~\ref{lemma:det_complement} and Lemma~\ref{lemma:CE_bound}}
\label{sec:determinant_bound}

\subsubsection{Proof of Lemma~\ref{lemma:det_complement}} 
Fix $(\bTheta,\bbV)\in\cM$.
Recalling the definitions in Eq.~\eqref{eq:SecDef} of $\bSec$ and $\tilde\bSec$, with sufficient patience, the desired Jacobian can be computed to be 
\begin{align}
    \bD:=\bJ\bG^\sT &= \begin{bmatrix}
   \grad^2\rho(\bTheta)\begin{bmatrix}
          \bI_k\otimes \btheta_1, &
         \dots,&
         \bI_k \otimes \btheta_k 
   \end{bmatrix}
   &
   \grad^2\rho(\bTheta)\begin{bmatrix}
          \bI_k\otimes \btheta_{0,1},&
         \dots,&
         \bI_k \otimes \btheta_{0,k_0} 
   \end{bmatrix} \\
    \frac1n
     \begin{bmatrix}
         \bSec (\bI_k \otimes \bv_1),& \dots ,&
         \bSec (\bI_k \otimes \bv_k) 
    \end{bmatrix} 
&
   \frac1n
     \begin{bmatrix}
         \bSec (\bI_k \otimes \bv_{0,1}), &
         \dots,&
         \bSec (\bI_k \otimes \bv_{0,k_0}) \\
    \end{bmatrix}\\
    \frac1n
   \begin{bmatrix}
         \tilde\bSec (\bI_k \otimes \bv_1), &
         \dots,&
         \tilde\bSec (\bI_k \otimes \bv_k) \\
   \end{bmatrix} 
   &
    \frac1n
   \begin{bmatrix}
         \tilde\bSec (\bI_k \otimes \bv_{0,1}), &
         \dots,&
         \tilde\bSec (\bI_k \otimes \bv_{0,k_0}) 
   \end{bmatrix} 
    \end{bmatrix}\\
    &+
     \begin{bmatrix}
       [\bI_k \otimes \bRho(\bTheta), \bzero_{dk \times kk_0}]\\
       \frac1n(\bI_{k + k_0} \otimes \bL) 
    \end{bmatrix} \in\R^{(nk + nk_0 + dk) \times k(k+k_0)}.
\end{align}

Recall the eigenvectors $\ba_{i,j}$ and $\ba_{0,i,j}$ of the nullspace of $\bLambda(\bTheta,\bbV)$ defined in Lemma~\ref{lemma:eig_vecs_NS_Sigma}. One can check that $\bJ\bzeta \ba_{i,j}$ and $\bJ\bzeta \ba_{0,i,j}$ correspond to columns of $\bD$ as $i,j$ range over their respective domains (on the event $\bzeta = \bzero)$.
So for the matrix $\bA \in \R^{(nk + nk_0 + dk)\times k(k+k_0)}$ whose columns are given by the collection of eigenvectors $\ba_{i,j},\ba_{0,i,j}$ (for some ordering) we have on the event $\bzeta = \bzero$,
\begin{equation}
    \bD = \frac1n
    \bJ \bzeta^\sT \bA.
\end{equation}

By definition of $\cM$ (and as asserted in Lemma~\ref{lemma:manifold_dim}), for any $(\bTheta,\bbV) \in\cM$, the columns of $\bD$ form a linearly independent basis the orthogonal complement of the tangent space at $(\bTheta,\bbV)$. Similarly, the columns of $\bA$ are a linearly independent basis of the null space of $\bLambda$ at the same point. So we can take basis matrices
\begin{equation}
\bB_{\bT^c} = \bD(\bD^\sT\bD)^{-1/2},\quad\quad
\bB_{\bLambda^c} = \bA(\bA^\sT\bA)^{-1/2}.
\end{equation}
Then
\begin{align}
    \det(\bB_{\bLambda^c}^\sT \bJ \bzeta \bB_{\bT^c})
    &= \det\left((\bA^\sT\bA)^{-1/2}\bA^\sT \bJ \bzeta\bD(\bD^\sT\bD)^{-1/2} \right)
    =
    \frac{n^{r_k}\cdot\det\left(\bD^\sT\bD\right)^{1/2} }{\det(\bA^\sT\bA)^{1/2}}.
\end{align}
Finally, one can check that, up to some permutation of the rows of $\bA$, we have 
\begin{equation}
    \bA^\sT\bA = [\bTheta,\bTheta_0]^\sT[\bTheta,\bTheta_0] \otimes \bI_k + \bI_{k+k_0} \otimes \bL^\sT\bL.
\end{equation}
This shows that Eq.~\eqref{eq:first_claim_det_reduction_lemma} holds.
To see the second claim, observe that by Assumption~\ref{ass:loss} there exists a constant $c>0$ such that
$\|\bL^\sT\bL\|_\op\le   C( k \|\bV\|_\op + \sqrt{nk})^2  \le C k n (\sfA_V^2 + 1) $
which along with $\bR(\hmu)\preceq \sfA_R^2$ proves the statement. 
\qed

\subsubsection{Conditioning and concentration}
Using our random matrix theory results of Section~\ref{sec:RMT},
we bound the conditional expectation of the Hessian of $\bzeta$.


Let us introduce the following quantities for this section
\begin{equation}
   \sfA_{L}  :=  1\vee
   \sup_{\substack{\|\bV\|_\op \le \sqrt{n}\sfA_V \\ 
  \|\bw\|_2 \le \sfA_{w}\sqrt{n}
   }} \frac1{\sqrt{n}}\|\bL(\bV,\bw)\|_\op,\quad
   \sfA_{\Rho} := 1 \vee
   \sup_{\|\bTheta\|_\op \le \sfA_R} \|\bRho(\bTheta)\|_\op.
\end{equation}
Note that by Assumptions~\ref{ass:loss} and~\ref{ass:regularizer}, we have $\sfA_L,\sfA_\Rho$ are bounded by some positive constant $C(\sfA_V,\sfA_w, k),C(\sfA_R)$, depending only on
$(\sfA_V,\sfA_w,k)$, $\sfA_R$ respectively.

We first start with the following lemma regarding concentration of Lipschitz functions of $\bH_0$.
\begin{lemma}[Concentration of Lipschitz functions of the Hessian]
\label{lemma:concentration_lipschitz_func}
Assume $f :\R \to\R$ is Lipschitz.  
Recall $\bH_0 = \left(\bX\otimes\bI_k\right)^\sT \bSec \left(\bX \otimes \bI_k\right)$, where 
$\bK$ was defined in Eq.~\eqref{eq:SecDef}, and let $\bS\in\R^{dk\times dk}$ be any deterministic symmetric matrix.
Then there exist absolute constants $c,C > 0$ such that, 
for any 
\begin{equation}
    t \ge \frac{\sfK\norm{f}_{\Lip}}{\alpha_n^{1/2}} \frac{k}{n^{1/2}},
\end{equation}
we have
\begin{equation}
    \P_\bX\left(\left|\frac1n\Tr(f((\bH_0 + \bS)/n)) - \frac1n\E_\bX\left[ \Tr(f((\bH_0 + \bS)/n))\right]\right| \ge t\right) 
\le
     C \exp \left\{ 
    -c\frac{ t n^{3/2} \alpha_n^{1/2} }{  k \sfK \norm{f}_{\Lip}}
    \right\}
\end{equation}
The same bound holds for any matrix $\bSec= (\bSec_{ij})_{i,j\le k}$ with $\bSec_{ij}\in\reals^{n\times n}$
a diagonal matrix with diagonal entries bounded (in absolute value) by $\sfK$.
\end{lemma}

\begin{proof}
First, we bound the variation of the function
\begin{equation}
\label{eq:lip_func_of_X}
   g(\bX) := \Tr f\Big((\bI \otimes \bX)^\sT \bSec (\bI \otimes \bX)/n + \bS/n\Big).
\end{equation}
Let $\bM = (\bI \otimes \bX)^\sT \bSec (\bI \otimes \bX)/n + \bS/n$ and
$\bM' = (\bI \otimes \bX')^\sT \bSec (\bI \otimes \bX')/n + \bS/n$.
Use $\{\lambda_i\}_{i\in[dk]}$ and $\{ \lambda_i'\}_{i \in[dk]}$ to denote their eigenvalues respectively. By Hoffman-Wielandt, we have
    \begin{align*}
        \left| \Tr f(\bM) - \Tr f(\bM') \right|
        &= \left| \sum_{i=1}^{dk} f(\lambda_i) - \sum_{i=1}^{dk}  f(\lambda_i')\right|
         \le \norm{f}_{\Lip} \min_{\sigma } \sum_{i=1}^{dk} \left|\lambda_i - \lambda_{\sigma(i)}' \right| \\
         &\le \norm{f}_{\Lip} (dk)^{1/2} \norm{\bM - \bM'}_F.
    \end{align*}
We can bound the norm of the difference 
\begin{align*}
    \norm{\bM - \bM'}_F &
    \le\frac1n\norm{(\bI \otimes \bX) - (\bI \otimes \bX')}_F \norm{\bK}_\op \left(\norm{\bI \otimes \bX}_\op + \norm{\bI \otimes \bX'}_\op \right)\\
    &\le \frac{k^{1/2}}{n} \sfK \norm{\bX - \bX'}_F (\norm{\bX}_\op + \norm{\bX'}_\op).
\end{align*}
Now for any $\gamma > 4$, define
\begin{equation}
\cA_\gamma :=     \left\{ \norm{\bX}_\op \le n^{1/2}\gamma^{1/2} \right\}.
\end{equation}
The above bound on the variation of $g$ implies that on $\cA_\gamma$, we have
\begin{equation}
    |g(\bX) - g(\bX') | \le \frac{C_0\,\sfK\, \norm{f}_{\Lip}}{\alpha_n^{1/2}} k\gamma^{1/2}    \norm{\bX - \bX'}_F\, .
\end{equation}
We can apply Gaussian concentration on this event. For $t>0$, we have for some universal constant $C_2,c_2,C_1,c_1 >0$,
\begin{align*}
\P\left(\left| \frac1n \Tr f(\bH_0) - \frac1n \E\left[\Tr f(\bH_0)\right] \right| \ge t \right)&\le 
    \P\left(\left\{\left| \frac1n \Tr f(\bH_0) - \frac1n \E\left[\Tr f(\bH_0)\right] \right| \ge t \right\} \cap \cA_\gamma\right)  \\
    &+
    \P\left( \cA_\gamma^c\right) \\
    &\le  C_1 \exp \left\{ 
    -c_1\frac{ t^2 n^2 \alpha_n }{ \gamma  k^2 \sfK^2 \norm{f}_{\Lip}^2}
    \right\}
    +
     C_2\exp\left\{  -c_2 n\gamma\right\}.
\end{align*}
Choosing $\gamma$ to satisfy
\begin{equation}
\gamma = 4\frac{t n^{1/2}\alpha_n^{1/2}}
{k \sfK \norm{f}_{\Lip}},
\end{equation}
for appropriate universal $c_3>0$ gives the desired bound, as long as 
\begin{equation}
    t \ge \frac{\sfK\norm{f}_{\Lip}}{\alpha_n^{1/2}} \frac{k}{n^{1/2}}.
\end{equation}
\end{proof}


\subsubsection{Proof of Lemma~\ref{lemma:CE_bound}}
\label{sec:pf_lemma_CE_bound}
Throughout the proof, we will use the notation
$\bS = \bS(\bTheta) := n\grad^2\rho(\bTheta).$ 
\noindent\textbf{Step 1: Conditioning as a perturbation.}
Let us first recall that the event $\{\bzeta = \bzero\}$, is equivalent
to $\bL^\sT\bX = - n\bRho^\sT$ and $\bX(\bTheta,\bTheta_0) = \bbV$. 
So letting $\bP_{\bTheta}, \bP_\bL$ be the projections onto the columns spaces of $(\bTheta,\bTheta_0),\bL$ respectively, we have on $\{\bzeta = \bzero\}$
\begin{equation}
    \bX = \bP_\bL^\perp \bX \bP_\bTheta^\perp - n \bL(\bL^\sT\bL)^{-1} \bRho^\sT \bP_{\bTheta}^\perp  + \bbV\bR^{-1} \bTheta^\sT.
\end{equation}
Hence for any function $g$,
\begin{equation}
    \E[g(\bX) | \bzeta = \bzero] = \E[g(\bX + \bDelta_{0,k})]
\end{equation}
for some matrix $\bDelta_{0,k} = \bDelta_{0,k}(\bTheta,\bbV)$ 
satisfying
\begin{align}
    &\rank(\bDelta_{0,k}) \le 4(k+k_0),\\
    &\norm{\bDelta_{0,k}}_\op \le C_0\sqrt{n}\max\left(\frac{\norm{\bX}_\op}{\sqrt{n}} ,
    C(\sfA_{V},\sfA_\Rho, \sfa_R,\sfa_L)
    \right)
\end{align}
for some $C_0 >0.$
Consequently, letting
\begin{equation}
\label{eq:Delta_1k_def}
   \Delta_{1,k}  :=  (\bI \otimes \bDelta_{0,k})^\sT \bSec (\bI\otimes \bX ) + (\bI \otimes \bX)^\sT \bSec (\bI \otimes \bDelta_{0,k}) + (\bI \otimes \bDelta_{0,k})^\sT \bSec (\bI\otimes \bDelta_{0,k}),
\end{equation}
we have
\begin{align}
&\E[|\det \big(\de \bz(\bt)\big)|\one_{\bH \succeq \bzero} \big|\bzeta(\bt) = 0 ]\\
&\quad\stackrel{(a)}{\le}
\frac{C_1(\sfA_R,\sfA_V,k)}{\sfa_G^{r_k}}\E\left[|\det\left(\bH\right)|   
\one_{\{\bH\succeq \bzero \}}
\bigg| \bzeta = \bzero\right]\\
&
\quad\stackrel{(b)}{\le}
\frac{C_1(\sfA_R,\sfA_V,k)}{ \sfa_G^{r_k}}
\E\left[  \big|\det\left(\bH + \bDelta_{1,k}\right)\big|  \one_{\{\bH + \bDelta_{1,k} \succeq \bzero\}}\right]
\label{eq:concentration_decomp_0}
\end{align}
where 
$(a)$ follows by 
Eq.~\eqref{eq:det_projection} and Lemma~\ref{lemma:det_complement},
and  $(b)$ follows by Eq.~\eqref{eq:conditioning_generic}.
Observe that $\rank(\bI\otimes \bDelta_{0,k}) \le 4r_k.$ 
Letting $r_k' := 12 r_k$, we have from the definition in~Eq.~\eqref{eq:Delta_1k_def}, 
that
\begin{align}
\rank(\bDelta_{1,k}) &\le r'_k\\
\|\bDelta_{1,k}\|_\op &\le  3\sfK \|\bX\|_\op (1 + \|\Delta_{0,k}\|_\op)^2 
\\
&\le 
3
C_2(\sfA_V, \sfA_\Rho, \sfa_{R},\sfa_L)
\sfK \|\bX\|_\op \left(\sqrt{n} +
\norm{\bX}_\op
\right)^2.
\end{align}
Then by Cauchy's interlacing theorem, we have for any $i\ge dk-r'_k$,
\begin{equation}
\label{eq:interlacing_2}
    \lambda_{i+r'_k}(\bH+  \bDelta_{1,k})\leq \lambda_{i}(\bH ),\quad \lambda_i(\bH + \bDelta_{1,k})\leq \|\bH +  \bDelta_{1,k}\|_\op.
\end{equation}

\noindent \textbf{Step 2: Needed bound on some moments of the operator norm}
Define
\begin{equation}
   \Delta_2(p) := 
 \E\left[ 
\left( \frac{\|\bH +  \bDelta_{1,k}\|_\op}{n}\right)^{p}\right] \quad\textrm{for}\quad p>1.
\end{equation}
This will reappear in several places in the proof, we let us preempt this by giving a bound on this quantity.
We have for any $p>1$,
\begin{align}
\Delta_2(p)
&\le 
\frac{C_3^p}{n^p}
\left( \E\left[ 
\sfK^p(\|\bX\|^{2p}_\op + \|\bDelta_{0,k}\|_\op^{2p})\right] + \|\bS \|_\op^p\right)\\
&\le \sfC_1^p
\end{align}
where 
\begin{equation}
    \sfC_1 := C_{4}(\sfK, \sfA_{V},\sfA_\Rho,\sfA_R,\sfa_L,\sfa_R).
\end{equation}

\noindent \textbf{Step 3: Proof of item \textit{1.}}
We will use the constraint on the minimum singular value of the Hessian to constrain the  asymptotic spectral measure $\mu_{\star}(\hnu, \hmu)$. 
Namely, fixing $\tau_0>0$,
we will show that for $\bbV$ satisfying
\begin{equation}
\label{eq:constraint_on_V_asymp}
\mu_{\star}(\hnu,\hmu)((-\infty, -\tau_0)) \ge \tau_0,
\end{equation}
 the value of the expectation in~\eqref{eq:concentration_decomp_0} is small.
To this ends, define the event
\begin{equation}
    \Omega_{3}:=\left\{
    \big|\left\{ \lambda \in \spec\left(\bH/n \right) : \lambda \leq 0\right\} \big| < r'_k
    \right\},
\end{equation}
Then by Eq.~\eqref{eq:interlacing_2},  $\{\lambda_{\min}(\bH + \bDelta_{1,k})/n> 0\}\subseteq \Omega_{3}$. 
We'll bound the probability of $\Omega_3$ for $\bbV$ satisfying~\eqref{eq:constraint_on_V_asymp}. 
Define the Lipschitz test function $f_{\tau_0}:\R\to\R$ as
\begin{equation}f_{\tau_0}(\lambda) = 
    \begin{cases}
        1 & \text{ if } \lambda\leq -\tau_0,\\
        1- \frac{1}{\tau_0}(\lambda +\tau_0)& \text{ if }-\tau_0<\lambda\leq 0\\
        0& \text{ if } 0<\lambda.
    \end{cases}
\end{equation}
This function has Lipschitz modulus bounded by $\tau_0^{-1}$. Furthermore,  
if
$\bbV$ satisfies Eq.~\eqref{eq:constraint_on_V_asymp}, then
$\tau_0 <
    \E_{\mu_\star}\left[ f_{\tau_0}(\Lambda) \right].$
On the other hand, on the event $\Omega_{3},$ we have
   $\Tr\,f_{\tau_0}(\bH /n)
   =   \sum_{i=1}^{n} f_{\tau_0}(\lambda_i(\bH /n))
   \le r_k'.$
Hence, we can bound for any such $\bbV$
\begin{align}
\P\left(\Omega_{3} \right)
    &\le  
    \P\left(\left|\frac1{dk} \Tr(f_{\tau_0}(\bH/n)) - \E_{\mu_{\star}}[f_{\tau_0}(\Lambda)] \right| > \tau_0- \frac{r_k'}{dk}\right)\\
    &\le
    \P\left(\left|\frac1{n} \Tr(f_{\tau_0}(\bH/n)) - 
    \frac1{n} \E\left[\Tr(f_{\tau_0}(\bH + /n))\right]
    \right| > \frac{dk}{n}\tau_0- \frac{r_k'}{n} - \omega_{n}(\tau_0)\right)\\
    &
\le
     C_{5} \exp \left\{ 
    -c_6\frac{ 
    \tau_0(k\tau_0- r_k'/d - \alpha_n\omega_{n}(\tau_0))
    n^{3/2} }{ \alpha_n^{1/2} k \sfK }
    \right\}
\end{align}
where $\omega_{n}(\tau_0) \to 0$ for any $\delta >0$ as $n\to\infty,$ by Proposition~\ref{prop:uniform_convergence_lipschitz_test_functions}.
So we conclude that
for any $\bbV$ satisfying~\eqref{eq:constraint_on_V_asymp}
\begin{align}
&\E\left[  \big|\det\left(\bH  + \bDelta_{1,k}\right)\big|    \one_{\{\bH  + \bDelta_{1,k} \succeq 0\}}\right]\\
&\hspace{20mm}\le 
\E\left[  \big|\det\left(\bH +  \bDelta_{1,k}\right)\big|    \one_{\Omega_3}\right]\\
&\hspace{20mm}\le n^{dk + 4r_k}  
C_{7}
\sfC_1^{dk} \exp \left\{ 
    -c_8\frac{ 
    \delta(k\delta- r_k'/d - \alpha_n\omega_{n}(\delta))
    n^{3/2} }{ \alpha_n^{1/2} k \sfK }
    \right\}\\
    &\hspace{20mm}\stackrel{(a)}{\le} \exp\left\{ - c_3 \delta^2 n^{3/2}\right\},
\end{align}
where in $(a)$ we took $n > C_{9}(\tau_0)$ for some $C_{9}$ so that so that
$\sfC_1 \le  e^{\sqrt{n}}$ and $\omega_{n}(\tau) <\tau/3.$

\noindent\textbf{Step 4: 
Proof of item \textit{2.}}
We now deal with the determinant term for $(\bTheta,\bbV)$ not satisfying~\eqref{eq:constraint_on_V_asymp}. 
Namely, we'll show concentration of the determinant.

First, note that since $t\mapsto \log t$ is monotonically increasing for $t>0$, when $\bH +  \bDelta_{1,k} \succ \bzero$, we have for any $\tau_1>0$,
  \begin{align}
\label{eq:log_det_to_log_eps_2}
      \log \left(\det((\bH + \bDelta_{1,k})/n)\right) =& \sum_{i=1}^{dk}\log 
      (\lambda_i(\bH+  \bDelta_{1,k} )/n)\\
      \leq& \sum_{i=1}^{dk -r'_k}\log \lambda_i(\bH/n) + r_k'\log\left( \frac{\|\bH +  \bDelta_{1,k}\|_\op}{n}\right)\\
      \leq& \Tr\left(\log^{(\tau_1)}(\bH /n) \right) + 
r_k'\log\left( \frac{\|\bH+  \bDelta_{1,k}\|_\op}{n \tau_1}\right),
  \end{align}
where we defined $\log^{(\tau_1)}(t) := \log(t \vee \tau_1).$
Combining with the result of Step 1, we conclude 
\begin{align}
\label{eq:step_2_det_conc_result}
&\E[|\det \big(\de \bz\big)|\one_{\bH  \succ \bzero} \big|\bzeta = 0 ]\\
&
\quad{\le}
\frac{n^{dk +r_k} C_{10}(\sfA_R,\sfA_L)}{(\sfa_\bG)^{r_k}}
\E\left[ 
\exp\left\{
\Tr\log^{(\tau_1)}\left(\frac{1}n\bH\right) \right\}  
\left( \frac{\|\bH + \bDelta_{1,k}\|_\op}{n \tau_1}\right)^{12 r_k}
 \one_{\{\bH +  \bDelta_{1,k} \succeq \bzero\}}\right].
\end{align}
Noting that for any $\eps>0$, 
$t\mapsto\log^\up{\eps}(t)$ has Lipschitz modulus bounded by $\eps^{-1}$, we now apply Lipschitz concentration (Lemma~\ref{lemma:concentration_lipschitz_func}) to $\Tr\log^\up{\tau_1}(\bH /n)$.
To this end, for any $t_n>0$, define the event
define the event 
\begin{equation}
    \Omega_4 := \left\{
    \left|\frac1{n}\Tr\Big(\log^{(\tau_1)}
    \left(\bH/{n}\right)\Big) - 
    \frac1{n}\E_\bX\left[ \Tr\Big(\log^{(\tau_1)}
    \left(\bH/{n}\right)
    \Big)\right]\right| \le t_n
    \right\}.
\end{equation}
Then the expectation in Eq.~\eqref{eq:step_2_det_conc_result} is bounded as
\begin{align}
&\E\left[ 
\exp\left\{
\Tr\log^{(\tau_1)}\left(\frac{\bH}n\right) \right\}  
\left( \frac{\|\bH + \bDelta_{1,k}\|_\op}{n \tau_1}\right)^{r_k'}
 \one_{\{\bH +  \bDelta_{1,k} \succeq \bzero\}}\right]\\
&\quad\le 
\exp\left\{
\E\left[\Tr\log^{(\tau_1)}\left(\frac{\bH}n\right) \right] + nt_n
\right\}
\E\left[ 
\left( \frac{\|\bH + \bDelta_{1,k}\|_\op}{n \tau_1}\right)^{r_k'}
 \one_{\{\bH +  \bDelta_{1,k} \succeq \bzero\}}\right]\\
&\quad+
\E\left[ 
\left( \frac{\|\bH +  \bDelta_{1,k}\|_\op}{n}\right)^{dk}
 \one_{\{\bH +  \bDelta_{1,k} \succeq \bzero\}} \one_{\Omega_4^c}\right]\\
&\le 
\Bigg(\exp\left\{
\E\left[\Tr\log^{(\tau_1)}\left(\frac{\bH}n\right) \right] + nt_n
\right\}
\left(\frac{1}{\tau_1}\right)^{12r_k}
\Delta_2(24 r_k)^{1/2}+
\P(\Omega_4^c)^{1/4}
\Delta_2(4dk)^{1/4}
\Bigg) ,
\label{eq:concentration_bound_decomp}
\end{align}
Choosing $t_n$ in the definition of $\Omega_4$ as 
\begin{equation}
\label{eq:choice_of_t_concentraion}    
t_n :=  \frac{ c_1\sfK k}{ \tau_1 \alpha_n^{1/2}} n^{-1/4}
\end{equation}
for appropriate constant $c$, we conclude 
by Lemma~\ref{lemma:concentration_lipschitz_func}
that
$\P(\Omega_4^c) \le  C_{13}(\sfK) \exp\left\{- n^{5/4}\right\}$.
Then after combining with Eq.~\eqref{eq:concentration_bound_decomp} along with the bounds on $\Delta_3$ and $\Delta_4$ derived previously we have
\begin{align}
&\E[|\det \big(\de \bz\big)|\one_{\bH  \succeq \bzero} \big|\bzeta = \bzero ]\\
&
\quad{\le}
\frac{\; C_{11}(\sfA_R,\sfA_L)\; n^{dk+5r_k}}{\sfa_G^{r_k}}
\Bigg(\exp\left\{
\E\left[\Tr\log^{(\tau_1)}\left(\frac{\bH}n\right) \right] + n\frac{c_1 \sfK k }{\tau_1 \alpha_n^{1/2}} n^{-1/4}
\right\}
\left(\frac{\sfC_1}{\tau_1}\right)^{12r_k}\\
&\quad+C_{12}\sfC_1^{dk} \exp\left\{- n^{5/4}  \right\}
\Bigg)\\
&\quad{\stackrel{(a)}{\le}}
{n^{dk}}
\Bigg(\exp\left\{
\E\left[\Tr\log^{(\tau_1)}\left(\frac{\bH}n\right) \right] + n \frac{C_{13} n^{-1/4}}{\tau_1}
\right\}
+C_{14}\exp\left\{- n^{5/4}\right\}
\Bigg)
\end{align}
where in $(a)$ we took $n > C_{15}$ so that  $\sfC_1 \le  e^{n^{1/2}}$ and
used that
$\sfa_{G}^{-1} =  e^{o(n)}$.
\qed

\subsection{Asymptotics of the Kac-Rice integral}
\label{sec:kr_asymptotics}
What remains now is to study the asymptotics of the integral to derive the upper bound of Theorem~\ref{thm:general}. 
Let us begin by proving Lemma~\ref{lemma:asymp_1} in the next section.

\subsubsection{Proof of Lemma~\ref{lemma:asymp_1}}
\label{sec:proof_prop_asymp_1}

Fix $\tau_0,\tau_1 \in(0,1)$ and $\beta$ as in the statement of the lemma.
Once again, we suppress the indices $(\bTheta,\bbV)$ in the arguments.

\noindent\textbf{Step 1: obtaining the hard constraint on the support.}
First, we show that
\begin{align}
\mathrm{(I)} &:= \limsup_{n\to\infty}\frac1n\log\left(
\E_\bw\left[
\int_{\cM}
\E[|\det \big(\de \bz\big)|\one_{\bH \succeq\bzero} \big|\bzeta = 0 ,\bw]
p_{\bTheta,\bbV}(\bzero)
\right)
\one_{\bw\in\Omega_\delta}
\right]
\one_{\mu_{\star}((-\infty,- \tau_0)) \ge \tau_0} \de_{\cM} V\\
&= -\infty.
\end{align}
Directly by item {2} of Lemma~\ref{lemma:CE_bound}, followed by the bound on $p_{\bTheta,\bbV}(\bzero)$ of Corollary~\ref{cor:uniform_density_bound}, we have for some $C_0,c_0$ independent of $n$,
\begin{align}
\textrm
{(I)}&\le 
\limsup_{n\to\infty}\frac1n \log\left(
 C_0 e^{-c_0 \tau_0^2 n^{3/2}} 
 \E_\bw\left[
\int_{\cM} 
p_{\bTheta,\bbV}(\bzero)
\de_{\cM} V\;
\one_{\bw \in\Omega_\delta}
\right]
\right)\\
&\le 
\limsup_{n\to\infty}\frac1n \log\left(
\frac{
 C_0 e^{-c_0 \tau_0^2 n^{3/2}}
\sfa_{R}^{-nk} \sfa_{L}^{-(d-r_k)}}{ (2\pi)^{(dk + nk + nk_0 -r_k)/2} n^{dk/2}}
 \E_\bw[\vol(\cM) \one_{\bw \in \Omega_\delta}]
\right).
\end{align}
To estimate $\vol(\cM)$, we use Lemma~\ref{lemma:manifold_integral} with $f=1$ and the $\beta$ chosen. Letting $\Err_{\textrm{blow-up}}(\beta,n)$ be the multiplicative error defined therein, we have
\begin{equation}
    \vol(\cM) \le 
    \Err_{\textrm{blow-up}}(\beta,n)
    \vol(\cM^\up{\beta}) \le 
    \Err_{\textrm{blow-up}}(\beta,n)
     \vol\left(\Ball_{(k+k_0)\sfA_{V}}^{n(k+k_0)}(\bzero)\right)
    \vol\left(\Ball_{(k+k_0)\sfA_{R}}^{dk}(\bzero)\right)
\end{equation}
where we used that $\cM \subseteq
\Ball_{(k+k_0)\sfA_{V}}^{n(k+k_0)}(\bzero) \times \Ball_{(k+k_0)\sfA_{R}}^{dk}(\bzero)$. Evaluating these terms, substituting into the upperbound on (I), then taking $n\to\infty$ shows the claim.  

\noindent\textbf{Step 2: bounding the asymptotically dominating term.}
Define
\begin{align}
F_{n,\tau_1}(\bbV,\bTheta,\bw)
&:=
 \frac{k}{2\alpha_n}\log(\alpha_n)+
\frac{k}{\alpha_n}\E\left[\frac1{dk}\Tr\log^{(\tau_1)}\left(\frac{\bH}n\right)\Big| \bw \right]  - \frac{1}{2\alpha_n}\log \det\left(\frac{\bL^\sT\bL}{n}\right)\\
&+ \frac{1}{2\alpha_n}\Tr(\bTheta^\sT\bTheta) 
   -\frac{n}{2}
\Tr\left(\bRho (\bL^\sT\bL)^{-1}\bRho^\sT\right) + \frac1{2}\Tr\left(\frac1n\bbV^\sT\bL (\bL^\sT\bL)^{-1}\bL^\sT \bbV \bR^{-1}\right)\\
&+ \frac12 \Tr\left(\frac1n\bbV(\bI - \bR^{-1})\bbV^\sT\right)
-\frac1{2}\log\det(\bR).
\end{align}
We show that
\begin{align}
\mathrm{(II)} &:=
\limsup_{n\to\infty}\frac1n\log\left(
\E_\bw\left[
\int_{\cM}
\E[|\det \big(\de \bz\big)|\one_{\bH \succeq \bzero} \big|\bzeta = 0, \bw ]
\; \one_{\bw \in \Omega_\delta}
\right]
\right)
\one_{\mu_{\star}(\hnu,\hmu)((-\infty, -\tau_0])< \tau_0} \de_{\cM} V \\
&\le \limsup_{n\to\infty}
   \frac1n\log\left( \E_\bw\left[\int_{\cM}\exp\left\{nF_{n,\tau_1}(\bbV,\bTheta)\right\} p_{1}(\bbV) p_{2}(\bTheta)
   \one_{\mu_{\star}(\hnu,\hmu)((-\infty, -\tau_0]) < \tau_0}
   \de_\cM V
   \; \one_{\bw\in\cG}
   \right]\right)
   \label{eq:bound_step_2_asymptotics_1}
\end{align}
In what follows, we use 
\begin{equation}
    K_{n,\tau_1}(\bbV,\bTheta,\bw) := 
\E\left[\frac1{dk}\Tr\log^{(\tau_1)}\left(\frac{\bH}n\right) \bigg| \bw\right],
\end{equation}
By item~1 of Lemma~\ref{lemma:CE_bound},
there exists $C_1$ independent of $n$ such that
for any $(\hnu,\hmu)$ satisfying the 
support condition
$\mu_{\star}(\hnu,\hmu)((-\infty, -\tau_0])< \tau_0$,
\begin{align}
\E[|\det \big(\de \bz\big)|\one_{\bH  \succeq \bzero} \big|\bzeta = \bzero,\bw ]\le
n^{dk}
\Bigg(\exp\left\{
\frac{nk}{\alpha_n} K_{n,\tau_1}(\bbV,\bTheta,\bw) +  \frac{C_1 n^{1-1/4}}{\tau_1}
\right\}
+ \exp\left\{ -n^{5/4} \right\}
\Bigg)\
\end{align}
Note that we have the uniform-bound
\begin{align}
    \exp\left\{- n^{5/4}\right\} \le \exp\left\{\frac{nk}{\alpha_n} \log(\tau_1) + \frac{C_1 n^{1-1/4}}{\tau_1}\right\} &\le 
    \exp\left\{
\frac{nk}{\alpha_n} K_{n,\tau_1}(\bbV,\bTheta,\bw) +  \frac{C_1 n^{1-1/4}}{\tau_1}
\right\},
\end{align}
holding for $n$ large enough, uniformly over all $(\bTheta,\bbV,\bw)$.
Then since the term $n^{1-1/4}C_1/\tau_1$ is exponentially trivial, we conclude that
\begin{align}
   \mathrm{(II)}  \le  \limsup_{n\to\infty} \frac1n \log \E_\bw\left[ \int_{\cM} 
   \exp\left\{
   \frac{nk}{\alpha} K_{n,\tau_1}(\bbV,\bTheta,\bw)
   \right\}
\one_{\{\mu_{\star}((-\infty, -\tau_0])< \tau_0\}}
   p_{\bbV,\bTheta}(\bzero)\de_\cM V \; \one_{\bw\in\Omega_\delta}\right].
\end{align}
What remains to
conclude Eq.~\eqref{eq:bound_step_2_asymptotics_1} is to recall the bound on $p_{\bbV,\bTheta}(\bzero)$ in Lemma~\ref{lemma:density_bounds} and ignore the exponentially trivial, and simplify to obtain a bound in terms of of $F_{n,\tau_1}$.

\noindent \textbf{Step 3: Estimating the integral over the manifold.}
We now rewrite the bound on $\mathrm{(II)}$ as an expectation over the blow-up of the manifold from Lemma~\ref{lemma:manifold_integral}. 
Choose a sequence $\beta_n = c_0 (\sfa_{G,n}^2\wedge 1/n)$ for sufficiently small constant $c_0>0$ so that $\beta_n$ satisfies the condition in Lemma~\ref{lemma:manifold_integral} for all $n$ sufficiently large.
We apply this lemma with this chosen value of $\beta_n$ to the function 
\begin{equation}
    f(\bTheta,\bbV) :=  e^{nF_{n,\tau_1}(\bTheta,\bbV)} p_1(\bbV)p_2(\bTheta).
\end{equation}
It's easy to verify that under Assumption~\ref{ass:loss} and~\ref{ass:regularizer}, guaranteeing the Lipschitzness and local Lipschitzness of the derivatives of the loss and the regularizer, respectively, that
\begin{equation}
 \|\log f\|_{\Lip,\cM^\up{1}} 
\le C_3(\tau_1, \sfA_{R},\sfA_{V},\sfA_{w},\sfa_{R},\alpha_n, r_k) \; n
\end{equation} 
for some $C_3$ that remains bounded for $\alpha_n$ in a compact subset of $(1,\infty)$, so that
\begin{equation}
 \lim_{n\to\infty}\frac1n \left(\beta_n\|\log f\|_{\Lip,\cM^\up{1}}  \right) = 0
\end{equation}
by the choice of $\beta_n$.
 Lemma~\ref{lem:intg-tube} further stablishes that 
\begin{align}
    &\limsup_{n\to\infty} \log(
    \Err_{\textrm{blow-up}}(\beta_n,n)) =0,
\end{align}
by the choice of $\beta_n$ and Assumption~\ref{ass:params} that $\sfa_{G,n} = e^{-o(n)}$.
Combining with \textbf{Step 2} we conclude that
\begin{align}
   \mathrm{(II)} &\le \limsup_{n\to\infty}\frac1n\log
   \E\left[\E\left[\exp\left\{nF_{n,\tau_1}(\bbV,\bTheta)\right\}
   \one_{\{\mu_{\star}(\hnu,\hmu)((-\infty,-\tau_0]) < \tau_0\} \cap \cM^{(\beta_n)}} \Big| \bw
   \right] \one_{\bw\in\Omega_\delta}\right],
\end{align}
where the expectation is under $p_1(\bbV),p_2(\bTheta).$

%
%

\noindent\textbf{Step 4: Concluding.} Finally, we write the bound on \textrm{(II)} in terms of empirical measures $\hmu,\hnu$ of $\sqrt{d}[\bTheta,\bTheta_0]$, $[\bbV,\bw]$ respectively.
Set
\begin{equation}
K_{\tau_1}(\hnu,\hmu) := \int \log(\lambda \vee \tau_1) \mu_{\star}(\hnu,\hmu)(\de \lambda)
\end{equation}
then note that
by the uniform bounds of Proposition~\ref{prop:uniform_convergence_lipschitz_test_functions}, we have
\begin{align}
\exp\left\{\frac{nk}{\alpha_n} K_{n,\tau_1}(\bbV,\bTheta) 
\right\}
\le
\exp\left\{\frac{nk}{\alpha_n} K_{\tau_1}(\hnu,\hmu) +   \frac{nk}{\alpha_n} \omega_{\textrm{ST}}(n, \tau_1)
\right\}
\end{align}
for some $\omega_{\textrm{ST}}(n, \tau_1) = o(1)$ uniformly over $\hnu,\hmu$ so that 
$n k\omega_{\textrm{ST}}(n, \tau_1)/\alpha_n$ is exponentially trivial.
Furthermore, it's easy to check that for any $\beta \in (0,1)$, 
we have
\begin{equation}
 \cM^\up{\beta_n} \subseteq \{
 (\bTheta,\bbV) : (\hmu,\hnu) \in \cuM^\up{\beta_n}  \}
 \subseteq \{
 (\bTheta,\bbV) : (\hmu,\hnu) \in \cuM^\up{\beta}  \}
\end{equation}
for $n$ sufficiently large,
since $\beta_n \to 0$.
So since the integrand is nonnegative, 
and combining with the bound on \textrm{(II)} from \textbf{Step 3}, and recalling the definition of $\phi_{\tau_1}$ from the statement of the proposition, we obtain
\begin{equation}
   \mathrm{(II)} \le \limsup_{n\to\infty}\frac1n\log
   \E_\bw\left[\E\left[\exp\left\{n\phi_{\tau_1}(\hnu,\hmu)\right\}
   \one_{\{\mu_{\star}(\hnu,\hmu)((-\infty, -\tau_0]) < \tau_0\} \cap \cuM^{(\beta)}}
   | \bw \right] \one_{\bw\in\Omega_\delta}\right] 
\end{equation}
Finally, combining this with the bound on $\textrm{(I)}$ from \textbf{Step 1}, and invoking Lemma~\ref{lemma:kac_rice_manifold} gives the result of the lemma.
\qed

\subsubsection{Large deviations and completing the proof of Theorem~\ref{thm:general}}
\label{sec:proof_thm1_large_deviations}

To obtain the asymptotic formula of Theorem~\ref{thm:general} and complete the proof, we study the limit obtained in Lemma~\ref{lemma:asymp_1} and obtain an upper bound via Varadhan's integral lemma.

First, by Sanov's Theorem applied to $\hmu$, viewing the marginals 
$\{\hmu_{\sqrt{d_n}\bTheta_0}\}_{n}$ as a deterministic sequence of points in the set $\cuP(\R^{k_0})$  converging to $\mu_0$ (and are hence exponentially tight), we have for any Borel measurable $\cU_1\subseteq\cuP(\R^{k+k_0})$,
\begin{equation}
 \lim_{n\to\infty} \frac1n\log \P_{\substack{\bTheta\sim \cN(\bI_{dk}/d) \\ \bTheta_0\sim \hmu_{\sqrt{d}\bTheta_0}} }\left(\hmu\in\cU_1\right)   \le -\inf_{\substack{\mu\in\overline{\cU_1}\\ \mu_{(\bt_0)} = \mu_0 }} 
 \frac1\alpha\KL(\mu_{\cdot|\bt_0}\|\cN(\bzero,\bI_k)), ,
\end{equation}
where we introduced
 $(\bt,\bt_0) \sim \mu$ so that $\mu_{\cdot|\bt_0}$ denotes the conditional measure of $\bt$ conditional on $\bt_0$, and $\mu_{(\bt_0)}$ denotes the marginal of $\bt_0$.
Similarly, again by Sanov theorem, for any measurable $\cU_2\subseteq\cuP(\R^{k+k_0+1}), $ we have
\begin{equation}
 \lim_{n\to\infty} \frac1n\log \P_{\substack{\bbV\sim \cN(\bI) \\ \bw\sim \P_\bw} }\left(\hnu\in\cU_2\right)   \le -\inf_{\substack{\nu\in\overline\cU_2}} 
 \KL(\nu\|\cN(\bzero,\bI_{k+k_0})\times \P_w).
\end{equation}
The contraction principle then gives the LDP for the pair $(\hmu,\hnu):$ 
\begin{equation}
    \lim_{n\to\infty} \frac1n \log \P((\hmu,\hnu) \in   \cU) 
    \le -\inf_{\substack{(\mu,\nu) \in\cU\\ \mu_{(\bt_0)} = \mu_0}} I(\mu,\nu)
\end{equation}
 for
    \begin{equation}
        I(\mu,\nu) :=      \frac1\alpha\KL(\mu_{\cdot|\bt_0}\|\cN(\bzero,\bI_k))+ \KL(\nu\|\cN(\bzero,\bI_{k+k_0})\times \P_w)\, ,
    \end{equation}
    and $\cU$ measurable subset of $\cuP(\R^{k+k_0}) \times \cuP(\R^{k+k_0+1})$.

Now we use this LDP above alongside Varadhan's integration lemma to bound the limit in Lemma~\ref{lemma:asymp_1}.
Observe first that one can directly show the exponent $\phi_{\tau_1}(\hmu,\hnu)$ of Lemma~\ref{lemma:asymp_1} is uniformly bounded for $(\hmu,\hnu) \in {\cuM^\up{\beta}}$. Indeed, in Section~\ref{sec:RMT} in the proof of Proposition~\ref{prop:uniform_convergence_lipschitz_test_functions}, we showed that $\mu_\star(\hnu,\hmu)$ is compactly supported, with support bounded uniformly in $(\hmu,\hnu)$, and hence its truncated logarithmic potential is bounded. Furthermore, the bounds $\sfA_{R},\sfA_{V},\sfa_{R},\sfa_{L}$ in the definition of $\cuM$ guarantee uniform bounds on the functionals of $\nu,\mu$ appearing in the definition of $\phi_{\tau_1}$, so we have sufficient exponential tightness to apply Lemma 4.3.6 of \cite{dembo2009large} as follows.
Define for $\beta\ge 0, \tau_0\ge 0, \delta\ge 0$,
    \begin{equation}
        \cuS(\beta,\tau_0,\delta):= 
        \overline{\cuM^{(\beta)}} \cap 
        \{(\mu,\nu):\; \mu_\star(\mu,\nu)((-\infty,-\tau_0))\le \tau_0, \;\;
        d_{\textrm{BL}}(\hnu_{w},\P_w)\le\delta,
        \},
    \end{equation} 
    and
    \begin{equation}
        \cuS_0(\beta,\tau_0,\delta) := \cuS(\beta,\tau_0,\delta) \cap
       \{\mu_{(\bt_0)}  = \mu_0\}.
    \end{equation}
    Then we have
    \begin{align} 
         \limsup_{n\to\infty}&\frac1n \log ( \E[ Z_n\one_{\bw \in \Omega_\delta}])
        \le
        \limsup_{n\to\infty}\frac1n\log
      \E\left[\exp\left\{n\phi_{\tau_1}(\hnu,\hmu)\right\}
       \one_{\{(\hmu,\hnu)\in\cuS(\beta,\tau_0,\gamma)\}}
       \right]\\
        \le& \sup_{\substack{(\mu,\nu)\in \cuS_0(\beta,\tau_0,\delta)}}
        \big\{ \phi_{\tau_1}(\mu,\nu) - I(\mu,\nu)
        \big\},
        \label{eq:varadhan+blowup}
    \end{align}
    where in the first inequality we used Lemma~\ref{lemma:asymp_1}, and in the second we used \cite[Lemma 4.3.6]{dembo2009large} and that of $\cuS_0(\beta,\tau_0,\delta)$ is closed. (To see that it is indeed closed, note that we have shown in Section~\ref{sec:RMT} that if $(\mu,\nu) \mapsto \mu_\star$ is continuous in the topology of weak convergence, meanwhile, for a weakly converging sequence of random variables $X_n \to X$, we have $\P(X \in \cU) \le \liminf_{n} \P(X_n \in\cU)$ for any open set $\cU$).

Now note that $\cuS_0(\beta,\tau_0,\delta)$ is a compact subset of $\cuP(\R^{k+k_0})\times \cuP(\R^{k+k_0+1})$. This can be verified through Prokhorov's Theorem: if we show $\cuM^\up{\beta}$ is tight, Prokhorov implies that $\overline{\cuM^\up{\beta}}$ is compact which gives the compactness of $\cuS_0(\beta,\tau_0,\delta);$ the closed subset of $\overline{\cuM^\up{\beta}}$. 

To prove $\cuM^\up{\beta}$ is tight, note that for any $(\mu,\nu)\in {\cuM^\up{\beta}}$, 
letting $\bar\btheta = (\btheta,\btheta_0),$
we have for some $C$ depending on $\beta$,
    $(\sfA_R^2 + C(\beta))\bI \succ \E_{\mu}[\bar\btheta \bar\btheta^\sT]$. Hence $\|\E_{\mu}[\bar\btheta]\|_2^2\le (k+k_0)(\sfA_{R}^2 + C(\beta))$, so an application of Markov's yields
    $\P_{\mu}\left(\|\bar\btheta\|_2 > t\right) \le (k+k_0)(C(\beta) + \sfA_R^2)/t^2 \to 0$ as $t\to 0$.
Similar reasoning applied to $\nu$ instead of $\mu$ then gives tightness of $\cuM^\up{\beta}$ for any $\beta\ge 0$.

What is left now is to show that we can send the parameters $\tau_0,\tau_1$ and $\beta$ to $0$.
Let us first take $\beta\rightarrow 0$
in the bound of
    Eq.~\eqref{eq:varadhan+blowup}.
Take a sequence $\{\beta_i\}_{i\in\N}$ such that $\beta_i\to 0$. Then by compactness of $\cuS_0$, there exists $\{(\mu_{\beta_i},\nu_{\beta_i})\in \cuS_0(\beta_i,\tau_0,\delta),\; i\in\N\}$ so that 
    \begin{equation}
        \limsup_{\beta\to 0}\sup_{(\nu,\mu)\in \cuS_0(\beta,\tau_0,\delta)} \{\phi_{\tau_1}(\mu,\nu) - I(\mu,\nu)  \}=
        \limsup_{i\to \infty} \{\phi_{\tau_1}(\mu_{\beta_i},\nu_{\beta_i}) - I(\mu_{\beta_i},\nu_{\beta_i})\}.
    \end{equation}
Noting that $\{\cuS_0(\beta_i,\tau_0,\delta)\}_{i\in \N}$ is a decreasing sequence of closed sets, every converging subsequence of $\{(\mu_{\beta_i},\nu_{\beta_i})\}$ converges to a point in the set $\bigcap_{i\in \N}\cuS_0(\beta_i,\tau_0,\delta) = \cuS_0(0,\tau_0,\delta)$. Since $I(\mu,\nu)$ is lower semi-continuous and $\phi_{\tau_1}(\mu,\nu)$ 
is continuous on $\cuM^\up{\beta}$ for $\beta$ sufficiently small, we conclude that $\phi_{\tau_1} - I$ is upper semi-continuous so that
    \begin{align}
         \limsup_{\beta\to 0}\sup_{(\nu,\mu)\in \cuS_0(\beta,\tau_0,\delta)} \phi_{\tau_1}(\mu,\nu) - I(\mu,\nu) &\le 
         \sup_{(\nu,\mu)\in \cuS_0(0,\tau_0,\delta)}\phi_{\tau_1}(\mu,\nu) - I(\mu,\nu).
    \end{align}
A similar argument allows us to pass to the limit $\tau_0\to 0$
after observing that the sequence of sets indexed by decreasing $\tau_0$ is indeed decreasing, and similarly for the limit $\delta \to 0$. Combining with Eq.~\eqref{eq:varadhan+blowup} and noting that 
$\cuV = \cuS_0(0, 0,0 )$ where $\cuV$ was defined in the statement of Theorem~\ref{thm:general}, we have
    \begin{align} 
    \label{eq:beta-tau1-gamma}
        \limsup_{\delta\to 0}\limsup_{n\to\infty} &\frac1n \log\left( \E[ Z_n(\cuA,\cuB)] \one_{\Omega_\delta}  \right) 
        \le \sup_{(\mu,\nu)\in \cuV }
        \left\{\phi_{\tau_1}(\nu,\mu) - I(\mu,\nu)\right\}.
    \end{align}

Now let us pass to the $\tau_1 \to 0$ limit.
    First, recalling the definition of $\phi_{\tau_1}$ from Lemma~\ref{lemma:asymp_1}, we note that we can write 
    \begin{equation}
       \phi_{\tau_1}(\mu,\nu) - I(\mu,\nu) = F(\mu,\nu) + \int \log(\lambda \vee \tau_1) \mu_{\star}(\mu,\nu)(\de \lambda)
    \end{equation}
    for some $F(\mu,\nu)$ that is uniformly upper bounded on $\cuM$: indeed, the definition of $\cuM$ guarantees that all terms (other than the logarithmic potential) are finite, and $-I(\mu,\nu) \le 0$ by definition.

Choosing sequences
    $\{(\mu_{i},\nu_{i})\in \cuV\}$ 
    and $\{\tau_{i}\}_{i\in\N}$ with $\tau_{i}\to 0$ satisfying
\begin{align}
 \mathrm{(I)}&:= \limsup_{\tau_1\to 0}\sup_{(\nu,\mu)\in \cuV} \{\phi_{\tau_1}(\mu,\nu) - I(\mu,\nu)  \}\\
        &= \limsup_{i\to \infty} \left\{
        F(\mu_i,\nu_i) + \int \log(\lambda \vee \tau_i) \mu_{\star}(\mu_i,\nu_i)(\de \lambda)
        \right\},
\end{align}
we have by compactness (after passing to a subsequence and relabeling) that
$(\mu_i,\nu_i)$ converge to $(\mu_0,\nu_0) \in \cuV$ so that $\mu_\star(\mu_i,\nu_i) \to \mu_\star(\mu_0,\nu_0)$ weakly.  
Since $\mu_\star$ is compactly supported for any $\mu,\nu$ so that the positive part of the logarithm is uniformly integrable, and $F(\mu,\nu)$ is uniformly upper bounded on $\cuM$, we have by Fatou's Lemma that
\begin{equation}
   \mathrm{(I)}  \le F(\mu_0,\nu_0) + \int \log(\lambda) \mu_\star(\mu_0,\nu_0)(\de \lambda) \le \sup_{(\mu,\nu)\in\cuV }  \left\{\phi_0(\mu,\nu) - I(\mu,\nu)\right\}.
\end{equation}
Finally, by noting that
$\KL(\nu\|\cN(\bzero,\bI_{k+k_0})\times\P_w) = \KL(\nu_{\cdot| w}\|\cN(\bzero,\bI_{k+k_0})) $ for $(\mu,\nu) \in\cuV$, we see that $\phi_0(\mu,\nu) - I(\mu,\nu) = -\Phi_\gen(\mu,\nu)$ for $\Phi_\gen$ as in the statement of Theorem~\ref{thm:general}.

%% file: Appendix_C_CONVEX.tex
\section{Technical lemmas for the proof of Theorems~\ref{thm:convexity},~\ref{thm:global_min} and~\ref{thm:simple_critical_point_variational_formula}}

\subsection{The logarithmic potential: Proof of Lemma~\ref{lemma:variational_log_pot}}
\label{sec:log_pot_proof}
Recall the definition of $K_z$ in Lemma~\ref{lemma:variational_log_pot}. We extend it below to complex $z$:
for any $\nu\in\cuP(\R^{k+k_0+1})$, $z \in \C \setminus \supp(\mu_{\star,0}(\nu))$ 
with $\Im(z) \ge 0$, $\bQ\in\bbH_+^k$,
\begin{equation}
    K_z(\bQ;\nu):= -\alpha z \Tr(\bQ) + \alpha \E_{\nu}[\log\det(\bI + \grad^2 \ell(\bv,\bv_0, w)\bQ) ]  - \log\det(\bQ) - k (\log(\alpha) + 1)
\end{equation}
where $\log$ denotes the complex logarithm (with a branch on the negative real axis). 

\begin{lemma}
\label{lemma:log_pot_z}
Under Assumptions \ref{ass:regime} to \ref{ass:params} of Section~\ref{sec:assumptions} along with
the additional Assumption~\ref{ass:convexity}, we have
    \begin{equation}
    \label{eq:log_pot}
        k\int\log(\zeta - z) \de \mu_{\star,0}(\zeta)
=  K_z(\bS_\star(z;\nu);\nu),
    \end{equation}  
where 
$\bS_\star(z;\nu)$ is the unique solution of \eqref{eq:fp_eq}, and was defined in Eq.~\eqref{eq:def_S_star} for $z\in\bbH_+$, and is extended by analytic continuation to $x\in\R \setminus \supp(\mu_{\star,0}(\nu))$.

Consequently, for any $\lambda \ge 0$, and $\nu$ with $\inf\supp(\mu_{\star,0}(\nu))\ge -\lambda$,
\begin{equation}
\label{eq:log_pot_0}
k\int \log(\zeta )\mu_{\star,\lambda}(\nu)(\de\zeta) \le  \limsup_{\delta \to 0+}K_{-(\lambda+\delta)}(\bS_{\star}(-(\lambda+\delta);\nu);\nu).
\end{equation}
\end{lemma}
\begin{proof}
Fix $z\in\bbH_+$.
By taking derivatives, one can easily see that $\bQ = \bS_\star(z;\nu)$ is a critical point of $K_z(\bQ;\nu)$,
whence
\begin{equation}
   \frac{\partial}{\partial z} K_z(\bS_\star(z;\nu);\nu) =  -\alpha \Tr(\bS_\star(z;\nu)) = - k  \int \frac1{\zeta - z} \mu_{\star,0}(\nu)(\de\zeta) =  
k \frac{\partial}{\partial z}\int \log(\zeta )\mu_{\star,0}(\nu)(\de\zeta).
\end{equation}
Equation \eqref{eq:log_pot} now follows by showing
\begin{align}
    \lim_{\Re(z)\to-\infty}\left| K_z(\bS_\star(z;\nu);\nu)-k\int \log(\zeta )\mu_{\star,\lambda}(\nu)(\de\zeta)
    \right| = 0\, .
\end{align}
Analytic continuation then gives the equality for $z$ on the real line outside of the support.

We next prove Eq.~\eqref{eq:log_pot_0}.
Since $\mu_{\star,\lambda}(\nu)$ is compactly supported
by Corollary~\ref{cor:S_star_min_singular_value_bound},
 we always have 
\begin{equation}
\int_1^{\infty} \log(\zeta )\mu_{\star,\lambda}(\nu)(\de\zeta)  <\infty\, . 
\end{equation}
Further, if 
\begin{equation}
\int_{0}^{1} \log(\zeta-z )\mu_{\star,\lambda}(\nu)(\de\zeta)  = -\infty,
\end{equation}
then Eq.~\eqref{eq:log_pot_0} holds trivially.
Therefore, it's sufficient to consider the case where $|\log(\zeta)|$ is integrable with respect to $\mu_{\star,\lambda}$ for any $\lambda\ge 0$. In this case, Eq.~\eqref{eq:log_pot_0} follows directly by domination:
\begin{align}
 \int \log(\zeta )\mu_{\star,\lambda}(\nu)(\de\zeta) 
= 
\lim_{\delta \to 0+}\int \log(\zeta  + \delta)\mu_{\star,\lambda}(\nu)(\de\zeta) 
=
\lim_{\delta \to 0+}
K_{-(\lambda+\delta)}(\bS_{\star}(-(\lambda+\delta);\nu);\nu).
\end{align}
\end{proof}

\begin{lemma}[Local strict convexity of $K$]
\label{lemma:strict_convexity_K}
Fix  $x\in\reals_{\ge 0}$.
Under the 
Assumptions \ref{ass:regime} to \ref{ass:params} 
 of Section~\ref{sec:assumptions} along with Assumption~\ref{ass:convexity}, if $\bS\succ\bzero$ satisfies
\begin{equation}
\label{eq:derivative_K_0}
    \alpha^{-1} \bS^{-1} -\E_\nu[(\bI + \grad^2\ell \bS)^{-1}\grad^2\ell] = x\bI,
\end{equation}
then $\bS \mapsto K_{-x}(\bS; \nu)$ is
strictly convex at $\bS$.
\end{lemma}

\begin{proof}
For ease of notation, denote $\bW :=  \grad^2\ell$ and suppress its arguments throughout.
 For any $\bZ\in\R^{k\times k}$ symmetric, let $\bH_\bS(\bZ)$ denote the second derivative tensor of $K_{-x}$ at $\bS$, applied to $\bZ$.
Now let $\bS_0$ be a point satisfying the critical point equation~\eqref{eq:derivative_K_0}.
To save on notation, 
we denote
\begin{equation}
   \bM := \bM(\bS_0;\bW):=  \bS_0^{1/2}\bW^{1/2}(\bI+\bW^{1/2} \bS_0\bW^{1/2})^{-1} \bW^{1/2}\bS_0^{1/2},\quad\textrm{and}\quad
   \bA := \bS_0^{-1/2}\bZ\bS_{0}^{-1/2};
\end{equation}
Note that since
 by Assumption~\ref{ass:convexity}, we have $\bW \succeq \bzero$ almost surely, $\bW^{1/2}$ exists. 
We have
\begin{align}
  \frac1\alpha\Tr( \bZ \bH_{\bS_0}( \bZ)) &= \frac1\alpha\Tr(\bZ \bS_0^{-1} \bZ \bS_0^{-1})\\
  &\quad- 
  \E[
\Tr(\bS_0^{-1/2}\bZ \bS_0^{-1/2}\bM(\bS_0;\bW) \bS_0^{-1/2}\bZ\bS_0^{-1/2}\bM(\bS_0;\bW) )]\\
&= \frac1\alpha \Tr(\bA^2) - \E[\Tr(\bA\bM\bA\bM)]
\label{eq:second_derivative_K_0}
\end{align}
Since $\bM$ is PSD, we have
\begin{equation}
   \E[\Tr(\bA\bM\bA\bM)]  \le \E[\|\bM\|_\op \Tr(\bA\bM\bA)] \stackrel{(a)}{\le} \E[ \Tr(\bA\bM\bA)],
\end{equation}
where in $(a)$, we used that $\bM \prec\bI$.
If $\E[\Tr(\bA\bM\bA)] =0,$ we are done since the lower bound on the derivative would be given by $\alpha^{-1}\Tr(\bA^2)$, which is strictly positive for any $\bZ\neq \bzero$, since $\bS_0 \succ\bzero.$
So it's sufficient to complete the proof assuming $\E[\Tr(\bA\bM\bA)] >0$.
In this case, on a non-zero probability set, $\Tr(\bA\bM\bA)> 0$, and
since we always have $\|\bM\|_\op < 1$, we conclude that inequality $(a)$ holds strictly.
Then using that 
for $\bS_0$ satisfying Eq.~\eqref{eq:derivative_K_0}, we have 
$\E[\bM(\bS_0;\bW)] = \alpha^{-1}\bI   - x\bS_0$, we have
\begin{align}
\frac1\alpha \Tr(\bA^2) - \E[\Tr(\bA\bM\bA)] = x  \Tr(\bZ\bS_0^{-1} \bZ) \ge 0
\end{align}
for all $x\ge0$.
Combining with Eq.~\eqref{eq:second_derivative_K_0} completes the proof.
\end{proof}

We move on to the proof of Lemma~\ref{lemma:variational_log_pot}
\begin{proof}[Proof of Lemma~\ref{lemma:variational_log_pot}]
Without loss of generality, assume that $\log(\zeta)$ is absolutely integrable under $\mu_{\star,\lambda}$
for all $\lambda\ge 0$. Indeed, for $\lambda>0$ this holds because $\mu_{\star,\lambda}$ is compactly supported inside
$[\lambda,\infty)$.  For $\lambda=0$, the positive part of $\log(\zeta)$ is integrable because
$\mu_{\star,0}$ is compactly supported, and  lack of absolute integrability implies that the integral diverges to $-\infty$.

Under Assumption~\ref{ass:convexity}, we have $\supp(\mu_{\star,0}(\nu))\subseteq[0,\infty),$  so for any $\lambda \ge0,$ Eq.~\eqref{eq:log_pot_0} of Lemma~\ref{lemma:log_pot_z} yields
\begin{equation}
k\int \log(\zeta )\mu_{\star,\lambda}(\nu)(\de\zeta) \le  \limsup_{\delta \to 0}K_{-(\lambda+\delta)}(\bS_{\star}(-(\lambda+\delta);\nu);\nu).
\end{equation}

Now applying Lemma~\ref{lemma:strict_convexity_K} with $x = \delta + \lambda$:
In particular, this gives that at the point $\bS_\star(-(\delta+\lambda);\nu)$ which satisfies Eq.~\eqref{eq:derivative_K_0}, the continuous function $\bS \mapsto K_{-(\delta+\lambda)}(\bS;\nu)$ is strictly convex, implying that $\bS_\star(-(\delta+\lambda);\nu)$ is the unique global minimum. 
Combining this with the above display gives
\begin{align}
k\int \log(\zeta )\mu_{\star,\lambda}(\nu)(\de\zeta) 
&\le \limsup_{\delta\to 0+}  \inf_{\bS\succ\bzero} K_{-(\lambda+\delta)}(\bS;\nu)
\le 
  \inf_{\bS\succ\bzero}
  \limsup_{\delta\to 0+} 
  K_{-(\lambda+\delta)}(\bS;\nu) \\
  &= 
  \inf_{\bS\succ\bzero}
  K_{-\lambda}(\bS;\nu),
\end{align}
as desired.
\end{proof}

\subsection{Simplifying the constraint set in Theorem~\ref{thm:global_min}}
\label{sec:simplifying_constraint_set}
We state and prove the two lemmas referenced in the proof of Theorem~\ref{thm:global_min} that allow us to simplify the set of critical points on which the rate function bound is applicable. 
\begin{lemma}[Lower bounding the smallest singular value of the Jacobian]
\label{lemma:jacobian_lb}
Assume $\rho(t) = \lambda \; t^2/2$ for $\lambda \ge0$.
For any critical point $\bTheta$ of $\widehat R_n(\bTheta)$, 
we have under Assumption~\ref{ass:loss},
\begin{equation}
    \sigma_{\min}\left( \bJ_{(\bTheta,\bbV)} \bG^\sT\right) \ge 
     \frac{
   \sigma_{\min}\Big((\bI_k \otimes [\bTheta,\bTheta_0]^\sT)\grad^2 \hat R_n(\bTheta)(\bI_k \otimes [\bTheta,\bTheta_0])\Big) 
     }{(\|\bX[\bTheta,\bTheta_0]\|_\op + \|[\bTheta,\bTheta_0]\|_\op )}.
\end{equation}
\end{lemma}

\begin{proof}
Recall that $\bJ_{(\bTheta,\bbV)}\bG\in\reals^{k(k+k_0)\times (dk + n (k+k_0))}$ denotes the Euclidean Jacobian of the function $\bg : \R^{dk + n (k+k_0)} \to \R^{k(k+k_0)}$ obtained from vectorizing $\bG$ and its arguments.
Namely,
\begin{equation}
    \bg(\bTheta,\bbV) = (g_{i,j}(\bTheta,\bbV)_{i\in[k],j\in[k+k_0]},\quad
    g_{i,j}(\bTheta,\bbV) = \begin{cases}
       \frac1n\bell_i^\sT\bv_j  + \lambda \btheta_i^\sT\btheta_j & j \le k\\
       \frac1n\bell_i^\sT\bv_{0,j}  + \lambda \btheta_i^\sT\btheta_{0,j - k} & 
       j  > k
    \end{cases},
    i \in[k].
\end{equation}

Recalling the definitions in Eq.~\eqref{eq:SecDef} of $\bSec$ and $\tilde\bSec$, with sufficient diligence, the desired Jacobian can be computed to be 
\begin{align}
    \bJ\bG^\sT &= \begin{bmatrix}
   \lambda\begin{bmatrix}
          \bI_k\otimes \btheta_1, &
         \dots,&
         \bI_k \otimes \btheta_k 
   \end{bmatrix}
   &
   \lambda\begin{bmatrix}
          \bI_k\otimes \btheta_{0,1},&
         \dots,&
         \bI_k \otimes \btheta_{0,k_0} 
   \end{bmatrix} 
    \\
    \frac1n
     \begin{bmatrix}
         \bSec (\bI_k \otimes \bv_1),& \dots ,&
         \bSec (\bI_k \otimes \bv_k) 
    \end{bmatrix} 
&
   \frac1n
     \begin{bmatrix}
         \bSec (\bI_k \otimes \bv_{0,1}), &
         \dots,&
         \bSec (\bI_k \otimes \bv_{0,k_0}) \\
    \end{bmatrix}\\
    \frac1n
   \begin{bmatrix}
         \tilde\bSec (\bI_k \otimes \bv_1), &
         \dots,&
         \tilde\bSec (\bI_k \otimes \bv_k) \\
   \end{bmatrix} 
   &
    \frac1n
   \begin{bmatrix}
         \tilde\bSec (\bI_k \otimes \bv_{0,1}), &
         \dots,&
         \tilde\bSec (\bI_k \otimes \bv_{0,k_0}) 
   \end{bmatrix} 
    \end{bmatrix}\\
    &+
     \begin{bmatrix}
       [\lambda(\bI_k \otimes \bTheta), \bzero_{dk \times kk_0}]\\
       \frac1n(\bI_{k + k_0} \otimes \bL) 
    \end{bmatrix} \in\R^{(nk + nk_0 + dk) \times k(k+k_0)}.
\end{align}

Define
\begin{equation}
    \bB := \begin{bmatrix}
       (\bI_k \otimes \bX^\sT)  &
        \bzero_{dk\times n k_0} &
        \bI_{dk}
    \end{bmatrix} \in\R^{dk \times (nk + nk_0 + dk)}.
\end{equation}
Recalling the definition of $\bH_0$ in Eq.~\eqref{eq:bH_def} and
noting that at any critical point of $\hat R_n(\bTheta)$, we have
\begin{equation}
  \bX^\sT\bK[\bv_i,\bv_{0,j}] = \bH_0 [\btheta_i , \btheta_{0,j}]
  \quad
  \textrm{and}
  \quad
  \bX^\sT\bell_{i}  = - \lambda \btheta_i \quad \textrm{for}\quad i\in[k],j\in[k_0],
\end{equation}
we can compute at a critical point
\begin{equation}
    \bB \bJ\bG^\sT = 
       \left(\frac1n \bH_0 + \lambda\bI_{dk}\right)\bA,
       \end{equation}
       where
       \begin{equation}
       \bA := \left[(\bI_k \otimes \btheta_1),\dots,(\bI_k\otimes\btheta_k),(\bI_k \otimes \btheta_{0,1},\dots,(\bI_k\otimes\btheta_{0,k_0})\right].
\end{equation}

Up to a permutation $\bP\in\R^{dk\times dk}$ of the columns of $\bA$, we have
\begin{equation}
\bA\bP = \bI_k \otimes [\bTheta,\bTheta_0].
\end{equation}

So 
\begin{align}
   \sigma_{\min}\Big((\bI_k \otimes [\bTheta,\bTheta_0]^\sT)\Big(\frac{\bH_0}{n} + \lambda\bI_{dk}\Big)(\bI_k \otimes [\bTheta,\bTheta_0])\Big) 
   &=
   \sigma_{\min}\Big( (\bI_k \otimes [\bTheta,\bTheta_0])^\sT\bB \bJ\bG^\sT\bP \Big)\\
   &\le \sigma_{\min}\big( \bJ \bG^\sT\big) \|(\bI_k \otimes [\bTheta,\bTheta_0])^\sT\bB\|_\op.
\end{align}

Using $\|(\bI_k\otimes [\bTheta,\bTheta_0])^\sT \bB\|_\op \le \|\bX[\bTheta,\bTheta_0]\|_\op + \|[\bTheta,\bTheta_0]\|_\op$ gives the claim.
\end{proof}

\begin{lemma}\label{lemma:VolumeBound}
Assume $\sigma_{\min}(\bTheta_0) \succ r\bI$ for some $r>0$.
Define
   \begin{equation}
       \cS_{\delta,R} := \{ \bTheta \in\R^{d\times k} : \|\bTheta\|_F \le R,\quad \sigma_{\min}([\bTheta,\bTheta_0]) \le \delta \}
   \end{equation}
   for $\delta < r/2$.
Then 
\begin{equation}
    \vol(\cS_{\delta,R}) \le
     k (C(R,r))^{dk} (\sqrt{k} \delta)^{d - k - k_0+1}.
\end{equation}
for constant $C(R,r)>0$ depending only on $r$ and $R.$
\end{lemma}
\begin{proof}
If $\bTheta \in\cS_{\delta,R}$, then there exists some $(\bbeta^\sT,\bbeta_0^\sT)^\sT\in\R^{k+k_0}$ with norm 1 such that $\bTheta\bbeta+\bTheta_0\bbeta_0 = \bu$ for some $\bu$ with $\|\bu\|_2 \le \delta.$
Let $j = \argmax_{i\in[k]} |\beta_i|$, where $\beta_i$ is the $i$-th coordinate of $\bbeta$, and denote by $\bP^\perp_{j}$ the projection onto the orthocomplement of 
\begin{equation}
    \textrm{span}\left(\{\btheta_{0,i}\}_{i\in[k_0]} \cap \{\btheta_{i}\}_{i\in[k], i\neq j} \right).
\end{equation}
Since
$\bu = \sum_{i=1}^k \beta_i \btheta_i + \sum_{i=1}^{k_0} \beta_{0,i} \btheta_{0,i},$
we have $\delta \ge |\beta_j|\|\bP_{j}^\perp \btheta_j\|_2$.
Now note that $\sigma_{\min}(\bTheta_0) > 2\delta$
and $\|\bTheta\|_F \le R$ implies that there must exist some constant $c_0(R,c)$ depending only on $R$ and $r$ and such that $\|\bbeta\|_2 \ge c_0(R,c).$ 
Indeed, we have
\begin{equation}
    r \|\bbeta_0\|_2 \le \|\bTheta_0 \bbeta_0\| = \|\bu - \bTheta\bbeta\| \le \delta + R \|\bbeta\|_2.
\end{equation}
Using that $\|\bbeta\|_2^2 + \|\bbeta_0\|_2^2 = 1$ and $\delta < r/2$, this then gives
\begin{equation}
   \|\bbeta\|_2^2 \ge \frac12\frac{r^2}{r^2 + 2 R^2}.
\end{equation}
Hence, for $j$ being the index of maximum mass as above, we have $|\beta_j| \ge c_0 k^{-1/2}$. This allows us to conclude that
$\cS_{\delta,R} \subseteq \bigcup_{j=1}^k  \cV_{j}(\delta, R)$
where
\begin{equation}
    \cV_{j}(\delta, R) := \left\{\|\bTheta\|_F \le R ,\quad   \|\bP_{-j} \btheta_j\|_2 \le \frac{\sqrt{k}}{c_0} \delta\right\}.
\end{equation}
Meanwhile, for any $j \in[k]$,
\begin{equation}
    \vol( \cV_j(\delta,R)) \le  (C R)^{dk} \left(\frac{\sqrt{k} \delta}{c_0}\right)^{d - k - k_0+1}.
\end{equation}
Bounding the volume of the union by the sum of the volumes gives the claim.

\end{proof}

\begin{lemma}
\label{lemma:min_sv_Theta}
Let $\bL(\bV,\bV_0,\bw)\in\R^{n\times k}$ be as defined in~\eqref{eq:def_bL_bRho}.
Assume $\sigma_{\min}(\bTheta_0) \succ c_0\bI$ for some $c_0 >0$.
Then under Assumption~\ref{ass:loss} on the loss, with the ridge regularizer $\rho(t) = \lambda t^2/2$,
for any fixed $C,c>0$ and $\lambda \ge 0$, there exists $\delta >0$ sufficiently small such that
\begin{equation}
\lim_{n\to\infty}\P\left( \exists \bTheta :\sigma_{\min}([\bTheta,\bTheta_0]) < \delta,\;  \sigma_{\min}(\bL(\bX\bTheta,\bX\bTheta_0,\bw)) \ge  \sqrt{n}\, c,\; \|\bTheta\|_F \le C,\;
\grad \hat R_n(\bTheta)  = \bzero
\right)   = 0.
\end{equation}
\end{lemma}
\begin{proof}
Let
\begin{equation}
    \bF(\bTheta) := \frac1{\sqrt{n}} \bX^\sT \bL(\bX\bTheta,\bX\bTheta_0,\bw) + \sqrt{n}\lambda \bTheta\, ,
\end{equation}
be the scaled gradient of the empirical risk.
Conditional on $\bX[\bTheta,\bTheta_0] = \bbV$ and $\bw,$ the random variable $\bP_{[\bTheta,\bTheta_0]}^\perp\bF(\bTheta)$ is  distributed as $\bU_{\bTheta}\bZ_0$ where  $\bU_{\bTheta}\in\reals^{n\times (d-k-k_0)}$ is a basis of the orthogonal complement of $[\bTheta,\bTheta_0]$
and
\begin{equation}
    \bZ_0 \sim \cN(\bzero, \bI_{d- k - k_0} \otimes \bL^\sT\bL/n)\, .
\end{equation}
Therefore we can bound for any fixed $\bTheta$,
\begin{align}
\label{eq:small_ball_prob}
&\P\left( \|\bF(\bTheta)\|_F \le \eps \sqrt{nk} , \bL^\sT\bL \succ c n  \right)\\
\nonumber
&\le\E\left[\P\left( \|\bF(\bTheta)\|_F \le \eps \sqrt{nk} , \bL^\sT\bL \succ c n  \Big| \bX[\bTheta,\bTheta_0]=\bbV, \bw \right)\right]\\
&\le (C_0 \eps)^{dk - k - k_0}\, ,\nonumber
\end{align}
for some $C_0>0$ depending only on $c$.
For fixed $\delta,\eps >0$, consider now the sets
\begin{equation}
    \cA_0(\delta) := \{\bTheta : \sigma_{\min}(\bTheta,\bTheta_0) \le \delta\},\quad\quad
    \cA(\delta,\eps) := \{\bTheta \in\cA_0(\delta) : \|\bF(\bTheta)\|_F \le \eps\}.
\end{equation}
Since Assumption~\ref{ass:loss} guarantees that for any $\bTheta_1,\bTheta_2$,
\begin{equation}
    \|\bF(\bTheta_1) - \bF(\bTheta_2)\|_F^2 \le C_1 \,k \frac1n \|\bX\|^4_\op  \|\bTheta_1 -\bTheta_2\|_F^2,
\end{equation}
We have on the high probability event $\Omega_0 := \{\|\bX\|_\op \le 2( \sqrt{n} +\sqrt{d})\},$
for any $\tilde\bTheta \in \cA_0(\delta)$ with $\bF(\tilde\bTheta) = \bzero$, 
\begin{equation}
\|\bTheta - \tilde \bTheta\|_F \le \eps/2 \;\;\Rightarrow\;\;
    \|\bF(\bTheta)\|_F  \le C_2 \sqrt{k n} \eps,\quad
\sigma_{\min}([\bTheta,\bTheta_0]) \le \delta + \eps\,,
\end{equation}
for some $C_2$ depending on $\alpha>0$.
Namely, letting $\Ball_{\eps/2}^{d\times k}(\bzero)$ be the Euclidean ball in $\R^{d\times k}$ of radius $\eps/2$, this 
shows that
\begin{equation}
\cA_0(\delta) + \Ball_{\eps/2}^{d\times k}(\bzero)  \subseteq   \cA(\delta + \eps , C_2 \sqrt{k n} \eps).
\end{equation}
And since $\cA(\delta + \eps , C_2 \sqrt{k n} \eps) \subseteq \cA_0(\delta +\eps)$, standard bounds on the covering number $\cN_{\eps}( \cA_0(\delta)\cap\Ball_{C}^{d\times k}(\bzero))$ of $\cA_0(\delta)\cap\Ball_{C}^{d\times k}(\bzero)$ with Euclidean balls of radius $\eps$ give
\begin{align}
   \cN_\eps(
   \cA_0(\delta)\cap\Ball_C^{d\times k}(\bzero))\;
   \vol(\Ball_{\eps/2}^{d\times k}(\bzero))  &\le  \vol\big((\cA_0(\delta)\cap\Ball_C^{d\times k}(\bzero)) + \Ball_{\eps/2}^{d\times k}(\bzero)\big) \\
   &\le   \vol(\cA_0(\delta +\eps)\cap\Ball_C^{d\times k}(\bzero)),
\end{align}
where $``+"$ denotes the Minkowski sum of sets.
Therefore by~\eqref{eq:small_ball_prob}, 
\begin{align*}
   &\P\left(\exists \bTheta \in \cA_0(\delta) : \bF(\bTheta) = \bzero, \bL^\sT\bL \succ n c,\; \|\bTheta\|_F \le C\right)   \\
  &\le 
  \P\left(
  \exists \bTheta \in \cN_\eps(\cA_0(\delta)) : \|\bF(\bTheta)\|_F \le 2 C_2 \sqrt{kn} \eps, \;
  \bL^\sT \bL \succ c_0 n,\; 
  \|\bTheta\|_F \le C
  \right)\\
  &\le \cN_\eps(\cA_0(\delta) \cap\Ball_C^{d\times k}(\bzero)) (2 C_2 C_0 \eps)^{dk - k - k_0}\\
  &\le (C_3 \eps)^{dk - k- k_0} \left(\frac{1}{C_4 \eps}\right)^{dk}   \vol(\cA_0(\delta +\eps) \cap \Ball_{C}^{d\times k}(\bzero))\\
  &\stackrel{(a)}{\le} (C_3 \eps)^{dk - k- k_0} \left(\frac{1}{C_4 \eps}\right)^{dk}  k (C_5)^{dk}  (\sqrt {k} (\delta + \eps) )^{d -k -k_0 -1}\\
  &\le k C_6^{dk} \frac{(\sqrt{k}(\delta+\eps))^{d-k-k_0-1}}{\eps^{k+k_0}}\, ,
\end{align*}
where in step $(a)$ we used Lemma \ref{lemma:VolumeBound}. 
Now choose $\delta,\eps$ sufficiently small so that the latter quantity converges to $0$ as $n\to\infty$.
\end{proof}

\subsection{Proof of Theorem~\ref{thm:simple_critical_point_variational_formula}}

\noindent  \emph{Claim  1.}
To see that $\Risk$ is convex, define $\cuF:   L^2\times\sfS^k_{\ge}\times \R^{k\times k_0}\to \R$  via
\begin{align}
\cuF(\bu,\bK,\bM):= \E[\ell(\bu + \bK \bz_1 + \bM \bz_0, \bR_{00}^{1/2}\bz_0, w)]  + \frac{\lambda}{2}(\bK^2 + \bM\bM^\sT)\, .\label{eq:cF_def}
\end{align}
Since $\cuF$ is jointly convex, the convexity of $F$ will follow if we conclude that the set
\begin{equation}
    \cA := \{(\bK,\bu) \in \sfS^k \times L_2 : \E[\bu\bu^\sT]\preceq \alpha^{-1}\bK^2\}
    \label{eq:ConvexSetA}
\end{equation}
is jointly convex. This follows by defining the
set
\begin{equation}
    \widetilde\cA := \big\{(\bK,\bu,\bB) \in \sfS^k \times L_2\times\sfS^k\, :
    \; \E[\bu\bu^\sT]\preceq \bB \, , \bK\bB^{-1}\bK \succeq \alpha\bI\big\}\, .
\end{equation}
We note that, by using matrix monotonicity of the function $\bA\mapsto \bA^{-1}$,
that the projection of  $\widetilde\cA$ onto  $(\bK,\bu)$ coincides with $\cA$.
Further,  the constraint  $\E[\bu\bu^\sT]\preceq \bB$ is obviously convex, and 
 $\bK\bB^{-1}\bK \succeq \alpha\bI$ is equivalent (since $\bB\succeq \bzero$) to 
\begin{align}
\left(
\begin{matrix}
    \alpha \bI & \bB\\
    \bB & \bK
\end{matrix}\right) \preceq \bzero\, ,
\end{align}
which is also convex.

\vspace{0.2cm}

\noindent\emph{Claim  2.}
Let $\bg = \bK\bz_1 + \bM \bz_0, \bg_0  = \bR_{00}^{1/2} \bz_0$,
and write
\begin{align}
    &\Risk(\bK,\bM)
=
\inf_{\bu \in \cS(\bK)} 
\left\{
\E\left[\ell(\bu+ \bg,\bg_0,w) 
\right]
+
\frac{\lambda}{2}\bR_{11}
\right\}\\
&=
\inf_{\bu \in L^2} 
\sup_{\bQ\succ \bzero} 
\left\{
\E\left[\ell(\bu+ \bg,\bg_0,w) 
\right]
+
\frac12\Tr\left(
(\E[
\bu\bu^\sT]
- \alpha^{-1} \bK^2
)\bQ\right)+
\frac{\lambda}{2}\bR_{11}
\right\}\\
&= 
\sup_{\bQ\succ \bzero} 
\inf_{\bu \in L^2} 
\left\{
\E\left[\ell(\bu+ \bg,\bg_0,w) 
\right]
+
\frac12\Tr\left(
(\E[
\bu\bu^\sT]
- \alpha^{-1} \bK^2
)\bQ\right)+
\frac{\lambda}{2}\bR_{11}
\right\}\\
&= 
\sup_{\bS\succ \bzero} 
\inf_{\bx \in L^2} 
\left\{
\E\left[\ell(\bx,\bg_0,w) 
+\frac12 (\bg - \bx)^\sT\bS^{-1}(\bg-\bx)
\right]
-\frac1{2\alpha} \Tr\left(\bS^{-1}\bK^2\right) +
\frac{\lambda}{2}\bR_{11}
\right\}\\
&=\sup_{\bS\succ \bzero}  \cuG(\bK,\bM,\bS).
\end{align}
where $G(\bK,\bM, \bS)$ is the objective in Eq.~\eqref{eq:min_max_critical_G_def}, which we repeat here for the reader's convenience,
\begin{equation}
    \cuG(\bK,\bM,\bS) :=
       \E\left[\More_{\ell(\cdot, \bg_0,w)}(\bg;\bS)\right] - \frac1{2\alpha}\Tr(\bS^{-1}\bK^2) 
       + \frac{\lambda}{2}\Tr(\bK^2 + \bM\bM^\sT)\, ,\label{eq:cuG_copied}
\end{equation}
with the Moreau envelope defined in Eq.~\eqref{eq:moreau_def}.
Now, by straightforward differentiation of $\cuG(\bK,\bM,\bS)$ with respect to each of $\bK,\bM,\bS$, one can show that the critical points of $\cuG(\bK,\bM,\bS)$ are given by $(\bK,\bM, \bS) = (\bR^\opt/\bR_{00},\bR_{10}^\opt \bR_{00}^{-1},\bS^\opt)$ by checking that the stationarity conditions corresponds to Eq.~\eqref{eq:opt_fp_eqs}.
Furthermore, by definition, $\bS^\opt(\bR)$, the solution of~\eqref{eq:opt_fp_eqs} is unique for each $\bR$ (as the limit of the Stieltjes Transform $\bS_\star$ or as the minimizer of a strongly convex program as seen in the proof of Theorem~\ref{thm:global_min}).
Then by differentiation of $\bG$ with respect to $\bS$, one can show that $G(\bK,\bM,\bS)$ is locally concave at $\bS = \bS^\opt$ for fixed $\bR.$ This shows that indeed $\bS^\opt(\bR)$ is the maximizer of $G(\bK,\bM,\bS).$
Combined with the convexity of $\Risk(\bK,\bM)$ proves the claim.

\vspace{0.2cm}

\noindent\emph{Claim  3.}
\emph{Point $(a)$:}  For $\lambda>0$, $(\bK,\bM)\mapsto \Risk(\bK,\bM)$ is strongly convex, and therefore
the minimizer exists and is unique.

\emph{Point $(b)$:} In this case $\cuF$ is strictly convex, namely, for any $t\in (0,1)$
and $(\bu_0,\bK_0,\bM_0)\neq (\bu_1,\bK_1,\bM_1)$, letting 
$(\bu_t,\bK_t,\bM_t) = (1-t)(\bu_0,\bK_0,\bM_0)+t (\bu_1,\bK_1,\bM_1)$, we have
\begin{align}
\cuF(\bu_t,\bK_t,\bM_t) < (1-t)\cuF(\bu_0,\bK_0,\bM_0)+t \cuF(\bu_1,\bK_1,\bM_1)\, ,
\end{align}
strictly. This immediately implies that $\Risk$ is strictly convex and therefore 
it cannot have multiple minimizers.

\emph{Point $(c)$.} By point $(a)$ we can assume $\lambda = 0$. Further, by continuity, there exists 
an open set $U_0\subseteq \reals^k$, $U_0\ni \bu_0$ and a constant $c_0>0$
such that $\nabla^2_{\bg}\ell(\bg;\bR_{00}^{1/2}\bz_0,w)\succeq c_0\bI_k$ for all $\bg\in U_0$
and $\by=(\bg_0,w):= (\bR_{00}^{1/2}\bz_0,w)\in \cE = \cE_0\times \cE_1$.
For brevity, we will write $(\bK, \bM) = \bL\in\reals^{k\times (k+k_0)}$
and $\bz= (\bz_1,\bz_0)$, $\bz\sim\cN(\bzero,\bI_{k+k_0})$.

Using these notations
\begin{align}
\Risk(\bL) = \sup_{\bS\succeq \bzero} \cuG(\bL,\bS) = \inf\Big\{\cuF(\bu,\bL) :\; (\bu,\bL) \in \cS_0\Big\}\, ,
\end{align}
 where $\cS_0:= \{(\bu,\bL=(\bK,\bM)): \, \E[\bu\bu^{\sT}]\preceq \bK^2/\alpha\}$.

 We recall that (using the definition in Eq.~\eqref{eq:Prox-Def})
 \begin{align}
 \bJ_{\bg}\Prox_{\ell(\,\cdot\, ;\by)}(\bg;\bS) &= \Big(\bI+\bS\nabla^2\ell(\bx;\by)\Big)^{-1}\, ,\\
 \bx & = \Prox_{\ell(\,\cdot\, ;\by)}(\bg;\bS)\, .
 \end{align}
 In particular (using the fact that $\bS^{\opt}\succ \bzero$), $\bJ_{\bg}\Prox_{\ell(\,\cdot\, ;\by)}(\bg;\bS^{\opt})|_{\bg=\bu_0}-\bI$ is nonsingular
 on the event $\by\in \cE$ which has strictly positive probability. As a consequence, if 
 $(\bg,\bg_0)\in\reals^{k+k_0}$ is a non-degenerate Gaussian (which is the case for $\bK\succ \bzero$)
 then
 \begin{align}
 \E\big\{(\bx-\bg)(\bx-\bg)^{\sT}\}\succ c_*\bI_k\, ,\label{eq:x-g}
 \end{align}
 for some $c_*>0$
 
 Since the supremum 
 $\sup_{\bS\succeq\bzero} \cuG(\bL^{\opt},bS)$
 is achieved at $\bS^{\opt}\succ \bzero$, it follows by strong duality that the 
 $\inf_{\bu}\cuF(\bu,\bL^{\opt})$ is achieved at:
 \begin{align}
 \bu^{\opt} =  \Prox_{\ell(\,\cdot\,,\by)}( \bL^{\opt}\bz; \bS^{\opt})-\bL^{\opt}\bz\, .
 \end{align}
Note that (for $\bK^{\opt}\succ \bzero$)  $\bL^{\opt}\bz$ is a non-degenerate Gaussian random vector
and hence has a density which is everywhere positive. 
Further  (for $\bS^{\opt}\succ \bzero$) $\bv\mapsto \Prox_{\ell(\,\cdot\,,\by)}( \bv; \bS^{\opt})$ is a diffeomorphism.
Hence, $\bv^{\opt} :=\bu^{\opt} +\bL^{\opt}\bz$ has a density which is everywhere
positive. 

Let $\ocuG(\bL,\bQ)= \cuG(\bL,\bQ^{-1})$. By differentiating 
Eq.~\eqref{eq:cuG_copied} we get, for any $\bZ\in\reals^{k\times k}$ symmetric
\begin{align}
\nabla_{\bQ}^2\ocuG(\bL^{\opt},\bQ)[\bZ,\bZ] &= 
\E\<(\bx-\bg),\bZ\Big(\nabla^2_{\bx,\bx}\ell(\bx;\by) +\bQ\Big)^{-1}\bZ(\bx-\bg)\>\, ,\\
\bx & = \Prox_{\ell(\,\cdot\,,\by)}( \bg; \bQ^{-1})\, ,\;\;\;\;\;\; \bg = \bL^{\opt}\bz\, .
\end{align}
Since $\bzero \preceq \nabla^2_{\bx,\bx}\ell(\bx;\by)\preceq C\bI$ for some constant $C$,
using Eq.~\eqref{eq:x-g} we have, for some $c_1>0$
\begin{align}
\nabla^2_{\bQ}\ocuG(\bL^{\opt},\bQ)[\bZ,\bZ] &\ge c_1
\Tr\Big\{\E[(\bx-\bg)(\bx-\bg)^{\sT}] \bZ^2\Big\}\ge c_2\|\bZ\|_F^2\, .
\end{align}
In particular $\bS^{\opt}$ is the unique maximizer of $\bS \mapsto \cuG(\bL^{\opt},\bS)$.
By a continuity argument, there exists a neighborhood $\Ball(\bL^{\opt},\eps)$
of $\bL^{\opt}$
such that, for $\bL\in \Ball(\bL^{\opt},\eps)$, $\bS \mapsto \cuG(\bL,\bS)$
is uniquely maximized as some $\bS_*(\bL)\succ \bzero$.
We will next show that $\Risk(\bL)>\Risk(\bL^{\opt})$ for any 
$\bL\in \Ball(\bL^{\opt},\eps)\setminus \{\bL^{\opt}\}$.

Fix $\bL\in \Ball(\bL^{\opt},\eps)\setminus \{\bL^{\opt}\}$
let $\bS = \bS_*(\bL)$ and note that 
\begin{align}
\Risk(\bL) = \cuF(\bu,\bL)\,,\;\;\;\; \bu = \Prox_{\ell(\,\cdot\;\by)}(\bL\bz;\bS)\, .
\end{align}
 Define $\bL_t = t\bL+(1-t)\bL^{\opt}$ and $\bu_t = t\bu+(1-t)\bu^{\opt}$.
 By convexity of $\cS_0$, $\Risk(\bL_t)\le \cuF(\bu_t,\bL_t)$ while
 while  $\Risk(\bL_0)= \cuF(\bu_0,\bL_0)$ by definition. 
 Further $t\mapsto \cuF(\bu_t,\bL_t)$ is twice differentiable by Assumption \ref{ass:loss}
 and dominated convergence.
 We then
 have
 \begin{align}
\Risk(\bL_t) \le\Risk(\bL_0)  + \left.\frac{\de\phantom{t}}{\de t}  \cuF(\bu_t,\bL_t) \right|_{t=0} t+o(t)\, .
 \end{align}
Since  $\Risk(\bL_0)=\Risk(\bL^{\opt})\le \Risk(\bL_t)$ by optimality, this implies 
$\left.\frac{\de\phantom{t}}{\de t}  \cuF(\bu_t,\bL_t) \right|_{t=0}\ge 0$.

We further have
 \begin{align}
    \Risk(\bL)-\Risk(\bL^{\opt}) &= \int_0^1  (1-t) \frac{\de^2\phantom{t}}{\de t^2}  \cuF(\bu_t,\bL_t) \, \de t \, .
 \end{align}
 A direct calculation yields
 \begin{align}
     \frac{\de\phantom{t}}{\de t}  \cuF(\bu_t,\bL_t) & = 
     \E\<\nabla\ell(\bu_t+\bL_t\bz;\by), [\bu-\bu_0+(\bL-\bL_0)\bz]\> \, ,\\
     \frac{\de^2\phantom{t}}{\de t^2}  \cuF(\bu_t,\bL_t) & = 
        \E\<\nabla^2\ell(\bu_t+\bL_t\bz;\by), [\bu-\bu_0+(\bL-\bL_0)\bz]^{\otimes 2}\> \, ,
 \end{align}
 whence
 \begin{align*}
  &\left.\frac{\de^2\phantom{t}}{\de t^2}  \cuF(\bu_t,\bL_t) \right|_{t=0} = 
        \E\<\nabla^2\ell(\bu^{\opt}+\bL^{\opt}\bz;\by), [\bu-\bu^{\opt}+(\bL-\bL^{\opt})\bz]^{\otimes 2}\> \\
        & \ge c_0\E\Big\{\|\bu-\bu^{\opt}+(\bL-\bL^{\opt})\bz \|^2\bfone_{\bu^{\opt}+\bL^{\opt}\bz\in U_0}\bfone_{\bR_{00}^{1/2}\bz_0\in \cE_0}
        \bfone_{w\in \cE_1}
        \Big\}\\
        & = c_0\E\Big\{\|\Prox_{\ell(\,\cdot\,;\by)}(\bL\bz;\bS)- \Prox_{\ell(\,\cdot\,;\by)}(\bL^{\opt}\bz;\bS^{\opt}) \|^2\bfone_{\Prox_{\ell(\,\cdot\,;\by)}(\bL^{\opt}\bz;\bS)\in U_0}\bfone_{\bR_{00}^{1/2}\bz_0\in \cE_0}
           \bfone_{w\in \cE_1}\Big\}\, .
\end{align*}
 Conditioning on $\bz_0, w$, we get
\begin{align*}
  \left.\frac{\de^2\phantom{t}}{\de t^2}  \cuF(\bu_t,\bL_t) \right|_{t=0}& \ge c_0\E\big\{\Delta(\bz_0,w)\, 
  \bfone_{\bR_{00}^{1/2}\bz_0\in \cE_0}
           \bfone_{w\in \cE_1}\big\}\, ,\\
           \Delta(\bz_0,w) & := \E_{\bg,\bg^{\opt}}\Big\{\|
           F(\bg)-F^{\opt}(\bg^{\opt})\|^2\bfone_{F^{\opt}(\bg^{\opt})\in U_0}\Big\}\, ,\\
           F(\bg)& :=\Prox_{\ell(\,\cdot\,;\by)}(\bg+\bM\bz_0;\bS)\, ,\;\;\; F^{\opt}(\bg^{\opt}) :=\Prox_{\ell(\,\cdot\, ;\by)}(\bg^{\opt}+\bM\bz_0;\bS^{\opt}) \, ,
\end{align*}
 where in the expectation $\E_{\bg,\bg^{\opt}}$ is over $\bg,\bg^{\opt}$ which are jointly Gaussian with 
 $\E\{\bg\bg^{\sT}\} = \bK^2$,  $\E\{\bg(\bg^{\opt})^{\sT}\} = \bK\bK^{\opt}$,  $\E\{\bg^{\opt}(\bg^{\opt})^{\sT}\} = (\bK^{\opt})^2$.
 We claim that $\Delta(\bz_0,w)>0$ for   $\bR_{00}^{1/2}\bz_0\in \cE_0$, $w\in \cE_1$, which completes the proof.
 To prove this claim, note that, for $\bK\neq \bK^{\opt}$, we have $\bg^{\opt}= \bA\bg+\bC\bg'$ for 
 $\bC\in\reals^{k\times q}$ matrix with full column rank and $\bg'\sim\cN(\bzero,\bI_q)$ a Gaussian vector independent of $\bg$
 and not-identically $0$. Since $\bJ_{\bg}F^{\opt}(\bg)$ is non-singular for $F^{\opt}(\bg) = \bu_0$,
 we can assume by continuity that it is non-singular for all  $F^{\opt}(\bg^{\opt})\in U_0$, whence  
 $\bg'\mapsto \|F^{\opt}(\bA\bg+\bC\bg')-F(\bg)\|$ only vanishes on a Lebesgue measure zero set. 

 \vspace{0.2cm}
 
\noindent\emph{Claim  4.}
In this proof, we will keep track of the dependency on the regularization parameter $\lambda$, and hence write
$\hR_{n,\lambda}(\bTheta)$ and $\Risk_{\lambda}(\bK,\bM)$ for the empirical risk and its asymptotic characterization.

We will first prove the claim for $\lambda>0$. In this case, by strong convexity
$\hR_{n,\lambda}(\bTheta)$ is minimized at $\hbTheta$ with $\|\hbTheta\|_F\le C(\lambda)< \infty$
eventually almost surely, whence
\begin{align}
\lim_{\rho\to\infty}\lim_{n,d\to\infty}\inf_{\|\bTheta\|_F\le \rho}\hR_{n,\lambda}(\bTheta)
\lim_{n,d\to\infty}\inf_{\bTheta}\hR_{n,\lambda}(\bTheta)\, .
\end{align}
By point 3.$(a)$ of the this theorem and by 
Theorem \ref{thm:global_min}.$2$ (whose proof does not use the present claim point), we have
\begin{align}
\lim_{n,d\to\infty}\inf_{\bTheta}\hR_{n,\lambda}(\bTheta)=
\int \ell(\bu,\bu_0,w)\, \nu^{\opt}(\de\bu,\de\bu_0,\de w)+\frac{\lambda}{2}\int \|\bt\|^2\mu^{\opt}(\de\bt,\de\bt_0)\, .
\end{align}
By Eq.~\eqref{eq:muopt} in Definition \ref{def:opt_FP_conds}, and the correspondence in Eq.~\eqref{eq:Corr_Ropt_Kopt},
we have
\begin{align}
\int \|\bt\|^2\mu^{\opt}(\de\bt,\de\bt_0) =  \Tr\big((\bK^{\opt})^2 + \bM^{\opt}(\bM^{\opt})^\sT\big)\, .
 \end{align}
Further, by Eq.~\eqref{eq:nuopt}, $(\bg^\sT,\bg_0^\sT)^\sT \sim \cN\left( \bzero_{k+k_0},\bR^{\opt}\right)$ independent
of $w\sim\P_w$, 
\begin{align*}
\int \ell(\bu,\bu_0,\bw)\, \nu^{\opt} (\de\bu,\de\bu_0,\de w) 
&= \E\Big\{\ell\big(\Prox(\bg;\bS^{\opt},\bg_0,w),\bg_0,w),\bg_0,w\big)\Big\}\\
& = \E\left[\More_{\ell(\cdot, \bg_0,w)}(\bg;\bS^{\opt})\right] - \frac1{2\alpha}\Tr((\bS^{\opt})^{-1}(\bK^{\opt})^2) \, ,
\end{align*}
where the last equality follows from the stationarity conditions.
finally, the proof is completed by applying strong duality  as in the proof of Theorem \ref{thm:convexity}.

Finally, we consider the case $\lambda=0$.
First notice that $\Risk_{\lambda}(\bK,\bM)$ is monotone non-increasing in $\lambda$, whence
\begin{align}
\lim_{\lambda\to 0+} \inf_{\bK\succeq \bzero,\bM}\Risk_{\lambda}(\bK,\bM) &= \inf_{\lambda>0} \inf_{\bK\succeq \bzero,\bM}\Risk_{\lambda}(\bK,\bM)\\
&=\inf_{\bK\succeq \bzero,\bM} \inf_{\lambda>0}  \Risk_{\lambda}(\bK,\bM) = \inf_{\bK\succeq \bzero,\bM} \Risk_{0}(\bK,\bM) \, .\nonumber
\end{align}
By a similar monotonicity argument
\begin{align}
\lim_{\rho\to 0}\lim_{n\to\infty}\inf_{\|\bTheta\|_F\le\rho}\hR_{n,0}(\bTheta) =
\lim_{\lambda\to 0+}\lim_{n\to\infty}\inf_{\bTheta}\hR_{n,\lambda}(\bTheta)\, ,
\end{align}
whence the claim follows.

%% file: Appendix_D_EXAMPLES.tex
\section{Proofs for applications}
\subsection{Analysis of multinomial regression: Proof of Proposition~\ref{prop:multinomial}}
   Consider the regularized multinomial regression problem with empirical risk
   \begin{equation}
       \hat R_{n,\lambda}(\bTheta) := \frac1n\sum_{i=1}^n \left\{ \log\left(1 + \sum_{j=1}^k e^{\bx_i^\sT\btheta_j}\right) - \by_i^\sT\bTheta^{\sT}\bx_i \right\}  + \frac{\lambda}{2} \|\bTheta\|_F^2
   \end{equation}
   for $\lambda \ge 0$. 
   The claims of Proposition~\ref{prop:multinomial}
   for $\lambda>0$ follow immediately from Proposition \ref{propo:Exponential} for the case $\lambda>0$. The condition
   $\lambda_{\min}(\E_{\hnu}[\grad\ell \grad \ell^\sT])\ge a_3$ 
   can be verified from Lemma \ref{lemma:LocalStrongMultinomial} 
   below, using the fact that $\|\hbTheta\|_F^2\le C/\lambda$ 
   for a constant $C$. (Convergence of empirical spectral distribution of the Hessian holds
   by applying Theorem \ref{thm:global_min}.)
   We will there $\lambda =0$, but it will be convenient to study the above problem for $\lambda$ near $0$ for additional regularity.

For future reference, we  introduce the notation
\begin{align}
\ell_{i}(\bv):=
   \log\Big( \sum_{j=1}^k e^{v_j} + 1\Big)  
   -  \by_{i}^\sT \bv\, .\label{eq:ell_i}
\end{align}
We will use the following lemma that is 
an adaptation of Lemmas S6.1 and S6.4 of~\cite{tan2024multinomial}.
Although the parametrization in these lemmas is slightly different,
the proof required here can be derived from the proofs of those lemmas. For the convenience of the reader, we provide it here.
\begin{lemma}\label{lemma:LocalStrongMultinomial}
Under the assumptions of Proposition~\ref{prop:multinomial}, for any constant
$\rho>0$, there exists $c = c(\rho;\bR_{00},\alpha,\rho)>0$ such that
\begin{align}
&   \lim_{n\to\infty}\P\Big(\frac1n\sum_{i=1}^n \grad\ell_{i}(\bTheta^\sT\bx_i) 
    \grad\ell_{i}(\bTheta^\sT\bx_i)^\sT \succeq c\bI_k
    \;\;\forall \bTheta:\|\bTheta\|_F^2\le \rho\Big) = 1\,,
    \label{eq:outer_prod_singular_value_lb}\\
&\lim_{n\to\infty}\P\Big(\nabla^2 \hR_{n,0}(\bTheta)\succeq c\bI_{dk}\;\;\forall \bTheta:\|\bTheta\|_F^2\le \rho\Big) = 1\,.
\label{eq:hessian_singular_value_lb_multinomial}
\end{align}
\end{lemma}
\begin{proof}[Proof (Due to~\cite{tan2024multinomial}).]

Define the event
\begin{equation}
\Omega_{1,n} := \left\{ \sigma_{\max}(\bX) \le C_0 \sqrt{n},\quad \sigma_{\min}(\bX)\ge c_0\sqrt{n}\right\}
\end{equation}
for $C_0,c_0$ chosen so that $\Omega_{1,n}$ are high probability sets.

\noindent\textbf{Proof of Eq.~\eqref{eq:hessian_singular_value_lb_multinomial}.}
For a subset $\cI\subseteq [n]$, let $\bX_\cI = (\bx_i^\sT)_{i\in\cI} \in\R^{|\cI| \times d}$.
Now since $n/d_n \rightarrow \alpha$, we can choose a $\beta \in(\alpha^{-1},1)$ so that for some constant $c_0>0$ independent of $n$, we have $\lim_{n\to\infty}\P(\Omega_{3,n}) =1$ for the event
\begin{equation}
  \Omega_{1,n} := \left\{\sigma_{\min}(\bX_{\cI}) > c_0\sqrt{n} \;\; \textrm{for all}\;\; \cI \subseteq [n] \;\;\textrm{with}\;\; |\cI| \ge \beta n \right\}.
\end{equation}
(See for instance Lemma S6.3 of~\cite{tan2024multinomial}. The argument is a standard union bound using lower bound on the singular values of Gaussian matrices.)
For any $D  > 1$, $\|\bTheta\|_F^2 \le \rho$, we have on the event $\Omega_{1,n}$
\begin{equation}
   D \cdot|\{i\in[n] : \|\bTheta^\sT\bx_i\|_F^2 \ge D\}| \le  \sum_{i=1}^n \|\bTheta^\sT\bx_i\|_F^2 \le \|\bX\|_\op^2 \|\bTheta\|_F^2 \le C_0^2 \rho n =: \rho_0 n,
\end{equation}
so that, choosing $D =D_\beta := \rho_0 (1-\beta)^{-1}$,  we have for
\begin{equation}
      |\cI_\beta| \ge \beta n  \quad\textrm{for}\quad \cI_\beta := \{i \in[n] : \|\bTheta^\sT\bx_i\|_F^2 \le D_\beta \}\, .
\end{equation}
Furthermore, we have for each $i\in\cI_\beta$,
\begin{equation}
\label{eq:entry_wise_bound_p_hessian}
    \bp(\bTheta^\sT\bx_i) \in  [c_1,1-c_1]^k
\end{equation}
for some $c_1\in(0,1)$ independent of $n$.

Now we can compute for $\bv\in\R^k$ (note that, with the definition \eqref{eq:ell_i}, $\grad^2 \ell_i(\bv)$ is independent of $\by_i$, so we omit the subscript $i$):
\begin{equation}
    \grad^2 \ell(\bv) = \Diag(\bp(\bv)) - \bp(\bv)\bp(\bv)^\sT,
\end{equation}
Hence, for any $\bu$, letting $\bu^2 = \bu\odot \bu$,
\begin{align}
  \bu^\sT \grad^2 \ell(\bv)\bu   &= \bu^2 \odot \bp(\bv) - (\bu^\sT\bp(\bv))^2\\
  &\stackrel{(a)}{\ge}
 \bu^2 \odot \bp(\bv) - (\bu^2 \odot \bp(\bv))\|\bp(\bv)\|_1  \stackrel{(b)}{=} 
 (\bu^2 \odot \bp(\bv)) (p_0(\bv)),
\end{align}
where $(a)$ is an application of Cauchy Schwarz, and $(b)$ follows from the identity  $\sum_{j\in[k]}p_j(\bv) + p_0(\bv) = 1.$
This, along with Eq.~\eqref{eq:entry_wise_bound_p_hessian} then implies
that for $i\in \cI_\beta$, 
\begin{equation}
    \lambda_{\min}(\grad^2 \ell(\bTheta^\sT\bx_i))  > c_3
\end{equation}
for some $c_3$ independent of $n$.
Now noting that
\begin{equation}
\grad^2 \hat R_{n,0}(\bTheta) = \frac1n\sum_{i=1}^n \grad^2\ell(\bTheta^\sT\bx_i) \otimes (\bx_i\bx_i^\sT),
\end{equation}
we can bound 
on the high probability event $\Omega_{1,n}\cap\Omega_{2,n}$,
\begin{align}
    \lambda_{\min}(n\grad^2 \hat R_{n,0}(\bTheta)) 
    &\ge
    \lambda_{\min}\left(\sum_{i\in\cI_\beta} \grad^2\ell(\bTheta^\sT\bx_i) \otimes (\bx_i\bx_i^\sT)\right)\\
    &\ge 
    c_3 \lambda_{\min}\left(\sum_{i\in\cI_\beta} \bI_k \otimes (\bx_i\bx_i^\sT)\right)
    \ge c_3 c_0^2 n,
\end{align}
showing the claim.

\noindent
\textbf{Proof of Eq.~\eqref{eq:outer_prod_singular_value_lb}.}
By the upper bound on $\|\bTheta_0\|_F$, we can find some $\gamma\in(0,1)$ constant in $n$ so that the event
\begin{align}
\Omega_{3,n} &:= \left\{\sum_{i=1}^n \one_{\{\by_i = \be_{j}\}} \ge \gamma n  \quad\textrm{for all}\quad j \in\{0,\dots,k\} \right\}
\end{align} 
holds with high probability.
Without loss of generality, for what follows, assume that $\gamma n$ is an integer.

We next construct $\gamma n$ subsets of $[n]$ so that, in each subset, there is exactly one index corresponding to a label of each class.  Namely, find disjoint index sets $\{\cI_l\}_{l\in [\gamma n]}$, $\cI_l\subseteq [n]$, so that for each $l\in [\gamma n]$, we have
\begin{equation}
|\cI_l| = k+1, \quad\quad\textrm{and,  for each}\; j\in\{0,\dots,k\},\;\sum_{i\in \cI_l} \one_{\by_i = \be_j} =1.
\end{equation}
We claim that,  on the event $\Omega_{1,n}$, for any $\bTheta$ with  $\|\bTheta\|_F^2\le \rho$,
there must exist a $\cL \subseteq [\gamma n]$ such that
\begin{equation}
\label{eq:cL_index_set_bound}
    |\cL| \ge \frac{\gamma n}{2} \quad\textrm{s.t. for all $l\in\cL$}.\quad \sum_{i \in\cI_l} \|\bTheta^\sT\bx_i\|_F^2 \le \frac{ 2\rho_0}{\gamma}\, .
\end{equation}
Indeed, this follows from
\begin{equation}
|\cL^c| \min_{l\in\cL^c} \sum_{i\in\cI_l} \|\bTheta^\sT\bx_i\|_F^2 \le 
\sum_{l\in[\gamma n]} \sum_{i\in\cI_k} \|\bTheta^\sT\bx_i\|_F^2  \le 
\|\bX\|_\op^2 \|\bTheta\|_F^2
\le C_0^2 \rho n.
\end{equation}
%
Therefore, by definition of $\bp$, we have for each $i\in \cup_{l\in\cL} \cI_l$,
\begin{equation}
\label{eq:entry_wise_bounds_p}
    \bp(\bTheta^\sT\bx_i)\in [c_4,1-c_4]^k
\end{equation}
for some $c_4 \in (0,1)$  independent of $n$.

By straightforward computation, we have for each $l\in[\gamma n]$
\begin{equation}
    \sum_{i \in\cI_l} \grad \ell_i(\bTheta^\sT\bx_i) \grad \ell_i(\bTheta^\sT\bx_i)^\sT = \bS^\sT\bB^\sT(\bI_{k+1} - \bP_l)^\sT(\bI_{k+1} - \bP_l)\bB \bS
\end{equation}
where $\bS \in\R^{k\times k}$ is a permutation matrix,
\begin{equation}
    \bP_l = (p_j(\bTheta^\sT\bx_i))_{i\in [\cI_l], j\in\{0,\dots,k\}} \quad\textrm{and}\quad
    \bB = \begin{bmatrix} \bI_k \\
    \bzero^\sT
    \end{bmatrix} \in\R^{(k+1) \times k}.
\end{equation}
The bound in Eq.~\eqref{eq:entry_wise_bounds_p} then implies 
that for each $l\in\cL$, $\bP_l\in\R^{(k+1)\times (k+1)}$, which is a stochastic matrix, has entries that are in $[c_5,1-c_5]$ for some $c_5>0$. Hence, it is a stochastic matrix of an irreducible and aperiodic Markov chain, so that $(\bI_{k+1} - \bP_l)\bu = \bzero$ if and only if $\bu =a\one_{k+1}$ for $a\in\R$. 
Letting $\cP$ be the space of these stochastic matrices whose entries are in $[c_5,1-c_5]$, we have by compactness of this space that 
for some $\bP_\star \in \cP$, 
\begin{align}
    \min_{l\in\cL}\min_{\substack{\bv \in\R^{k}
    \\
    \|\bv\|_2 = 1
    }}\bv^\sT\bB^\sT(\bI_{k+1} - \bP_l)^\sT(\bI_{k+1} - \bP_l)\bB\bv
    &\ge  
    \min_{\substack{\bu\in\R^{k+1}\\ \|\bu\|_2 =1\\
    u_{k+1} = 0
    }}\bu^\sT(\bI_{k+1} -\bP_\star)^\sT(\bI_{k+1} -\bP_\star )\bu \ge c_5
\end{align}
for some $c_6>0$ independent of $n$ (but dependent on $c_1$).
So we can bound
\begin{align}
    \lambda_{\min}\left( \frac1n\sum_{i=1}^n \grad \ell_i(\bTheta^\sT\bx_i)
    \grad \ell_i(\bTheta^\sT\bx_i)^\sT
    \right) 
    &\ge 
    \lambda_{\min}\left(\frac1n \sum_{l\in |\cL|} 
    \bB^\sT (\bI_{k+1} - \bP_l)^\sT (\bI_{k+1} - \bP_l) \bB
    \right)\\
    &\ge \frac{|\cL|}n  c_6.
\end{align}
The bound on $|\cL|$ now gives the desired result of Eq.~\eqref{eq:outer_prod_singular_value_lb}.
\end{proof}

%
%

\begin{lemma}[Bounded empirical risk minimizer for multinomial regression]
\label{lemma:equivalence_multinomial}
Under the assumptions of Proposition~\ref{prop:multinomial}, 
the following are equivalent:
\begin{enumerate}
 \item There exists $C>0$ independent of $n$  such that, for all $\lambda>0$ 
    \begin{equation}
     \lim_{n\to\infty} \P\left(\|\hat\bTheta_\lambda\|_F < C \right) = 1.
    \end{equation}
    \item There exists $C>0$ independent of $n$ and $\lambda$, such that, 
    \begin{equation}
        \lim_{\lambda \to 0+} \lim_{n\to\infty} \P\left(\|\hat\bTheta_\lambda\|_F < C \right) = 1.
    \end{equation}
\item There exists $C>0$ such that
    \begin{equation}
         \lim_{n\to\infty} \P\left( \hat\bTheta \;\mbox{\rm exists}\;,\|\hbTheta\|_F 
 < C\right) = 1.
    \end{equation}
\end{enumerate}
If any of the above holds, then, for any $\delta>0$, we have
 \begin{equation}
        \lim_{\lambda \to 0+} \lim_{n\to\infty} \P\left(\|\hbTheta-\hbTheta_\lambda\|_F < \delta \right) = 1.\label{eq:LambdaPerturbation}
    \end{equation}
On the other hand, if for all $C>0$, we have
    \begin{equation}
    \label{eq:diverging_in_lambda}
        \lim_{\lambda \to 0} \lim_{n\to\infty} \P\left(\|\hat\bTheta_\lambda\|_F > C\right) = 1,
    \end{equation}
    then for all $C>0,$
    \begin{equation}
    \label{eq:any_seq_minimzers_diverges}
         \lim_{n\to\infty} \P\left( \hat\bTheta \;\mbox{\rm exists}\;,\|\hbTheta\|_F 
 < C\right) = 0.
    \end{equation}
\end{lemma}
\begin{proof}
We will show $\textit{1}\Rightarrow \textit{2}\Rightarrow \textit{3} \Rightarrow \textit{1}$. 

\vspace{0.15cm}

\noindent $\textit{1}\Rightarrow \textit{2}$. This is obvious.

\vspace{0.15cm}

\noindent $\textit{3}\Rightarrow \textit{1}$. Define $r_n(\lambda) := \inf_{\bTheta}\hR_{n,\lambda}(\bTheta)$.
This is a non-decreasing non-negative concave function of $\lambda$. By the envelope theorem, for any
$0\le \lambda_1<\lambda_2$, we have 
\begin{align}
\frac{1}{2}\|\hbTheta_{\lambda_1}\|^2_F \ge \frac{r(\lambda_2)-r(\lambda_1)}{\lambda_2-\lambda_1} \ge \frac{1}{2}\|\hbTheta_{\lambda_2}\|^2_F\, ,
\end{align}
where, for $\lambda=0$, $\|\hbTheta_{0}\|_F$ is the norm of any minimizer when this exists.
It follows in particular that $\|\hbTheta_{\lambda}\|_F\le \|\hbTheta\|_F$ for any $\lambda>0$, and therefore the claim follows. 

\vspace{0.15cm}

\noindent $\textit{2}\Rightarrow \textit{3}$. Fix $C$ as in point \textit{2}, $\delta_0>0$ and we chose $\lambda_0>0$
such that
$\lim_{n\to\infty}\P(\|\hbTheta_{\lambda}\|_F<C)\ge 1-\delta_0$ for all $\lambda\in(0,\lambda_0)$. Let $c_0 = c_0(2C)>0$
be given as per Lemma \ref{lemma:LocalStrongMultinomial}. Hence, with probability $1-\delta_0-o_n(1)$,
\begin{align}
\|\bTheta\|_F\le 2C \;\;\Rightarrow\;\;
\hR_{n,0}(\bTheta)&\ge \hR_{n,0}(\hbTheta_{\lambda}) +\< \hR_{n,0}(\hbTheta_{\lambda}),\bTheta-\hbTheta_{\lambda}\>
+\frac{c_0}{2}\|\bTheta-\hbTheta_{\lambda}\|_F^2\\
&\ge \hR_{n,0}(\hbTheta_{\lambda}) -\lambda\< \hbTheta_{\lambda},\bTheta-\hbTheta_{\lambda}\>
+\frac{c_0}{2}\|\bTheta-\hbTheta_{\lambda}\|_F^2\, .
\end{align}
Recalling that $\|\hbTheta_{\lambda}\|_F<C$, this implies
\begin{align}
\frac{2\lambda C}{c_0}<C \;\;\Rightarrow \;\; \|\bTheta-\hbTheta_{\lambda}\|_F\le \frac{2\lambda}{c_0}\|\hbTheta_{\lambda}\|_F\le \frac{2\lambda C}{c_0}\, .
\end{align}
The first condition can be satisfied by eventually decreasing $\lambda_0$. We thus conclude that, for each 
$\delta_0>0$, there exists $\lambda_0>0$ such that 
\begin{align}
\lim_{n\to\infty}\P\Big(\|\hbTheta\|_F\le 2C, \|\hbTheta_{\lambda}-\hbTheta\|_F\le \frac{2\lambda C}{c_0}\Big)\ge 1-\delta_0\, .
\end{align}
The claim \textit{3} follows by dropping the second inequality
 $\|\hbTheta_{\lambda}-\hbTheta\|_F\le 2\lambda C/c_0$ and noting that $\delta_0$ is arbitrary.

Equation \eqref{eq:LambdaPerturbation} follows by dropping the 
inequality $\|\hbTheta\|_F\le 2C$ in the last display.
 
\vspace{0.15cm}

Finally Eq.~\eqref{eq:diverging_in_lambda} implies Eq.~\eqref{eq:any_seq_minimzers_diverges} ~ by
the monotonicity of $\|\hbTheta_{\lambda}\|_F$ in $\lambda$ proven above.
\end{proof}

\begin{lemma}
\label{lemma:multinomial_regularized}
Consider the setting of Proposition~\ref{prop:multinomial}.
For any $\lambda >0$, the equations 
\begin{align}
\label{eq:multinomial_FP_regularized}
   \alpha \E\left[(\bp(\bv) - \by)(\bp(\bv) - \by)^\sT\right]  &= \bS^{-1} (\bR/\bR_{00}) \bS^{-1}\\
  \E[(\bp(\bv) - \by) (\bv^\sT,\bg_0^\sT)]  &= \lambda (\bR_{11},\bR_{10})
\end{align}
have a unique solution $(\bR_{\opt}(\lambda),\bS_{\opt}(\lambda))$.

Furthermore, letting $\hat\bTheta_\lambda$ be the unique minimizer of $\hat R_{n,\lambda}$, and $(\mu^\opt(\lambda),\nu^\opt(\lambda))$ by defined in terms of $\bR^\opt(\lambda),\bS^\opt(\lambda)$ via Definition~\ref{def:opt_FP_conds},
then points $\textit{1.-3.}$ of Theorem~\ref{thm:global_min} hold.
\end{lemma}
\begin{proof}
Uniqueness of the solution follows from Theorem \ref{prop:simple_critical_point_variational_formula}
(point \textit{3.(a)} holds because $\lambda>0$).

In order to prove that the conclusions of Theorem~\ref{thm:global_min} holds, we
will apply that theorem to a modification of the problem under consideration.
Namely, we will perform a smoothing of the labels $\by_i$, and check that the assumption holds.

For $\eps>0$, let $\by_{i,\eps}\in\R^k$ be a smoothing of $\by_i$ so that it has entries in $[0,1]$, 
$\by_{i,\eps} := \boldf_{\eps}(\bTheta_0^{\sT}\bx_i, w_i)$ for some $\boldf_\eps$ that 
is a $C^2$, Lipschitz function of $\bTheta_0^{\sT}\bx_i$ for all fixed $\eps>0$, where $w_i$ is a uniform random variables; and satisfies 
$\P\left(\by_i \neq \by_{i,\eps}| \bTheta_0^\sT\bx_i\right) \le \eps$ (such a smoothing can be constructed, for instance, by writing $\by_i$ as a function of a uniformly distributed random variable and indicators, then smoothing the indicator appropriately).
Define the smoothed regularized MLE, for $\lambda>0$,
\begin{equation}
     \hat\bTheta_{n,\eps,\lambda} := 
     \argmin_{\bTheta\in\R^{d\times k}}  \hat R_{n,\eps,\lambda}(\bTheta),
    \quad
\hat R_{n,\eps,\lambda}(\bTheta) :=
     \frac1n \sum_{i=1}^n
    \ell_{i,\eps}(\bTheta^\sT\bx_i) + \frac{\lambda}{2} \|\bTheta\|_F^2,
\end{equation}
where
\begin{equation}
    \ell_{i,\eps}(\bTheta^\sT\bx_i):=
   \log\Big( \sum_{j=1}^k e^{\bx_i^\sT\btheta_j} + 1\Big)  
   -  \by_{i,\eps}^\sT \bTheta \bx_i
  \, .
\end{equation}
Note that $\ell_{i,\eps}(\bTheta^\sT\bx_i) = 
L(\bTheta^\sT\bx_i,\boldf_{\eps}(\bTheta_0^\sT\bx_i,w_i))$  depends on $\bTheta_0^\sT\bx_i,w_i$ through the labeling function $\boldf_{\eps}$. 
To avoid clutter, we use the notation $\ell_{i,\eps}(\bTheta^\sT\bx_i)$ instead of $\ell(\bTheta^\sT\bx_i,\bTheta_0^\sT\bx_i,w_i)$ which is used in other sections.

We check the conditions of Theorem~\ref{thm:global_min} for this setting.
Assumptions~\ref{ass:regime} and~\ref{ass:theta_0} are given.
It's easy to check that Assumptions~\ref{ass:loss},~\ref{ass:Data} and ~\ref{ass:convexity} hold.

We next show that the minimizer $\hat\bTheta_{n,\eps,\lambda}$ is, with high probability, in the set of critical points $\cE(\bTheta_0,\bw)$ defined in Theorem~\ref{thm:global_min} 
for which our theory applies.
For this, we will need to show that $\sigma_{\max}(\hat\bTheta_{\eps,\lambda}, \bTheta_0) \le C,$  for some $C$ independent of $n$,
and that the Hessian  along with the gradient outer product are 
lower bounded
at $\hat\bTheta_{\eps,\lambda}$ by some $c>0$ independent of $n$.
In what follows, let 
\begin{equation}
    \Omega_{1,n} := \{   C_0\sqrt{n} \ge \sigma_{\max}(\bX) \ge \sigma_{\min}(\bX)\ge \sqrt{n} c_0 \}
\end{equation}
for some $C_0,c_0 >0$ chosen so that $\Omega_{1,n}$ is a high probability event.

\noindent\textbf{Upper bound on $\bR(\hmu(\hat\bTheta_{\eps,\lambda}))$.}
For any $\lambda >0$, it is easy to see $\|\hat\bTheta_{\eps,\lambda}\|_F \le C_0/\lambda$ for some $C_0$ independent of $n,\eps$, since the multinomial loss is lower bounded by zero.
This along with the assumption that $\bR_{00}=\lim_{n\to\infty}\bTheta_0^{\sT}\bTheta_0$
is finite  implies  $
\bR(\hmu(\hat\bTheta_{\eps,\lambda}))\prec C\bI$ for all fixed $\lambda >0$.

\noindent\textbf{Lower bound on the Hessian $\grad^2_{\bTheta}\hat R_{n,\eps,\lambda}(\bTheta)$.} Clearly, since 
$\bTheta\mapsto \ell_{i,\eps}(\bTheta)$ is convex, we have for any $\lambda >0$, $\grad^2_{\bTheta}\hat R_{n,\eps,\lambda}(\bTheta) \succeq \lambda/2\bI.$ 

\noindent\textbf{Lower bound on the gradient outer product $\E_{\hnu}[\nabla\ell\nabla\ell^{\sT}]$.}

Let
\begin{equation}
    \bA_i := 
    \grad\ell_{i}(\hat\bTheta_\lambda^\sT\bx_i)
    \grad\ell_{i}(\hat\bTheta_\lambda^\sT\bx_i)^\sT,
    \quad\quad
    \bA_{i,\eps} := 
    \grad\ell_{i,\eps}(\hat\bTheta_{\eps,\lambda}^\sT\bx_i)
    \grad\ell_{i,\eps}(\hat\bTheta_{\eps,\lambda}^\sT\bx_i)^\sT
\end{equation}

By Lemma \ref{lemma:LocalStrongMultinomial}, using the definition of $\ell_i$ therein,
 for any fixed $\lambda >0$ there exists $c_1(\lambda)$ independent of $n$
 such that, with high probability
\begin{equation}\label{eq:nablaell_1}
    \frac1n\sum_{i=1}^n \bA_i \succ c_1(\lambda) \bI\, .
\end{equation}
We'll show that the smallest singular value of $n^{-1}\sum_i \bA_i$ is sufficiently close to that of $n^{-1}\sum_{i}\bA_{i,\eps}$.

First, on $\Omega_{1,n}$, we have by strong convexity for $\lambda >0$, 
\begin{equation}
\label{eq:min_lipschitz_in_y}
    \|\hat\bTheta_{\eps,\lambda} - \hat\bTheta_{\lambda}\|_F \le \frac{C}{ \lambda \sqrt{n}} \|\bY_{\eps} - \bY\|_F
\end{equation}
where $\bY,\bY_{\eps} \in\R^{n\times k}$ are the matrices whose rows are the labels $\by_i^\sT,\by_{i,\eps}^\sT$, respectively, for $i\in[n]$.

Meanwhile $\|\grad\ell_{i}(\hat\bTheta_{\lambda}^\sT\bx_i)\|_2,\|\grad\ell_{i,\eps}(\hat\bTheta_{\eps,\lambda}^\sT\bx_i)\|_2   \le C$ for some $C > 0$ independent of $n$.
Further using the  Lipschitz continuity of the minimizer in the labels, we have, on $\Omega_{1,n}$,
\begin{align}
\label{eq:grad_diff_to_y_diff}
 \|\grad\ell_{i}(\hat\bTheta_{\lambda}^\sT\bx_i) - \grad\ell_{i,\eps}(\hat\bTheta_{\eps,\lambda}^\sT\bx_i)\|_2  
 &\le \|\bp(\hat\bTheta_{\eps,\lambda}^\sT \bx_i) - \bp(\hat\bTheta_{\lambda}^\sT\bx_i)\|_2 + 
 \|\by_{i} - \by_{\eps,i}\|_2\\
 &\le C 
 \|(\hat\bTheta_{\eps,\lambda}- \hat\bTheta_{\lambda})^\sT\bx_i\|_2 +  
 \|\by_{i} - \by_{\eps,i}\|_2.
\end{align}

Now note that we have for all $i,j\in[n]$,
\begin{align}
\nonumber
\Tr\left(\bA_i(\bA_j - \bA_{j,\eps}))\right)
&= 
\left(\grad\ell_{j}(\hat\bTheta_\lambda^\sT\bx_i)^\sT
\grad\ell_{i}(\hat\bTheta_\lambda^\sT\bx_i)
\right)^2 -
\left(\grad\ell_{i}(\hat\bTheta_{\eps,\lambda}^\sT\bx_i)^\sT
\grad\ell_{j}(\hat\bTheta_\lambda^\sT\bx_j)
\right)^2\\
\label{eq:outer_prod_to_grad_diff}
&\quad\le
C\|\grad\ell_{i}(\hat\bTheta_\lambda^\sT\bx_i) 
-\grad\ell_{i,\eps}(\hat\bTheta_{\eps,\lambda}^\sT\bx_i)
\|_2
\end{align}
where in the inequality we used Cauchy Schwarz and that $\|\grad \ell_{i,\eps}\|_2,\|\grad \ell_{i}\|_2$ are uniformly bounded by a constant $C>0$ independent of $n$.
A similar bound clearly holds for $\Tr(\bA_{i,\eps}(\bA_j - \bA_{j,\eps})).$
So we can bound
\begin{align}
\label{eq:grad_outer_product_diff}
    \left\|\frac1n\sum_{i=1}^n  
    \bA_i -
    \frac1n\sum_{i=1}^n  \bA_{i,\eps} \right\|_F^2
    &=
    \frac1{n^2} \sum_{i,j\in[n]} \Tr\left((\bA_i-\bA_{i,\eps})(\bA_j-\bA_{j,\eps})^\sT \right)\\
    &\le
    \frac1{n^2} \sum_{i,j\in[n]}\left\{ |\Tr\left(\bA_i(\bA_j - \bA_{j,\eps}))\right)| + |\Tr\left(\bA_{i,\eps}(\bA_j - \bA_{j,\eps})\right)|
    \right\}\\
    &\stackrel{(a)}\le 
    \frac{C}{n}  \sum_{i=1}^n
    \|\grad\ell_{i}(\hat\bTheta_\lambda^\sT\bx_i) 
-\grad\ell_{i,\eps}(\hat\bTheta_{\eps,\lambda}^\sT\bx_i)
\|_2\\
&\stackrel{(b)}{\le}
\frac{C}{n}\sum_{i=1}^n \| (\hat\bTheta_{\eps,\lambda} - \hat\bTheta_{\lambda})^\sT\bx_i \|_2 +\|\by_i - \by_{\eps,i}\|_2\\
&\stackrel{(c)}{\le} \frac{C}{\sqrt{n}} \left( \| \bX (\hat\bTheta_{\eps,\lambda} - \hat\bTheta_{\lambda})\|_F + \| \bY - \bY_{\eps}\|_F\right)\\
& \stackrel{(d)}{\le} \frac{C}{\sqrt{n}}\left(\frac{\|\bX\|_\op}{\sqrt{n}\lambda} + 1\right) \|\bY - \bY_{\eps}\|_F,
\end{align}
where in $(a)$ we used Eq.~\eqref{eq:outer_prod_to_grad_diff}, in $(b)$ we used
Eq.~\eqref{eq:grad_diff_to_y_diff}, in $(c)$ we used $\|\bv\|_1 \le \sqrt{n}\|\bv\|_2$ for $\bv\in\R^{n}$, and in $(d)$ we used Eq.~\eqref{eq:min_lipschitz_in_y}.
So on the high probability events $\Omega_{1,n}\cap\Omega_{2,n}$,
\begin{align}
\label{eq:nablaell_2}
    &\frac1n\sum_{i=1}^n \grad\ell_{i,\eps}(\hat\bTheta_{\eps,\lambda}^\sT\bx_i) 
    \grad\ell_{i,\eps}(\hat\bTheta_{\eps,\lambda}^\sT\bx_i)^\sT\\
    &\quad\succeq
    \frac1n\sum_{i=1}^n \grad\ell_{i}(\hat\bTheta_{\lambda}^\sT\bx_i)
    \grad\ell_{i}(\hat\bTheta_{\lambda}^\sT\bx_i)^\sT -  \left(\frac{C}{\sqrt{n}}(\lambda^{-1} + 1) \| \bY - \bY_{\eps}\|_F\right)^{1/2} \bI_k.
    \nonumber
\end{align}
By construction of the smoothing, we have
    $\E[\|\by_i - \by_{i,\eps}\|_2^2 | \bTheta_0^\sT\bx_i] \le C \eps$
and $\|\by_i - \by_{i,\eps}\|_2^2 < C$ almost surely, for some $C >0$ independent of $n$. So by Hoeffding's inequality, for any $\delta> 0$ we can choose $\eps = \eps(\delta,\lambda)>0$
such that
\begin{equation}
    \lim_{n\to\infty }\P\left(\frac1{\sqrt{n}} \|\bY - \bY_\eps\|_F > \delta\right) =0\, .
\end{equation}
By choosing $\delta$ sufficiently small and using Eqs.~\eqref{eq:nablaell_1} 
\eqref{eq:nablaell_2}, we conclude that, with high probability,
\begin{equation}
    \frac1n\sum_{i=1}^n \grad\ell_{i,\eps}(\hat\bTheta_{\eps,\lambda}^\sT\bx_i) 
    \grad\ell_{i,\eps}(\hat\bTheta_{\eps,\lambda}^\sT\bx_i)^\sT
    \succeq \frac{c_1(\lambda)}{2}\bI\, .
\end{equation}
This shows that for all $\lambda,\eps >0$, with high probability  the minimizer $\hat\bTheta_{\eps,\lambda} \in \cE(\bTheta_0)$ for all $\eps >0$, i.e.,
\begin{equation}
\lim_{n\to\infty}\P\big(\hat\bTheta_{\eps,\lambda} \in \cE(\bTheta_0) \big) = 1\,.
\end{equation}
Hence we have shown that the conditions of Theorem~\ref{thm:global_min} are satisfied.
This allows us to conclude the statement of the lemma when $\by$ is replaced with $\by_{\eps}$ in equations~\eqref{eq:multinomial_FP_regularized}, and $\hat\bTheta_{\lambda}$ is replaced by $\hat\bTheta_{\eps,\lambda}$ for $\eps >0$ sufficiently small.
Continuity of these equations in $\eps$ along with the consequence of strong convexity~\eqref{eq:min_lipschitz_in_y} allows us to then pass to the limit $\eps\to 0$ giving the desired claim.
\end{proof}

\begin{proof}[Proof of Proposition~\ref{prop:multinomial}]
\noindent{\emph{Claim \textit{1(a)}.}}
Since the system~\eqref{eq:FP_multinomial} has a (finite) solution $(\bR^\opt,\bS^\opt)$,
by Theorem~\ref{thm:simple_critical_point_variational_formula}, this solution is unique implying the claim. 

\noindent{\emph{Claim \textit{1(b)}.}}
Let $(\bR^\opt(\lambda),\bS^\opt(\lambda))$ the unique solution for $\lambda>0$,
which corresponds to the unique minimizer of $F(\bK,\bM)$ defined in 
Eq.~\eqref{eq:FKM_Def}, by Theorem \ref{prop:simple_critical_point_variational_formula}. Since $F(\,\cdot\, )$ depends continuously on $\lambda$
and has a unique minimizer for $\lambda=0$ (by the previous point), it follows that $(\bR^\opt(\lambda),\bS^\opt(\lambda)) \rightarrow (\bR^\opt,\bS^\opt)$ as $\lambda\to 0$.
In particular, there exists $C>0$ independent of $n$ and $\lambda$, such that,
for all $\lambda>0$ smalle enough
    \begin{equation}
         \lim_{n\to\infty} \P\left(\|\hat\bTheta_\lambda\|_F < C \right) = 1.\label{eq:LastBDD}
    \end{equation}
The equivalence of Lemma~\ref{lemma:equivalence_multinomial} then 
implies claim~\textit{(b)}.

\noindent{\emph{Claims \textit{1(c)}, \textit{1(d)}.}}
Since Eq.~\eqref{eq:LastBDD} implies Eq.~\eqref{eq:LambdaPerturbation}, 
and $\lim_{\lambda\to 0}(\bR^\opt(\lambda),\bS^\opt(\lambda)) = (\bR^\opt,\bS^\opt)$,
statements $(c)$ and $(d)$ follow from Lemma \ref{lemma:multinomial_regularized}.

\noindent{\emph{Claims \textit{2}.}}
If  the system~\eqref{eq:FP_multinomial} does not have a (finite) solution,
then we claim that  $\lim_{\lambda\to 0}\Tr(\bR^\opt(\lambda))=\infty$.
Indeed, if it by contradiction $\lim\inf_{\lambda\to 0}\Tr(\bR^\opt(\lambda))<\infty$,
then we can find a sequence $\bR^{\opt}(\lambda_i)$, $i\in \N$ 
with $\lambda_i\to 0$ and $\bR^{\opt}(\lambda_i)$ converging to a finite
limit  $\bR^{\opt}$ (recall that $\Tr(\bR)\le C$ is a compact subset of $\bR\succeq \bzero$), and this would be a solution of the system \eqref{eq:FP_multinomial}
with $\lambda=0$, leading to a contradiction.

Since  $\Tr(\bR^\opt(\lambda))$  is unbounded as $\lambda\to 0$, by Lemma~\ref{lemma:multinomial_regularized} there exists a sequence $\lambda_i\to 0$ 
such that for any $C>0$,
\begin{equation}
   \lim_{i\to \infty} \lim_{n\to\infty} \P(\|\hat\bTheta_{\lambda_i}\|_F > C)  = 1.
\end{equation}
Applying Lemma~\ref{lemma:equivalence_multinomial}, 
we thus conclude that Eq.~\eqref{eq:any_seq_minimzers_diverges}.
\end{proof}

%% file: Appendix_E_NUMERICAL.tex
\section{Additional numerical simulations}
\label{app:Numerical}

\begin{figure}[t]
    \centering
     \begin{subfigure}[t]{0.45\textwidth}
        \includegraphics[width=\textwidth]{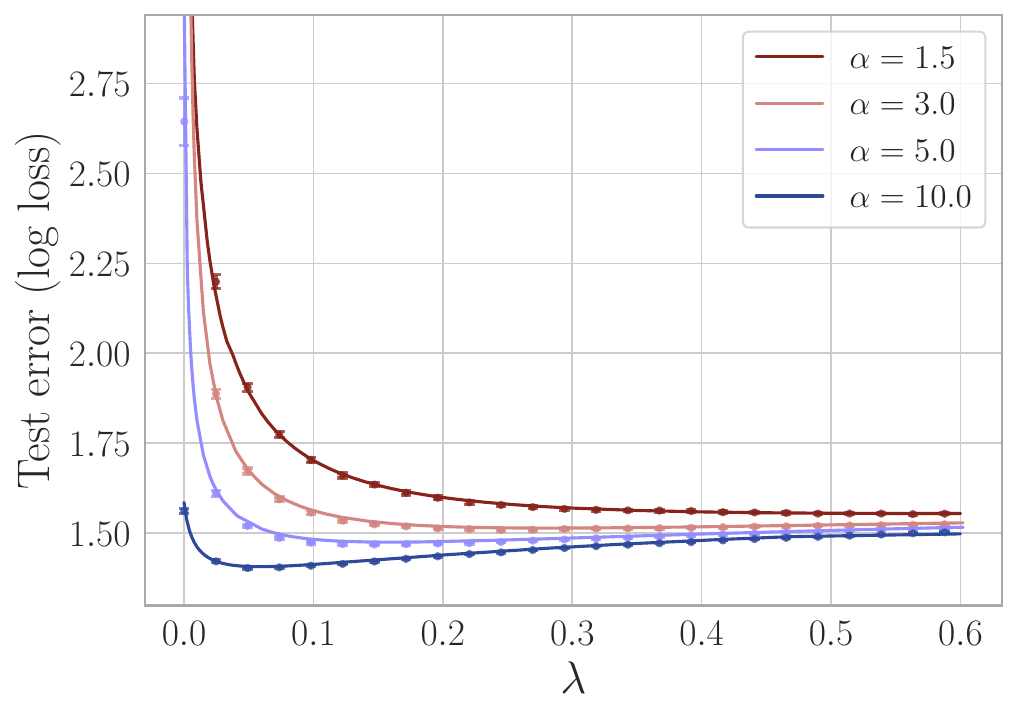}
    \end{subfigure}
    \begin{subfigure}[t]{0.45\textwidth}
        \includegraphics[width=\textwidth]{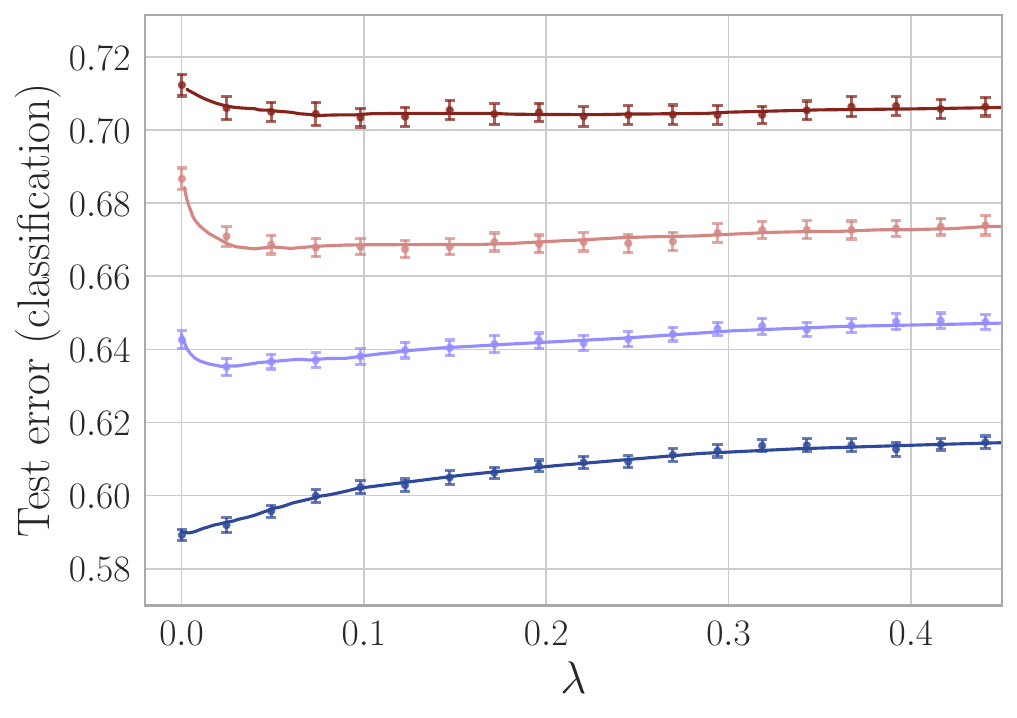}
    \end{subfigure}
    \begin{subfigure}[t]{0.45\textwidth}
        \centering
        \includegraphics[width=\textwidth]{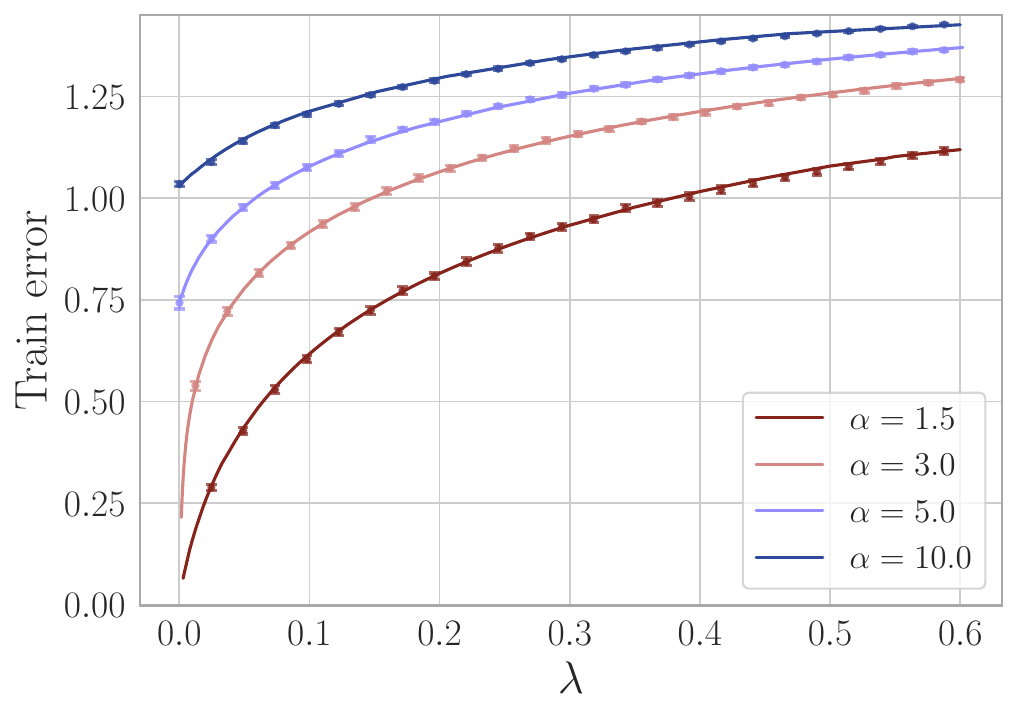}
    \end{subfigure}
    \begin{subfigure}[t]{0.45\textwidth}
       \centering
        \includegraphics[width=\textwidth]{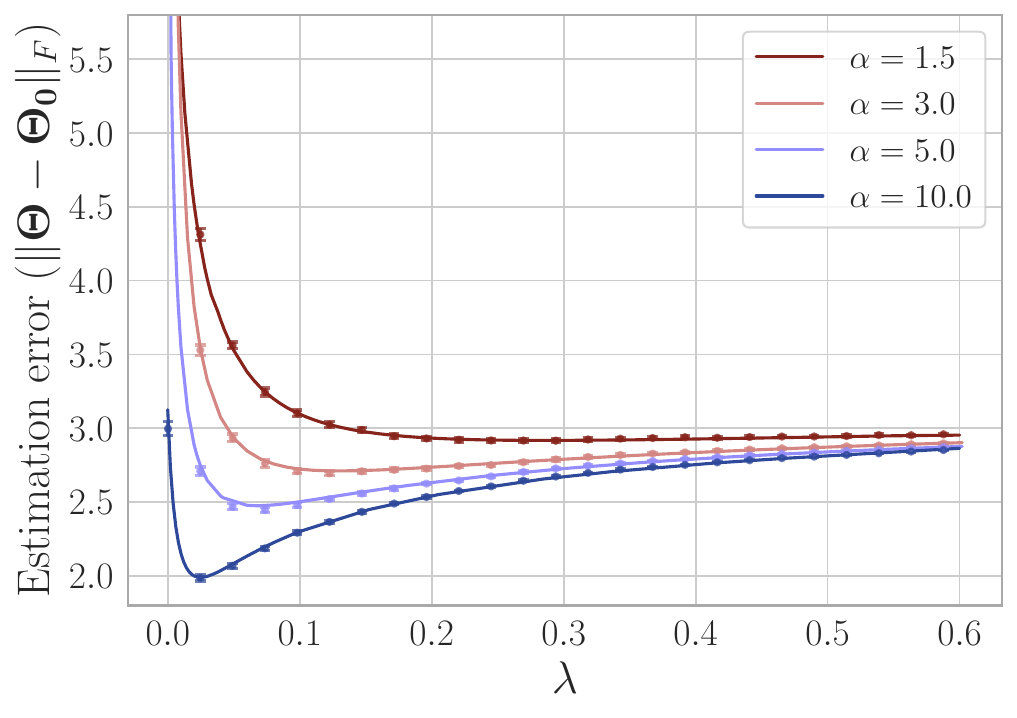}
    \end{subfigure}

    \caption{Train/test error (log loss), estimation error, and classification error
    of ridge regularized multinomial regression, for $(k+1)=5$ symmetric classes,
    as a function of the regularization parameter $\lambda$ for several values of $\alpha$. 
    Empirical results are averaged over 100 independent trials,  with $d = 250$. 
    Continuous lines are theoretical predictions obtained by solving numerically the system \eqref{eq:FP_multinomial}.}
    \label{fig:regularized_error_5}
\end{figure} 

\begin{figure}[t]
    \centering
    \begin{subfigure}[t]{0.45\textwidth}
        \includegraphics[width=\textwidth]{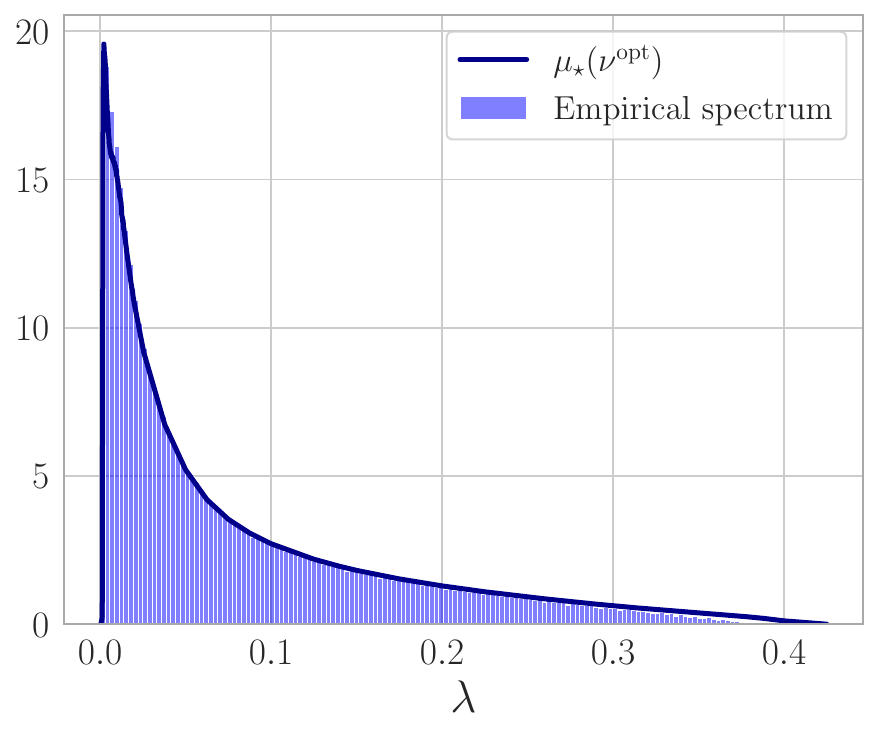}
    \end{subfigure}
    \begin{subfigure}[t]{0.45\textwidth}
       \centering
        \includegraphics[width=\textwidth]{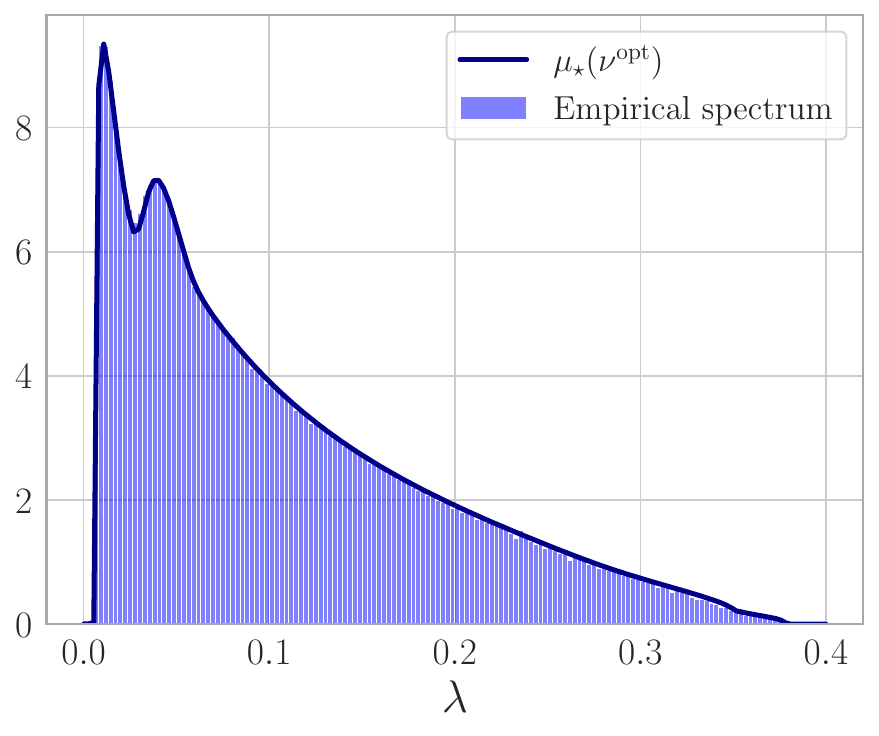}
    \end{subfigure}
    \begin{subfigure}[t]{0.45\textwidth}
        \centering
        \includegraphics[width=\textwidth]{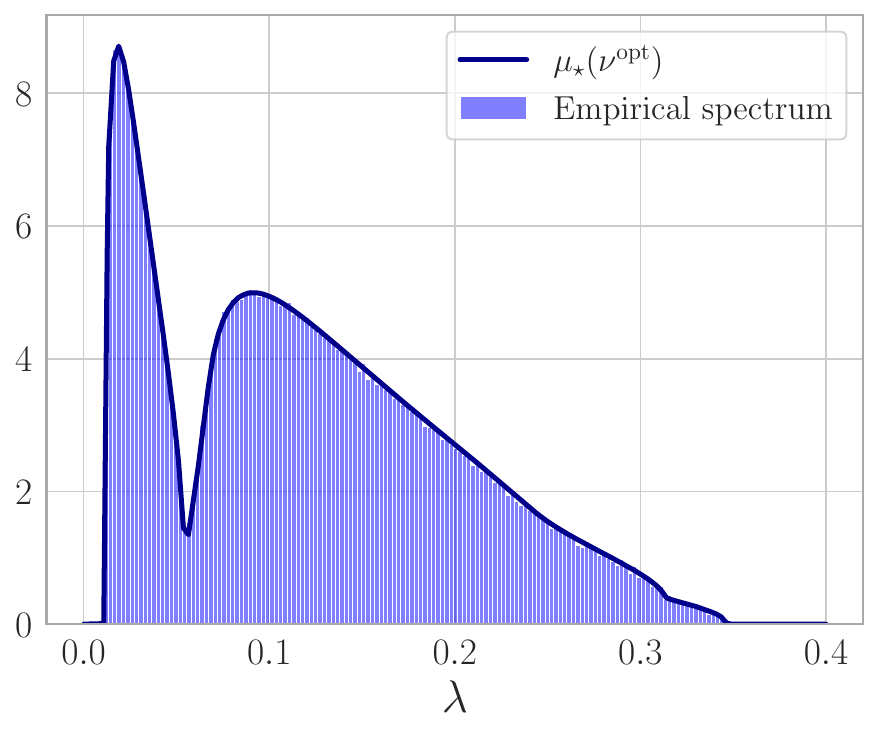}
    \end{subfigure}
    \begin{subfigure}[t]{0.45\textwidth}
        \includegraphics[width=\textwidth]{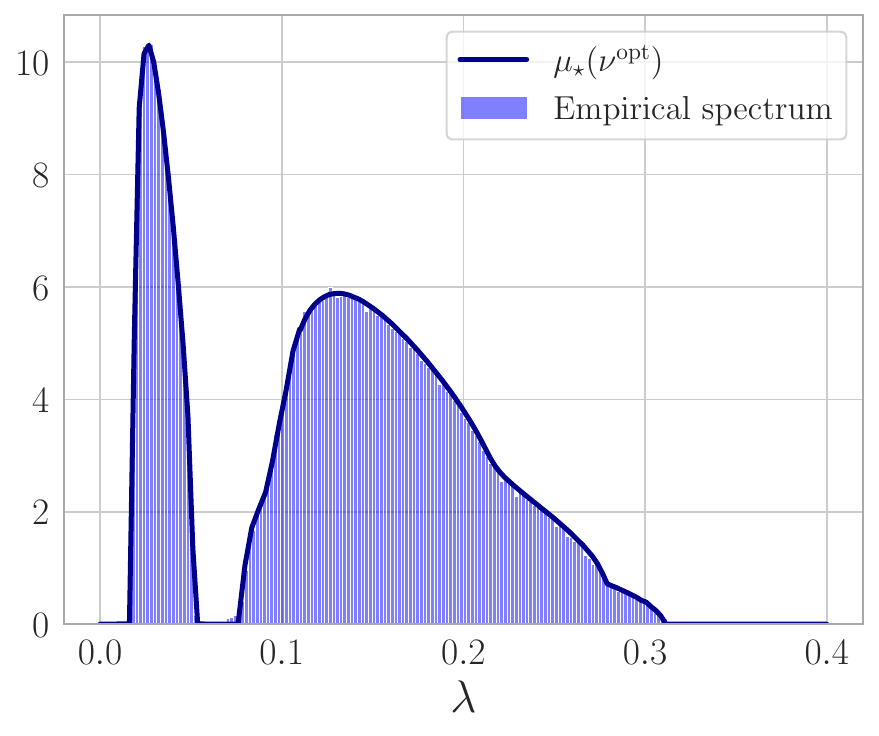}
    \end{subfigure} 
    \caption{Histograms of the empirical spectral distribution of the Hessian at the MLE
    for multinomial regression with $(k+1)=5$ and $\bR_{00}=\bR_{00}^{(2)}$ as defined in~\eqref{eq:numerical_configuration_three_clusters}, in $d = 250$ dimensions, aggregated over $100$ 
    independent realizations.
    From left to right, $\alpha = 4$, $\alpha = 6$, $\alpha=10$, and $\alpha = 20$. Blue lines represent the theoretical distribution derived from Proposition \ref{prop:multinomial}.}
    \label{fig:Spectrum_5}
\end{figure} 

\begin{figure}[t]
    \centering
    
     \begin{subfigure}[t]{0.45\textwidth}
        \includegraphics[width=\textwidth]{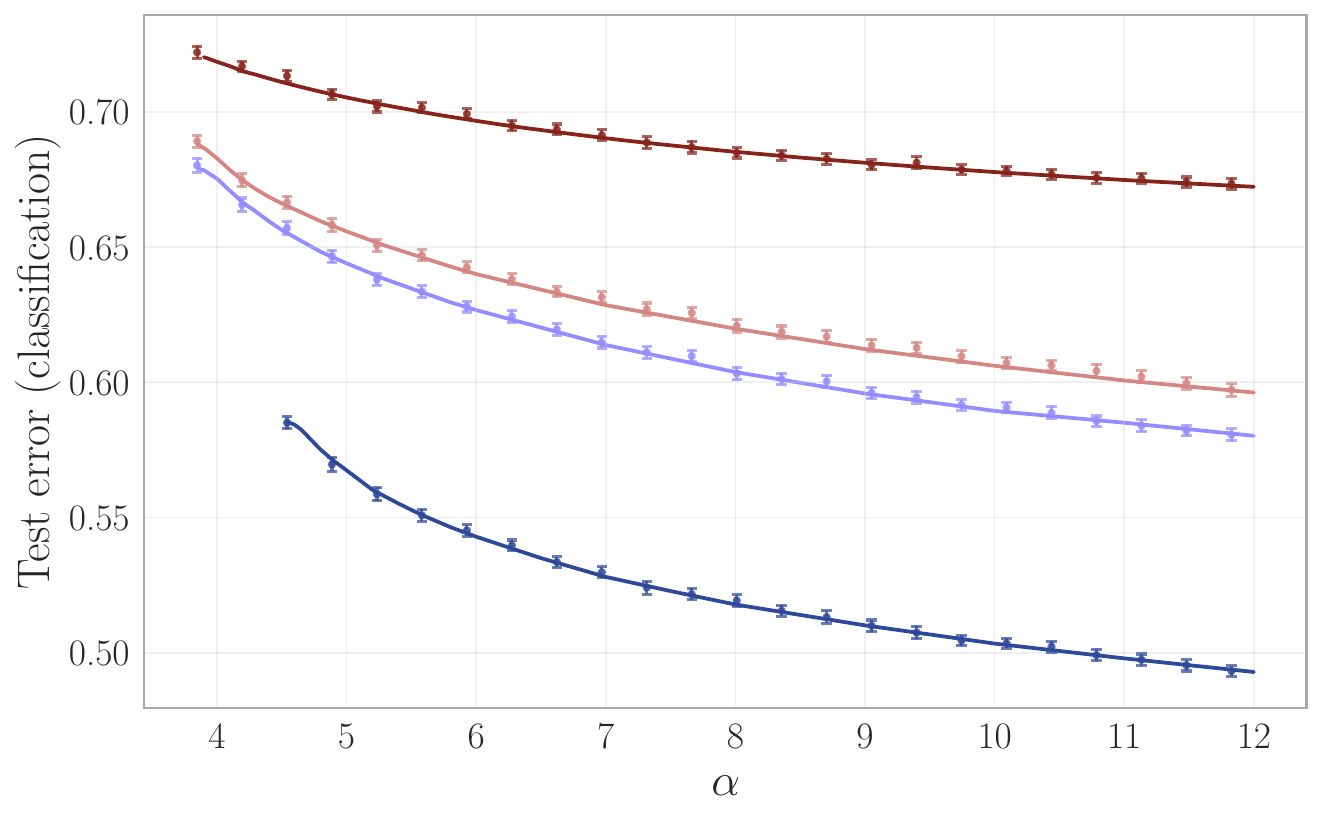}
    \end{subfigure}
    \begin{subfigure}[t]{0.45\textwidth}
        \includegraphics[width=\textwidth]{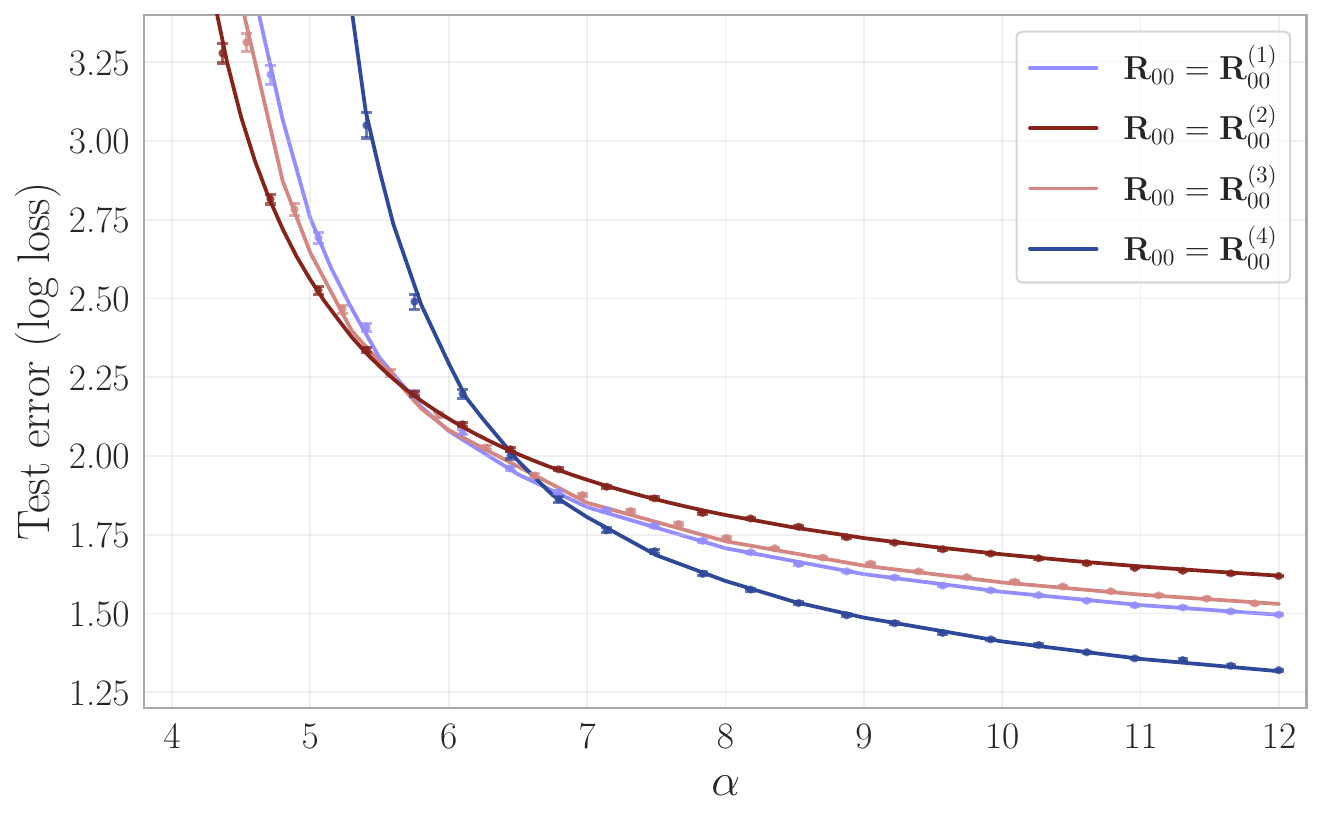}
    \end{subfigure}
    \begin{subfigure}[t]{0.45\textwidth}
       \centering
        \includegraphics[width=\textwidth]{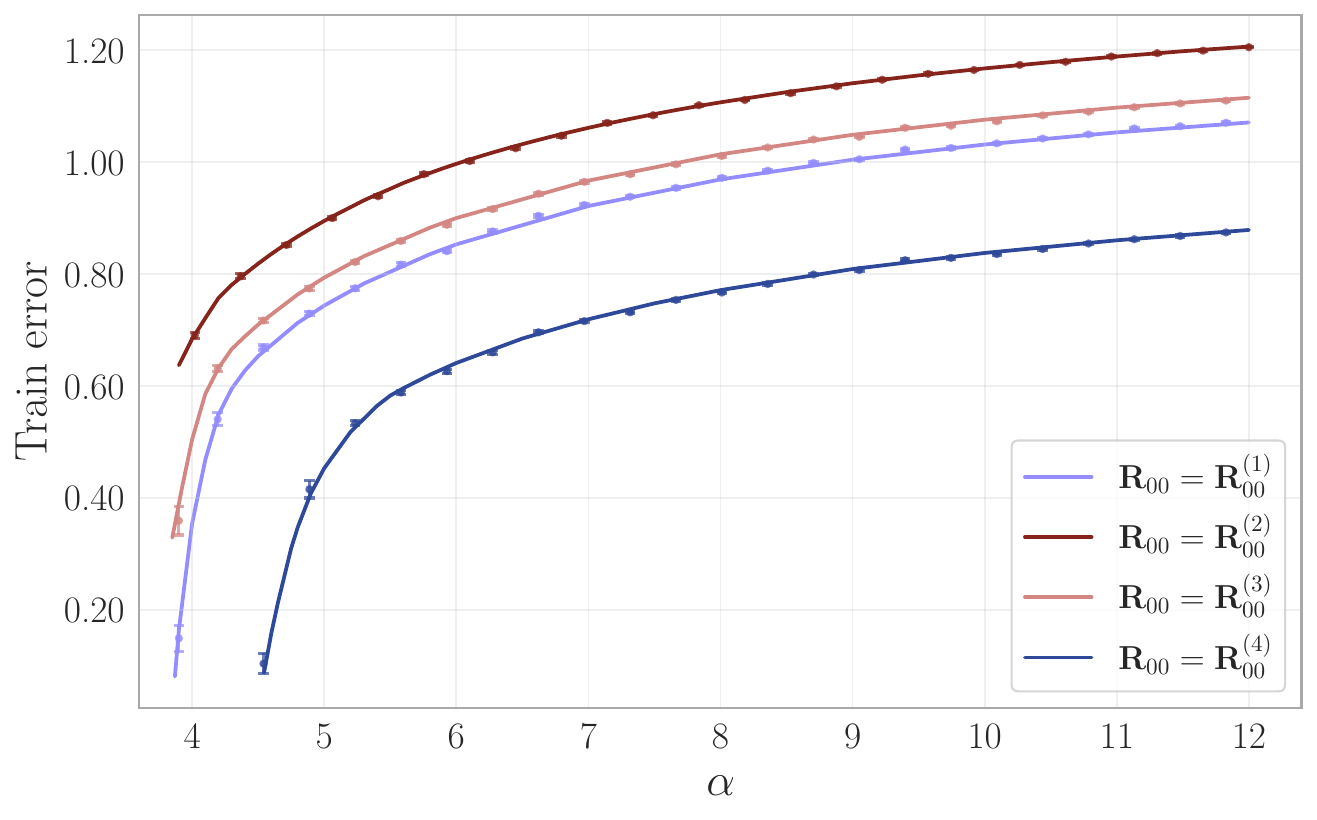}
    \end{subfigure}
    \begin{subfigure}[t]{0.45\textwidth}
        \centering
        \includegraphics[width=\textwidth]{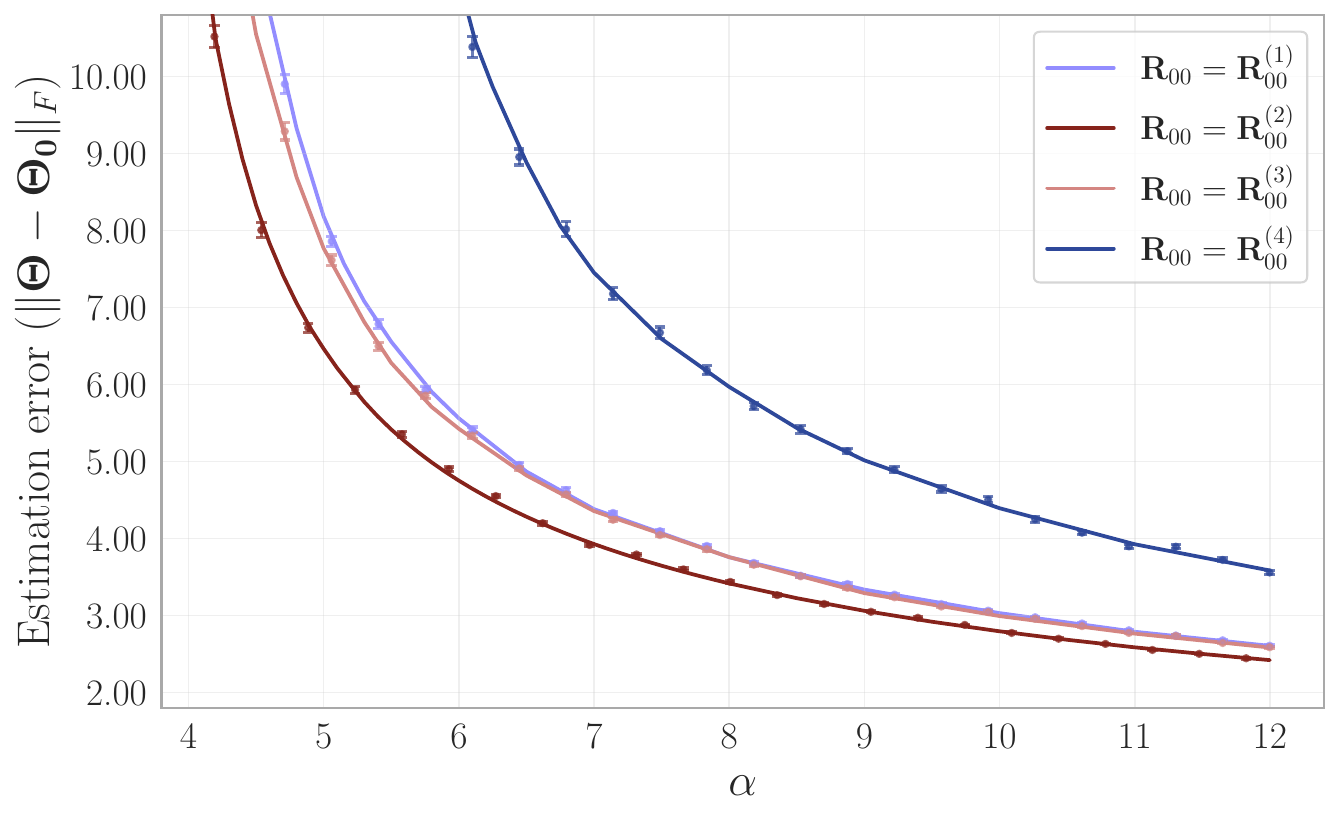}
    \end{subfigure}
    \caption{Train/test error (log loss), estimation error, and classification error
    of unregularized multinomial regression, for $(k+1)=5$ symmetric classes, as a function of $\alpha$ for different values of $\bR_{00}$ specified in the text (here $\lambda=0$).
    Empirical results are averaged over 100 independent trials,  with $d = 250$. }
    \label{fig:error_vs_alpha_5}
\end{figure}

In the main text of the paper, we reported the results of numerical simulations for 
multinomial regression with $(k+1) = 3$ classes. 
For completeness, we repeated the experiments with $(k+1)=5$ classes.
We report the results in Figures \ref{fig:regularized_error_5},
\ref{fig:Spectrum_5}, and \ref{fig:error_vs_alpha_5}.

In Figures \ref{fig:regularized_error_5}, we report the histograms of the empirical spectral distribution of the Hessian at MLE for $\lambda=0$, using a configuration with
\begin{equation}
    \bR_{00}^{(1)} :=\;
    \frac52\cdot\begin{bmatrix}
        \quad 1\quad & 0.5  & 0.5 & 0.5 \\
        0.5 &\quad 1\quad & 0.5 & 0.5 \\
        0.5 & 0.5 & \quad 1\quad & 0.5 \\
        0.5 & 0.5 & 0.5 & \quad 1\quad
    \end{bmatrix},
\end{equation}
which encodes 5 completely symmetrical classes, each with a unit effect size 

In Figure~\ref{fig:Spectrum_5}, we use a configuration with
\begin{equation}\label{eq:numerical_configuration_three_clusters}
    \bR_{00}^{(2)} := \begin{bmatrix}
       \qquad 1\qquad         & 0.975 & 0.75       & 0.75       \\
        0.975 &\qquad 1\qquad         & 0.945 & 0.937     \\
        0.75       & 0.945 &\qquad 3 \qquad        & 2.975 \\
        0.75       & 0.937     & 2.975 & \qquad3\qquad
    \end{bmatrix},
\end{equation}
    which corresponds to five classes forming three clusters of sizes 1, 2, and 2, with class parameters that lie close together within each cluster.

In Figure \ref{fig:error_vs_alpha_5} we compare the same empirical and theoretical quantities, for different values of the ground truth parameters$\bR_{00}^{(1)}$,  $\bR_{00}^{(2)}$, along with the  additional configurations
\begin{equation}
    \bR_{00}^{(3)} 
    := 3.41\cdot\begin{bmatrix}
    \quad 1\quad & 0.81 & 0.81 & 0.81 \\
    0.81 &\quad 1\quad & 0.81 & 0.81 \\
    0.81 & 0.81 & \quad 1\quad & 0.81 \\
    0.81 & 0.81 & 0.81 & \quad 1\quad
    \end{bmatrix} \text{ and}
    \quad
      \bR_{00}^{(4)}  
      := 14.24\cdot\begin{bmatrix}
    \quad 1\quad & 0.95 & 0.95 & 0.95 \\
    0.95 & \quad 1\quad & 0.95 & 0.95 \\
    0.95 & 0.95 & \quad 1\quad & 0.95 \\
    0.95 & 0.95 & 0.95 & \quad 1\quad
      \end{bmatrix}.
    \end{equation}
We observe that the theoretical predictions show a perfect match with the empirical results, as noted in Section~\ref{sec:Multinomial}.